\documentclass{article}

\PassOptionsToPackage{numbers, sort&compress}{natbib}
 \usepackage[preprint]{neurips_2026}

\usepackage[utf8]{inputenc} 
\usepackage[T1]{fontenc}    
\usepackage[pagebackref,breaklinks,colorlinks,allcolors=magenta]{hyperref}       
\usepackage{url}            
\usepackage{booktabs}       
\usepackage{amsfonts}       
\usepackage{nicefrac}       
\usepackage{microtype}      
\usepackage{xcolor}         

\usepackage{amsmath,mathrsfs,amsthm,bm}
\usepackage{graphicx,subcaption}
\usepackage{amssymb}
\usepackage{array}
\usepackage{multirow}

\usepackage{amsfonts,mathrsfs,bm}
\usepackage{textcomp} 

\AtBeginDocument{
\def\a{{\alpha}}

\def\s{{\sigma}}

\def\th{{\theta}}

\def\s{{\sigma}}
\def\t{{\tau}}

}

\def\cM{{\mathcal{M}}}
\def\cN{{\mathcal{N}}}

\def\sF{{\mathscr{F}}}
\def\sM{{\mathscr{M}}}
\def\sR{{\mathscr{R}}}
\def\sS{{\mathscr{S}}}
\def\sT{{\mathscr{T}}}

\def\vzero{{\bm{0}}}
\def\vone{{\bm{1}}}

\newcommand{\RR}{\mathbb{R}}

\newcommand*{\dto}{\downarrow}

\newcommand{\figWidth}{1.03} 

\def\th{{\theta}}
\def\cM{{\mathcal{M}}}
\def\cN{{\mathcal{N}}}
\def\sF{{\mathscr{F}}}
\def\sM{{\mathscr{M}}}
\def\sR{{\mathscr{R}}}
\def\sS{{\mathscr{S}}}
\def\sT{{\mathscr{T}}}
\def\vzero{{\bm{0}}}
\def\vone{{\bm{1}}}
\theoremstyle{plain}
\newtheorem{theorem}{Theorem}[section]
\newtheorem{proposition}[theorem]{Proposition}
\newtheorem{lemma}[theorem]{Lemma}

\theoremstyle{definition}

\theoremstyle{remark}
\newtheorem{remark}[theorem]{Remark}
\newtheorem{finding}[theorem]{Finding}

\title{On the Intrinsic Limited Robustness of\\Latent-Based Watermarking}

\author{Cheng-Han Yeh\thanks{Equal contribution.}\quad Kuan-chun Yu\footnotemark[1]\quad Cheng-Chang Tsai\footnotemark[1]\quad Chun-Shien Lu \\
Institute of Information Science \\
Academia Sinica \\
Taipei, Taiwan (ROC) \\
\texttt{\{jerryyeh,mikeyu1012,cctsai,lcs\}@iis.sinica.edu.tw}
}

\begin{document}

\maketitle

\begin{abstract}
  Existing latent-based watermarking methods for diffusion models have overestimated their robustness to image distortions, including geometric transformations such as rotation, scaling, and translation (RST). Moreover, this paradigm of watermarking approaches may suffer from inherent limitations arising from the domain in which the watermark is embedded. In this paper, we provide the first theoretical analysis explaining why these methods lack invariance to perturbations. By relaxing the invariant relation, we derive a maximum perturbation bound that characterizes the relationship between pixel-space perturbations and their corresponding effects in latent space. In addition, we present the first analytical formulation that captures all components of practical detection mechanisms. Finally, we conduct experiments to validate the theoretical findings and the limitations of latent-based watermarking methods. Our theoretical and empirical results indicate that, under the current design paradigm, latent-based watermarking methods intrinsically exhibit limited robustness. We conclude by providing the analytical tool and design guidelines that future research could follow.
\end{abstract}

\section{Introduction}
    \subsection{Background}
        The dissemination of misinformation or disinformation severely causes financial loss and reputation damage. In particular, due to the maturation of diffusion models~\cite{dhariwal2021diffusion,ho2020denoising,rombach2021high,song2020denoising} in deep learning, the generated (fake) images have been indistinguishable from the natural images, deteriorating the distribution of AI-generated contents (AIGCs) and further obscuring the facts. Several realistic cases indicate that if AIGCs cannot be efficiently detected, not only is intellectual property infringed, but also financial losses may occur, and reputation may be damaged. In addressing these concerns, governments and industries start to take precautionary measures internationally~\cite{EUAIlaws, WhiteHouseAI, thewhitehouselaws, eulaws}. These policies clearly suggest to employ digital watermarking technologies for AIGCs detection and tracing, revealing its importance.

    \subsection{Motivation}\label{motivation}
        Watermarking methods~\cite{ci2024ringid, fernandez2023stable, wen2023treerings} on AIGCs have been proposed, to name a few, but can only claim robustness against limited types of perturbations with a limited set of parameters. 
        This obscures the potential risk, \textit{i.e}., insufficient robustness, of these well-known watermarking schemes. Specifically, we are curious about the factors, including watermark patterns, the shift in latent distribution, etc, that affect the robustness claimed by these methods.

        To answer this question, we establish a new formulation that combines the RST framework~\cite{ruanaidh1998rotation} and the diffusion model~\cite{kadkhodaie2023generalization}. We prove that the inverse diffusion process lacks invariance and robustness to perturbations in the pixel space.
        Next, by relaxing the detection criterion, we extend this analysis and derive an upper bound that characterizes pixel-space perturbations induced by variations around the watermark. The results show that sufficiently shifting the mean of the standard normal distribution via watermark injection improves robustness to distortions in the corresponding images. In addition, we propose the first analytical form to characterize the components used in latent-based watermarking. This form also delineates the relationships between the detection mechanism, the setting of False Positive Rate (FPR), and different perturbations beyond RST.

        In summary, our contributions are fivefold:
        \begin{enumerate}
            \item We are the first to bridge Vector Approximate Message Passing~\cite{rangan2019vector} with DDIM reverse process, deliver the proof to invert the reverse process, and give the closed form of the Lipschitz constant of DDIM reverse and inverse process.
            \item We identify the factors that influence the robustness of watermarking on the latent representation, and explain the lack of invariance to perturbations.
            \item We derive the upper bound of the perturbation, measured by the $\ell_p$-norm, for a watermarked image, and show that even if such an image is perturbed, it could still be recognized by the detection algorithm.
            \item We provide the first analytical form to describe the robustness of a watermarking scheme, involving the design of watermark, perturbations, setting of False Positive Rate (FPR), and Lipschitz constant of perturbations.
            \item Our findings, supported by both theoretical and experimental evidences, highlight a key limitation in that latent-based watermarking schemes may lack robustness under the current design paradigm.
        \end{enumerate}

\section{Related Work}
    An increasing number of studies have been proposed for generative watermarking in the latent space of diffusion models, which require the inverse process of DDIM~\cite{song2020denoising} to return to the time step at which the watermarks are embedded in that the image generation and watermarking processes are coupled together. The concept of generative watermarking or in-process watermarking\footnote{By contrast, post-processing watermarking~\cite{solachidis2001circularly,jia2021mbrs,lu2006media,luo2020distortion,zhang2024attack,zhu2018hidden} means that an image has been generated or already existed and then watermarked.} was first proposed in Tree-Ring~\cite{wen2023treerings}.

    Although Tree-Ring demonstrates somewhat robustness in verifying the existence of a watermark, its performance deteriorates significantly in watermark identification due to the inherent low-capacity property. To address this issue, RingID~\cite{ci2024ringid} proposes using a more meticulously designed pattern as the watermark, along with a multi-channel heterogeneous watermarking framework, to achieve superior robustness and identification capability. Unlike Tree-Ring and RingID, which use watermarks with designed patterns for specific purposes, Gaussian Shading~\cite{yang2024gaussian} uses a binary sequence as its watermark. In addition, unlike Tree-Ring and RingID, which both cause the distribution of initial noises to deviate from the standard Gaussian distribution, Gaussian Shading and its variant~\cite{mao2025max} map the watermark to the initial noise sampled from the standard Gaussian distribution.

    In contrast to the aforementioned works that embed watermarks into the initial noise, ROBIN~\cite{huang2024robin} embeds the watermark into the latent at an intermediate time step, whereas Latent Watermark~\cite{meng2025latent} embeds the watermark immediately before decoding. Finally, ZoDiac~\cite{zhang2024attack} employs a pre-trained stable diffusion~\cite{rombach2021high} model to embed the watermark into the trainable latent space.

    Another line of research explores content-dependent watermark (CDW). This concept was first introduced in~\cite{lu2003content} to counter watermark estimation attacks (WEAs), which are used in collusion and copy attacks~\cite{voloshynovskiy2001attack}. Specifically, CDW comprises an informative watermark that carries information about the owner and an image/content hash that represents the cover carrier, \textit{i.e}., content-dependent information~\cite{lu2003content}. For watermarking on AIGCs,~\citet{yang2024can} proposes averaging over a certain number of images to extract an additive pattern, which can be used for watermark forgery and removal. To defend against this attack,~\citet{yang2024can} also suggests that content-adaptive watermarking methods, which employ the same underlying concept as CDW, demonstrate robust resistance. 

\section{Insufficient Robustness in Latent-Based Watermarking}\label{rethink latent wm}
    In this section, we investigate the limited robustness of latent-based watermarking methods. First, we introduce the notations used in our analysis in Sec.~\ref{notation}.
    Second, we provide the analytical tool to explain the reverse and inverse processes of DDIM from the Vector Approximate Message Passing (VAMP)~\cite{rangan2019vector} and MMSE denoiser perspective in Sec.~\ref{Sec: DDIM formulation and L bound}.
    Third, we examine the perturbation invariance of latent-based watermarking~\cite{ci2024ringid,huang2024robin,wen2023treerings,yang2024gaussian,zhang2024attack,zhu2026towards,11462859,rebuffi2026learning,gao2026slice} in Sec.~\ref{Sec: Fragility}. Then, as mentioned in Sec.~\ref{motivation}, by relaxing the invariant relation, we provide a framework describing the relationship between the perturbations on pixel space and the corresponding effects on latent space of diffusion models, and give insights regarding the tolerance (\textit{i.e}., the maximum perturbation on the image) and the resistance (\textit{i.e}., the upper bound of the Lipschitz constants of perturbations that a watermarking scheme can resist) of the latent-based watermarking methods against more severe perturbations in the pixel space in Sec.~\ref{max bound}. Finally, we discuss the effects of thresholds and latent distribution shift of these methods in Sec.~\ref{dist shift}. For the background relevant to our analysis, please refer to Sec.~\ref{sec: background knowledge} of Appendix.

    \subsection{Notation}\label{notation}
        Throughout this paper, $\mathbf{I}\in\RR^{C\times H\times W}$ denotes an image, $\mathbf{x}_0$ represents the latent representation at time step $0$, and $\mathbf{y}$ is the corresponding latent representation at time step $T$, where $\mathbf{x}_0, \mathbf{y} \in \RR^n$. $\|\cdot\|$ represents the $\ell_p$ norm, $p\ge1$, unless specifically stated. $\rm Id$ denotes the identity matrix. $D_\theta$ is the minimum mean-squared error (MMSE) denoiser~\cite{metzler2016denoising, rangan2019vector}, parameterized by $\theta$. We use $\epsilon_\theta$ to represent the noise prediction network used in~\cite{song2020denoising} and parameterized by $\theta$. We denote the encoder and decoder used in~\cite{rombach2021high} as $\mathcal{E}_{\phi_e}$ and $\mathcal{G}_{\phi_g}$, where $\mathbf{I}=\mathcal{G}_{\phi_g}(\mathbf{x}_0)$ and $\mathbf{x}_0=\mathcal{E}_{\phi_e}(\mathcal{G}_{\phi_g}(\mathbf{x}_0))$. We use $\sR$, $\sS$, and $\sT$ to denote the rotation, scaling, and translation operations, respectively. $\sF$ and $\sF_\sM$ represent 
        the Fourier and Fourier--Mellin transforms, respectively. $\lambda(A)$ collects eigenvalues of a square matrix $A$.

    \subsection{Formulation of DDIM Reverse and Inverse Process}\label{Sec: DDIM formulation and L bound}
        Recently, diffusion models and their corresponding denoisers used in the reverse process have been studied in~\cite{kadkhodaie2023generalization} to analyze the inductive biases of the network, under the assumption that the denoiser is a ``piecewise linear, bias-free neural network''. Given a denoiser and a noisy input image, prior works~\cite{romano2017little,mohan2019robust,kadkhodaie2023generalization} represent
        the denoised image as the product of the Jacobian matrix of the denoiser evaluated at the noisy input image and the noisy input image itself. Although these assumptions make the analysis more concise, they also restrict its applicability and practicality.

        In our analysis, a $T$-step reverse process $f_\theta$ can be represented as an unrolling VAMP using the MMSE denoiser $D_\theta$ without the above assumption. Specifically, we have the following Theorem.
        \begin{theorem}\label{thm:Global inverse of rev}
            Given an MMSE denoiser $D_\theta$ and an induced $T$-step DDIM reverse process $f_\theta$, $f_\theta$ is globally invertible and expressed as:
            \begin{align}\label{eq:1}
                \mathbf{y} = f_\theta^{-1}(\hat{\mathbf{x}}_0) = f_\theta^{-1}(f_\theta(\mathbf{y})),
            \end{align}
        where $\hat{\mathbf{x}}_0$ is the denoised representation given the latent representation $\mathbf{y}$. At the same time, for all $\tilde{\mathbf{y}}$ in latent space and the corresponding denoised representation $\tilde{\mathbf{x}}_0$, $\tilde{\mathbf{x}}_0 = f_\th(\tilde{\mathbf{y}}) = f_\th(f^{-1}_\th(\tilde{\mathbf{x}}_0))$.
        \end{theorem}
        The detailed derivations are in Sec.~\ref{Derivation of MMSE-AMP and DDIM and Lipschitz} and Sec.~\ref{Sec: proofs for thm lem} of Appendix.

    \subsection{Limited Perturbation Invariance in Latent-based Watermark}\label{Sec: Fragility}
        We begin by analyzing the lack of invariance of the composition of RST transformations, which is formalized in the following theorem. The proofs of the following theorems are in Sec.~\ref{Sec: proofs for thm lem} of Appendix.
        \begin{theorem}\label{thm:RST invariance}
            Given an invertible $f_\theta$ defined in Eq.~(\ref{eq:1}), an RST-perturbed representation $\hat{\mathbf{x}}^{RST}_0$, and its original counterpart $\hat{\mathbf{x}}_0$ with $\hat{\mathbf{x}}^{RST}_0\ne\hat{\mathbf{x}}_0$, we have
            \begin{align}\label{eq:RST noninvariance}
                [\mathscr{F_M} \circ \mathscr{F}]\mathbf{y}^{RST} \neq [\mathscr{F_M} \circ \mathscr{F}]\mathbf{y},
            \end{align}
            where $\mathbf{y}^{RST}$ and $\mathbf{y}$ are the latent representations of $\hat{\mathbf{x}}^{RST}_0$ and $\hat{\mathbf{x}}_0$, respectively. In other words, the combined process, \textit{i.e}., the DDIM inversion process operating on an RST-transformed input, does not possess RST invariance, which results in the inequality of the transformed outputs in the joint Fourier and Fourier-Mellin domains.
        \end{theorem}
        Theorem~\ref{thm:RST invariance} shows that even after transforming the latent representations, $\mathbf{y}^{RST}$ and $\mathbf{y}$, into the Fourier domain, the results are not equal. To reach equality, it should satisfy the assumption called ``fragility'' in that no perturbation is allowed between either $\hat{\mathbf{x}}^{RST}_0$ and $\hat{\mathbf{x}}_0$ or $\mathbf{y}^{RST}$ and $\mathbf{y}$.

        Furthermore, we extend our analysis to perturbations beyond RST transformations and formalize the corresponding result in the following theorem.
        \begin{theorem}\label{thm: non-perturb invariant}
            If the output of
            $f_\theta$, defined in Theorem~\ref{thm:Global inverse of rev}, is non-RST invariant, then the output
            is also not invariant to any perturbations.
        \end{theorem}

    \subsection{Maximum Perturbation on Image and Resilience against Perturbations}\label{max bound}
        Different from the fragility assumption made in Theorems~\ref{thm:RST invariance} and~\ref{thm: non-perturb invariant}, here we derive the upper bound of the maximum allowed perturbation on images for robustness purposes, given a certain tolerance in terms of distance in latent space. Specifically, all images within such a bound can go through DDIM inversion to obtain latent representations close to the unperturbed one under a certain distance metric.

        Given the Lipschitz constants of DDIM reverse and inverse processes at each time step provided in Sec.~\ref{Sec: derive Spectral Bounds} of Appendix, we can derive the Lipschitz constants of $T$-step reverse and inverse processes.
        \begin{theorem}\label{thm: global L bound}
            Let $f_\theta: \RR^n \to \RR^n$ be the generative process and $f^{-1}_\theta: \RR^n \to \RR^n$ be the inverse process. The global Lipschitz bounds with monotonically decreasing noise schedule $\left\{\alpha_t\right\}^T_{t=0}$ are:
            \begin{align}
                \text{Generation: } \quad &L_{f_\theta} \le \frac{\sqrt{\alpha_0}}{\sqrt{\bar{\alpha}_T}}, \nonumber\\
                \text{Inversion: } \quad &L_{f^{-1}_\theta} \le \frac{\sqrt{1-\bar{\alpha}_T}}{\sqrt{1-\alpha_0}}, \nonumber
            \end{align}
        \end{theorem}
        where $\bar{\alpha}_t:=\prod_{s=0}^t\alpha_s$.
        The detailed proof is in Sec.~\ref{Sec: proofs for thm lem} of Appendix.
        
        Next, for the operations commonly used in the generation and detection of latent-based watermarks, we define the following operators: $\mathcal{W}$ injects a $\mathbf{w}$ in a latent representation $\mathbf{y}$; $\mathcal{M}$ denotes a masking operation used to extract the watermark pattern; and $\mathcal{T}\in\mathcal{S}_{\mathcal{T}}$ represents an image transformation from a set of operator $\mathcal{S}_{\mathcal{T}}$, including geometric, non-geometric, and model-based perturbations under different parameter settings. Note that we treat different parameter settings of a perturbation as distinct operators in $\mathcal{S}_{\mathcal{T}}$.
        Additionally, we denote the Lipschitz constants of $\mathscr{F}^{-1}$, $\mathscr{F}$, $\mathcal{W}$, $\mathcal{M}$, $\mathcal{E}_{\phi_e}$, $\mathcal{G}_{\phi_g}$, and $\mathcal{T}$, as $L_{\mathscr{F}^{-1}}$, $L_{\mathscr{F}}$, $L_{\mathcal{W}}$, $L_{\mathcal{M}}$, $L_{\mathcal{E}_{\phi_e}}$, $L_{\mathcal{G}_{\phi_g}}$, and $L_{\mathcal{T}}$, respectively.

        Now, we can define the image generation process from a watermarked latent. Given a latent $\mathbf{y}$ sampled from standard normal distribution, the watermarked image $\mathbf{I}$ can be generated as:
        \begin{align}\label{eq: detection space to Image space}
            \mathbf{I}&=\mathbf{F}(\mathbf{y}, \mathbf{w})\stackrel{\text{def}}{=}\mathcal{G}_{\phi_g}(f_\theta(\mathscr{F}^{-1}(\mathbf{y}_\mathbf{w})), \nonumber\\
            \mathbf{y}_\mathbf{w}&=\mathcal{W}(\mathbf{y}, \mathbf{w})
            \stackrel{\text{def}}{=}\mathscr{F}(\mathbf{y})\odot\vone_{\cM^c}+\mathbf{w},
        \end{align}
        where $\vone_{(\cdot)}$ is an indicator function, $\odot$ is Hadamard product, $\mathbf{y}_\mathbf{w}$ is the watermarked latent, and $\mathbf{w}=\mathbf{W}\odot\vone_{\cM}$ for some $\mathbf{W}\in\mathbb{C}^n$ denoting the ground-truth watermark pattern. In Eq.~(\ref{eq: detection space to Image space}), as $\mathcal{W}$ injects a pattern into $\mathbf{y}$, we formulate the Lipschitz constant of $\mathcal{W}$ with image size of $H\times W$ as $L_{\mathcal{W}}=\sqrt{HW}$. 
        Similarly, we can formally define a mapping from a watermarked image $\mathbf{I}$ to the representation employed by the watermark detection algorithm as:
        \begin{align}
            &\hat{\mathbf{w}} = \mathbf{F}^{-1}(\mathcal{T}(\mathbf{I}))\stackrel{\text{def}}{=}\mathcal{M}(\mathscr{F}(f^{-1}_\theta(\mathcal{E}_{\phi_e}(\mathcal{T}(\mathbf{I}))))\label{eq: wm detection space},
        \end{align}
        where $\hat{\mathbf{w}}$ is the estimated watermark from the distorted image $\mathcal{T}(\mathbf{I})$.

        Due to numerical errors, asymmetric guidance scales, and value clipping during reverse and inverse processes, the invariant relation, as depicted in Theorem~\ref{thm: non-perturb invariant}, is highly unachievable even if perturbations are not present. Therefore, we relax the invariant relation and define a scalar $r$ to capture these practical effects. We examine these practical effects in Sec.~\ref{Sec: ordered perturbations on w-hatw'}. 
        
        With these practical considerations, we define the ball of radius $r$ centered at $\mathbf{w}$ as $\mathcal{B}_r(\mathbf{w})=\{\mathbf{w}^{\prime}:\| \mathbf{w} - \mathbf{w}^{\prime} \| \leq r \}$. Our goal below is to demonstrate the relationship between the distortions within the ball $\mathcal{B}_r(\mathbf{w})$ in the latent space and the upper bound of corresponding perturbations in the pixel space.


        \begin{theorem}\label{thm: condition for success detection}
            Given 
            a watermarked image $\mathbf{I}$ generated from $\mathbf{w}$ and $\mathbf{y}\sim\cN(\vzero,\rm Id)$, and a ball $\mathcal{B}_r(\mathbf{w})$, if the watermarked image $\mathbf{I}$ and its perturbed counterpart $\mathcal{T}(\mathbf{I})$ satisfy
            \begin{align*}
                \forall \mathbf{I},\|\mathbf{I}-\mathcal{T}(\mathbf{I})\|\leq L_{\mathbf{F}} r,
            \end{align*}
            then for any watermarked images, $\mathbf{I}=\mathbf{F}(\mathbf{y},\mathbf{w})$ and $\mathbf{I}^{\prime}=\mathbf{F}(\mathbf{y},\mathbf{w}^{\prime})$, with $\mathbf{w}'\in\mathcal{B}_r(\mathbf{w})$, we have:
            \begin{align}\label{eq: rev inv total bound}
                \|\mathbf{w}-\hat{\mathbf{w}}^{\prime}\|\leq L_{\mathbf{F}^{-1}}(1+L_{\mathcal{T}})L_{\mathbf{F}} r,
            \end{align}
            where $\hat{\mathbf{w}}^{\prime}$ is the estimated watermark of $\mathcal{T}(\mathbf{I}^{\prime})$, obtained from Eq.~(\ref{eq: wm detection space}), $L_{\mathbf{F}}=L_{\mathcal{G}_{\phi_g}}L_{f_\theta}L_{\mathscr{F}^{-1}}L_{\mathcal{W}}$, and $L_{\mathbf{F}^{-1}}=L_{\mathcal{E}_{\phi_e}}L_{f^{-1}_\theta}L_\mathscr{F}L_\mathcal{M}$.
        \end{theorem} 

 \begin{table*}[t]
            \centering\small\caption{Comparison of different watermarking schemes with operations involved in the generation and detection processes. $*$ denotes the use of the noise layer for fine-tuning the network model or extra modules. $\dagger$ denotes the extra module for compensating for or correcting the errors caused by perturbations. $\triangle$ denotes the need for modifications of the reverse process. The ``$\times$ number'' indicates the number of modules used in the generation or detection process.}\label{Table: comparison of operations}
            \begin{tabular}{|p{3.8cm}|p{1.98cm}|p{1.98cm}|p{1.98cm}|p{1.98cm}|}
                \toprule
                Watermarking Methods & $\mathcal{M}$ / $\mathcal{W}$ & $\mathscr{F}^{-1}$ / $\mathscr{F}$ & $f_{\theta}$ / $f^{-1}_\theta$ & $\mathcal{G}_{\phi_g}$ / $\mathcal{E}_{\phi_e}$\\
                \midrule[2pt]\hline
                Tree-Ring (TR)~\cite{wen2023treerings} & $\checkmark$ / $\checkmark$ & $\checkmark$ / $\checkmark$ & $\checkmark$ / $\checkmark$ & $\checkmark$ / $\checkmark$ \\
                RingID~\cite{ci2024ringid} & $\checkmark$ / $\checkmark$ & $\checkmark$ / $\checkmark$ & $\checkmark$ / $\checkmark$ & $\checkmark$ / $\checkmark$ \\
                HSQR/TR~\cite{lee2025semantic} & $\checkmark$ / $\checkmark$ & $\checkmark$ / $\checkmark$ & $\checkmark$ / $\checkmark$ & $\checkmark$ / $\checkmark$ \\
                Gaussian Shading (GS)~\cite{yang2024gaussian} & $\times$ / $\checkmark$ & $\times$ / $\times$ & $\checkmark$ / $\checkmark$ & $\checkmark$ / $\checkmark$ \\
                PRC Watermark (PRCW)~\cite{gunn2024undetectable} & $\times$ / $\checkmark$ & $\times$ / $\times$ & $\checkmark$ / $\checkmark$ & $\checkmark$ / $\checkmark$ \\
                WIND~\cite{arabi2024hidden} & $\times$ / $\checkmark$ & $\times$ / $\times$ & $\checkmark$ / $\checkmark$ & $\checkmark$ / $\checkmark$ \\
                \hline\hline
                TR-SynTag~\cite{fang2025syntag} & $\checkmark$ / $\checkmark$ & $\checkmark$ / $\checkmark$ & $\checkmark$ / $\checkmark$ $(\dagger\times2)$ & $\checkmark$ / $\checkmark$ $(\dagger)$ \\
                GS-SynTag~\cite{fang2025syntag} & $\times$ / $\checkmark$ & $\times$ / $\times$ & $\checkmark$ / $\checkmark$ $(\dagger\times2)$ & $\checkmark$ / $\checkmark$ $(\dagger)$ \\
                TR-CoSDA~\cite{fang2025cosda} & $\checkmark$ / $\checkmark$ & $\checkmark$ / $\checkmark$ & $\checkmark$ / $\checkmark$ ($*,\triangle$) & $\checkmark$ / $\checkmark$ \\
                GS-CoSDA~\cite{fang2025cosda} & $\times$ / $\checkmark$ & $\times$ / $\times$ & $\checkmark$ / $\checkmark$ ($*,\triangle$) & $\checkmark$ / $\checkmark$ \\
                MaXsive~\cite{mao2025max} & $\times$ / $\checkmark$ & $\times$ / $\times$ & $\checkmark$ / $\checkmark$ ($\dagger,\triangle$) & $\checkmark$ / $\checkmark$ \\
                WMAdaptor~\cite{ci2024wmadapter} & $\times$ / $\times$ & $\times$ / $\times$ & $\times$ / $\times$ & $\checkmark$ / $\checkmark$ $(*)$ \\
                LaWa~\cite{rezaei2024lawa} & $\times$ / $\times$ & $\times$ / $\times$ & $\times$ / $\times$ & $\checkmark$ / $\checkmark$ $(*)$ \\
                Stable Signature~\cite{fernandez2023stable} & $\times$ / $\times$ & $\times$ / $\times$ & $\times$ / $\times$ & $\checkmark$ / $\checkmark$ $(*)$ \\
                HiDDeN~\cite{zhu2018hidden} & $\times$ / $\times$ & $\times$ / $\times$ & $\times$ / $\times$ & $\checkmark$ / $\checkmark$ $(*)$ \\
                \bottomrule
            \end{tabular}
        \end{table*}

        The derivation is in Sec.~\ref{Sec: proofs for thm lem} of Appendix.
        For $\hat{\mathbf{w}}^{\prime}$ to be identified as the embedded watermark under a given FPR and corresponding threshold $R$, we need $\|\mathbf{w}-\hat{\mathbf{w}}^{\prime}\| \leq R$, where $r<R$. To clarify the difference between $R$ and $r$, $R$ describes the threshold of a given FPR that encloses the distance between the watermark $\mathbf{w}$ and any perturbed watermark $\hat{\mathbf{w}}^{\prime}$, whereas $r$ is the upper bound of the distance between any two unperturbed watermarks $\mathbf{w}$ and $\mathbf{w}^{\prime}$, and the distance between $\mathbf{w}$ and its perturbed counterpart $\hat{\mathbf{w}}$.

        To connect $R$ and Eq.~(\ref{eq: rev inv total bound}) for analyzing the robustness of watermark $\mathbf{w}$, we first look at the Lipschitz constants of each component used in the reverse and inverse processes. Actually, the Lipschitz constant $L_{\mathcal{M}}$ of masking operation $\mathcal{M}$ is less than $1$. However, from Theorem~\ref{thm: global L bound}, we know that the reverse and inverse processes are expansive mappings, with both $L_{f_\theta}$ and $L_{f^{-1}_\theta}$ being larger than $1$.

        \begin{finding}
            Suppose that an image is attacked by two perturbations, including $\mathcal{T}_{\mathrm{min}}$, which has the smallest Lipschitz constant $L_{\mathcal{T}_{\mathrm{min}}}$ among $\mathcal{S}_{\mathcal{T}}$, and $\mathcal{T}_{\mathrm{max}}$, which has the largest Lipschitz constant $L_{\mathcal{T}_{\mathrm{max}}}>L_{\mathcal{T}_{\mathrm{min}}}$ among $\mathcal{S}_{\mathcal{T}}$. We then can derive from Eq.~(\ref{eq: rev inv total bound}) to have:
            \begin{align}\label{eq: inequality detection bound}
                \|\mathbf{w}-\hat{\mathbf{w}}^{\prime}\| \leq (1+L_{\mathcal{T}_{\mathrm{min}}})L_{\mathbf{F}^{-1}}L_{\mathbf{F}} r
                \leq (1+L_{\mathcal{T}_{\mathrm{max}}})L_{\mathbf{F}^{-1}}L_{\mathbf{F}} r.
            \end{align}
        \end{finding}

        \begin{finding}
            Eq.~(\ref{eq: inequality detection bound}) shows that whether a perturbed image $\mathcal{T}(\mathbf{I}^{\prime})$ can be identified as a watermarked image depends on $\mathcal{M}$, $\mathcal{W}$, and $R$. That is, given a diffusion model, $\mathcal{M}$, $\mathbf{w}$, and $R$, the watermark is robust against any perturbation $\mathcal{T} \in \mathcal{S}_{\mathcal{T}}$ whose Lipschitz constant is at most $\frac{R}{L_{\mathbf{F}^{-1}}L_{\mathbf{F}}r}-1$ such that
        \begin{align}\label{eq: ideal detection threshold}
            \|\mathbf{w}-\hat{\mathbf{w}}^{\prime}\| &\leq (1+L_{\mathcal{T}_{\mathrm{min}}})L_{\mathbf{F}^{-1}}L_{\mathbf{F}} r
            \leq (1+L_{\mathcal{T}})L_{\mathbf{F}^{-1}}L_{\mathbf{F}} r \leq R.
        \end{align}
        \end{finding}
        

        From Eq.~(\ref{eq: ideal detection threshold}), the impact of perturbations depends on the modules used in $\mathbf{F}$ and $\mathbf{F}^{-1}$. In particular, expansive modules employed to extract the estimated watermark degrade watermarking robustness. Constraining these modules to be non-expansive in diffusion models, however, restricts the diversity of the generated samples, revealing an inherent trade-off in latent-based watermarking.

        \begin{remark}\label{remark: expansive map}
            Latent-based watermarking uses expansive modules that amplify perturbations in the pixel space, whereas using non-expansive modules results in generated samples with limited diversity.
        \end{remark}

        The ratio between $R$ and $r$ further limits robustness under a low FPR, which requires a smaller $R$, as discussed in Sec.~\ref{dist shift} below. When $R\leq r$, the detector inevitably classifies a non-perturbed watermarked image as unwatermarked based on the estimated watermark. This behavior highlights a fundamental limitation of latent-based watermarking in practical settings.

        \begin{remark}\label{remark: r & R}
            The factor $r$ is always present due to inevitable errors in diffusion models. When the latent-based watermark is used in real-world applications, $R$ cannot be smaller than $r$.
        \end{remark}

        \textbf{\textit{Design Guideline}}
        Based on our theoretical and empirical analyses, we fundamentally question the long-term viability of latent-based watermarking paradigm for AI-generated content. Our findings indicate that any robust watermarking scheme must navigate three inherent mathematical constraints---all of which latent-based methods structurally violate. 
        
        First, the reverse and inverse processes of diffusion models lack perturbation invariance, rendering their latent representations
        highly sensitive to even minor spatial distortions. Second, the expansive nature of these processes (with Lipschitz constants $L > 1$) mathematically guarantees that any pixel-space attack or numerical error will be
        significantly amplified during latent extraction. Consequently, both injection and detection must bypass such expansive modules. Third, latent-based methods face an
        inherent trade-off: injecting a more robust signal requires shifting the model's sampling distribution, which
        consequently degrades generation quality and
        increases the FID score~\cite{heusel2017gans}. 
        
        In light of these
        fundamental bottlenecks, the continued pursuit of latent-based watermarking appears mathematically counterproductive. A practical alternative is readily available: several existing post-hoc watermarking methods~\cite{10045698,10321686,10663684,10516610,10539606,souvcek2025pixel} already meet these rigorous guidelines. By operating entirely outside the
        sensitive latent space and avoiding expansive diffusion modules, these existing post-hoc schemes naturally bypass the aforementioned vulnerabilities, ensuring high detection reliability while strictly adhering to prescribed fidelity constraints ({\em e.g.,} PSNR or FID).

        We provide a comparison of different watermarking methods that describe all the operations involved in the generation and detection processes in Table~\ref{Table: comparison of operations}. In the upper block of Table~\ref{Table: comparison of operations}, we include the methods that do not train and fine-tune any components, and embed a designed watermark pattern in the latent representation at time step $T$. The lower block indicates
        methods that introduce extra modules, modify the sampling process $f_\theta$, or leverage the \textit{noise layer} to improve the robustness of watermarking. However, none of the methods in Table~\ref{Table: comparison of operations} satisfy the aforementioned guidelines.

    \subsection{Effects of $R$ and Distribution Shift in Latent Space}\label{dist shift}
        We can see that setting $R$ is closely related to Eq.~(\ref{eq: rev inv total bound}) and Eq.~(\ref{eq: ideal detection threshold}), which determines the robustness of the watermarking scheme against perturbations. Hence, we discuss the effect of $R$ and its connection with the False Positive rate (FPR). Suppose we have two datasets: one with watermarked images $\mathcal{D}_w$; the other with unwatermarked images $\mathcal{D}$. According to~\cite{an2024waves}, the detection threshold was set based on the ratio of the number of false positive examples to the number of samples in $\mathcal{D}$.

        Specifically, since unwatermarked images are inverted by $\mathbf{F}^{-1}_{\theta}$ into the latent space, this induces a distance between the estimated watermark of sample from $\mathcal{D}$ and watermark $\mathbf{w}$. As the number of samples in $\mathcal{D}$ increases, the distribution of distances between the unwatermarked latents and ground-truth watermark $\mathbf{w}$ converges to a non-central $\chi$ distribution. Hence, we can set the threshold $R$ such that the ratio of area satisfying $\|\mathbf{w}-\hat{\mathbf{w}}^{\prime}\|\leq R$ to area under this non-central $\chi$ distribution equals the desired FPR. The condition $\|\mathbf{w}-\hat{\mathbf{w}}^{\prime}\|\leq R$ thus serves as the detection criterion.
        
        Given $R$ and based on Eq.~(\ref{max perturb bound}) in Sec.~\ref{Sec: proofs for thm lem} of Appendix, increasing the upper bound $L_{\mathbf{F}}r$ around the image $\mathbf{I}$, within which the bound covers more images and these images can be projected back to $\mathcal{B}_r(\mathbf{w})$, provides a wider area to tolerate perturbations. Thus, according to Eq.~(\ref{max perturb bound}), it is possible to increase the value of $r$ or norm of $\mathbf{w}$ in the latent space, which correspond to relaxing the criterion $\|\mathbf{w}-\hat{\mathbf{w}}\|$ or increasing the strength of the injected $\mathbf{w}$, respectively.

        To reveal the importance of $r$, we vary it from $0$ to $\infty$, allowing us to explore similarity based on the distance between $\mathbf{w}$ and inversed latent representation $\mathbf{w}^{\prime}$, as defined in Eq.~(\ref{max perturb bound}). In fact, setting $r$ is equivalent to defining the region in which latent representations lie within a distance $r$ of $\mathbf{w}$.
        If it is set to $\infty$, every latent representation is trivially considered close to $\mathbf{w}$, which can result in high FPRs. Interestingly, this condition occurs when channels are masked by $\mathcal{M}$. For example, Tree-Ring~\cite{wen2023treerings} and RingID~\cite{ci2024ringid} only consider the difference in one channel while discarding the rest of the channels. Hence, even if the rest of the channels differ significantly, as long as the difference in the selected dimension remains within the constraint, the data is still regarded as being watermarked. Please also refer to the visual explanations and the discussions of $\mathcal{B}_r(\mathbf{w})$ in Sec.~\ref{Sec: Ball constraints} of Appendix. Conversely, if it equals $0$, indicating $r=0$ without allowing any tolerance, which depicts the ``fragility'' scenario discussed in Sec. \ref{Sec: Fragility}.

        On the other hand, by moving $\mathbf{w}$ to another position, where $\|\mathbf{w}\|$ increases greatly while not increasing $r$, this does not enlarge the maximum perturbation bound of $\mathbf{I}$ accordingly, yet corresponds to a larger allowed perturbation in the pixel space. 
        Adding a shift on each latent representation not only moves the normal distribution $\mathcal{N}(\mathbf{0}, \mathbf{I})$ to $\mathcal{N}(\mathbf{\mu}_\mathbf{w}, \mathbf{\Sigma}_\mathbf{w})$ but also moves a set of $\mathbf{y}$'s to an area that induces a larger norm for every $\mathbf{y}_\mathbf{w}\in\mathcal{N}(\mathbf{\mu}_\mathbf{w}, \mathbf{\Sigma}_\mathbf{w})$, which creates a larger ``buffer'' between $\mathcal{N}(\mathbf{0}, \mathbf{I})$ and $\mathcal{N}(\mathbf{\mu}_\mathbf{w}, \mathbf{\Sigma}_\mathbf{w})$ that accommodates perturbed latent representations.
        As for moving to near the normal distribution, which shrinks the ``buffer,''
        there are higher chances that the latent representations are from $\mathcal{N}(\mathbf{0}, \mathbf{I})$. Hence, it is argued that the robustness of latent-based watermarking comes from shifting the normal distribution $\mathcal{N}(\mathbf{0}, \mathbf{I})$ to $\mathcal{N}(\mathbf{\mu}_\mathbf{w}, \mathbf{\Sigma}_\mathbf{w})$, as shown in Fig.~\ref{distribution shift} in Sec.~\ref{sim of dist shift} of Appendix. We also provide an alternative explanation of the detection mechanism from the perspective of forgery and removal attacks in Sec.~\ref{lens of f r attack} of Appendix.

\section{Experiments}\label{exps}
    In Sec.~\ref{rethink latent wm}, our theoretical analysis shows that the latent space of diffusion models is not invariant to perturbations in the pixel space. We provide Eqs.~(\ref{eq: inequality detection bound}) and~(\ref{eq: ideal detection threshold}) in Sec.~\ref{max bound} to describe distortions in the latent space stemming from perturbations in the pixel space.
    To examine our theoretical findings, we provide empirical results using different perturbations and parameter settings on existing watermarking schemes. Our goal is to reveal why latent-based watermarking methods require auxiliary modules to withstand perturbations and identify the components that limit their robustness. Therefore, we selected some of the methods in the upper block of Table~\ref{Table: comparison of operations} as our evaluation targets. The implementation details are first described in Sec. \ref{implements}. In Secs.~\ref{Sec: ordered perturbations on w-hatw'},~\ref{Sec: Mask and wm effects}, and~\ref{Sec: Lip effects}, we show that Eq.~(\ref{eq: ideal detection threshold}) provides explanations regarding the effects of perturbations and Lipschitz constant $L_{\mathcal{T}}$, mask $\mathcal{M}$, and watermark pattern $\mathbf{w}$ on the distance $\|\mathbf{w}-\hat{\mathbf{w}}^{\prime}\|$. Also, we can observe the roles of $r$ and $R$ mentioned in Secs.~\ref{max bound} and~\ref{dist shift}. 

    \subsection{Implementation Details}\label{implements}
        All experiments were conducted in Docker containers, each equipped with an Intel\textsuperscript{\textregistered} Xeon\textsuperscript{\textregistered} Gold 6154 CPU and a single NVIDIA Tesla V100-SXM2-32GB GPU. The implementations used the Diffusers~\cite{von-platen-etal-2022-diffusers} package from Hugging Face\footnote{\url{https://huggingface.co/}}, along with the codebases from Tree-Ring~\cite{wen2023treerings}, RingID~\cite{ci2024ringid}, PRCW~\cite{gunn2024undetectable}, and HSTR/QR~\cite{lee2025semantic}. We adopted Stable Diffusion v2-1~\cite{rombach2021high} from Diffusers. The prompt dataset is the stable-diffusion-prompts from Gustavosta\footnote{\url{https://huggingface.co/datasets/Gustavosta/Stable-Diffusion-Prompts}}. The image distortions were implemented with Pillow\footnote{\url{https://pillow.readthedocs.io/en/stable/}}, torchvision~\cite{torchvision2016} , and Stirmark~\cite{petitcolas1998attacks}.

        To investigate how the strength of attacks affects the detection, each attack was conducted with five parameter settings. Additionally, we included two types of regeneration from WAVES~\cite{an2024waves}: Regen-VAE and Regen-DiffPure. See Sec.~\ref{Sec: exp settings} of Appendix for detailed descriptions. Due to limited space, we present five perturbations, each with three different parameter settings. Please refer to Secs.~\ref{sec: full figure exp},~\ref{Sec: w norm increase on RID HSTR}, and~\ref{sec: full lip exp} of Appendix for complete experiments.

        For experiment settings, we set the ring radius as $0\sim10$ for Tree-Ring~\cite{wen2023treerings} and $3\sim14$ for RingID~\cite{ci2024ringid}. The reverse process used $50$ steps to generate watermarked and unwatermarked images with guidance scales set to $7.5$. We followed the settings in~\cite{wen2023treerings,ci2024ringid,lee2025semantic,gunn2024undetectable} and sampled 1,000 pairs of these images to compute the detection thresholds of 1e-2 in terms of FPR. Before watermark detection, we utilized the DDIM inversion~\cite{dhariwal2021diffusion, song2020denoising, wen2023treerings} to execute the inverse process with the same $50$ steps from images to latent representations with guidance scales set to $1.0$, which followed the same settings in~\cite{huang2024robin, wen2023treerings}.

        \begin{figure}[t]
            \begin{subfigure}[c]{0.185\textwidth}
                \centering\subfloat{\centering\makebox[10pt][l]{\raisebox{7pt}{\rotatebox[origin=l]{90}{\tiny Jitter}}}
                \includegraphics[width=\figWidth\linewidth]{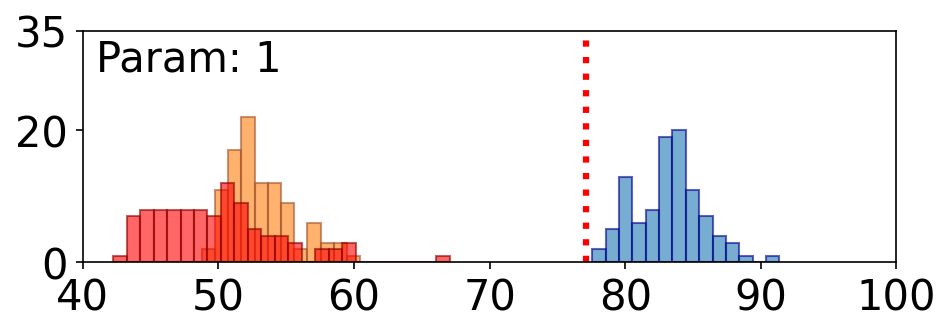}
                \includegraphics[width=\figWidth\linewidth]{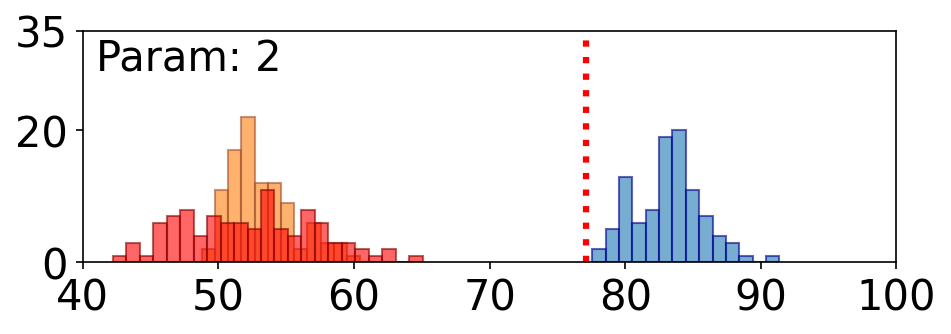}
                \includegraphics[width=\figWidth\linewidth]{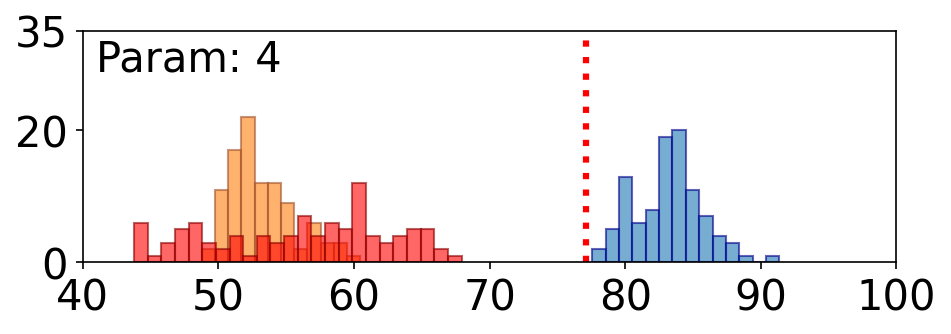}
                \includegraphics[width=\figWidth\linewidth]{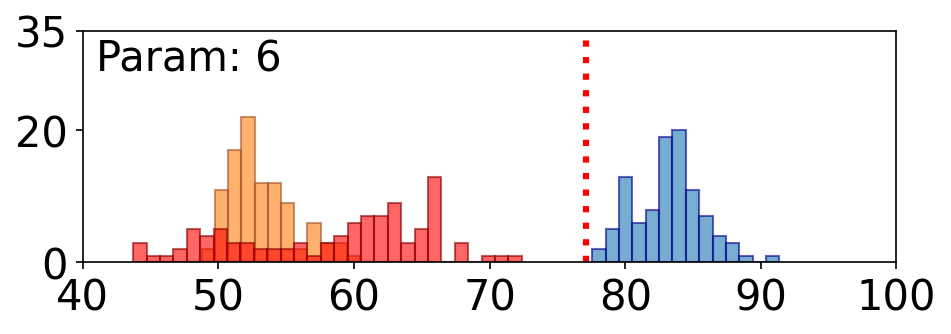}
                \includegraphics[width=\figWidth\linewidth]{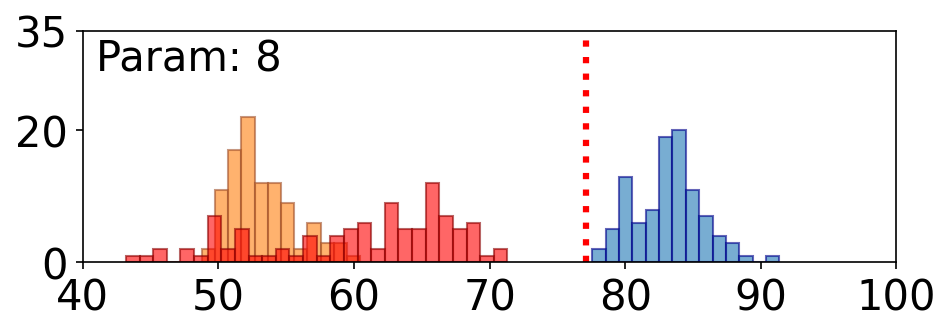}}

                \subfloat{\centering\makebox[10pt][l]{\raisebox{7pt}{\rotatebox[origin=l]{90}{\tiny JPEG}}}
                \includegraphics[width=\figWidth\linewidth]{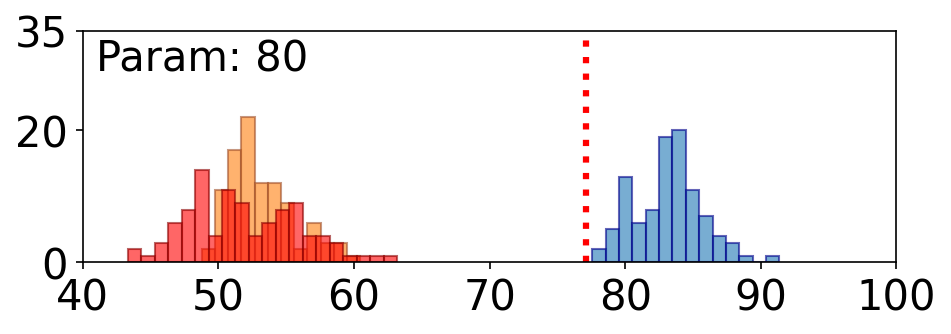}
                \includegraphics[width=\figWidth\linewidth]{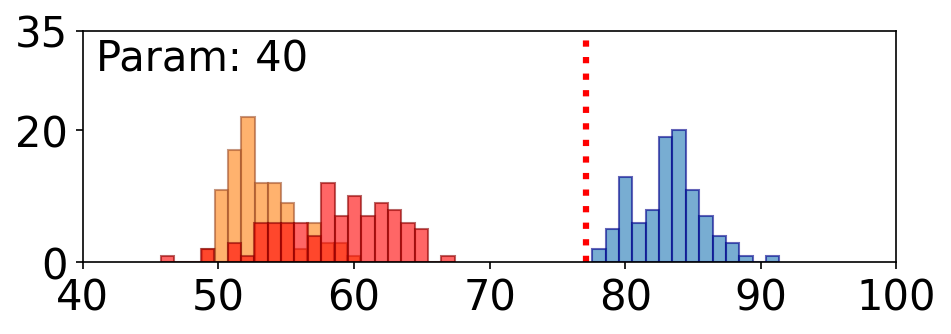}
                \includegraphics[width=\figWidth\linewidth]{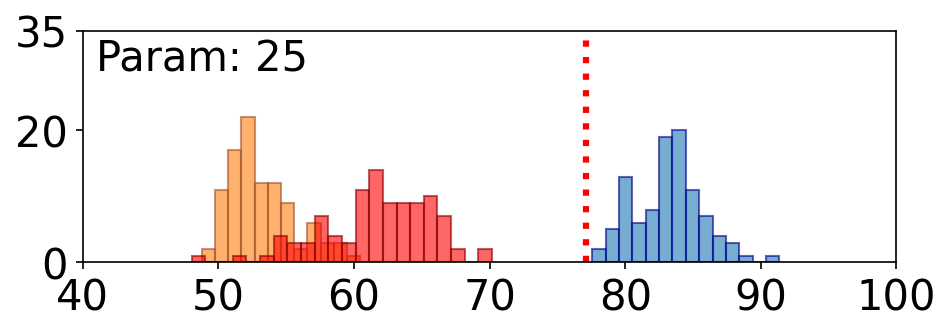}
                \includegraphics[width=\figWidth\linewidth]{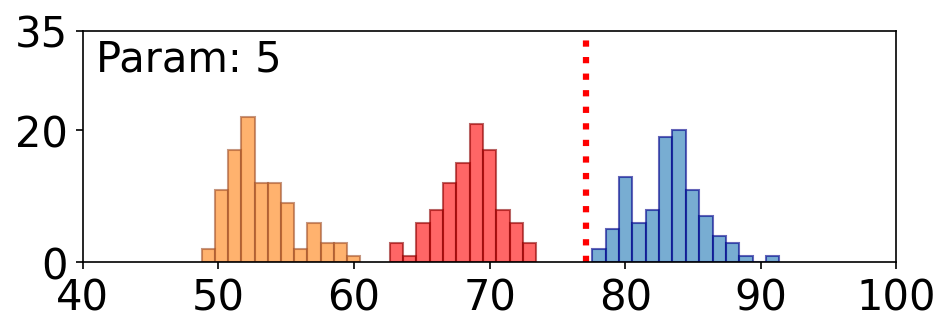}
                \includegraphics[width=\figWidth\linewidth]{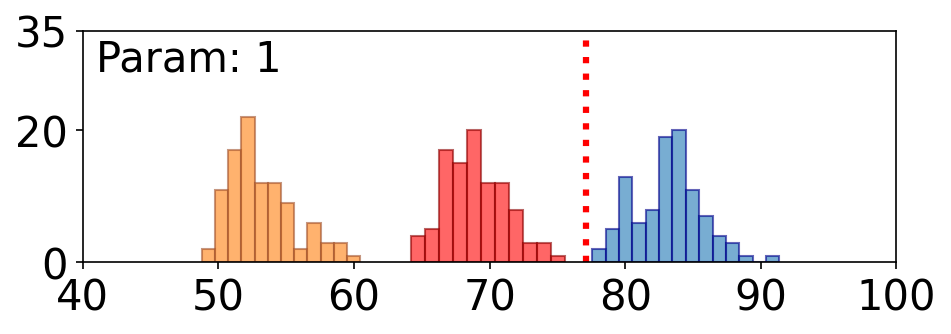}}

                \subfloat{\centering\makebox[10pt][l]{\raisebox{3pt}{\rotatebox[origin=l]{90}{\tiny Rotation}}}
                \includegraphics[width=\figWidth\linewidth]{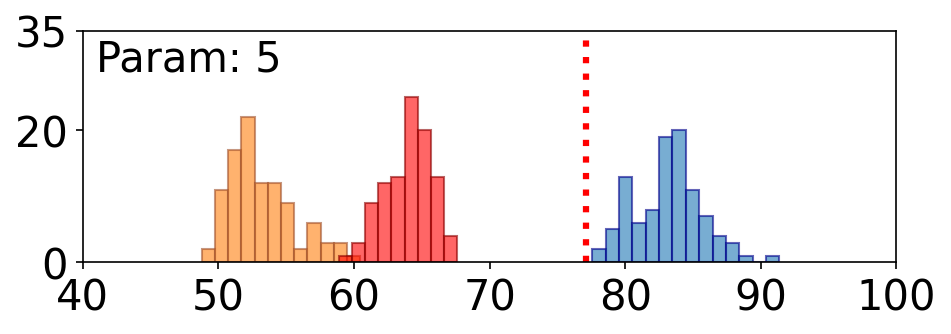}
                \includegraphics[width=\figWidth\linewidth]{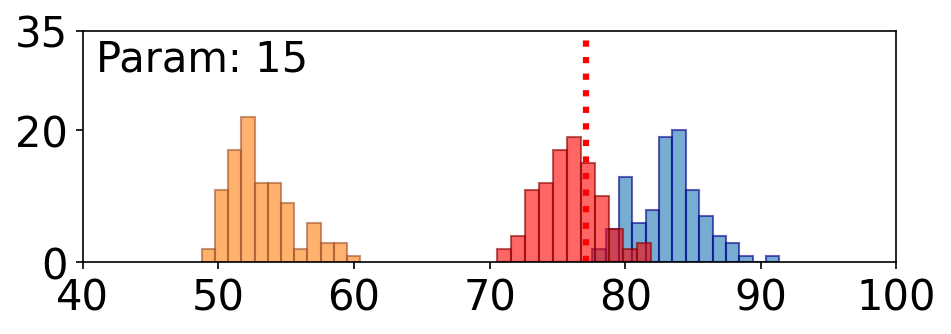}
                \includegraphics[width=\figWidth\linewidth]{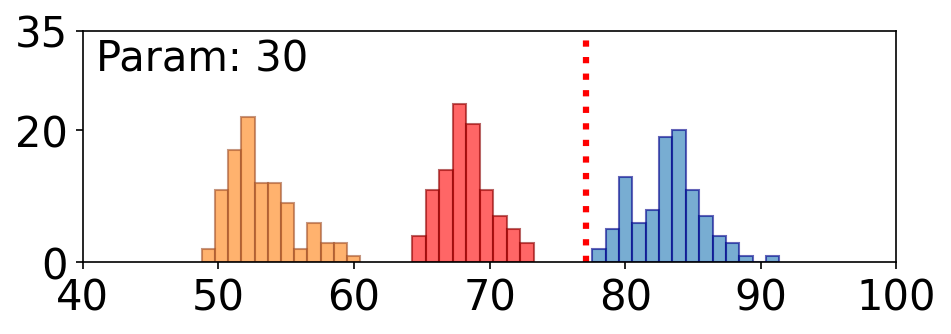}
                \includegraphics[width=\figWidth\linewidth]{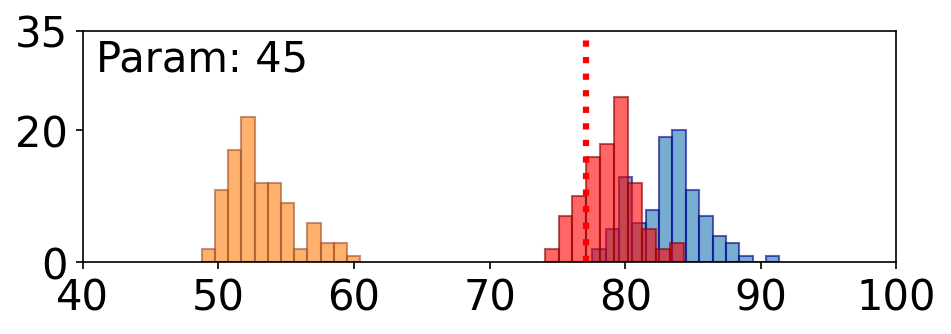}
                \includegraphics[width=\figWidth\linewidth]{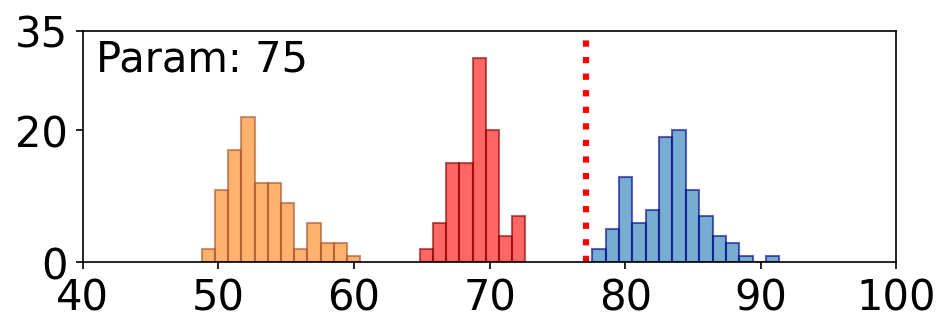}}

                \subfloat{\centering\makebox[10pt][l]{\raisebox{3pt}{\rotatebox[origin=l]{90}{\tiny Cropping}}}
                \includegraphics[width=\figWidth\linewidth]{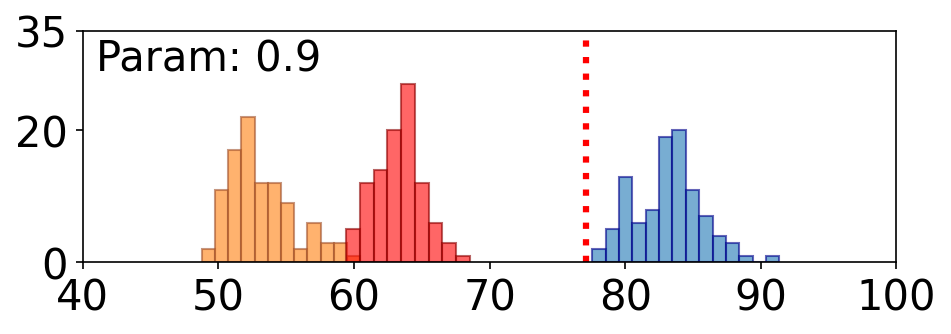}
                \includegraphics[width=\figWidth\linewidth]{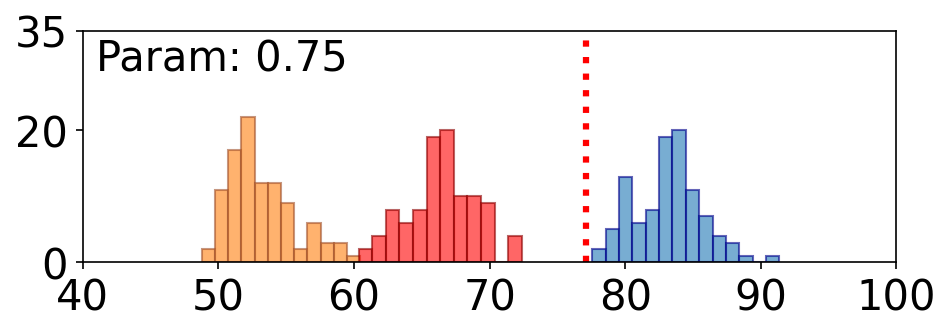}
                \includegraphics[width=\figWidth\linewidth]{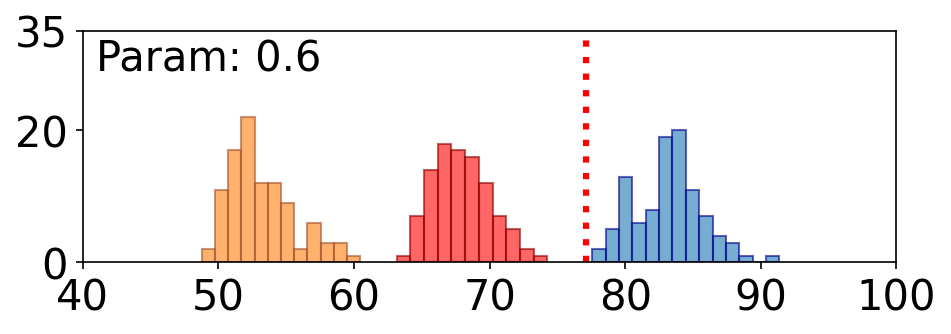}
                \includegraphics[width=\figWidth\linewidth]{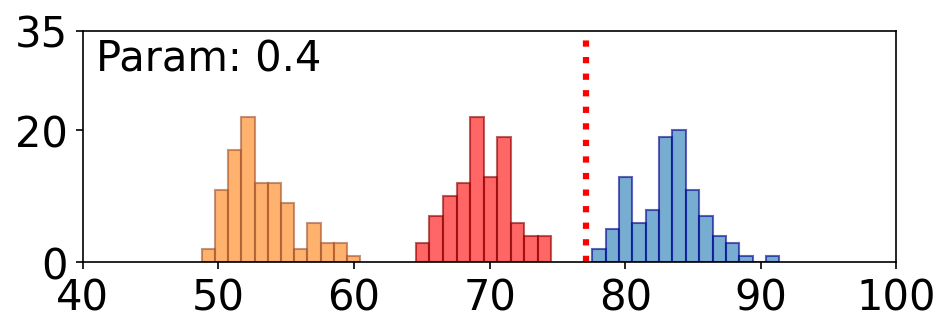}
                \includegraphics[width=\figWidth\linewidth]{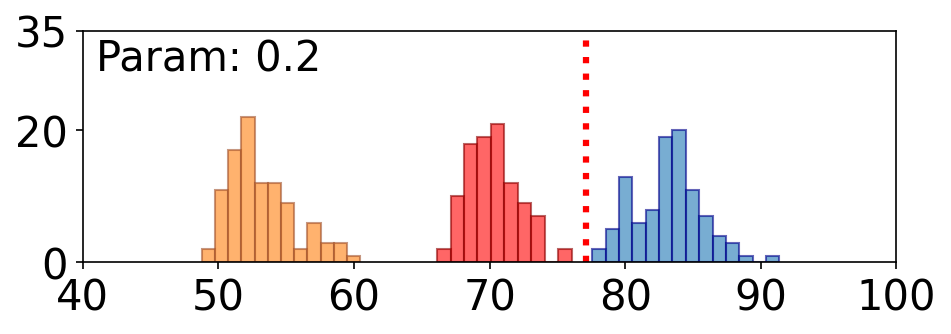}}

                \subfloat{\centering\makebox[10pt][l]{\raisebox{0pt}{\rotatebox[origin=l]{90}{\tiny Regen-VAE}}}
                \includegraphics[width=\figWidth\linewidth]{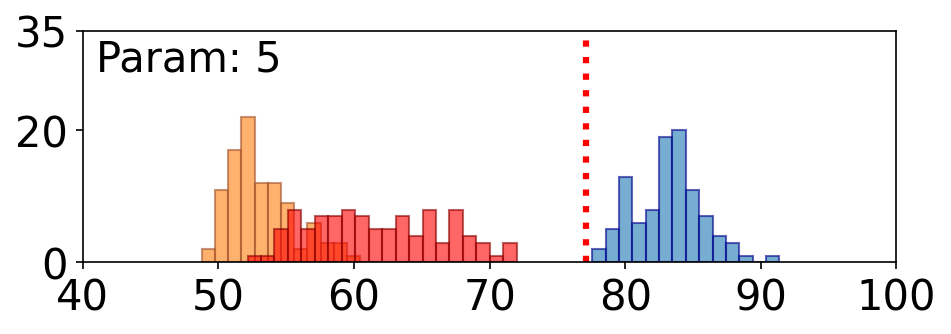}
                \includegraphics[width=\figWidth\linewidth]{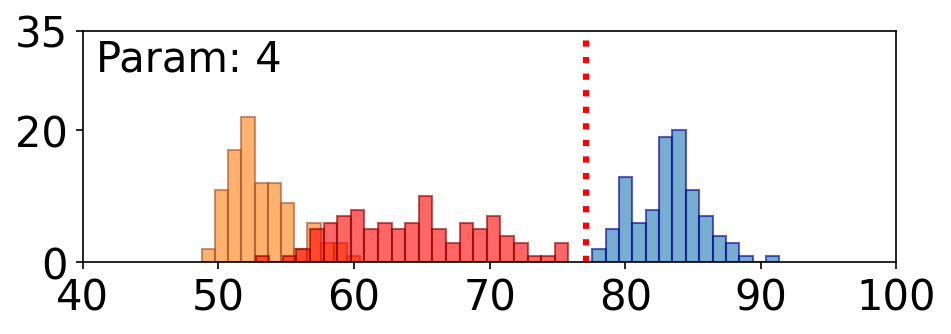}
                \includegraphics[width=\figWidth\linewidth]{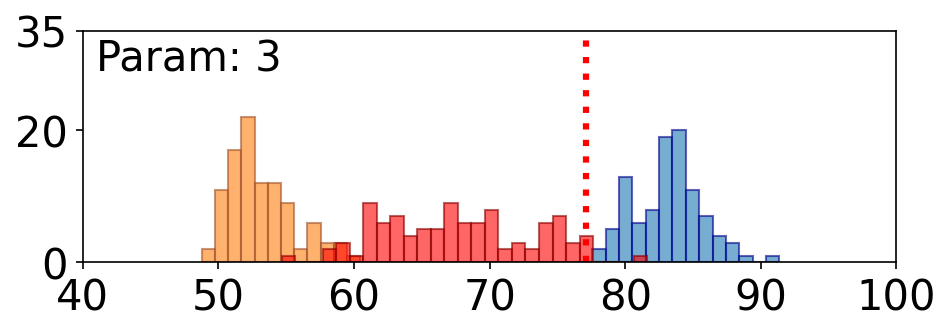}
                \includegraphics[width=\figWidth\linewidth]{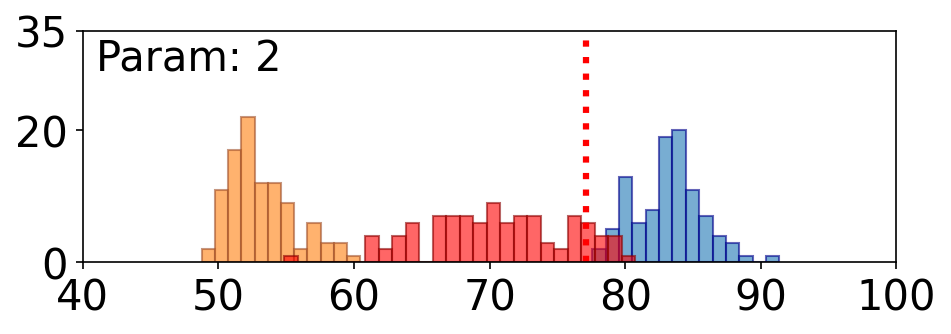}
                \includegraphics[width=\figWidth\linewidth]{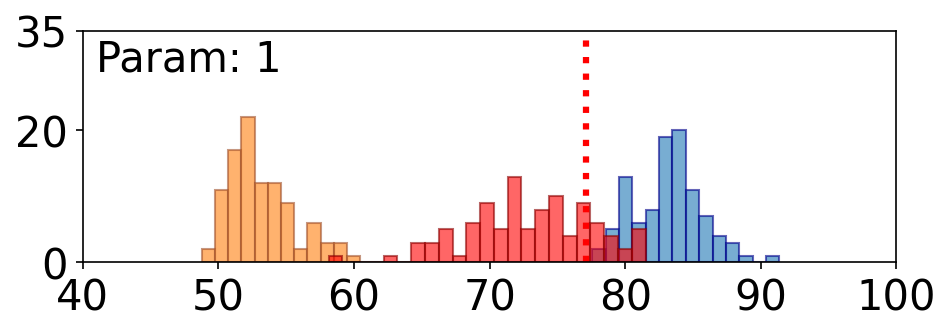}}
            \end{subfigure}
            \caption{Effects of different perturbations on $\|\mathbf{w}-\hat{\mathbf{w}}^{\prime}\|$. From left to right, the perturbations were set with different parameters (Param). Watermark: Tree-Ring. Thresholds: 77.12 (1e-2 FPR).}\label{fig: TR attack shifting subfig}
        \end{figure}

    \subsection{Effects of perturbations on detection distance}\label{Sec: ordered perturbations on w-hatw'}
        In this experiment, we analyzed the implications of Eq.~(\ref{eq: ideal detection threshold}) in that the maximum perturbed distance $\|\mathbf{w}-\hat{\mathbf{w}}^{\prime}\|$ can be ordered by the magnitudes of applied perturbations. We used Tree-Ring~\cite{wen2023treerings} and HSQR~\cite{lee2025semantic} as representative examples, with the corresponding results presented in Figs.~\ref{fig: TR attack shifting subfig} and~\ref{fig: HSQR attack shifting subfig}. For PRCW~\cite{gunn2024undetectable}, see Figs.~\ref{fig: PRCW decoding 1e-2fpr attack shifting},~\ref{fig: PRCW detection 1e-2fpr attack shifting},~\ref{fig: PRCW decoding 1e-5fpr attack shifting}, and~\ref{fig: PRCW detection 1e-5fpr attack shifting} in Sec.~\ref{sec: full figure exp} of Appendix.
        
        In these figures, the orange distribution denotes the distance between the watermarked latent and ground-truth watermark, the blue distribution corresponds to the distance between the unwatermarked latent and ground-truth watermark, and the red distribution represents the distances between the perturbed watermarked latent and ground-truth watermark. The red dotted lines indicate the detection threshold $R$ at FPRs of 1e-2. Ideally, the orange distribution should be concentrated at $0$. However, numerical errors, different guidance scales, stochastic sampling effects, and value clipping introduced during the reverse and inverse processes cause the estimated watermark to deviate from $\mathbf{w}$, resulting in a nonzero shift of $\|\mathbf{w}-\hat{\mathbf{w}}^{\prime}\|$. These deviations, captured by the term $r$, are explicitly accounted for in our derivation in Sec.~\ref{max bound}. We also present experiments of these deviations on different watermarking schemes in Table~\ref{Table: numerical error} and Fig.~\ref{fig: numerical error} in Sec.~\ref{sec: full figure exp} of Appendix. It shows that as the number of time steps $t$ increases, the mean value and maximum of a set of $\|\mathbf{w}-\hat{\mathbf{w}}^{\prime}\|$ are slightly reduced. This provides a remedy to reduce the numerical errors by the number of time steps $t$ so that the orange distribution shifts towards $0$ and leaves more distance for unknown perturbations.

        It can be seen from Figs.~\ref{fig: TR attack shifting subfig} and~\ref{fig: HSQR attack shifting subfig} that increasing perturbation strength leads to progressively larger shifts of the red distributions. If the stricter threshold is adjusted from the red dotted line to the left of the red dotted line, several perturbations are observed to push the red distributions beyond the stricter detection thresholds. According to Eq.~(\ref{eq: ideal detection threshold}), a lower FPR setting, corresponding to a stricter threshold, excludes more perturbations than a looser threshold, thereby reducing the watermark robustness. However, in practice, a watermarking scheme should be robust to perturbations even under a lower FPR setting.

        \begin{figure}[t]
            \begin{subfigure}[c]{0.185\textwidth}
                \centering\subfloat{\centering\makebox[10pt][l]{\raisebox{5pt}{\rotatebox[origin=l]{90}{\tiny Jitter}}}
                \includegraphics[width=\figWidth\linewidth]{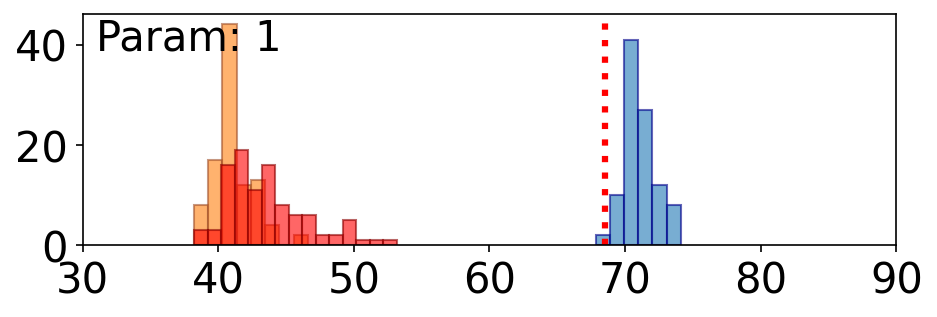}
                \includegraphics[width=\figWidth\linewidth]{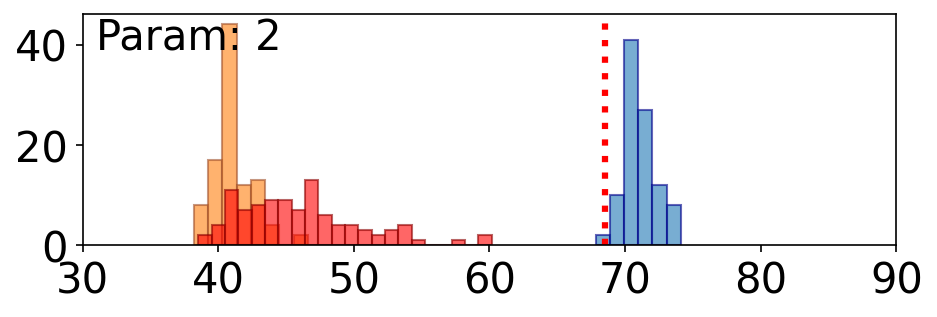}
                \includegraphics[width=\figWidth\linewidth]{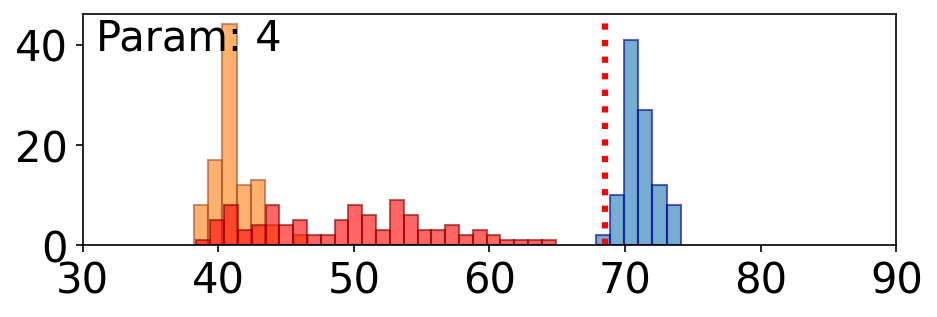}
                \includegraphics[width=\figWidth\linewidth]{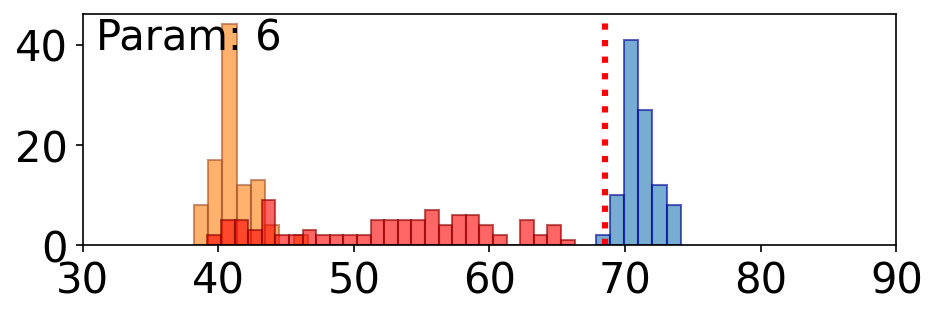}
                \includegraphics[width=\figWidth\linewidth]{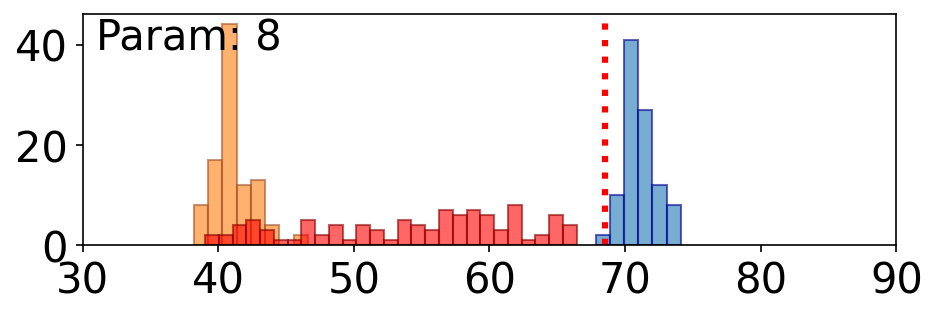}}
    
                \subfloat{\centering\makebox[10pt][l]{\raisebox{5pt}{\rotatebox[origin=l]{90}{\tiny JPEG}}}
                \includegraphics[width=\figWidth\linewidth]{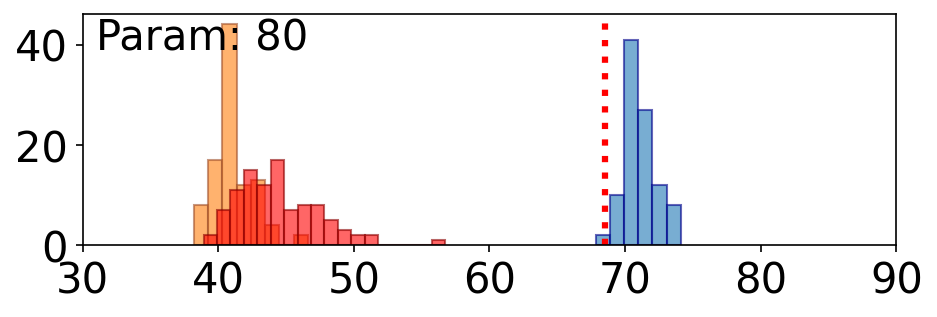}
                \includegraphics[width=\figWidth\linewidth]{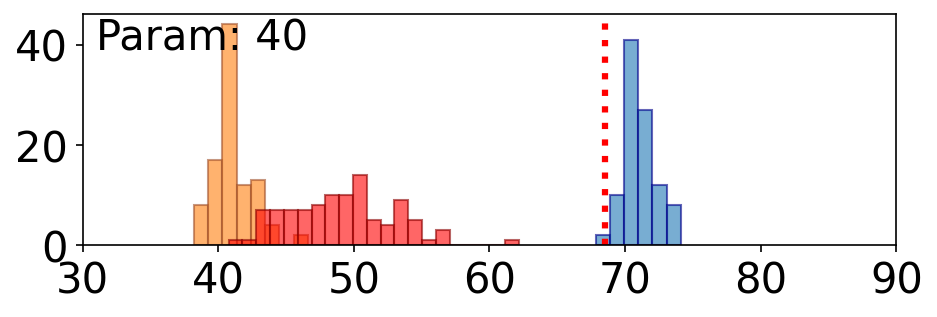}
                \includegraphics[width=\figWidth\linewidth]{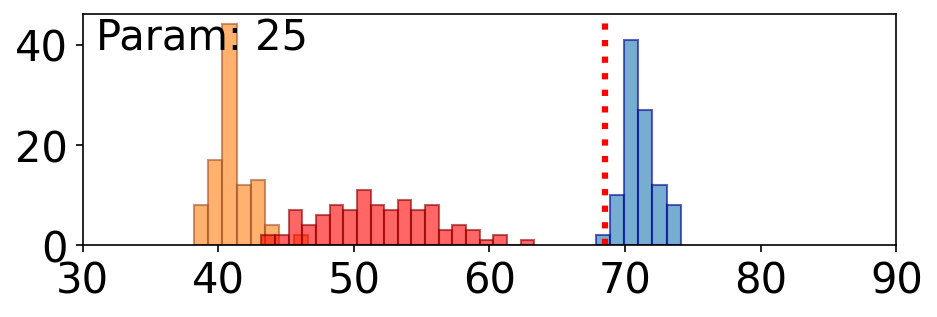}
                \includegraphics[width=\figWidth\linewidth]{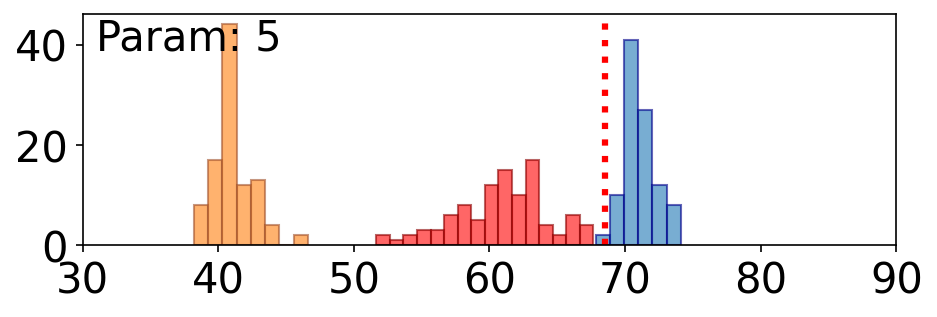}
                \includegraphics[width=\figWidth\linewidth]{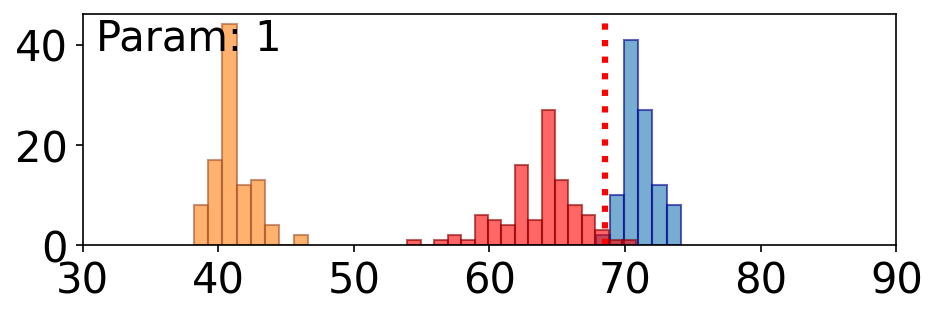}}
    
                \subfloat{\centering\makebox[10pt][l]{\raisebox{3pt}{\rotatebox[origin=l]{90}{\tiny Rotation}}}
                \includegraphics[width=\figWidth\linewidth]{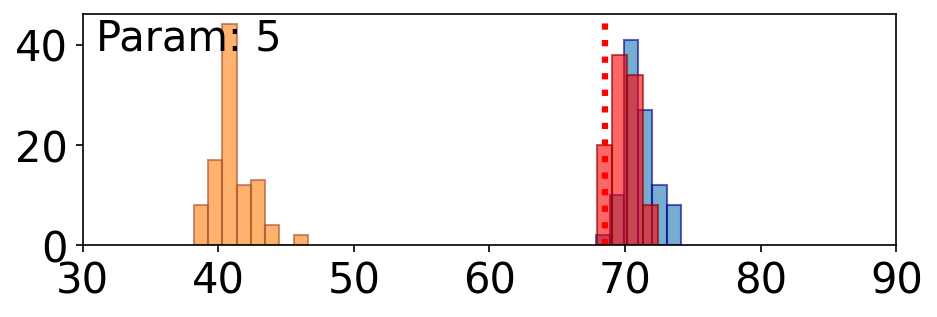}
                \includegraphics[width=\figWidth\linewidth]{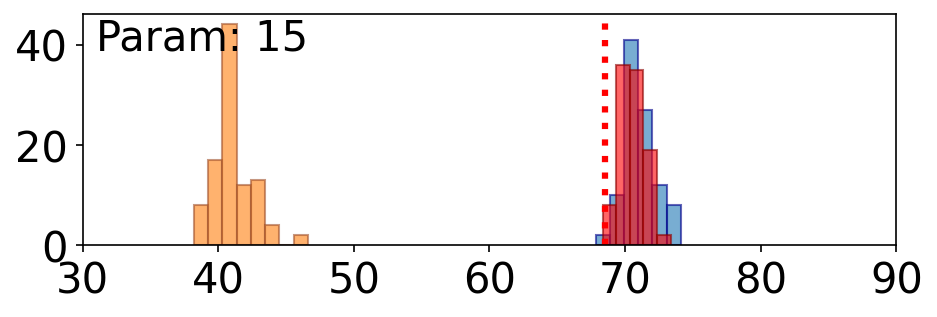}
                \includegraphics[width=\figWidth\linewidth]{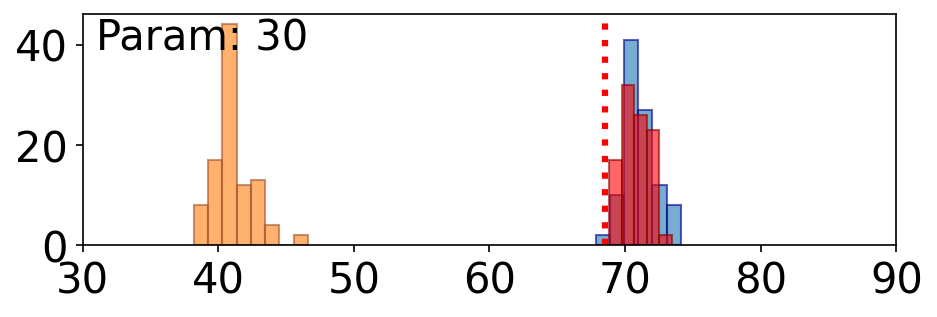}
                \includegraphics[width=\figWidth\linewidth]{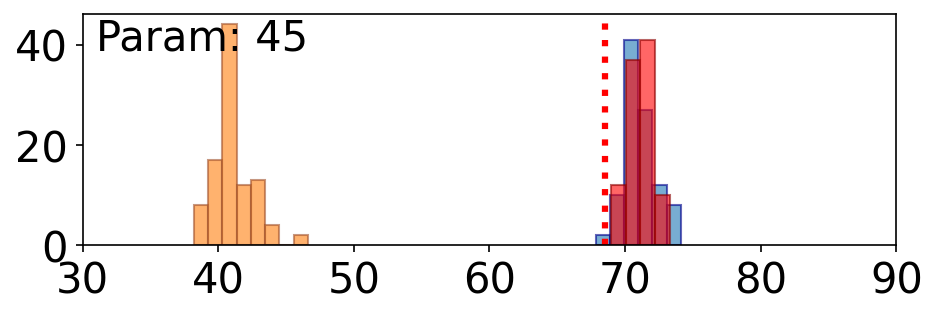}
                \includegraphics[width=\figWidth\linewidth]{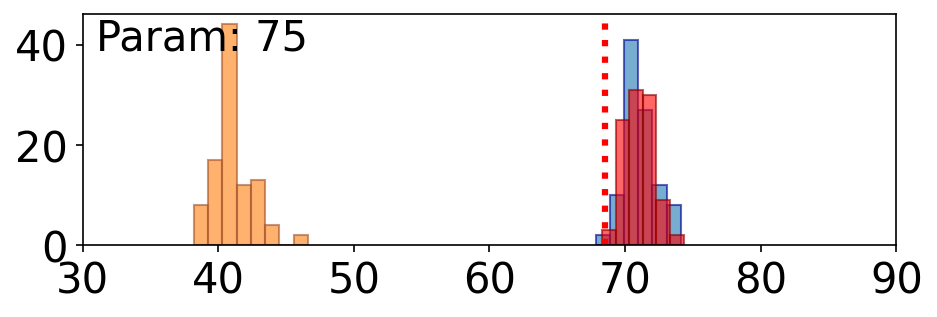}}
    
                \subfloat{\centering\makebox[10pt][l]{\raisebox{3pt}{\rotatebox[origin=l]{90}{\tiny Cropping}}}
                \includegraphics[width=\figWidth\linewidth]{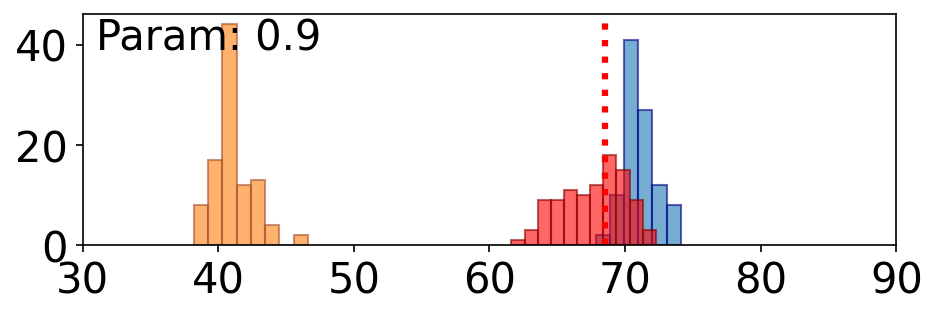}
                \includegraphics[width=\figWidth\linewidth]{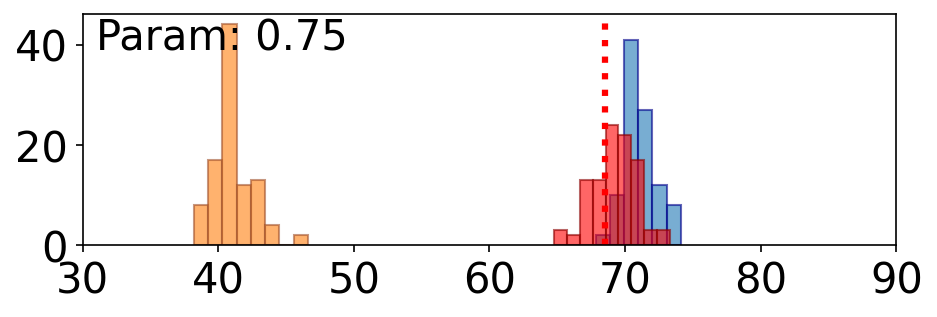}
                \includegraphics[width=\figWidth\linewidth]{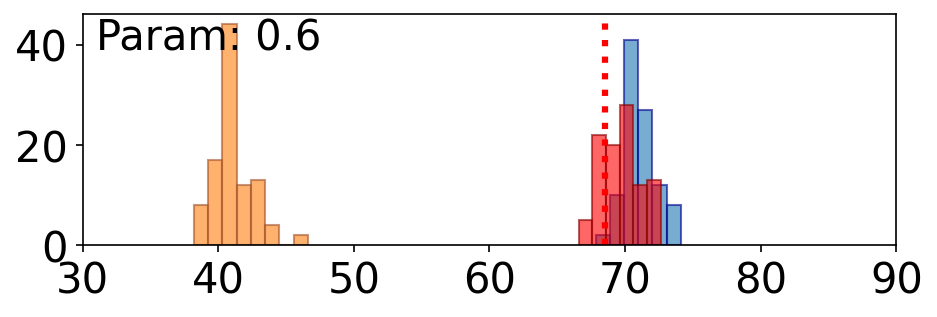}
                \includegraphics[width=\figWidth\linewidth]{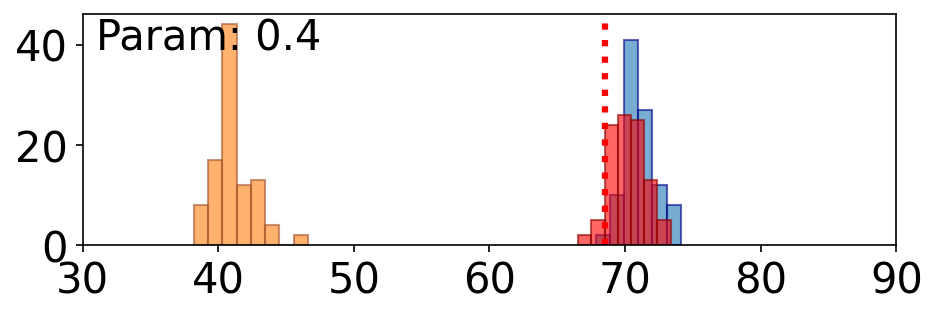}
                \includegraphics[width=\figWidth\linewidth]{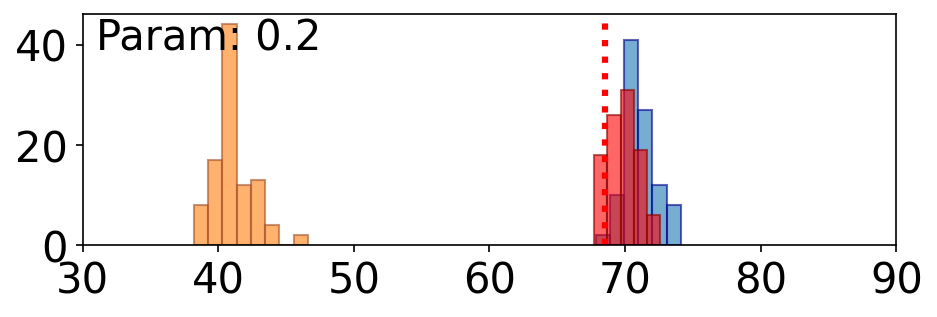}}

                \subfloat{\centering\makebox[10pt][l]{\raisebox{0pt}{\rotatebox[origin=l]{90}{\tiny Regen-VAE}}}
                \includegraphics[width=\figWidth\linewidth]{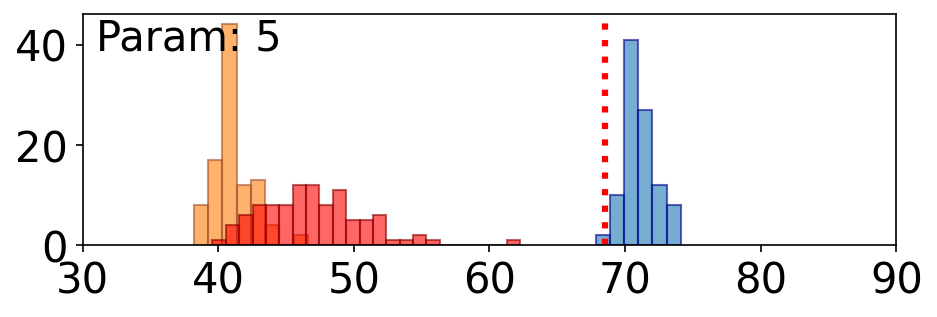}
                \includegraphics[width=\figWidth\linewidth]{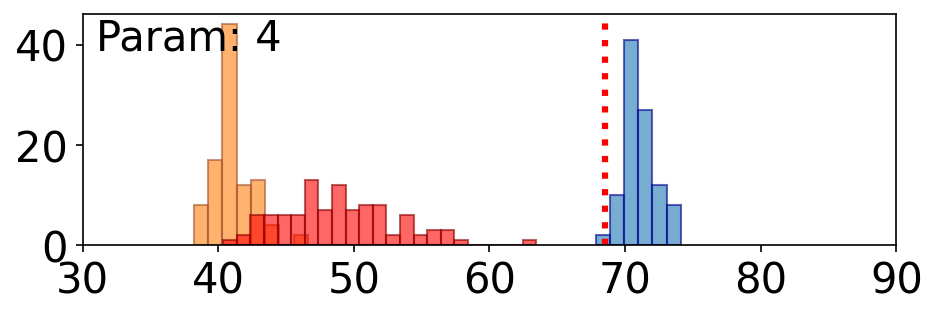}
                \includegraphics[width=\figWidth\linewidth]{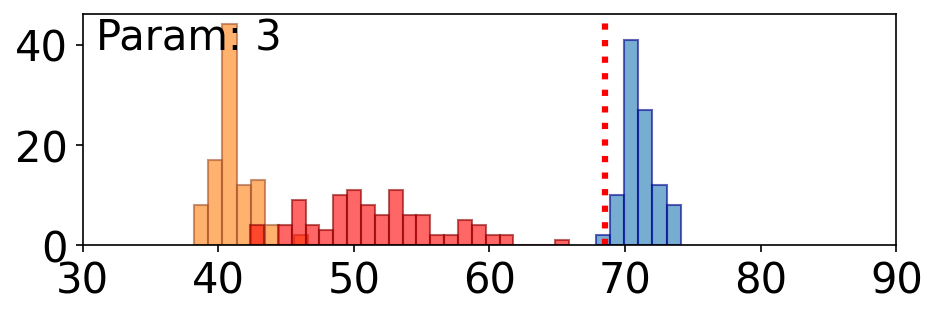}
                \includegraphics[width=\figWidth\linewidth]{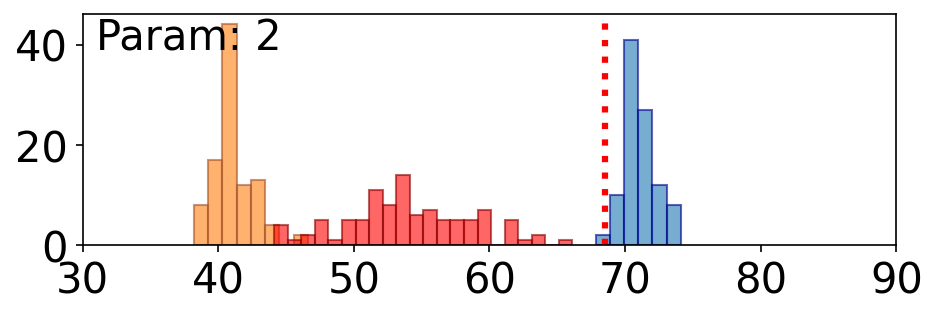}
                \includegraphics[width=\figWidth\linewidth]{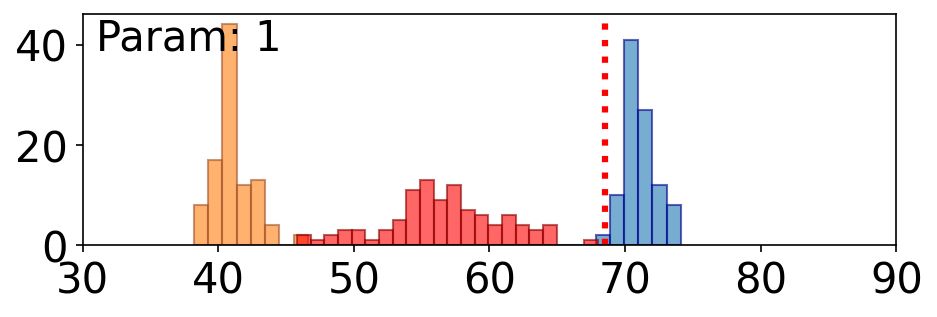}}
            \end{subfigure}
            \caption{Effects of different perturbations on $\|\mathbf{w}-\hat{\mathbf{w}}^{\prime}\|$. From left to right, the perturbations were set with different parameters (Param). Watermark: HSQR. Thresholds: 68.52 (1e-2 FPR).}\label{fig: HSQR attack shifting subfig}
        \end{figure}

    \subsection{Effects of Masking $\mathcal{M}$ and Watermark $\mathbf{w}$}\label{Sec: Mask and wm effects}
        We can observe from Figs.~\ref{fig: TR attack shifting subfig} and~\ref{fig: HSQR attack shifting subfig} that HSQR~\cite{lee2025semantic} is vulnerable to rotation but Tree-Ring~\cite{wen2023treerings} is not. The difference comes from the design of mask $\mathcal{M}$ and watermark $\mathbf{w}$. According to HSQR~\cite{lee2025semantic}, it utilizes a square mask and injects a QR-code-like pattern as the watermark. We hypothesize that the vulnerability comes from the desynchronization of the extracted watermark. However, even if the mask and pattern are circular and ring-like in Tree-Ring~\cite{wen2023treerings}, we further verified Theorems~\ref{thm:RST invariance} and~\ref{thm: non-perturb invariant} in Sec.~\ref{Sec: Fragility}. In Fig.~\ref{fig: TR attack shifting subfig}, the perturbation was set to rotate an image by 5 degrees, but the red distribution shifts from the orange distribution. Hence, we can conclude that theoretically and empirically, the diffusion process, unlike the Fourier and Fourier–Mellin transforms, is not rotation-invariant. Moreover, according to these experiments, the diffusion process is not invariant to perturbation.
        In Sec.~\ref{Sec: w norm increase on RID HSTR} of Appendix, according to our reasoning about $\|\mathbf{w}\|$ and the corresponding ``buffer'' mentioned in Sec.~\ref{dist shift}, we verify such an effect of moving a distribution toward the one with a larger norm by increasing the strength of injected watermarks.

        \begin{figure}[h]
            \centering
            \includegraphics[width=1.0\linewidth]{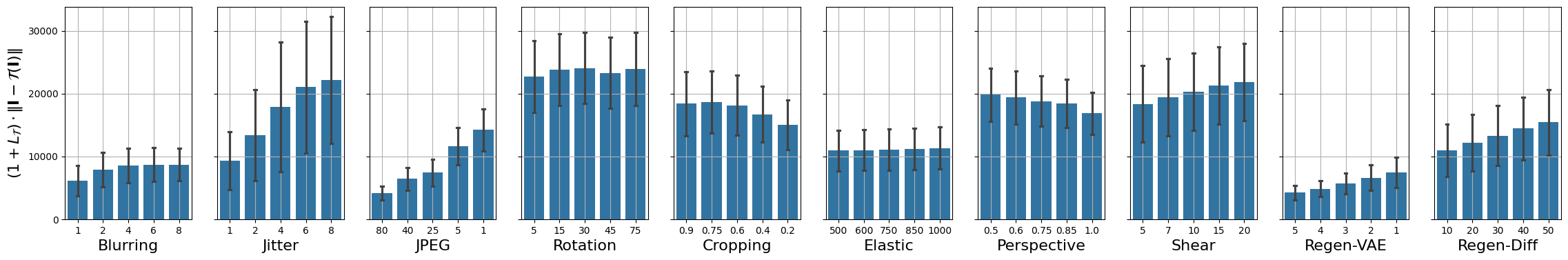}
            \caption{Experimental thresholds for different perturbations on $\|\mathbf{w}-\hat{\mathbf{w}}^{\prime}\|$, each with different parameters on x-axis. The y-axis represents the potential threat that $\hat{\mathbf{w}}'$ is getting far from $\mathbf{w}$, which corresponds to Eq.~(\ref{eq: ideal detection threshold}). Watermarking method: Tree-Ring.}\label{fig: experimental detection threshold}
        \end{figure}

    \subsection{Effects of Lipschitz constant $L_{\mathcal{T}}$}\label{Sec: Lip effects}
        As demonstrated in Eq.~(\ref{eq: rev inv total bound}) and Eq.~(\ref{eq: ideal detection threshold}), the distance $\|\mathbf{w}-\hat{\mathbf{w}}'\|$ is upper-bounded by a term proportional to $(1+L_\mathcal{T})r$ for a fixed $\mathbf{F}$, provided that the condition $\|\mathbf{I}-\mathcal{T}(\mathbf{I})\|\le L_\mathbf{F}r$ is satisfied. This leads us to investigate whether the size of $(1+L_\mathcal{T})\|\mathbf{I}-\mathcal{T}(\mathbf{I})\|$ is positively correlated with the separation of $\mathbf{w}$ from $\hat{\mathbf{w}}'$. The result for Tree-Ring is presented in Fig.~\ref{fig: experimental detection threshold} and those for other watermarking methods are shown in Figs.~\ref{fig: experimental detection threshold (ringid 64)} and~\ref{fig: experimental detection threshold (HSQR)}. Essentially, these patterns in these figures appear highly similar because the value $(1+L_\mathcal{T})\|\mathbf{I}-\mathcal{T}(\mathbf{I})\|$ is independent of the watermark pattern. In comparison with Fig.~\ref{fig: TR attack shifting subfig}, we observe that lower values on the $y$-axis in Fig.~\ref{fig: experimental detection threshold} correspond to smaller perturbation shifts, such as Blurring and high-quality factors of JPEG. Conversely, higher $y$-axis values typically indicate greater shifts, such as Cropping and certain Rotation cases. Interestingly, Jitter exhibits large standard deviations, suggesting that the shift varies significantly across different images. This effect corresponds to the dispersed distribution of the red bars in Jitter of Figs.~\ref{fig: TR attack shifting subfig} and~\ref{fig: HSQR attack shifting subfig}. For thresholds computed with different watermarking schemes, please refer to Sec.~\ref{sec: full lip exp} of Appendix.

\section{Conclusion and Limitations}\label{conclusion}
    In this paper, we establish that latent-based watermarking methods lack invariance to perturbations. Through both theoretical analysis and empirical evaluation, we show that their performance is strongly influenced by the components used in the diffusion reverse and inverse processes. Furthermore, our theoretical and experimental findings demonstrate that existing latent-based watermarking schemes, utilizing DDIM or other expansive modules, are fundamentally limited in practical settings, particularly when operating under low FPR constraints. Taken together, these results highlight fundamental limitations of current latent-based watermarking approaches and suggest that the prevailing design paradigm may be insufficient to achieve robust watermarking in practice.


    Our aim is not to hinder the development of latent-based watermarking methods, but to provide guidance for future research. We acknowledge that our analysis relies on assumptions about the MMSE denoiser and the invertibility of DDIM, which may not always hold in practice. Despite these limitations, the proposed framework offers a useful analytical perspective for identifying potential weaknesses in future watermarking designs.




\newpage
\clearpage

{
    \small
    \bibliographystyle{ieeenat_fullname}
    \bibliography{refs}
}


\newpage
\appendix

\section*{Appendix}
\section{Background}\label{sec: background knowledge}
    We present the background knowledge relevant to our research.

    \subsection{Rotation, Scale, and Translation Invariance}
        Geometric attacks (GAs), as one of the primary kinds of watermarking attacks~\cite{voloshynovskiy2001attack}, are greatly ignored nowadays in the evaluation of the robustness of generative watermarking. GAs aims to cause synchronization errors to disable watermark detection without removing embedded watermarks or degrading the quality of watermarked images. Rotation, scale, and translation are common global geometric distortions, which are included in Stirmark~\cite{petitcolas2000watermarking,petitcolas1998attacks}. In earlier studies,~\cite{ruanaidh1998rotation} presented a Fourier—Mellin-based approach to embed watermarks in resisting a combination of rotation and scaling transformations, whereas~\cite{lu2006media} proposed robust feature-based image watermarking with resistance to Stirmark benchmark.

    \subsection{DDIM Inversion}\label{sec:DDIM Inversion}
        Diffusion models~\cite{sohl2015deep} generate synthetic images through the reverse process, which progressively denoises the initial noise sampled from the standard Gaussian distribution. Since this reverse process is inherently not deterministic, it is virtually impossible to retrieve the initial noise that generates a given synthetic image. However, the introduction of DDIM~\cite{song2020denoising} renders this possible. \citet{song2020denoising} suggested that we can ``reverse'' the reverse process to obtain the initial noise for downstream applications that require it~\cite{ci2024ringid,mokady2023null,wang2023dire,wen2023treerings}. Note that in~\cite{song2020denoising}, this ``inverse process'' was only mentioned, without any explicit formulation. The first explicit formulation of the inverse process appeared in~\cite{dhariwal2021diffusion}, although it was not yet called the inverse process. Recently, in the literature on watermarking on AIGCs, the term ``inverse process'' was first introduced in Tree-Ring~\cite{wen2023treerings}. Since then, numerous methods exploiting the inverse process have been proposed~\cite{ci2024ringid,huang2024robin,wen2023treerings,yang2024gaussian,zhang2024attack,zhu2026towards,11462859,rebuffi2026learning,gao2026slice}.

\section{The Lipschitz bounds of DDIM Reverse and Inverse Process\\ (Prerequisite of Sec. \ref{Sec: proofs for thm lem})}\label{Derivation of MMSE-AMP and DDIM and Lipschitz}
In this section, we provide a derivation that connects VAMP and the MMSE denoiser to the DDIM reverse process, and prove that such a process is invertible, {\em a.k.a} the inverse process.
\subsection{Theoretical Analysis: Spectral Stability of Diffusion via MMSE-AMP}\label{Sec: Analysis}
In this section, we derive the spectral properties of the diffusion process by establishing an equivalence between the Denoising Diffusion Implicit Model (DDIM) and Vector Approximate Message Passing (VAMP). Our key insight is that DDIM acts as a \textit{normalized} iterative denoiser, where the normalization constant induces an expansive regime essential for generation. By treating the diffusion step as a regularized MMSE estimation update, we derive rigorous spectral bounds for both the generative (reverse) and encoding (inverse) processes.

The DDIM sampling process generates a sequence $\mathbf{x}_T \,(i.e.,\,\mathbf{y}), \mathbf{x}_{T-1}, \dots, \mathbf{x}_0$ via the update rule:
\begin{align}\label{eq:ddim_standard}
    \mathbf{x}_{t-1} =& \sqrt{\bar{\alpha}_{t-1}} \underbrace{\left( \frac{\mathbf{x}_t - \sqrt{1-\bar{\alpha}_t}\epsilon_\theta(\mathbf{x}_t)}{\sqrt{\bar{\alpha}_t}} \right)}_{\hat{\mathbf{x}}_0(\mathbf{x}_t)} + \sqrt{1-\bar{\alpha}_{t-1}} \epsilon_\theta(\mathbf{x}_t),
\end{align}
where $\bar{\alpha}_t:=\prod_{s=0}^t\alpha_s$, $\{\alpha_t\}_{t=0}^T\subset(0,1]$ is the noise schedule, and $\epsilon_\theta(\cdot)$ is the learned denoiser parameterized by $\theta$, and $\hat{\mathbf{x}}_0(\mathbf{x}_t)$ is the DDIM prediction at $t$.

\subsubsection{Connection: DDIM as Normalized AMP}
To apply the VAMP theory, we must align the signal models. Standard VAMP assumes a signal-plus-noise model $\mathbf{y} = \mathbf{x}_0 + \xi$, $\xi\sim\cN(\vzero,\sigma^2\rm Id)$. In contrast, the diffusion state $\mathbf{x}_t$ scales the signal by $\sqrt{\alpha_t}$. 
We define the \textbf{Normalized AMP State} $\mathbf{u}_t$ as:
\begin{align*}
    \mathbf{u}_t = \frac{\mathbf{x}_t}{\sqrt{\bar{\alpha}_t}} = \mathbf{x}_0 + \sqrt{\frac{1-\bar{\alpha}_t}{\bar{\alpha}_t}}\epsilon.
\end{align*}
This transforms $\mathbf{x}_t$ into the standard form expected by an MMSE estimator. The DDIM prediction $\hat{\mathbf{x}}_0(\mathbf{x}_t)$ is equivalent to applying a standard MMSE denoiser $D_\theta(\cdot)$ to this normalized state:
\begin{align}\label{eq: DDIM predict x0}
    \hat{\mathbf{x}}_0(\mathbf{x}_t) = D_\theta(\mathbf{u}_t) = D_\theta\left( \frac{\mathbf{x}_t}{\sqrt{\bar{\alpha}_t}} \right).
\end{align}

To connect $\epsilon_\theta$ in DDIM and $D_\theta$ in VAMP, we have:
\begin{align}\label{eq: noise prediction as denoiser}
    \epsilon_\theta(\mathbf{z}) = \frac{\mathbf{z} - \sqrt{\bar{\alpha}_t}D_\theta\left(\frac{\mathbf{z}}{\sqrt{\bar{\alpha}_t}}\right)}{\sqrt{1-\bar{\alpha}_t}},
\end{align}
where $\mathbf{z}$ is the noisy representation of each time step $t$ in DDIM.

\subsubsection{The Jacobian Analysis}
We first reformulate the DDIM update rule in Eq.~(\ref{eq:ddim_standard}) to isolate the denoising operator by employing Eq.~(\ref{eq: noise prediction as denoiser}) to yield the \textbf{MMSE-AMP Affine Form} with
$\mathbf{z}=\mathbf{x}_t$.
The DDIM update, the reverse step, from $t \to t-1$ is given by the affine form:
\begin{align}\label{eq: update_rule}
    \mathbf{x}_{t-1} = \mu_t \mathbf{x}_t + \gamma_t \hat{\mathbf{x}}_0(\mathbf{x}_t) = \mu_t \mathbf{x}_t + \gamma_t D_\theta\left( \frac{\mathbf{x}_t}{\sqrt{\bar{\alpha}_t}} \right),
\end{align}
where the coefficients are determined by the noise schedule:
\begin{align*}
    \mu_t &= \frac{\sqrt{1-\bar{\alpha}_{t-1}}}{\sqrt{1-\bar{\alpha}_t}}, \quad \text{(Noise Carryover)} \\
    \gamma_t &= \sqrt{\bar{\alpha}_{t-1}} - \mu_t \sqrt{\bar{\alpha}_t}. \quad \text{(Signal Injection)}
\end{align*}

To determine the stability of this update, we compute the Jacobian matrix $J_t = \frac{\partial \mathbf{x}_{t-1}}{\partial \mathbf{x}_t}$. Applying the chain rule to Eq. (\ref{eq: update_rule}) introduces the critical scaling factor $1/\sqrt{\bar{\alpha}_t}$:
\begin{align}\label{eq: Jacob of DDIM step t}
    J_t = \mu_t \mathrm{Id} + \left.\frac{\gamma_t}{\sqrt{\bar{\alpha}_t}}\nabla_\mathbf{u} D_\theta(\mathbf{u})\right\vert_{\mathbf{u}=\frac{\mathbf{x}_t}{\sqrt{\bar{\alpha}_t}}}.
\end{align}


\subsection{Derivation of Spectral Bounds}\label{Sec: derive Spectral Bounds}
We apply the Fundamental Property of MMSE Estimators used in VAMP theory: the MMSE denoiser is contractive, i.e., its eigenvalues are within $[0, 1]$. Thus, we can have the following proposition.
\begin{proposition}[Spectral Bounds of MMSE Denoiser]\label{prop:spectral bound of MMSE}
Assuming the denoiser $D_\theta(\cdot)$ approximates the optimal MMSE estimator under Gaussian noise, the eigenvalues of its Jacobian matrix satisfy $\lambda(\nabla D_\theta) \subset [0, 1]$.
\end{proposition}
\begin{proof}
    Our goal is to show that $\nabla D_\theta$ is positive semi-definite, and that $0\preceq\nabla D_\theta\preceq\mathrm{Id}$. Given $\mathbf{u}\sim\mathcal{N}(\mathbf{x},\s^2\mathrm{Id})$, the denoiser $D=D_\theta$ is defined as posterior expectation conditioned on $\mathbf{u}$:
    \begin{align*}
        D(\mathbf{u})
        =\mathbb{E}[\mathbf{x}|\mathbf{u}]
        =\int \mathbf{x} p(\mathbf{x}|\mathbf{u})\,d\mathbf{x}.
    \end{align*}
    Applying Tweedie's formula,
    \begin{align*}
        D(\mathbf{u})
        =\mathbf{u}+\sigma^2\nabla\log p(\mathbf{u}),
    \end{align*}
    where $p(\mathbf{u}):=\int p(\mathbf{u}|\mathbf{x})p(\mathbf{x})\,d\mathbf{x}$. So the Jacobian matrix is
    \begin{align}\label{eq: Jacobian form}
        \nabla D(\mathbf{u})=\mathrm{Id}+\sigma^2\nabla^2\log p(\mathbf{u}).
    \end{align}
    Using the formula (will be proved in Lemma~\ref{lemma: 2nd order formula}):
    \begin{align*}
        \nabla^2\log p(\mathbf{u})
        =\frac{1}{\sigma^4}\mathrm{Cov}(\mathbf{x}|\mathbf{u})-\frac{1}{\sigma^2}\mathrm{Id},
    \end{align*}
    where $\mathrm{Cov}(\mathbf{x}|\mathbf{u}):=\mathbb{E}[\mathbf{x}\mathbf{x}^\top|\mathbf{u}]-\mathbb{E}[\mathbf{x}|\mathbf{u}]\mathbb{E}[\mathbf{x}|\mathbf{u}]^\top$, and taking it into Eq.~(\ref{eq: Jacobian form}) gives
    \begin{align*}
        \nabla D(\mathbf{u})
        =\mathrm{Id}+\sigma^2\left(\frac{1}{\sigma^4}\mathrm{Cov}(\mathbf{x}|\mathbf{u})-\frac{1}{\sigma^2}\mathrm{Id}\right)
        =\frac{1}{\sigma^2}\mathrm{Cov}(\mathbf{x}|\mathbf{u}).
    \end{align*}
    This immediately shows that $\nabla D\succeq0$ since covariance matrices are positive semi-definite. To show that $\nabla D\preceq\mathrm{Id}$, we observe that, $p(\mathbf{u})$ is log-concave, which means $\nabla^2\log p(\mathbf{u})\preceq0$. Thus using Eq.~(\ref{eq: Jacobian form}) again, we obtain $\nabla D\preceq\mathrm{Id}$.
\end{proof}

\begin{lemma}\label{lemma: 2nd order formula}
    $\nabla^2\log p(\mathbf{u})
        =\frac{1}{\sigma^4}\mathrm{Cov}(\mathbf{x}|\mathbf{u})-\frac{1}{\sigma^2}\mathrm{Id}$.
\end{lemma}
\begin{proof}
    Note that $p(\mathbf{u}|\mathbf{x})=\mathcal{N}(\mathbf{x},\s^2\mathrm{Id})$, so we have the useful formula, the derivative of p.d.f.~of Gaussian distribution:
    \begin{align}\label{eq: Gaussian distribution derivative}
        \nabla_\mathbf{u}p(\mathbf{u}|\mathbf{x})
        =-\frac{1}{\sigma^2}(\mathbf{u}-\mathbf{x})p(\mathbf{u}|\mathbf{x}).
    \end{align}
    We begin with computing the Jacobian. From the chain rule, $\nabla\log p(\mathbf{u})=\frac{\nabla p(\mathbf{u})}{p(\mathbf{u})}$, and calculate the numerator:
    \begin{align}
        \nabla p(\mathbf{u})
        &=\nabla_\mathbf{u}\int p(\mathbf{u}|\mathbf{x})p(\mathbf{x})\,d\mathbf{x}\nonumber\\
        &=\int\nabla_\mathbf{u}p(\mathbf{u}|\mathbf{x})p(\mathbf{x})\,d\mathbf{x}\nonumber\\
        &=-\frac{1}{\s^2}\int(\mathbf{u}-\mathbf{x})p(\mathbf{u}|\mathbf{x})p(\mathbf{x})\,d\mathbf{x},
        \label{eq: numerator}
    \end{align}
    where the last equality follows by Eq.~(\ref{eq: Gaussian distribution derivative}). Thus, applying Bayes' formula,
    \begin{align}
        \nabla\log p(\mathbf{u})
        &=-\frac{1}{p(\mathbf{u})\s^2}\int(\mathbf{u}-\mathbf{x})p(\mathbf{u}|\mathbf{x})p(\mathbf{x})\,d\mathbf{x}\nonumber\\
        &=-\frac{1}{\s^2}\int(\mathbf{u}-\mathbf{x})\frac{p(\mathbf{u}|\mathbf{x})p(\mathbf{x})}{p(\mathbf{u})}\,d\mathbf{x}\nonumber\\
        &=-\frac{1}{\s^2}\int(\mathbf{u}-\mathbf{x})p(\mathbf{x}|\mathbf{u})\,d\mathbf{x}\nonumber\\
        &=-\frac{1}{\s^2}\mathbb{E}[\mathbf{u}-\mathbf{x}|\mathbf{u}].
        \label{eq: 1st order result}
    \end{align}
    Next, we compute the Hessian:
    \begin{align}
        \nabla^2\log p(\mathbf{u})
        =\nabla\left(\frac{\nabla p(\mathbf{u})}{p(\mathbf{u})}\right)
        &=\frac{p(\mathbf{u})\nabla^2p(\mathbf{u})-\nabla p(\mathbf{u})(\nabla p(\mathbf{u}))^\top}{p(\mathbf{u})^2}\nonumber\\
        &=\frac{\nabla^2p(\mathbf{u})}{p(\mathbf{u})}-\left(\frac{\nabla p(\mathbf{u})}{p(\mathbf{u})}\right)\left(\frac{\nabla p(\mathbf{u})}{p(\mathbf{u})}\right)^\top\nonumber\\
        &=\frac{\nabla^2p(\mathbf{u})}{p(\mathbf{u})}-\frac{1}{\s^4}\mathbb{E}[\mathbf{u}-\mathbf{x}|\mathbf{u}]\mathbb{E}[\mathbf{u}-\mathbf{x}|\mathbf{u}]^\top,
        \label{eq: 2nd order result (part one)}
    \end{align}
    where the second term of the last equality follows by Eq.~(\ref{eq: 1st order result}). To deal with the first term of Eq.~(\ref{eq: 2nd order result (part one)}), we need to calculate $\nabla^2p(\mathbf{u})$. From Eqs.~(\ref{eq: Gaussian distribution derivative}) and~(\ref{eq: numerator}),
    \begin{align*}
        \nabla^2p(\mathbf{u})
        &=-\frac{1}{\s^2}\int\nabla_\mathbf{u}\left[(\mathbf{u}-\mathbf{x})p(\mathbf{u}|\mathbf{x})\right]p(\mathbf{x})\,d\mathbf{x}\\
        &=-\frac{1}{\s^2}\int\left[\mathrm{Id}\cdot p(\mathbf{u}|\mathbf{x})+(\mathbf{u}-\mathbf{x})(\nabla_\mathbf{u}p(\mathbf{u}|\mathbf{x}))^\top\right]p(\mathbf{x})\,d\mathbf{x}\\
        &=-\frac{1}{\s^2}\int\left[\mathrm{Id}\cdot p(\mathbf{u}|\mathbf{x})-\frac{1}{\s^2}(\mathbf{u}-\mathbf{x})(\mathbf{u}-\mathbf{x})^\top p(\mathbf{u}|\mathbf{x})\right]p(\mathbf{x})\,d\mathbf{x}\\
        &=-\frac{1}{\s^2}\mathrm{Id}\int p(\mathbf{u}|\mathbf{x})p(\mathbf{x})\,d\mathbf{x}
        +\frac{1}{\s^4}\int(\mathbf{u}-\mathbf{x})(\mathbf{u}-\mathbf{x})^\top p(\mathbf{u}|\mathbf{x})p(\mathbf{x})\,d\mathbf{x}\\
        &=-\frac{1}{\s^2}\mathrm{Id}\cdot p(\mathbf{u})+\frac{1}{\s^4}\int(\mathbf{u}-\mathbf{x})(\mathbf{u}-\mathbf{x})^\top p(\mathbf{u}|\mathbf{x})p(\mathbf{x})\,d\mathbf{x}.
    \end{align*}
    Therefore, using Bayes’ formula,
    \begin{align}
        \frac{\nabla^2p(\mathbf{u})}{p(\mathbf{u})}
        &=-\frac{1}{\s^2}\mathrm{Id}+\frac{1}{\s^4}\int(\mathbf{u}-\mathbf{x})(\mathbf{u}-\mathbf{x})^\top p(\mathbf{x}|\mathbf{u})\,d\mathbf{x}\nonumber\\
        &=-\frac{1}{\s^2}\mathrm{Id}+\frac{1}{\s^4}\mathbb{E}[(\mathbf{u}-\mathbf{x})(\mathbf{u}-\mathbf{x})^\top|\mathbf{u}].
        \label{eq: 2nd order (part 2)}
    \end{align}
    Finally, taking Eq.~(\ref{eq: 2nd order (part 2)}) into Eq.~(\ref{eq: 2nd order result (part one)}) gives
    \begin{align*}
        \nabla^2\log p(\mathbf{u})
        &=\frac{1}{\s^4}\left(\mathbb{E}[(\mathbf{u}-\mathbf{x})(\mathbf{u}-\mathbf{x})^\top|\mathbf{u}]-\mathbb{E}[\mathbf{u}-\mathbf{x}|\mathbf{u}]\mathbb{E}[\mathbf{u}-\mathbf{x}|\mathbf{u}]^\top\right)-\frac{1}{\s^2}\mathrm{Id}\\
        &=\frac{1}{\s^4}\mathrm{Cov}(\mathbf{u}-\mathbf{x}|\mathbf{u})-\frac{1}{\s^2}\mathrm{Id}\\
        &=\frac{1}{\s^4}\mathrm{Cov}(\mathbf{x}|\mathbf{u})-\frac{1}{\s^2}\mathrm{Id},
    \end{align*}
    so the Hessian is obtained. This completes the proof.
\end{proof}

Using this property, the spectrum of the DDIM Jacobian matrix $J_t$ in Eq.~(\ref{eq: Jacob of DDIM step t}) is bounded by the linear combination of the identity and the denoiser spectrum. The spectrum of the reverse step is thus:
\begin{align}\label{eq: spectrum of Jacob}
    \lambda(J_t) = \left\{\mu_t + \frac{\gamma_t}{\sqrt{\bar{\alpha}_t}} \lambda:\lambda\in\lambda(\nabla D_\theta)\right\}.
\end{align}

Also, we can verify whether our derived sampling step is invertible. First, since noise variance never hits zero during intermediate steps, $\mu_t > 0$. Also, for time step $t$ to $t-1$, $\bar{\alpha}_{t-1} > \bar{\alpha}_t$, which implies $\gamma_t > 0$. For an MMSE denoiser $D_\theta(\cdot)$, its Jacobian matrix is positive semi-definite since its eigenvalues are within $[0, 1]$. Therefore, the eigenvalues of Eq.~(\ref{eq: Jacob of DDIM step t}), {\em a.k.a} Eq.~(\ref{eq: spectrum of Jacob}), are strictly positive. As a result, Eq.~(\ref{eq: update_rule}) is non-singular.

\subsubsection{Generative Expansion (Reverse Process)}
The maximum expansion is the supremum over the spectrum $L_{\text{gen}}^{(t)} = \sup_{\lambda_D \in [0,1]} \left( \mu_t + \frac{\gamma_t}{\sqrt{\bar{\alpha}_t}} \lambda_D \right)$. The Lipschitz constant of the generation step is the maximum singular value of $J_t$. This occurs along the signal direction, where the denoiser preserves structure ($\lambda_{\mathrm{max}}(\nabla D_\theta) = 1$). From Eq.~(\ref{eq: spectrum of Jacob}), we can obtain:
\begin{align}\label{eq: step t rev L constant}
    L_{\text{gen}}^{(t)} &\leq \mu_t + \frac{\gamma_t}{\sqrt{\bar{\alpha}_t}} = \mu_t + \frac{1}{\sqrt{\bar{\alpha}_t}} \left( \sqrt{\bar{\alpha}_{t-1}} - \mu_t \sqrt{\bar{\alpha}_t} \right) \nonumber \\
    &= \mu_t + \frac{\sqrt{\bar{\alpha}_{t-1}}}{\sqrt{\bar{\alpha}_t}} - \mu_t = \frac{\sqrt{\bar{\alpha}_{t-1}}}{\sqrt{\bar{\alpha}_t}}.
\end{align}
\subsubsection{Inversion Singularity}
For the inverse process $x_{t-1} \to x_t$, we ask ``where does the reverse process contract the input the most?'' The maximum expansion of the inverse process is dictated by the infimum of the forward step $L_{\text{inv}}^{(t)} = \left( \inf_{\lambda_D \in [0,1]} \left( \mu_t + \frac{\gamma_t}{\sqrt{\bar{\alpha}_t}} \lambda_D \right) \right)^{-1}$. During the reverse process, the denoiser suppresses the noisy signal the most ($\lambda_{\mathrm{min}}(\nabla D_\theta) = 0$). Hence, the stability is governed by the noise subspace, where in this regime, the Jacobian matrix in Eq.~(\ref{eq: Jacob of DDIM step t}) simplifies to $J_t = \mu_t \mathrm{Id}$. The inverse Lipschitz constant is therefore $1/\mu_t$:
\begin{equation}\label{eq: step t inv L constant}
    L_{\text{inv}}^{(t)} \leq \frac{1}{\mu_t} = \frac{\sqrt{1-\bar{\alpha}_t}}{\sqrt{1-\bar{\alpha}_{t-1}}}.
\end{equation}

\section{Proofs of Theorems and Lemmas}\label{Sec: proofs for thm lem}
In the following, we provide detailed proofs of Theorem~\ref{thm:Global inverse of rev}, Theorem~\ref{thm:RST invariance}, Theorem~\ref{thm: non-perturb invariant}, Theorem~\ref{thm: global L bound}, 
and Theorem~\ref{thm: condition for success detection}.
\begin{proof}[\textbf{Proof of Theorem~\ref{thm:Global inverse of rev}}]
In our specific MMSE-VAMP framework, the mapping naturally satisfies the "stronger conditions" required for global bijectivity. Specifically, the Jacobian of our update step is not merely non-singular; it is strictly positive definite globally. We prove with the following reasoning:

First, by Tweedie's formula, the optimal MMSE denoiser under Gaussian noise is $D_\th(\mathbf{x}) = \mathbf{x} + \sigma^2 \nabla_{\mathbf{x}} \log p(\mathbf{x})$. Because it is the gradient of a scalar potential function, its Jacobian $\nabla D_\th(\mathbf{x})$ is symmetric. Furthermore, standard MMSE properties dictate that the eigenvalues of $\nabla D_\th(\mathbf{x})$ are bounded in $[0, 1]$ everywhere in the domain. Therefore, $\nabla D_\th(\mathbf{x})$ is a positive semi-definite (PSD) matrix globally.

Recall our derived Jacobian for a single DDIM step ($\mathbf{x}_t \to \mathbf{x}_{t-1}$):
$$J_t = \mu_t \mathbf{I} + \frac{\gamma_t}{\sqrt{\bar{\alpha}_t}} \nabla D_\th(\mathbf{u}_t)$$
Because $\mu_t > 0$ strictly (based on the diffusion noise schedule), $\gamma_t > 0$, and $\nabla D_\th \succeq 0$ globally, the entire Jacobian matrix $J_t$ is strictly positive definite everywhere in $\mathbb{R}^d$. Specifically, its minimum eigenvalue is bounded globally by $\mu_t > 0$.

Finally, according to the global inverse function theorem (and specifically the Gale-Nikaido theorem for mappings with positive principal minors, or standard results for strongly monotone operators), a continuously differentiable mapping $F: \mathbb{R}^d \to \mathbb{R}^d$ whose Jacobian is everywhere strictly positive definite is a global diffeomorphism (a smooth bijection). Because every individual step $t \to t-1$ is a global bijection, the full $T$-step composition is also globally invertible.
\end{proof}

\begin{proof}[\textbf{Proof of Theorem~\ref{thm:RST invariance}}]
We use proof by contradiction. Suppose both $\hat{\mathbf{x}}^{RST}_0$ and $\hat{\mathbf{x}}_0$ have the same latent representation, which means $\mathbf{y}^{RST}=\mathbf{y}$, given the invertible $f_\theta$ and $\mathbf{I}^{RST}=\mathscr{RST}(\mathbf{I})$, we can derive:
\begin{align*}
    \hat{\mathbf{x}}^{RST}_0 =\mathcal{E}_{\phi_e}(\mathbf{I}^{RST})= f_\th(\mathbf{y}^{RST}),
\end{align*}
which contradicts $\hat{\mathbf{x}}^{RST}_0\neq\hat{\mathbf{x}}_0$. 
Actually, the above statement is true only when $\mathscr{RST}=\rm Id$, that is to say, there are no $\mathscr{RST}$ perturbations on $\mathbf{I}$.

Moreover, with the invertible $f_\theta$, we have:
\begin{align*}
    \mathbf{y} &= f^{-1}_\theta(\hat{\mathbf{x}}_0), \\
    \mathbf{y}^{RST} &= f^{-1}_\theta(\hat{\mathbf{x}}^{RST}_0).
\end{align*}
If the assumption $\mathbf{y}^{RST}=\mathbf{y}$ is imposed, we have:
\begin{align}\label{inverse contradiction}
    f^{-1}_\theta(\hat{\mathbf{x}}^{RST}_0)=f^{-1}_\theta(\hat{\mathbf{x}}_0),
\end{align}
which
contradicts the definition of inverse function and $\mathscr{RST}\neq \rm Id$ since $\hat{\mathbf{x}}^{RST}_0=\hat{\mathbf{x}}_0$ and is true only when $\mathscr{RST}=\rm Id$, which means no perturbation is imposed on $\mathbf{I}$.
As a result, we can conclude that if $\hat{\mathbf{x}}^{RST}_0\neq\hat{\mathbf{x}}_0$ (\textit{i.e.,} $\mathscr{RST}\neq \rm Id$), then $\mathbf{y}^{RST}\neq \mathbf{y}$.

Now, given $\mathbf{I}^{RST}=\mathscr{RST}(\mathbf{I})$ and $\hat{\mathbf{x}}^{RST}_0\neq\hat{\mathbf{x}}_0$, we can derive by combining the DDIM inversion $f^{-1}_\theta$ with $\sF$ and $\sF_\sM$ in the RST-invariant watermarking framework~\cite{ruanaidh1998rotation} as:
\begin{align}
    \mathcal{I}_1 &= [\mathscr{F_M} \circ \mathscr{F}]f^{-1}_\theta(\hat{\mathbf{x}}^{RST}_0) = [\mathscr{F_M} \circ \mathscr{F}]\mathbf{y}^{RST}, \label{eq:fmf 1} \\
    \mathcal{I}_2 &= [\mathscr{F_M} \circ \mathscr{F}]f^{-1}_\theta(\hat{\mathbf{x}}_0) = [\mathscr{F_M} \circ \mathscr{F}]\mathbf{y}. \label{eq:fmf 2}
\end{align}
where $\mathcal{I}_1$ and $\mathcal{I}_2$ are outputs of $\mathscr{F_M} \circ \mathscr{F}$ operations, respectively.
To demonstrate that $\mathcal{I}_1\ne\mathcal{I}_2$, we leverage the inherent properties of DDIM inversion—and generative models more broadly. Specifically, even within the framework of a Probability Flow ODE, due to the expansive property of $f_\theta^{-1}$ ({\em i.e.}, large $L_{f_\theta^{-1}}$), the resulting latent representations $\mathbf{y}$ and $\mathbf{y}^{RST}$ exhibit amplitude inconsistency. By treating the inconsistency as a random variable (see proof of Theorem~\ref{thm: non-perturb invariant}), it follows that $\mathcal{I}_1\ne\mathcal{I}_2$ holds almost surely.
In conclusion, when the DDIM inversion process operates with an RST-invariant framework, the combined process, {\em a.k.a} DDIM inversion and RST-invariant framework, loses RST invariance.
\end{proof}

\begin{proof}[\textbf{Proof of Theorem~\ref{thm: non-perturb invariant}}]
In this proof, we will continue the discussion of the proof of Theorem~\ref{thm:RST invariance}, focusing on proving the expansion property of DDIM inverse process. Let $\mathbf{x}_t,\mathbf{y}_t$ be the trajectories of two distinct data points under the deterministic DDIM process, and let $\Delta_t:=\mathbf{x}_t - \mathbf{y}_t$ denote the difference vector between the two particles at time $t$. Using the linearized approximation of the DDIM trajectory:
\begin{align*}
\mathbf{x}_t \approx \sqrt{\bar{\alpha}_t} \mathbf{x}_0 + \sqrt{1-\bar{\alpha}_t} \epsilon,
\end{align*}
For two points starting at $\mathbf{x}_0$ and $\mathbf{y}_0$ with random initial difference $\Delta_0=\mathbf{x}_0 - \mathbf{y}_0$ ({\em i.e.}, any perturbation), their inverted states at time $t$ relate to the initial states as (with high probability):
\begin{align}\label{eq:thm3.3-1}
\|\Delta_t\|\ge\frac{\sqrt{1-\bar{\alpha}_t}}{\sqrt{1-\alpha_0}}\|\Delta_0\|.
\end{align}
To show Eq.~(\ref{eq:thm3.3-1}), we follow the rule of DDIM inversion, for each step,
\begin{align}
\mathbf{x}_t&=\sqrt{\bar{\alpha}_t}\hat{\mathbf{x}}_0+\sqrt{1-\bar{\alpha}_t}\epsilon\nonumber\\
&=\sqrt{\bar{\alpha}_t}\hat{\mathbf{x}}_0+\sqrt{1-\bar{\alpha}_t}\frac{\mathbf{x}_{t-1}-\sqrt{\bar{\alpha}_{t-1}}\hat{\mathbf{x}}_0}{\sqrt{1-\bar{\alpha}_{t-1}}}\nonumber\\
&=\sqrt{\frac{1-\bar{\alpha}_t}{1-\bar{\alpha}_{t-1}}}\mathbf{x}_{t-1}+\left(\sqrt{\bar{\alpha}_t}-\frac{\sqrt{(1-\bar{\alpha}_t)\bar{\alpha}_{t-1}}}{\sqrt{1-\bar{\alpha}_{t-1}}}\right)\hat{\mathbf{x}}_0,\nonumber
\end{align}
where $\hat{\mathbf{x}}_0=\hat{\mathbf{x}}_0(\mathbf{x}_{t-1})$ is the predicted $\mathbf{x}_0$ given $\mathbf{x}_{t-1}$. It follows that
\begin{align}\label{eq:Jacobian of inversion for each step}
\frac{\partial\mathbf{x}_t}{\partial\mathbf{x}_{t-1}}
&=\sqrt{\frac{1-\bar{\alpha}_t}{1-\bar{\alpha}_{t-1}}}\mathrm{Id}
+\left(\sqrt{\bar{\alpha}_t}-\frac{\sqrt{(1-\bar{\alpha}_t)\bar{\alpha}_{t-1}}}{\sqrt{1-\bar{\alpha}_{t-1}}}\right)\frac{\partial\hat{\mathbf{x}}_0}{\partial\mathbf{x}_{t-1}}\nonumber\\
&=\underbrace{\sqrt{\frac{1-\bar{\alpha}_t}{1-\bar{\alpha}_{t-1}}}}_{=:\eta_t}\mathrm{Id}
+\underbrace{\left(\sqrt{\frac{\bar{\alpha}_t}{\bar{\alpha}_{t-1}}}-\sqrt{\frac{1-\bar{\alpha}_t}{1-\bar{\alpha}_{t-1}}}\right)}_{=:-\kappa_t}\left.\nabla D_\theta(\mathbf{u}_t)\right\vert_{\mathbf{u}_t=\mathbf{x}_{t-1}/\sqrt{\bar{\a}_{t-1}}},\\
&=\eta_t\left(\mathrm{Id}-\frac{\kappa_t}{\eta_t}\nabla D_\theta(\mathbf{u}_t)\right),\nonumber
\end{align}
where Eq.~(\ref{eq:Jacobian of inversion for each step}) follows from Eq.~(\ref{eq: DDIM predict x0}), and $\eta_t,\kappa_t>0$ denote the corresponding coefficients (see Fig.~\ref{fig:etas kappas} for numerically computing $\eta_t,\kappa_t$, and $\kappa_t/\eta_t$). We may multiply all the terms together and gives the Jacobian:
\begin{align*}
J_{0\to t}:=\frac{\partial\mathbf{x}_t}{\partial\mathbf{x}_0}
=\prod_{s=1}^t\frac{\partial\mathbf{x}_s}{\partial\mathbf{x}_{s-1}}
&=\left(\prod_{s=1}^t\eta_s\right)\prod_{s=1}^t\left(\mathrm{Id}-\frac{\kappa_s}{\eta_s}\nabla D_\theta(\mathbf{u}_s)\right)\\
&=\frac{\sqrt{1-\bar{\alpha}_t}}{\sqrt{1-\alpha_0}}\left(\mathrm{Id}+\sum_{m=1}^tR_{0\to t}^{(m)}\right),
\end{align*}
where
\begin{align*}
R_{0\to t}^{(m)}
:=(-1)^m\sum_{1\le i_1<i_2<\cdots<i_m\le t}\prod_{j=1}^m\frac{\kappa_{i_j}}{\eta_{i_j}}\nabla D_\theta(\mathbf{u}_{i_j}).
\end{align*}
Since $\Delta_0$ is a high-dimensional vector, and note that $0\preceq\nabla D_\theta(\mathbf{u}_{i_j})\preceq\mathrm{Id}$ (Proposition~\ref{prop:spectral bound of MMSE}), the vector $R_{0\to t}^{(m)}\Delta_0$, for each $m$, is perpendicular to $\Delta_0$ with high probability. Therefore,
\begin{align*}
    \|\Delta_t\|
    =\|J_{0\to t}\Delta_0\|
    =\frac{\sqrt{1-\bar{\alpha}_t}}{\sqrt{1-\alpha_0}}\left\|\Delta_0+\sum_{m=1}^tR_{0\to t}^{(m)}\Delta_0\right\|
    \ge\frac{\sqrt{1-\bar{\alpha}_t}}{\sqrt{1-\alpha_0}}\|\Delta_0\|
\end{align*}
with high probability.

The inverse process maps data from a low-variance manifold ($\alpha_0 \approx 1$) to a high-variance Gaussian sphere ($\alpha_t \to 0$). The scaling factor $\sqrt{1-\bar{\alpha}_t}/\sqrt{1-\alpha_0}$ acts as an \textbf{amplifier} for any $\Delta_0$. Consequently, the distance between two distinct particles is significantly magnified as $t$ increases.
\end{proof}

\begin{figure}[htbp] 
  \centering 
  \includegraphics[width=0.6
  \textwidth]{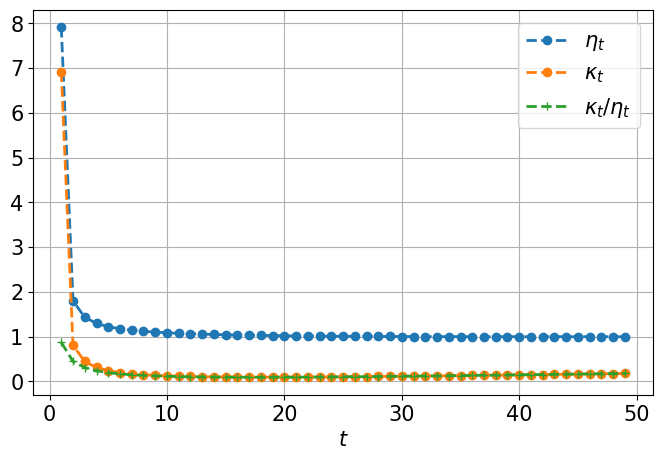} 
  \caption{Coefficients $\eta_t,\kappa_t,\kappa_t/\eta_t$ of all time steps, under the linear noise schedule.}
  \label{fig:etas kappas}
\end{figure}

\begin{proof}[\textbf{Proof of Theorem~\ref{thm: global L bound}}]
Since $\bar{\alpha}_{t-1} > \bar{\alpha}_t$, $L_{\text{gen}}^{(t)} > 1$. Eq.~(\ref{eq: step t rev L constant}) proves that generation is locally expansive. The total Lipschitz constant for the full trajectory telescopes perfectly:
\begin{align*}
    L_{f_\theta} \le \prod_{t=1}^T \frac{\sqrt{\bar{\alpha}_{t-1}}}{\sqrt{\bar{\alpha}_t}} = \frac{\sqrt{\alpha_0}}{\sqrt{\bar{\alpha}_T}}.
\end{align*}
Similarly, the global bound for inversion telescopes can be derived using Eq.~(\ref{eq: step t inv L constant}) as:
\begin{align*}
    L_{f^{-1}_\theta} \le \prod_{t=1}^T \frac{\sqrt{1-\bar{\alpha}_t}}{\sqrt{1-\bar{\alpha}_{t-1}}} = \frac{\sqrt{1-\bar{\alpha}_T}}{\sqrt{1-\alpha_0}}.
\end{align*}
\end{proof}

Before proving Theorem~\ref{thm: condition for success detection}, we need two lemmas.
\begin{lemma}\label{lem: max bound}
    Given $L_{\mathcal{G}}$, $L_{f_\theta}$, $L_{\mathcal{W}}$, $L_{\mathscr{F}^{-1}}$, $\mathbf{y}\sim\cN(\vzero,\rm Id)$, $\mathbf{w}$, and $r>0$, we denote the Lipschitz constant of $\mathbf{F}$ in Eq.~(\ref{eq: detection space to Image space}) as $L_{\mathbf{F}}=L_{\mathcal{G}_{\phi_g}}L_{f_\theta}L_{\mathscr{F}^{-1}}L_{\mathcal{W}}$. Furthermore, given $\mathbf{w}^{\prime} \in \mathcal{B}_{r}(\mathbf{w})$, we can obtain the maximum perturbation bound in the pixel space with $\mathbf{I}=\mathbf{F}(\mathbf{y},\mathbf{w})$ and $\mathbf{I}^{\prime}=\mathbf{F}(\mathbf{y},\mathbf{w}^{\prime})$ as:
    \begin{align}\label{max perturb bound}
        \|\mathbf{I}-\mathbf{I}^{\prime}\|\leq L_{\mathbf{F}}r.
    \end{align}
\end{lemma}

Specifically, the inequality in Eq.~(\ref{max perturb bound}) can be interpreted as the maximum discrepancy in the pixel space, given that the perturbed watermark $\mathbf{w}^{\prime}\in\mathcal{B}_r(\mathbf{w})$ in the latent space.

\begin{proof}[Proof of Lemma~\ref{lem: max bound}]
According to the results in Theorem~\ref{thm: global L bound}, given the Lipschitz constants of $\mathcal{W}$, $\mathcal{M}$, and $\mathcal{G}_{\phi_g}$, and Eq.~(\ref{eq: detection space to Image space}) and $L_{\mathcal{W}}=\sqrt{HW}$, we can have the following property:
\begin{align}\label{eq: gen bound from w to I}
    &\|\mathbf{I}-\mathbf{I}^{\prime}\|=\|\mathbf{F}(\mathbf{y}, \mathbf{w})-\mathbf{F}(\mathbf{y}, \mathbf{w}^{\prime})\|\leq L_{\mathbf{F}}\|\mathbf{w}-\mathbf{w}^{\prime}\|, \nonumber\\
    &L_{\mathbf{F}}=L_{\mathcal{G}_{\phi_g}}L_{f_\theta}L_{\mathscr{F}^{-1}}L_{\mathcal{W}},
\end{align}
and, with a radius $r$ in the latent space at point $\mathbf{w}$, we define a ball centered at $\mathbf{w}$ as $\mathcal{B}=\{\mathbf{w}^{\prime}:\| \mathbf{w}^{\prime} - \mathbf{w} \| \leq r \}$. Hence, the Eq.~(\ref{eq: gen bound from w to I}) can be simplified into:
\begin{align}
    \|\mathbf{I}-\mathbf{I}^{\prime}\|\leq L_{\mathbf{F}}r, 
\end{align}
\end{proof}

\begin{lemma} \label{lem: inv max bound}
    Given the Lipschitz constants $L_{\mathcal{E}_{\phi_e}}$, $L_{f^{-1}_\theta}$, $L_\mathscr{F}$, and $L_{\mathcal{M}}$, we define the Lipschitz constant of the watermark estimation of an image as $L_{\mathbf{F}^{-1}}=L_{\mathcal{E}_{\phi_e}}L_{f^{-1}_\theta}L_\mathscr{F}L_\mathcal{M}$. With Eq.~(\ref{eq: wm detection space}), the ground-truth watermark $\mathbf{w}$, and a distorted image $\mathcal{T}(\mathbf{I})$, for all $\mathbf{I}=\mathbf{F}(\mathbf{y}, \mathbf{w})$ and $\mathbf{y}\sim\cN(\vzero,\rm Id)$, the maximum perturbation that affects the estimated watermark $\hat{\mathbf{w}}$ can be derived as:
    \begin{align}
        \|\mathbf{w}-\hat{\mathbf{w}}\|\leq \|\mathbf{w}-\bar{\mathbf{w}}\|+L_{\mathbf{F}^{-1}}\|\mathbf{I}-\mathcal{T}(\mathbf{I})\|, \label{practical inv max perturb bound}
    \end{align}
    where $\bar{\mathbf{w}} = \mathbf{F}^{-1}(\mathbf{I}) = \mathcal{M}(\mathscr{F}(f^{-1}_\theta(\mathcal{E}_{\phi_e}(\mathbf{I})))$ indicates the estimated watermark from image $\mathbf{I}$. If $\mathbf{F}^{-1}\mathbf{F}=\mathrm{Id}$, Eq.~(\ref{practical inv max perturb bound}) can be reduced into
    \begin{align}
        \|\mathbf{w}-\hat{\mathbf{w}}\|\leq L_{\mathbf{F}^{-1}}\|\mathbf{I}-\mathcal{T}(\mathbf{I})\|. \label{inv max perturb bound}
    \end{align}
\end{lemma}

\begin{proof}[Proof of Lemma~\ref{lem: inv max bound}]

According to the results in Theorem~\ref{thm: global L bound}, given the Lipschitz constants $L_{\mathcal{E}_{\phi_e}}$, $L_{f^{-1}_\theta}$, $L_\mathscr{F}$, and $L_{\mathcal{M}}$, and Eq.~(\ref{eq: wm detection space}), we can have the following property:
\begin{align*}
    &\|\bar{\mathbf{w}}-\hat{\mathbf{w}}\|=\|\mathbf{F}^{-1}(\mathbf{I})-\mathbf{F}^{-1}(\mathcal{T}(\mathbf{I}))\|\leq L_{\mathbf{F}^{-1}}\|\mathbf{I}-\mathcal{T}(\mathbf{I})\|, \\
    &L_{\mathbf{F}^{-1}}=L_{\mathcal{E}_{\phi_e}}L_{f^{-1}_\theta}L_\mathscr{F}L_\mathcal{M}.
\end{align*}
Using the triangular inequality, we have the result
\begin{align}\label{eq: practical inv L bound}
    \|\mathbf{w}-\hat{\mathbf{w}}\|\leq \|\mathbf{w}-\bar{\mathbf{w}}\|+\|\bar{\mathbf{w}}-\hat{\mathbf{w}}\| \leq \|\mathbf{w}-\bar{\mathbf{w}}\|+L_{\mathbf{F}^{-1}}\|\mathbf{I}-\mathcal{T}(\mathbf{I})\|.
\end{align}
Given the inverse process $f^{-1}_\theta$ and the encoder $\mathcal{E}_{\phi_e}$, $\|\mathbf{w}-\bar{\mathbf{w}}\|=0$ since $\mathbf{F}^{-1}\mathbf{F}=\mathrm{Id}$, hence, Eq.~(\ref{eq: practical inv L bound}) can be reduced into Eq.~(\ref{inv max perturb bound}).
\end{proof}

\begin{proof}[\textbf{Proof of Theorem~\ref{thm: condition for success detection}}]
With the above lemmas, given the Lipschitz constant $L_{\mathcal{T}}$ and any two watermarked images $\mathbf{I}$ and $\mathbf{I}^{\prime}$, we have
\begin{align}\label{eq: perturb L ineq}
    \|\mathcal{T}(\mathbf{I})-\mathcal{T}(\mathbf{I}^{\prime})\|\leq L_{\mathcal{T}}\|\mathbf{I}-\mathbf{I}^{\prime}\|.
\end{align}
Using the triangle inequality and Eq.~(\ref{eq: perturb L ineq}), we have
\begin{align}\label{eq: I-T(I') tri upper bound}
    \|\mathbf{I}-\mathcal{T}(\mathbf{I}^{\prime})\|&\leq\|\mathbf{I}-\mathcal{T}(\mathbf{I})\| + \|\mathcal{T}(\mathbf{I})-\mathcal{T}(\mathbf{I}^{\prime})\| \leq \|\mathbf{I}-\mathcal{T}(\mathbf{I})\|+L_{\mathcal{T}}\|\mathbf{I}-\mathbf{I}^{\prime}\|.
\end{align}
Given Lemma~\ref{lem: max bound}, we have
\begin{align}\label{eq: w upper bound by triangle ineq}
    \|\mathbf{I}-\mathcal{T}(\mathbf{I}^{\prime})\|\leq \|\mathbf{I}-\mathcal{T}(\mathbf{I})\|+L_{\mathcal{T}}\|\mathbf{I}-\mathbf{I}^{\prime}\| \leq \|\mathbf{I}-\mathcal{T}(\mathbf{I})\|+L_{\mathcal{T}}L_{\mathbf{F}}r.
\end{align}
Next, given Lemma~\ref{lem: max bound}, if all watermarked images are resistant to the perturbation $\mathcal{T}$ on themselves, the following must be satisfied with Eq.~(\ref{max perturb bound}):
\begin{align}
    \forall \mathbf{I},\|\mathbf{I}-\mathcal{T}(\mathbf{I})\|\leq L_{\mathbf{F}}r. \label{eq: I-T(I) upper bound}
\end{align}
Hence, from Eqs.~(\ref{eq: w upper bound by triangle ineq}) and~(\ref{eq: I-T(I) upper bound}), we get
\begin{align}\label{eq: I-T(I') upper bound}
    \|\mathbf{I}-\mathcal{T}(\mathbf{I}^{\prime})\|\leq L_{\mathbf{F}}r+L_{\mathcal{T}}L_{\mathbf{F}}r=(1+L_{\mathcal{T}})L_{\mathbf{F}}r.
\end{align}
Define the following
\begin{align*}
    \bar{\mathbf{w}} = \mathbf{F}^{-1}(\mathbf{I}),\quad
    \hat{\mathbf{w}}^{\prime} = \mathbf{F}^{-1}(\mathcal{T}(\mathbf{I}^{\prime})),
\end{align*}
by the definition of Lipschitz constant $L_{\mathbf{F}^{-1}}$, we derive
\begin{align}\label{eq: barw - hatw' upper bound}
    \|\bar{\mathbf{w}}-\hat{\mathbf{w}}^{\prime}\|=\|\mathbf{F}^{-1}(\mathbf{I})-\mathbf{F}^{-1}(\mathcal{T}(\mathbf{I}^{\prime}))\|\leq L_{\mathbf{F}^{-1}} \|\mathbf{I}-\mathcal{T}(\mathbf{I}^{\prime})\|.
\end{align}
Using the triangle inequality, we turn Eq.~(\ref{eq: barw - hatw' upper bound}) into
\begin{align}\label{eq: w-hatw' bound}
    \|\mathbf{w}-\hat{\mathbf{w}}^{\prime}\|\leq\|\mathbf{w}-\bar{\mathbf{w}}\|+\|\bar{\mathbf{w}}-\hat{\mathbf{w}}^{\prime}\|\leq \|\mathbf{w}-\bar{\mathbf{w}}\|+L_{\mathbf{F}^{-1}}\|\mathbf{I}-\mathcal{T}(\mathbf{I}^{\prime})\|.
\end{align}
Given $f^{-1}_\theta$ and $\mathcal{E}_{\phi_e}$, $\|\mathbf{w}-\bar{\mathbf{w}}\|=0$ since $\mathbf{F}^{-1}\mathbf{F}=\mathrm{Id}$, hence, Eq.~(\ref{eq: w-hatw' bound}) can be reduced into
\begin{align}\label{eq: w-hatw' upper bound}
    \|\mathbf{w}-\hat{\mathbf{w}}^{\prime}\|\leq L_{\mathbf{F}^{-1}}\|\mathbf{I}-\mathcal{T}(\mathbf{I}^{\prime})\|.
\end{align}
Combining Eqs.~(\ref{eq: I-T(I') upper bound}) and~(\ref{eq: w-hatw' upper bound}), we finally have Eq.~(\ref{eq: rev inv total bound}).
\begin{align*}
    \|\mathbf{w}-\hat{\mathbf{w}}^{\prime}\|\leq L_{\mathbf{F}^{-1}}(1+L_{\mathcal{T}})L_{\mathbf{F}}r
\end{align*}
\end{proof}

\section{$\mathcal{B}$ defined by Tolerance Constraints}\label{Sec: Ball constraints}

\begin{figure*}[ht]
    \centering
        \begin{subfigure}[b]
        {0.32\textwidth}
        \centering
            \includegraphics[height=0.49927\textwidth]{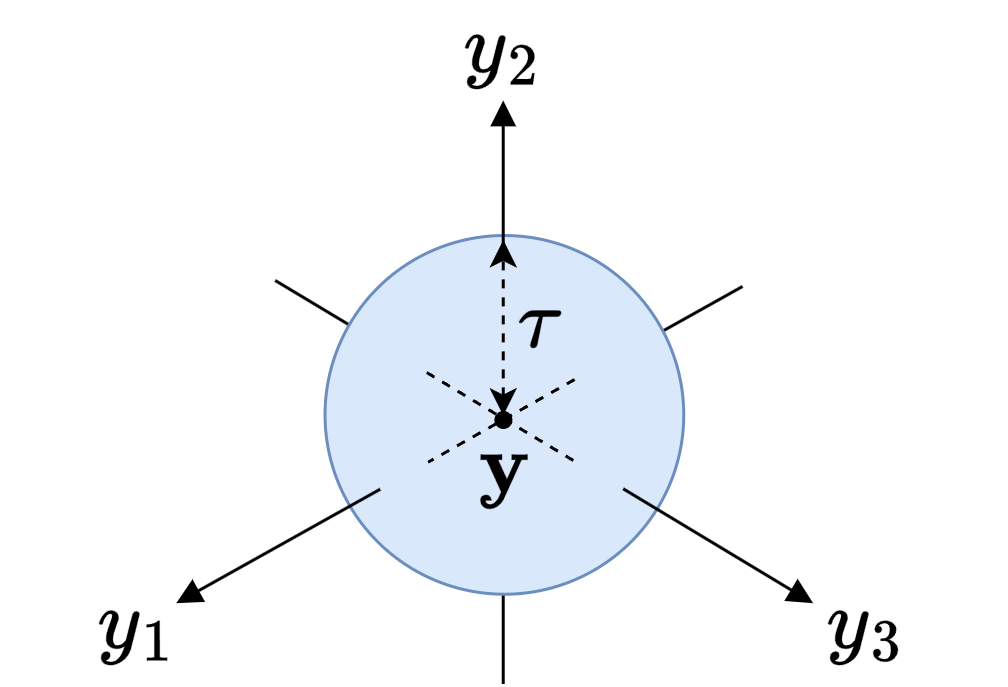}
            \caption{}
            \label{tolerance normal}
        \end{subfigure}
        \begin{subfigure}[b]
        {0.32\textwidth}
        \centering
            \includegraphics[height=0.75\textwidth]{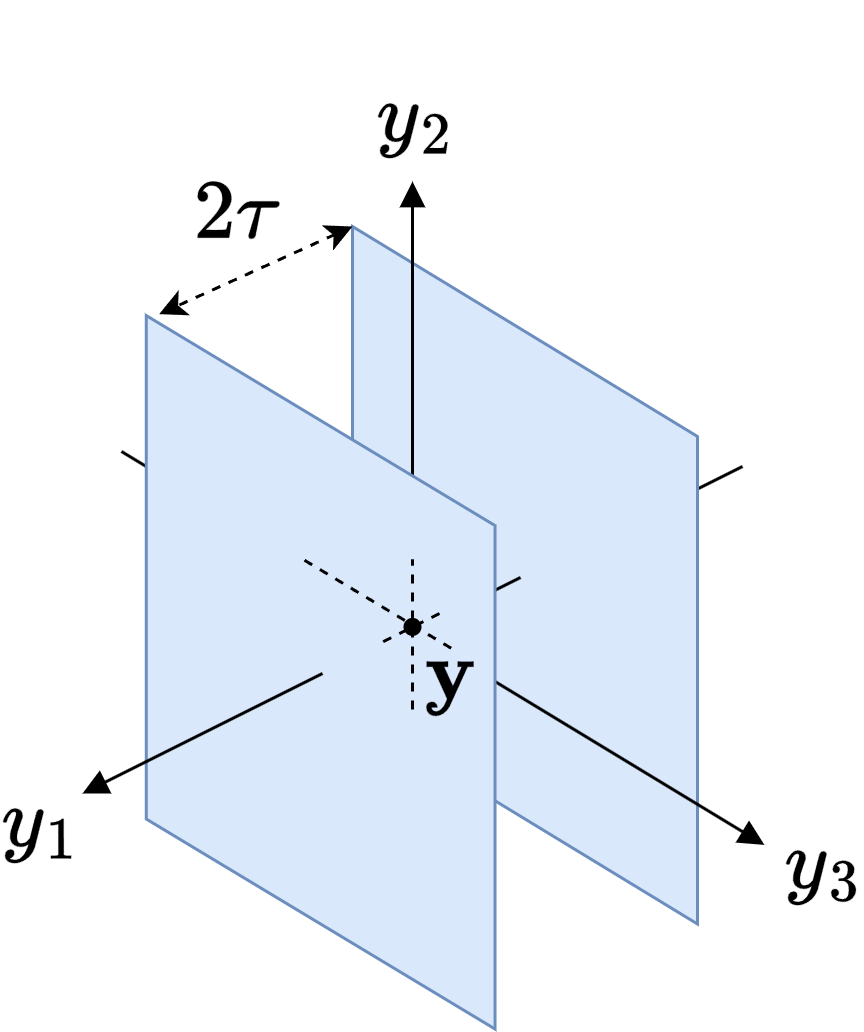}
            \caption{}
            \label{tolerance treering}
        \end{subfigure}
        \begin{subfigure}[b]
        {0.32\textwidth}
        \centering
            \includegraphics[height=0.75\textwidth]{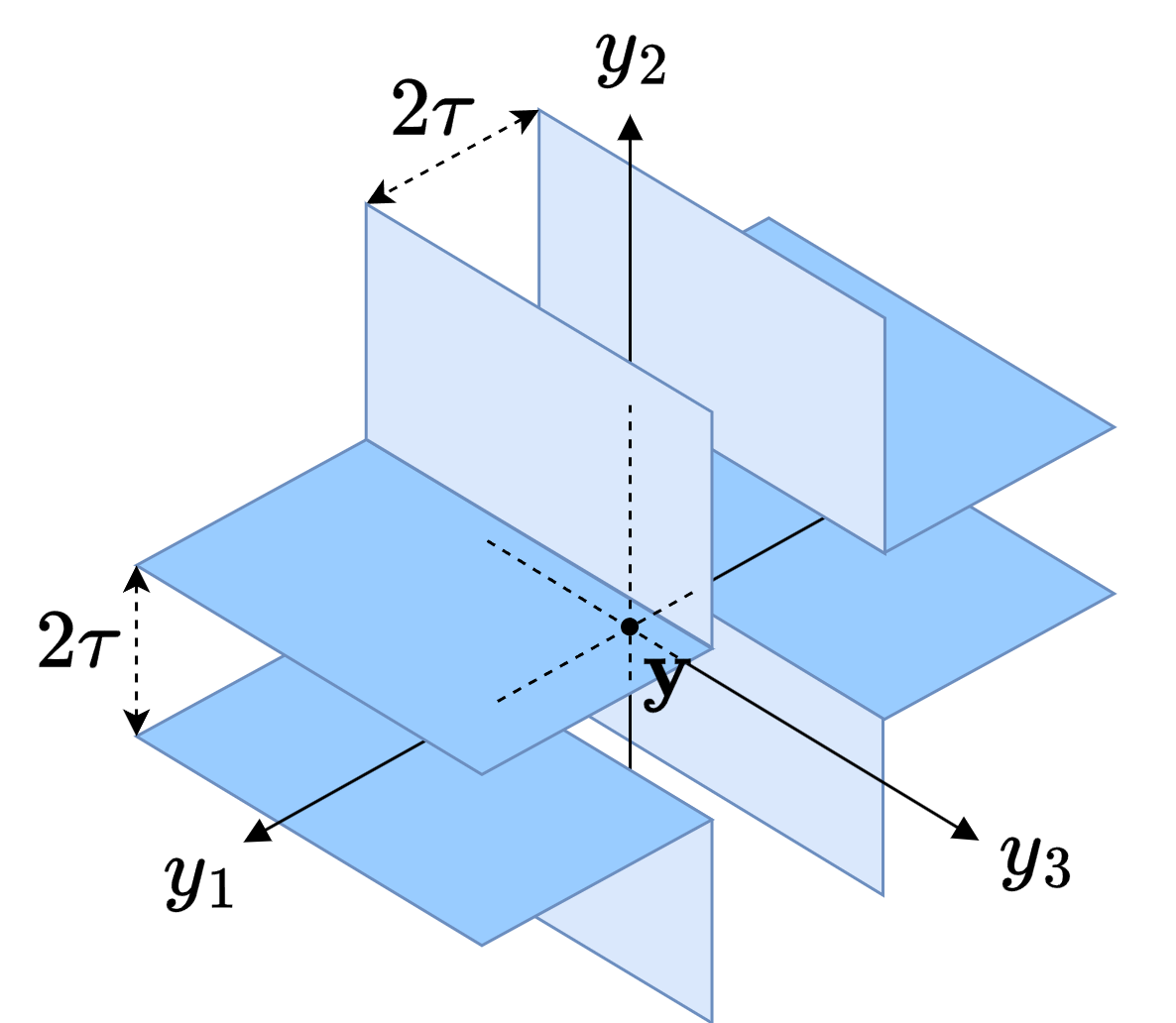}
            \caption{}
            \label{tolerance ringid}
        \end{subfigure}
        \caption{
        Visual explanation of $\mathcal{B}$. (a) The tolerance constraint is described in Sec.~\ref{max bound}. (b) The constraint only considers the $y_1$-axis (corresponding to watermark and detection on the Channel 3 in Tree-Ring~\cite{wen2023treerings}), leaving other dimensions unbounded. As long as the $\ell_1$ distance between any points and $\mathbf{y}$ (centre) on the $y_1$-axis is below a certain threshold $\t$ (corresponding to falling within the two planes in the figure), it is considered to be in $\mathcal{B}$. (c) This constraint considers the $y_1$- and $y_2$-axes, respectively (corresponding to the Channels 0 and 3 in RingID~\cite{ci2024ringid}), leaving the $y_3$-axes unbounded. The constraint requires calculating the $\ell_1$ distance between any point and $\mathbf{y}$ on $y_1$- and $y_2$-axes, respectively, in a dimension-wise manner, and then takes the minimum of these two values. A point is considered to fall in $\mathcal{B}$ if this minimum value is below a certain threshold $\t$ (corresponding to falling within the cross-shaped region in the figure).}
        \label{tolerance visualization}
\end{figure*}

In this section, we give an intuitive explanation of the high FPR under different tolerance constraints as described and illustrated in Fig.~\ref{tolerance visualization}. 
Given a tolerance $\t$ and the point $\mathbf{y}=(y_1,y_2,y_3) \in \RR^3$, we define the following subspaces
\begin{align}
    \mathcal{B}_a&=\{\mathbf{y}^{\prime}|\| \mathbf{y}^{\prime} - \mathbf{y} \|_2 \leq \t, \mathbf{y}^{\prime} \in \RR^3 \}, \nonumber \\
    \mathcal{B}_b&=\{\mathbf{y}^{\prime}|\| y^{\prime}_1 - y_1 \|_1 \leq \t, \mathbf{y}^{\prime} \in \RR^3 \},  \nonumber \\
    \mathcal{B}_{b^\prime}&=\{\mathbf{y}^{\prime}|\| y^{\prime}_2 - y_2 \|_1 \leq \t, \mathbf{y}^{\prime} \in \RR^3 \},  \nonumber \\
    \mathcal{B}_c&=\{\mathbf{y}^{\prime}|\min(\| y^{\prime}_1 - y_1 \|_1,\| y^{\prime}_2 - y_2 \|_1) \leq \t, \mathbf{y}^{\prime} \in \RR^3 \}. \nonumber
\end{align}
where $\mathcal{B}_a$, $\mathcal{B}_b$, and $\mathcal{B}_c$ correspond to Fig.~\ref{tolerance visualization}(a), (b), and (c), respectively. We can easily observe that three different constraints can cover different volumes $vol(\cdot)$ in a three-dimensional space, denoted as $vol(\mathcal{B}_a)$, $vol(\mathcal{B}_b)$, and $vol(\mathcal{B}_c)$. First, $\mathcal{B}_a \subset \mathcal{B}_b, \mathcal{B}_{b^\prime}, \mathcal{B}_c$. Second, $vol(\mathcal{B}_c)$ can be seen as the union of two subspaces, $\mathcal{B}_b$ and $\mathcal{B}_{b^\prime}$. We could conclude that
\begin{align}
    \mathcal{B}_a \subset \mathcal{B}_b \subset \mathcal{B}_c.
\end{align}
Accordingly, it is easier for a data point $\mathbf{y}^\prime$ to be included in $\mathcal{B}_c$, while it is harder to be included in $\mathcal{B}_a$ since $\mathcal{B}_c$ covers the most volume. More importantly, to perturb/attack a data point in/out of subspaces $\mathcal{B}_b$ and $\mathcal{B}_c$, the effort can be made by moving the point along the bounded axes, respectively. This results in a harder attack success rate for $\mathcal{B}_a$ since all axes should be considered. Hence, we can conclude that, under the same tolerance $\tau$, the constraint leads to different structures of $\mathcal{B}$, resulting in looser or tighter criteria to determine whether a point $\mathbf{y}^\prime$ lies within $\mathcal{B}$.

\section{Simulations of Distribution Shift}\label{sim of dist shift}
To demonstrate that relocation of $\mathbf{y}_\mathbf{w}$ caused by watermarks moves the latent distributions away from those of normally generated images with the corresponding latent representations $\mathbf{y}$ close to $\mathbf{0}$, 
we examine this effect in this section
for the simulation of Gaussian distribution in Fig.~\ref{distribution shift}(a) and the corresponding $\chi$ distribution in Fig.~\ref{distribution shift}(b) under different shifts in mean $\mathbf{\mu}$. This gives a larger tolerance for perturbing the images and leads to 
better robustness of watermarking in the generative process, despite the possible high FPRs. 
\begin{figure}[ht]
    \centering
        \begin{subfigure}[b]
        {0.4\textwidth}
            \includegraphics[height=0.8\textwidth]{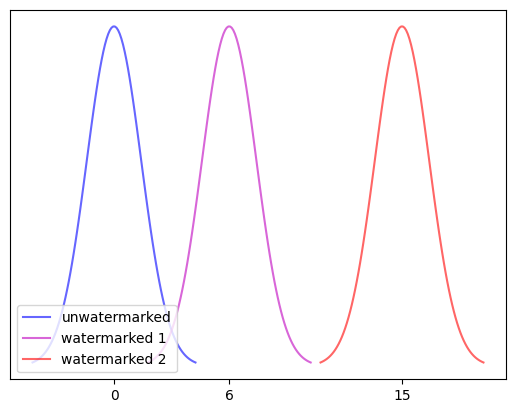}
            \caption{}
            \label{normal shift}
        \end{subfigure}
        \begin{subfigure}[b]
        {0.4\textwidth}
            \includegraphics[height=0.8\textwidth]{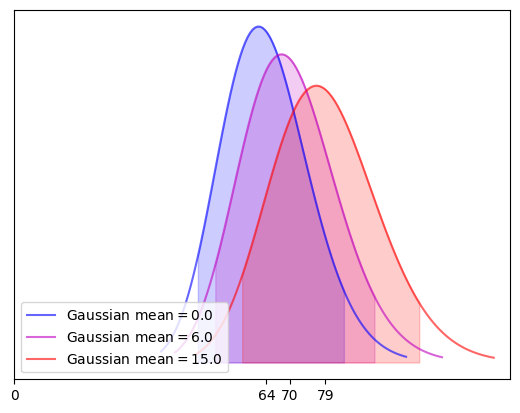}
            \caption{}
            \label{chi square}
        \end{subfigure}
        \caption{(a) Distribution Shift in latent space. The norm of watermarked images will shift away from $\mathbf{0}$. By injecting a larger-norm watermark in latent representation, it is easier to distinguish whether an image is from a normal distribution. Also, the allowed perturbation will be larger since the latent representations of perturbed images will fall within the gap between the unwatermarked and watermarked latents. However, they are still away from the unwatermarked latent. (b) The distribution of the $\ell_2$-norm of samples from each normal distribution. These distributions follow the $\chi$ distribution. The shaded areas represent the 0.3 and 99.7 percentiles for each distribution.}
        \label{distribution shift}
\end{figure}

\section{Analysis of the Detection Distance through the Lens of Forgery and Removal Attacks}\label{lens of f r attack}

\begin{figure*}[ht]
\centering
    \begin{subfigure}[b]{0.4\textwidth}
        \centering
        \includegraphics[width=0.9\textwidth]{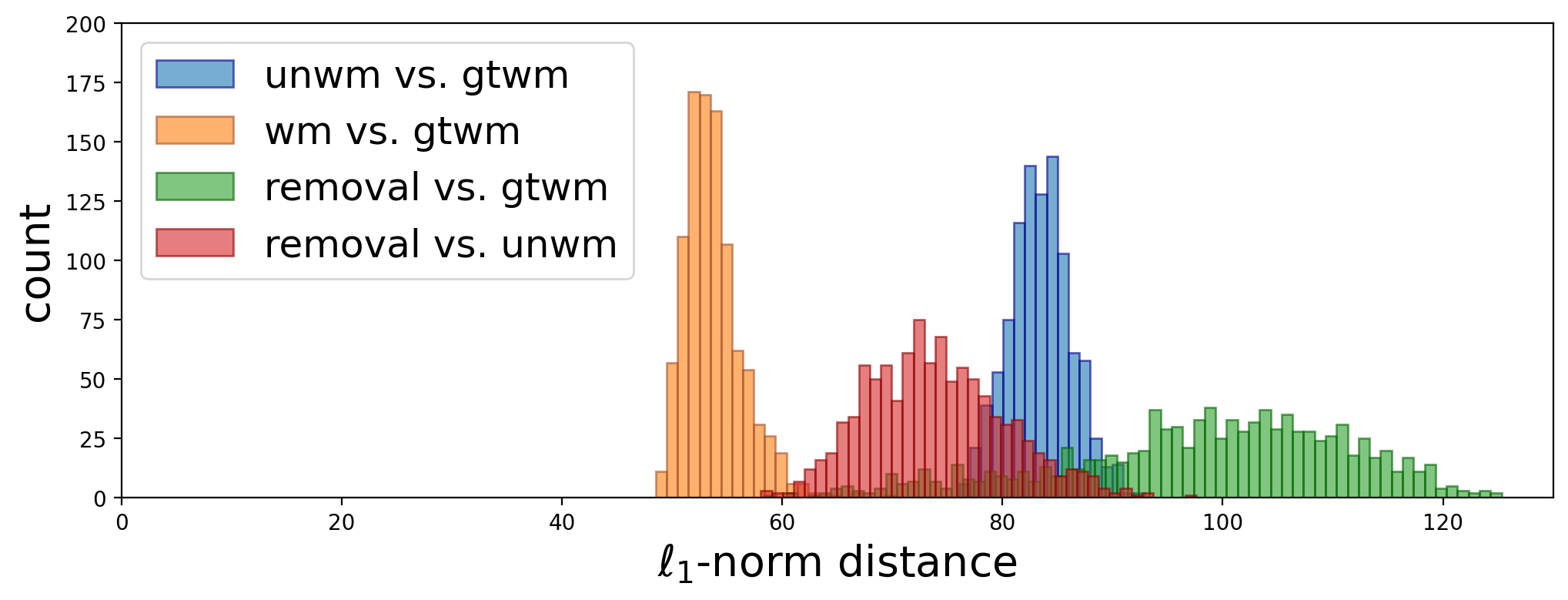}
        \caption{Watermark: Tree-Ring}
        \label{hist distance 1}
    \end{subfigure}
    \begin{subfigure}[b]{0.4\textwidth}
        \centering
        \includegraphics[width=0.9\textwidth]{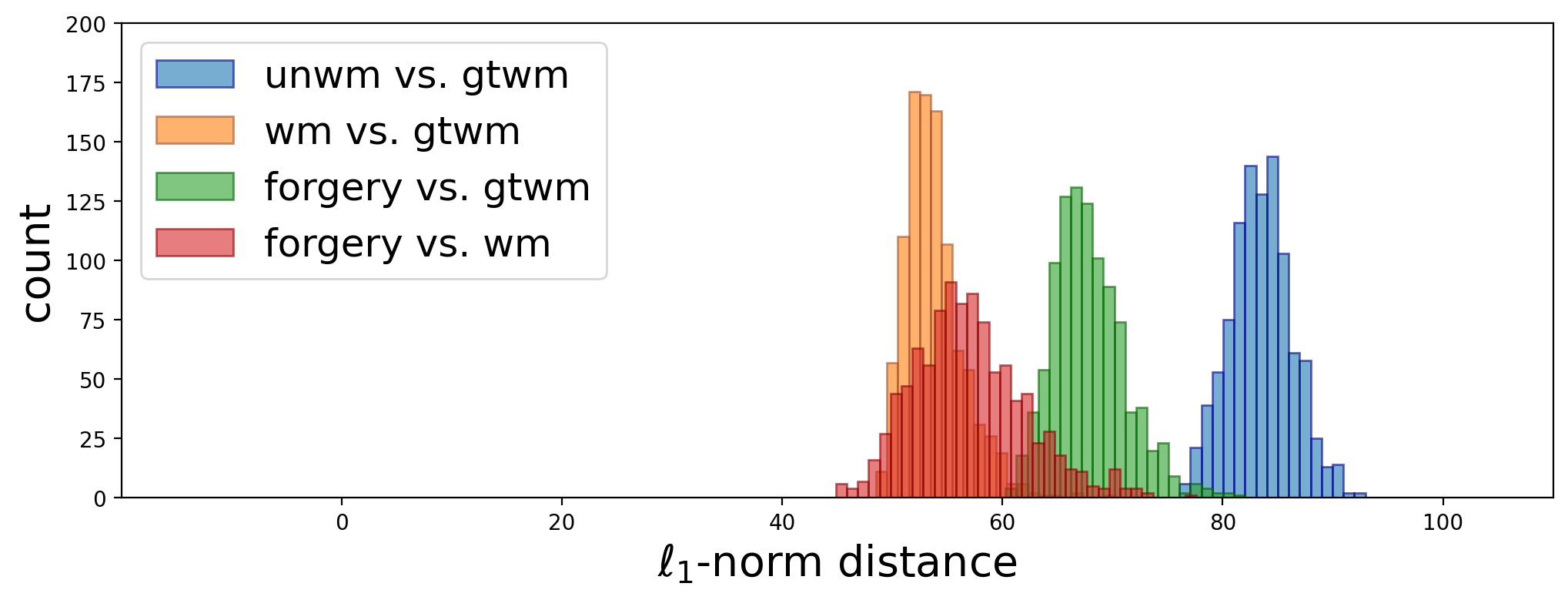}
        \caption{Watermark: Tree-Ring}
        \label{hist distance 2}
    \end{subfigure}
    \begin{subfigure}[b]{0.4\textwidth}
        \centering
        \includegraphics[width=0.9\textwidth]{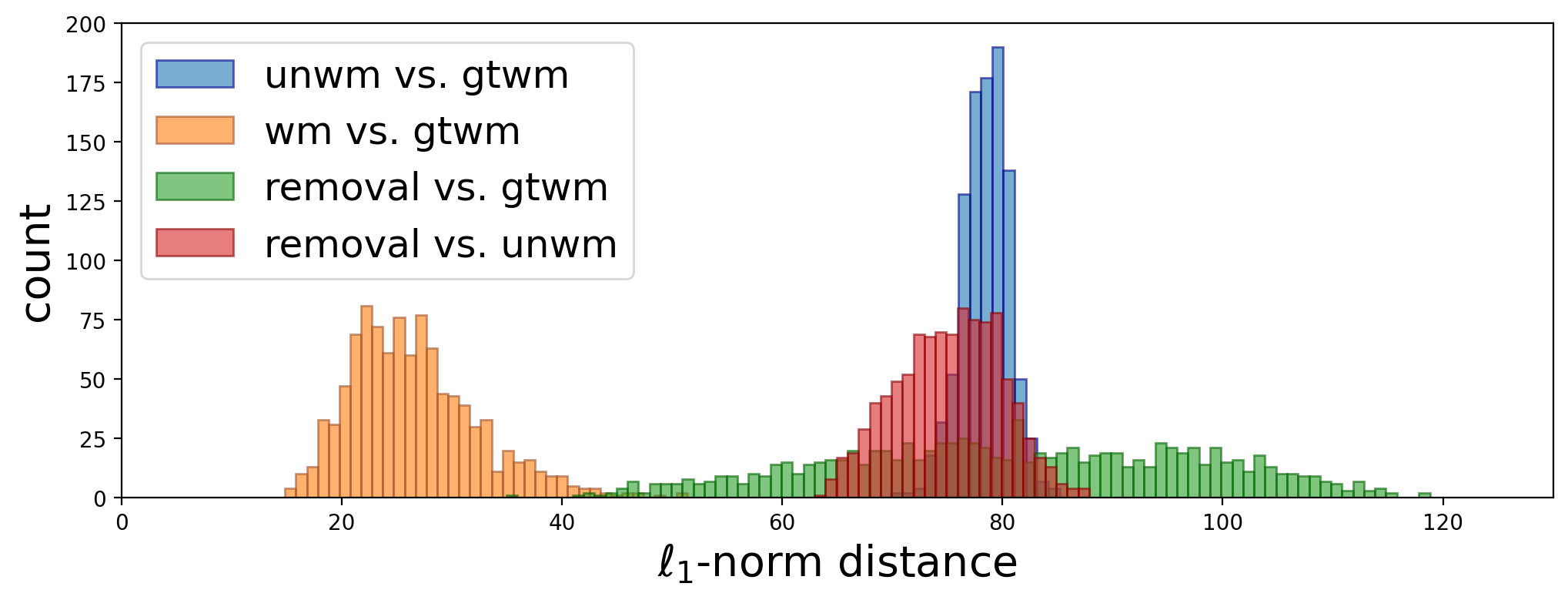}
        \caption{Watermark: RingID}
        \label{hist distance 3}
    \end{subfigure}
    \begin{subfigure}[b]{0.4\textwidth}
        \centering
        \includegraphics[width=0.9\textwidth]{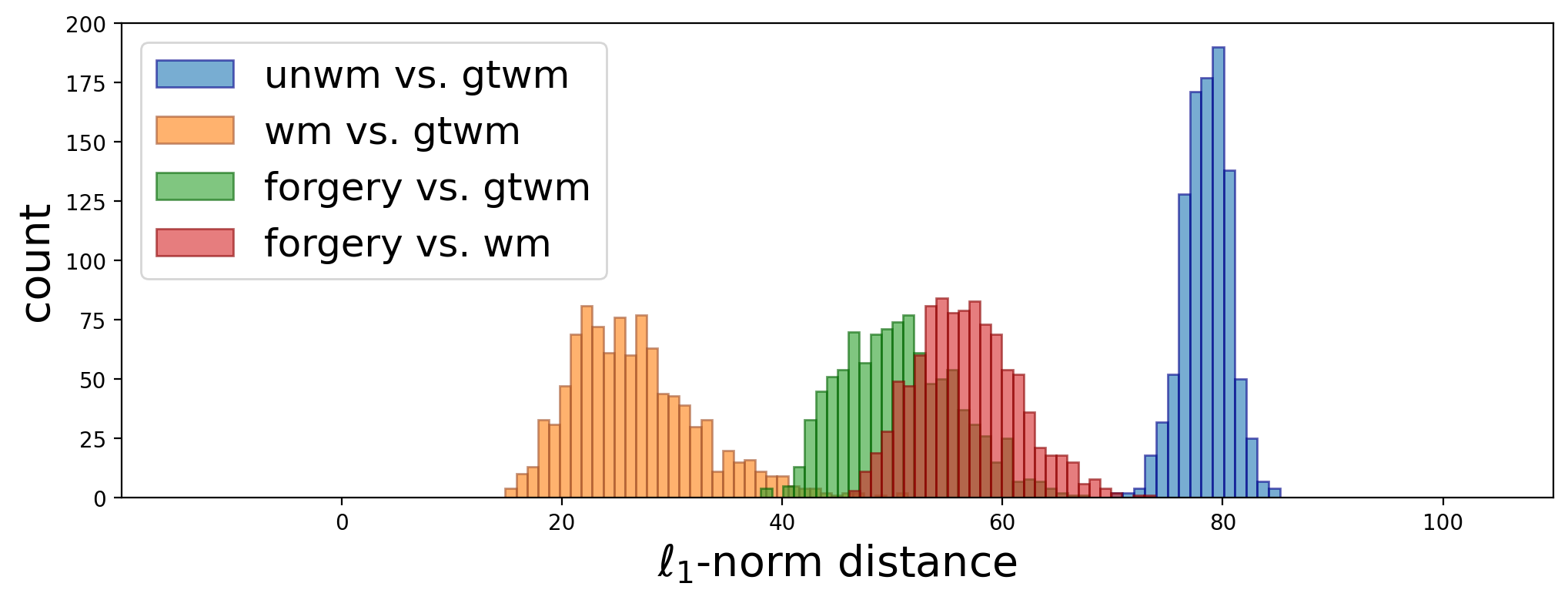}
        \caption{Watermark: RingID}
        \label{hist distance 4}
    \end{subfigure}
    \caption{Histograms of distance. When comparing the latent of the attacked images to samples from the normal distribution ({\em e.g.}, unwatermarked latent), we find that the attacked latent still maintains a certain distance from the unwatermarked latent, especially for the removal attack case. Here, ``gtwm'' means ``ground-truth watermark''; ``wm'' means ``watermarked images''; and ``unwm'' means ``unwatermarked images''. (Best viewed in a color display)}
    \label{remove and forgery}
\end{figure*}

To provide another perspective, we have looked closely at the forgery and removal attacks on images to examine the effect of distribution shifts in latent representations in Fig.~\ref{remove and forgery}. 
The interpretation of the x-axis of Fig.~\ref{remove and forgery} should not be confused with that of Fig.~\ref{distribution shift} in that the former describes the distance between the estimated watermark and the ground-truth watermark while 
the latter describes the mean value $\mathbf{\mu}_\mathbf{w}$ of watermarked distribution. 
Please also note that the interpretation of Fig.~\ref{remove and forgery} should consider each histogram's reference point and distance together, and that, in standard practice, we compute the distance between the inversed test latent and the ground truth watermark to decide if it is watermarked.

Specifically, in Fig.~\ref{remove and forgery}, we present the histograms of distance between two watermarks from two latent-based watermarking methods, Tree-Ring~\cite{wen2023treerings} and RingID~\cite{ci2024ringid}, under the forgery and removal attacks~\cite{yang2024can}. 
We use the blue, orange, and green histograms to represent the distances between the latent representation of ground truth watermark and those of (1) unwatermarked images, (2) watermarked images, and (3) forgery/removal attacked images, respectively. 
As for the red histogram, it represents the distance between the latent representations of unwatermarked images and those of forgery/removal attacked images.

We can observe from Fig.~\ref{remove and forgery}(a) and Fig.~\ref{remove and forgery}(c) that to remove a watermark successfully, the ideal way is to push the watermarked images away from the ground truth watermark distribution, {\em i.e.,} away from data that contain watermarks, which corresponds to the ``green'' histogram so that data being removed no longer stay in the ball. For forgery attack in Fig.~\ref{remove and forgery}(b) and Fig.~\ref{remove and forgery}(d), the attack adds an estimated watermark on other possibly unwatermarked images, which means that the attacker pushes unwatermarked image toward the ground truth watermark distribution that corresponds to the ``green'' histogram, indicating those data belong to the ball and hence are close to the watermark distribution. 

Interestingly, the “red” histogram establishes that the attacked images (via forgery or removal) do not remain closer to the desired distribution (watermarked or unwatermarked images). This suggests key weaknesses in such methods: they allow excessive tolerance, exhibit small distribution shifts, and leak significant information (as shown in~\cite{yang2024can}), enabling attackers to estimate the watermark. Consequently, these methods are vulnerable to forgery and removal attacks. 
As long as the perturbations move data into or out of the target region, the attack needs not to be precisely aligned with—or directly oppose—the watermark direction. Unlike traditional non-learning-based watermarking techniques, which resist such attacks while preserving image quality, these generative watermarking methods trade robustness for flexibility, making the watermarks easier to be compromised.

\section{Experimental Settings}\label{Sec: exp settings}
In this section, we give a detailed description of the settings of perturbations. The perturbations and the settings of each perturbation is as follows:
\begin{enumerate}
    \item[--] Blurring: blur with the kernel size set to $1,2,4,6,8$.
    \item[--] Color Jitter: increase the brightness of the image by setting the factor to $1,2,4,6,8$.
    \item[--] JPEG: compress the image with JPEG quality score set to $1,5,25,40,80$.
    \item[--] Rotation: rotate the image $5^{\circ}, 15^{\circ}, 30^{\circ}, 45^{\circ}, 75^{\circ}$ clock-wise.
    \item[--] Cropping: crop $10\%,25\%,40\%,60\%,80\%$ of the image ({\em i.e.}, remaining factor $0.9,0.75,0.6,0.4,0.2$, respectively).
    \item[--] Elastic deformation: produces a see-through-water-like effect on the image, with $(\alpha,\sigma)=(500,50),(600,60),(750,75),(850,85),(1000,100)$.
    \item[--] Random perspective: perform random perspective on the image, with distortion scale set to $0.5,0.6,0.75,0.85,1.0$.
    \item[--] Shear: shear the image with the parameter set to $5,7,10,15,20$.
    \item[--] Regen-VAE: regenerate the image with VAE with the quality level set to $1,2,3,4,5$.
    \item[--] Regen-DiffPure: regenerate the image with Stable Diffusion v1-4 with the number of noise and denoise step set to $10,20,30,40,50$.
\end{enumerate}
For all cases of Tree-Ring, RingID, HSQR, HSTR, and PRCW, 100 watermarked images were processed through ten types of distortions and subsequently fed into the model for detection. 

\section{Complete Experiments of Tree-Ring, HSQR, and PRCW}\label{sec: full figure exp}
We present the experiments related to the effects of perturbations on Tree-Ring, HSQR, and PRCW. Each watermarking scheme is attacked by seven perturbations with five different parameters to control the strength of the perturbation. Please refer to Figs.~\ref{fig: TR attack shifting},~\ref{fig: HSQR attack shifting},~\ref{fig: PRCW decoding 1e-2fpr attack shifting},~\ref{fig: PRCW detection 1e-2fpr attack shifting},~\ref{fig: PRCW decoding 1e-5fpr attack shifting}, and~\ref{fig: PRCW detection 1e-5fpr attack shifting}.

Following the official implementation of the PRCW method~\cite{gunn2024undetectable}, first we initialize a key pair $(\texttt{encode\_key},\texttt{decode\_key})$, and subsequently generate 100 watermarked images using \texttt{encode\_key}. We then apply seven types of image distortions listed in Sec.~\ref{Sec: exp settings} of Appendix, each evaluated across five distinct intensity levels. Subsequently, decoding and detection are performed on the distorted images and their corresponding unattacked counterparts, using the \texttt{decode\_key}. The detection is determined by two binary criteria: (i) Decoding component: a bit-wise comparison between the \texttt{recovered\_bits} (from the input image) and the \texttt{test\_bit} (extracted by \texttt{decode\_key}), which requires an exact match to be considered valid, and (ii) Detection component: a threshold-based detection result. A final detection is deemed successful if either of these two criteria is satisfied.

To facilitate the visualization of the distance between perturbed watermarked images and their ground-truth watermarks, we employ the Hamming distance $\|\cdot\|_H$ for the decoding component. For the detection component, we utilize the following discriminant from Algorithm 3 of PRCW~\cite{gunn2024undetectable}:
\begin{align}
\label{eq: PRCW Alog 3 L 5}
\sum_{\mathbf{w}\in\mathbf{P}}\log\left(\frac{1+\hat{s}_\mathbf{w}}{2}\right)-\sqrt{C\log(1/F)}+\frac{1}{2}\sum_{\mathbf{w}\in\mathbf{P}}\log\left(\frac{1-\hat{s}_\mathbf{w}^2}{4}\right)\,
\left\{\begin{array}{lll}
    \ge0 & \Rightarrow & \mathrm{True},\\
    <0 & \Rightarrow & \mathrm{False}.
\end{array}\right.
\end{align}
The corresponding results are illustrated in Figs.~\ref{fig: PRCW decoding 1e-2fpr attack shifting} and~\ref{fig: PRCW detection 1e-2fpr attack shifting}, respectively, under 1e-2 FPR. For 1e-5 FPR, please refer to Figs.~\ref{fig: PRCW decoding 1e-5fpr attack shifting} and~\ref{fig: PRCW detection 1e-5fpr attack shifting}, where blue dotted lines indicate the detection threshold.

We also present the numerical errors in Table~\ref{Table: numerical error}. For visualizing the effects of numerical errors in DDIM inversion on $\|\mathbf{w}-\hat{\mathbf{w}}^{\prime}\|$, we use histograms to demonstrate the shifting of $\|\mathbf{w}-\hat{\mathbf{w}}^{\prime}\|$ in Fig.~\ref{fig: numerical error}.

\begin{figure}[t]
        \begin{subfigure}[c]{0.185\textwidth}
            \centering
            \subfloat{\centering
            \makebox[10pt][l]{\raisebox{5pt}{\rotatebox[origin=l]{90}{\tiny Blurring}}}
            \includegraphics[width=1.03\linewidth]{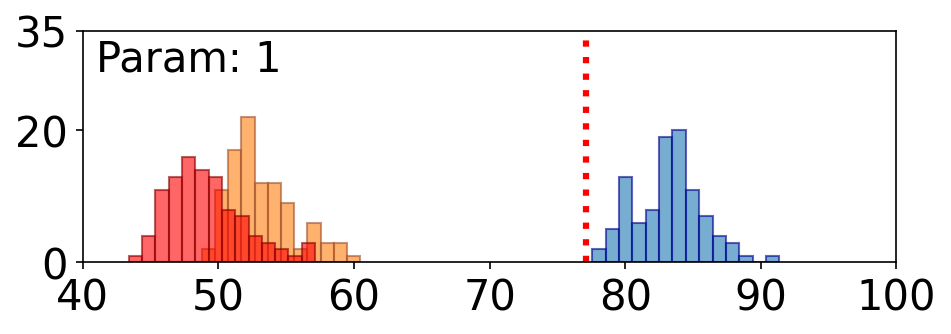}
            \includegraphics[width=1.03\linewidth]{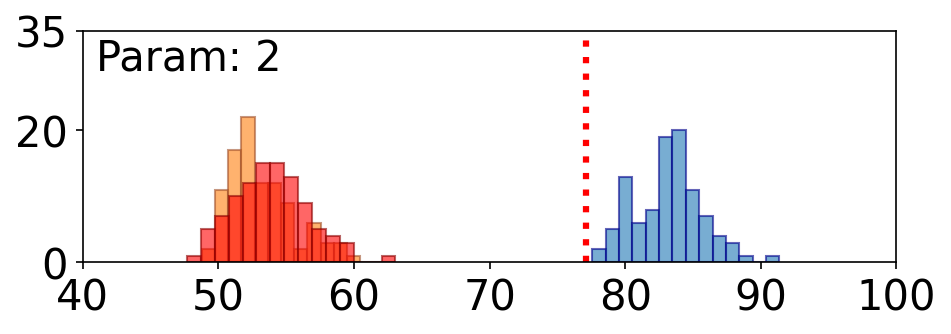}
            \includegraphics[width=1.03\linewidth]{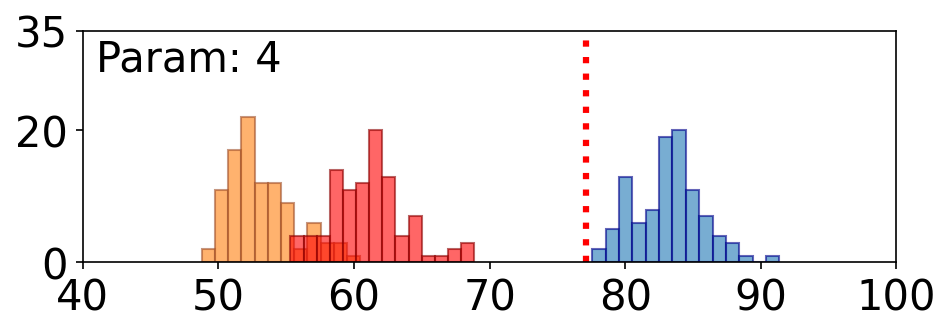}
            \includegraphics[width=1.03\linewidth]{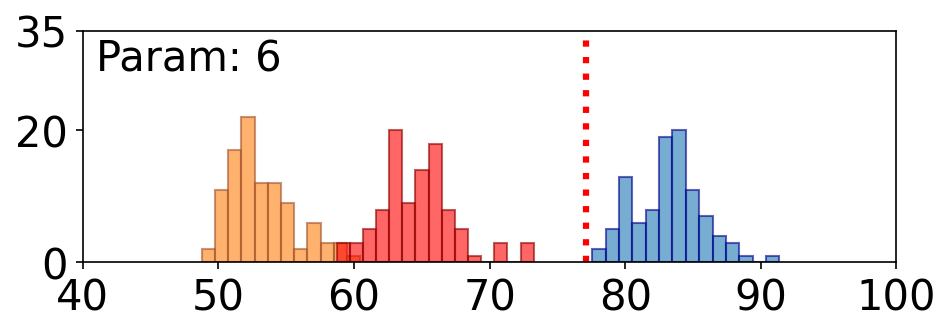}
            \includegraphics[width=1.03\linewidth]{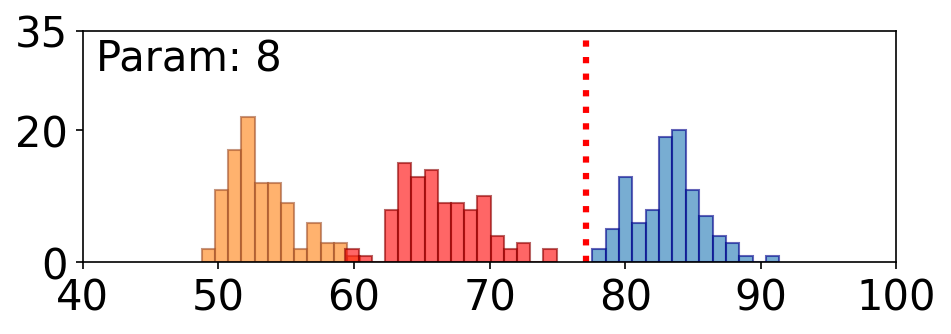}}
    
            \subfloat{\centering
            \makebox[10pt][l]{\raisebox{7pt}{\rotatebox[origin=l]{90}{\tiny Jitter}}}
            \includegraphics[width=1.03\linewidth]{figs/figs/treering/hist_dist_treering_color_jitter_1.png}
            \includegraphics[width=1.03\linewidth]{figs/figs/treering/hist_dist_treering_color_jitter_2.png}
            \includegraphics[width=1.03\linewidth]{figs/figs/treering/hist_dist_treering_color_jitter_4.png}
            \includegraphics[width=1.03\linewidth]{figs/figs/treering/hist_dist_treering_color_jitter_6.png}
            \includegraphics[width=1.03\linewidth]{figs/figs/treering/hist_dist_treering_color_jitter_8.png}}
    
            \subfloat{\centering
            \makebox[10pt][l]{\raisebox{7pt}{\rotatebox[origin=l]{90}{\tiny JPEG}}}
            \includegraphics[width=1.03\linewidth]{figs/figs/treering/hist_dist_treering_jpeg_80.png}
            \includegraphics[width=1.03\linewidth]{figs/figs/treering/hist_dist_treering_jpeg_40.png}
            \includegraphics[width=1.03\linewidth]{figs/figs/treering/hist_dist_treering_jpeg_25.png}
            \includegraphics[width=1.03\linewidth]{figs/figs/treering/hist_dist_treering_jpeg_5.png}
            \includegraphics[width=1.03\linewidth]{figs/figs/treering/hist_dist_treering_jpeg_1.png}}
    
            \subfloat{\centering
            \makebox[10pt][l]{\raisebox{3pt}{\rotatebox[origin=l]{90}{\tiny Rotation}}}
            \includegraphics[width=1.03\linewidth]{figs/figs/treering/hist_dist_treering_rotation_5.png}
            \includegraphics[width=1.03\linewidth]{figs/figs/treering/hist_dist_treering_rotation_15.png}
            \includegraphics[width=1.03\linewidth]{figs/figs/treering/hist_dist_treering_rotation_30.png}
            \includegraphics[width=1.03\linewidth]{figs/figs/treering/hist_dist_treering_rotation_45.png}
            \includegraphics[width=1.03\linewidth]{figs/figs/treering/hist_dist_treering_rotation_75.png}}
    
            \subfloat{\centering
            \makebox[10pt][l]{\raisebox{3pt}{\rotatebox[origin=l]{90}{\tiny Cropping}}}
            \includegraphics[width=1.03\linewidth]{figs/figs/treering/hist_dist_treering_cropping_0.9.png}
            \includegraphics[width=1.03\linewidth]{figs/figs/treering/hist_dist_treering_cropping_0.75.png}
            \includegraphics[width=1.03\linewidth]{figs/figs/treering/hist_dist_treering_cropping_0.6.png}
            \includegraphics[width=1.03\linewidth]{figs/figs/treering/hist_dist_treering_cropping_0.4.png}
            \includegraphics[width=1.03\linewidth]{figs/figs/treering/hist_dist_treering_cropping_0.2.png}}
    
            \subfloat{\centering
            \makebox[10pt][l]{\raisebox{3pt}{\rotatebox[origin=l]{90}{\tiny Elastic}}}
            \includegraphics[width=1.03\linewidth]{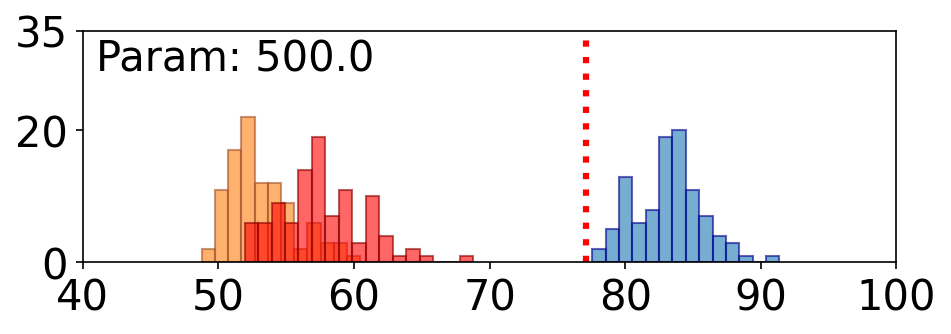}
            \includegraphics[width=1.03\linewidth]{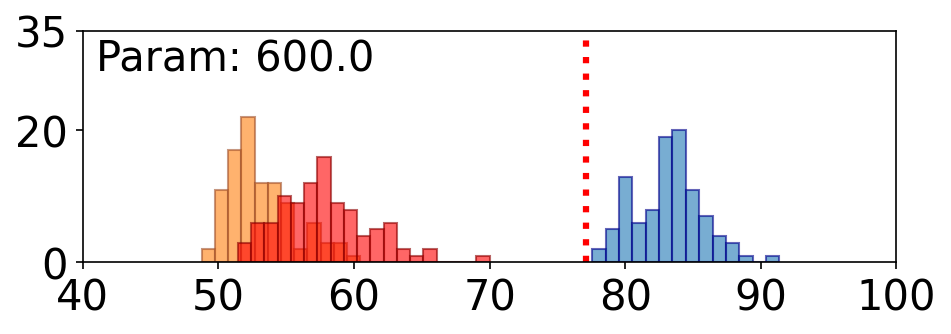}
            \includegraphics[width=1.03\linewidth]{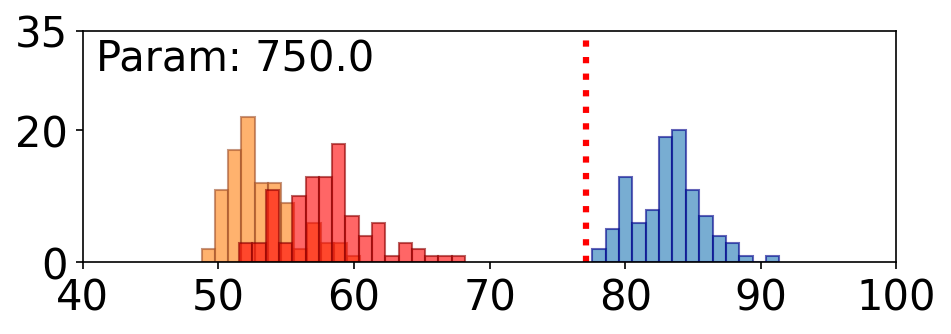}
            \includegraphics[width=1.03\linewidth]{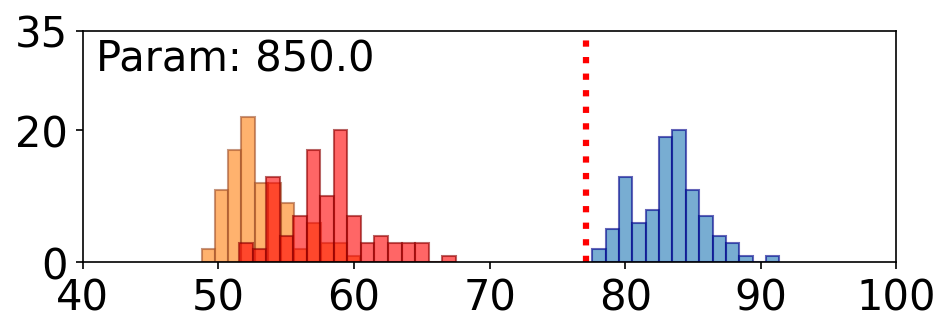}
            \includegraphics[width=1.03\linewidth]{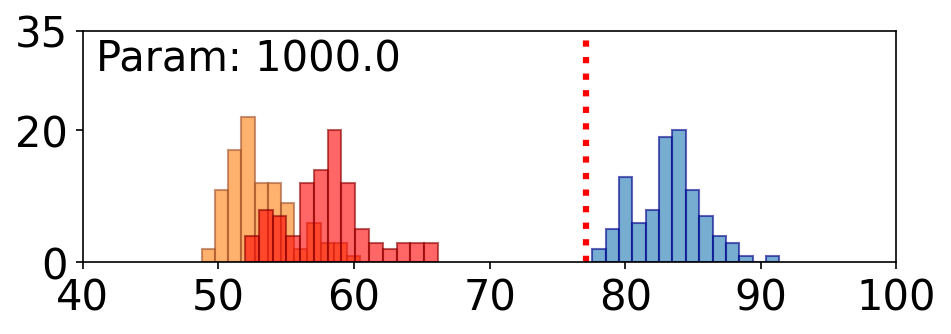}}
    
            \subfloat{\centering
            \makebox[10pt][l]{\raisebox{3pt}{\rotatebox[origin=l]{90}{\scalebox{.4}{Perspective}}}}
            \includegraphics[width=1.03\linewidth]{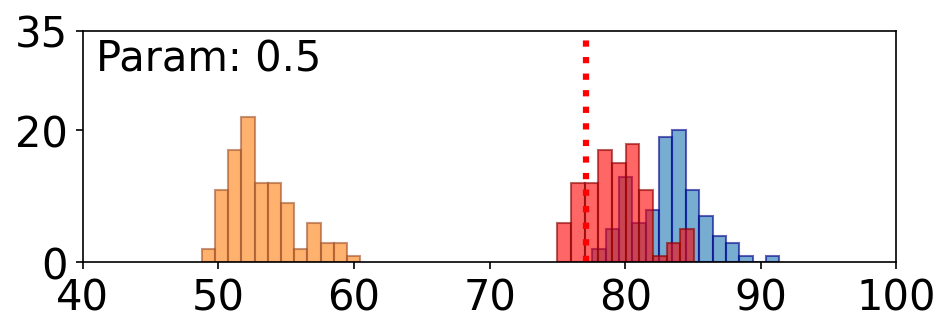}
            \includegraphics[width=1.03\linewidth]{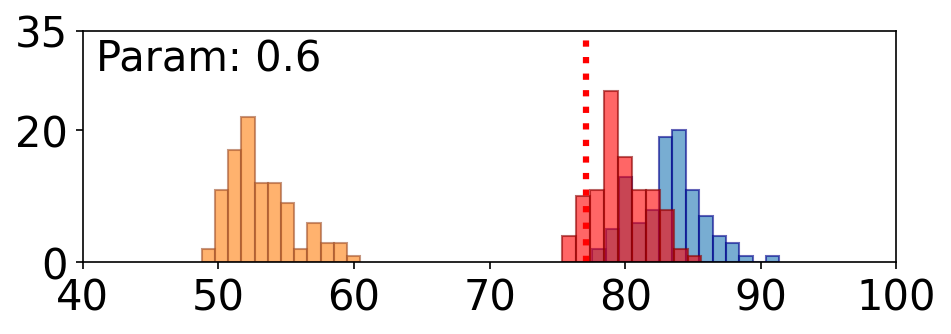}
            \includegraphics[width=1.03\linewidth]{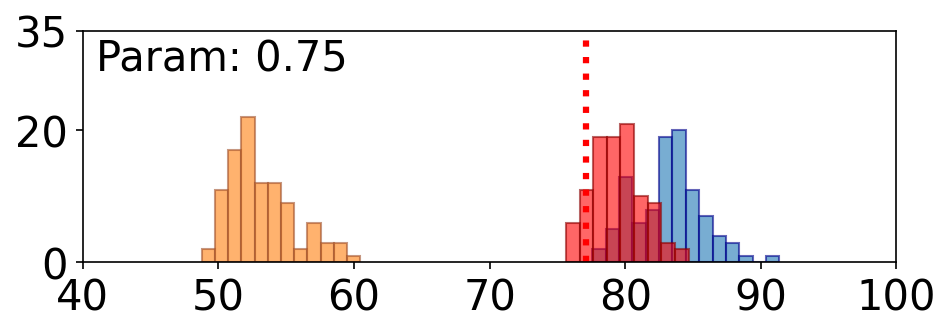}
            \includegraphics[width=1.03\linewidth]{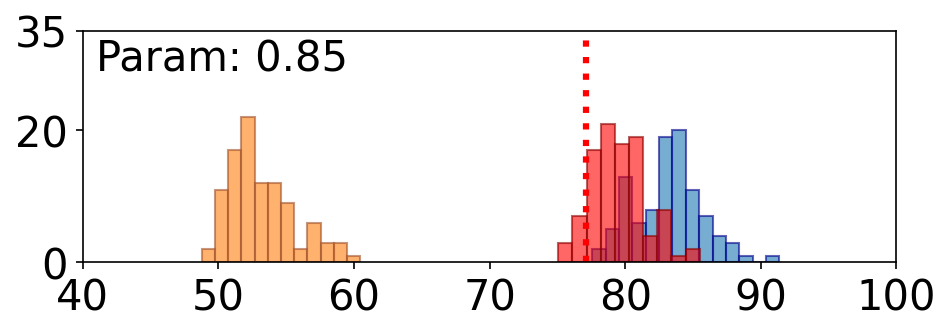}
            \includegraphics[width=1.03\linewidth]{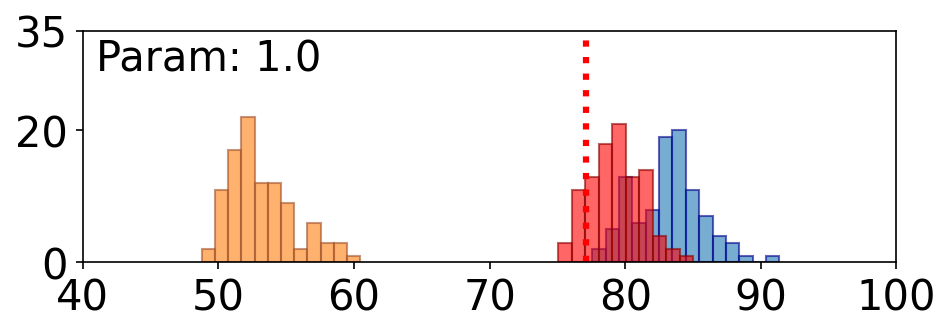}}
    
            \subfloat{\centering
            \makebox[10pt][l]{\raisebox{3pt}{\rotatebox[origin=l]{90}{\tiny Shear}}}
            \includegraphics[width=1.03\linewidth]{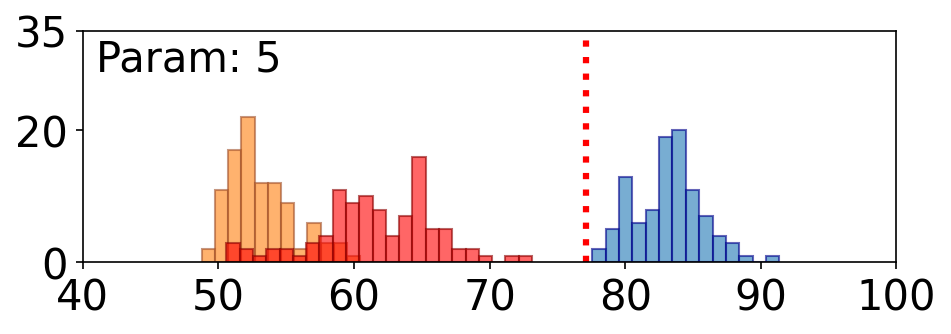}
            \includegraphics[width=1.03\linewidth]{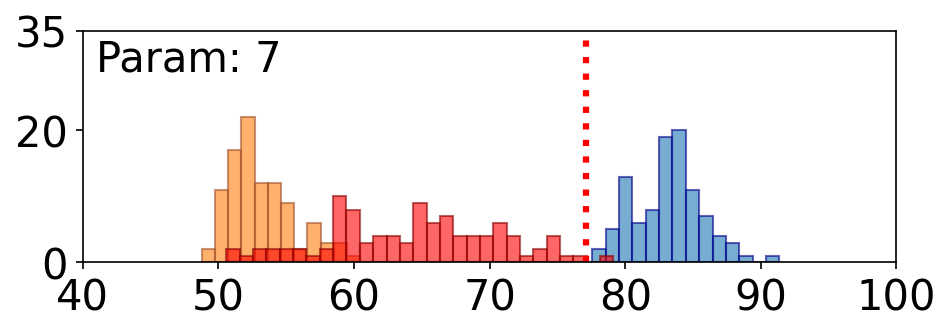}
            \includegraphics[width=1.03\linewidth]{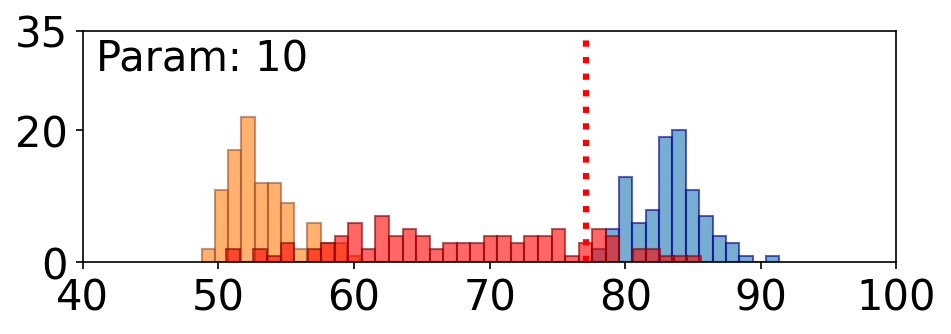}
            \includegraphics[width=1.03\linewidth]{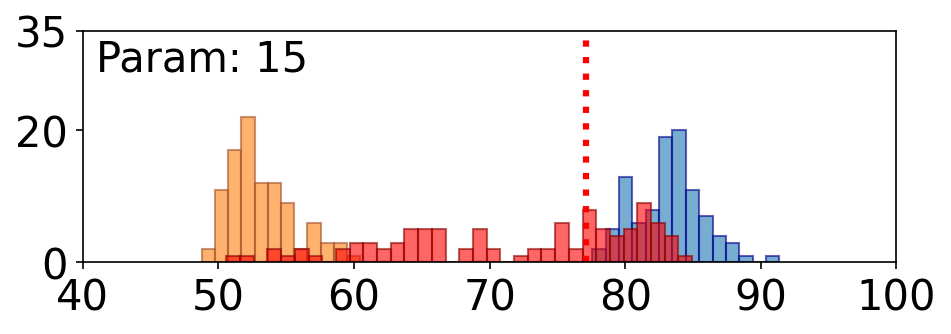}
            \includegraphics[width=1.03\linewidth]{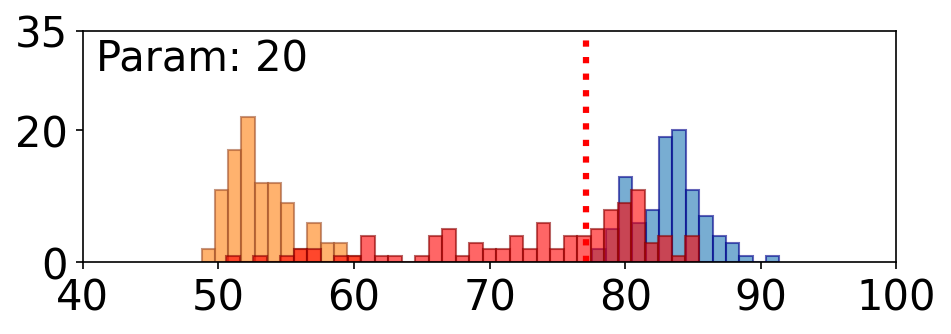}}

            \subfloat{\centering
            \makebox[10pt][l]{\raisebox{1pt}{\rotatebox[origin=l]{90}{\scalebox{.4}{Regen-VAE}}}}
            \includegraphics[width=1.03\linewidth]{figs/figs/treering/hist_dist_treering_regen_vae_5.png}
            \includegraphics[width=1.03\linewidth]{figs/figs/treering/hist_dist_treering_regen_vae_4.png}
            \includegraphics[width=1.03\linewidth]{figs/figs/treering/hist_dist_treering_regen_vae_3.png}
            \includegraphics[width=1.03\linewidth]{figs/figs/treering/hist_dist_treering_regen_vae_2.png}
            \includegraphics[width=1.03\linewidth]{figs/figs/treering/hist_dist_treering_regen_vae_1.png}}
    
            \subfloat{\centering
            \makebox[10pt][l]{\raisebox{0pt}{\rotatebox[origin=l]{90}{\scalebox{.4}{Regen-Diff}}}}
            \includegraphics[width=1.03\linewidth]{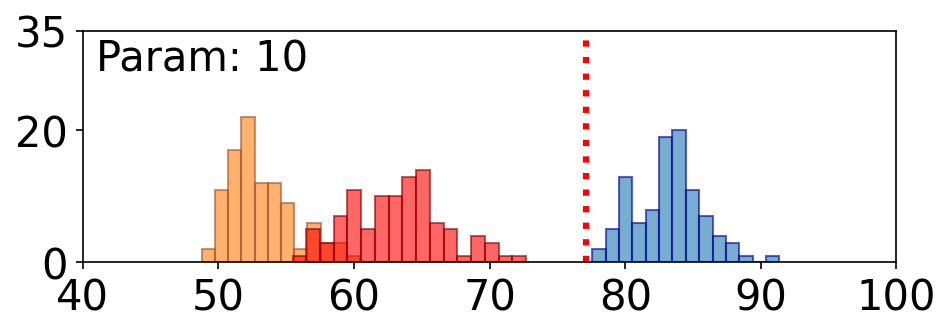}
            \includegraphics[width=1.03\linewidth]{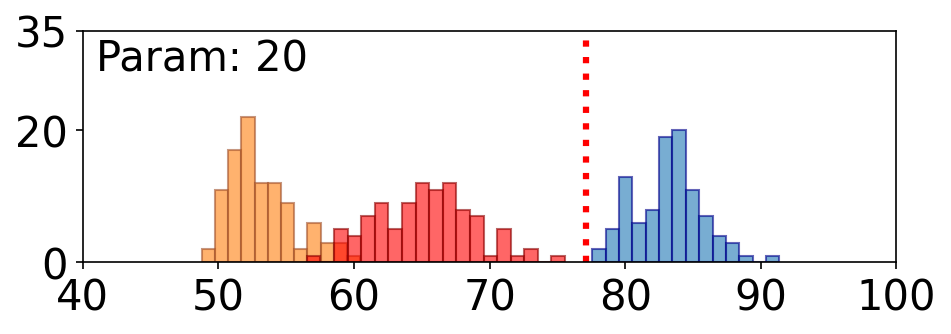}
            \includegraphics[width=1.03\linewidth]{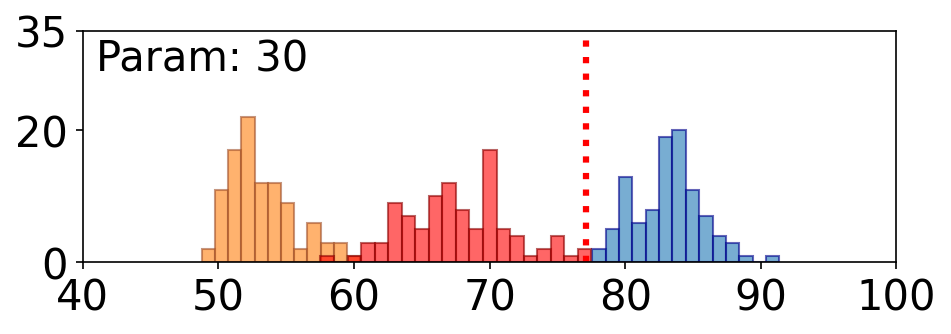}
            \includegraphics[width=1.03\linewidth]{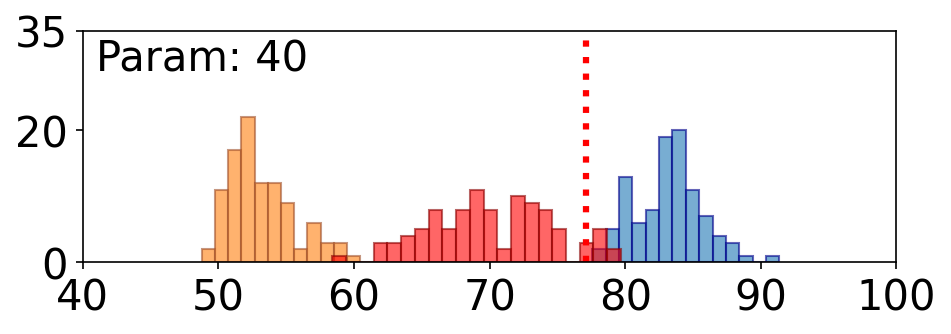}
            \includegraphics[width=1.03\linewidth]{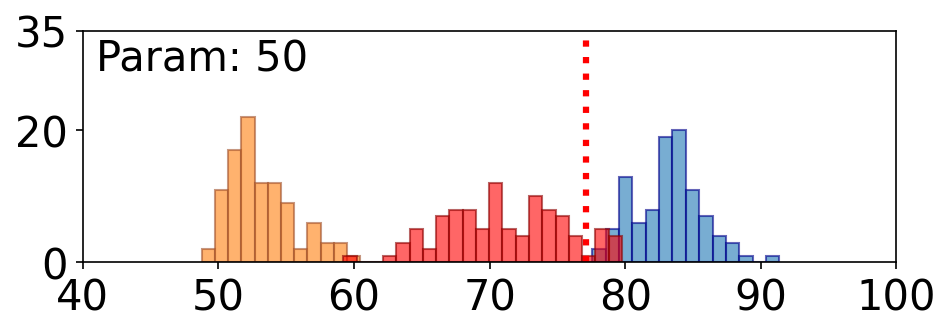}}
        \end{subfigure}
        
    \caption{The effects of different perturbations on $\|\mathbf{w}-\hat{\mathbf{w}}^{\prime}\|$. From left to right, the perturbations are set with different parameters (Param). 
    Watermark: Tree-Ring. Thresholds: 77.12 (1e-2 FPR).}
    \label{fig: TR attack shifting}
\end{figure}

\begin{figure}[t]
        \begin{subfigure}[c]{0.185\textwidth}
            \centering
            \subfloat{\centering
            \makebox[10pt][l]{\raisebox{5pt}{\rotatebox[origin=l]{90}{\tiny Blurring}}}
            \includegraphics[width=1.03\linewidth]{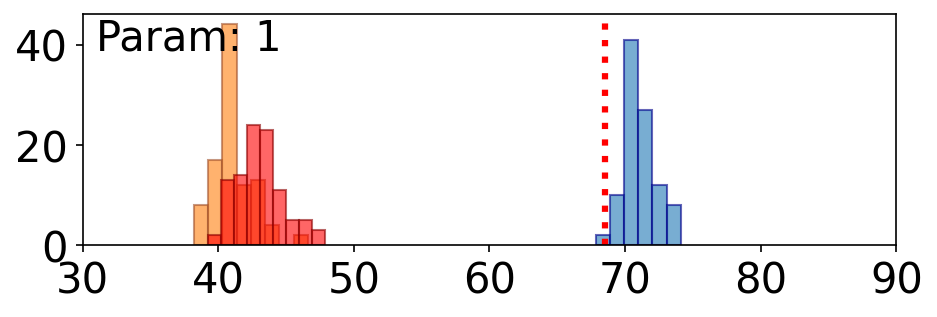}
            \includegraphics[width=1.03\linewidth]{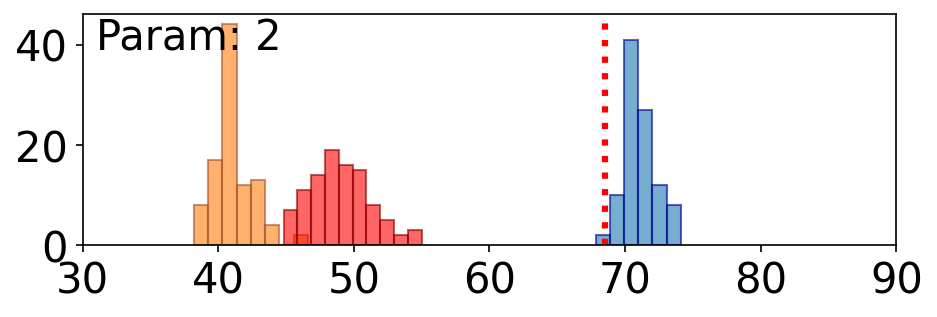}
            \includegraphics[width=1.03\linewidth]{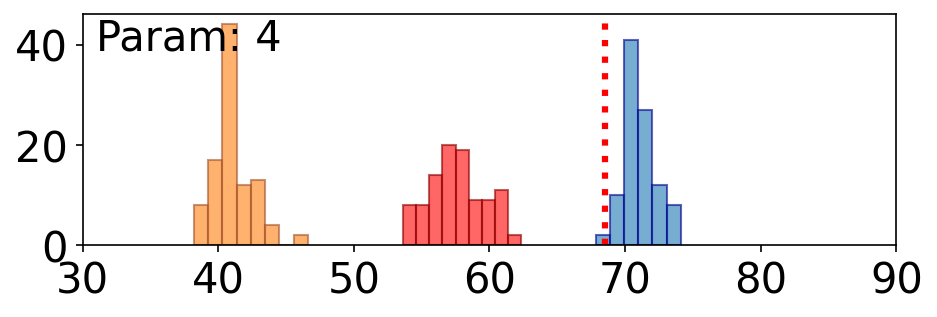}
            \includegraphics[width=1.03\linewidth]{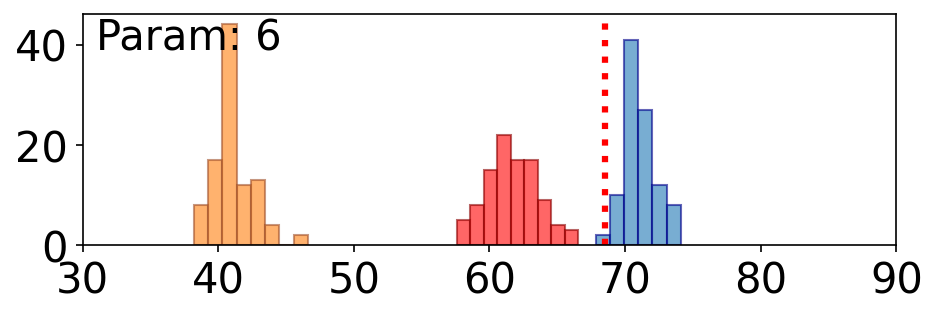}
            \includegraphics[width=1.03\linewidth]{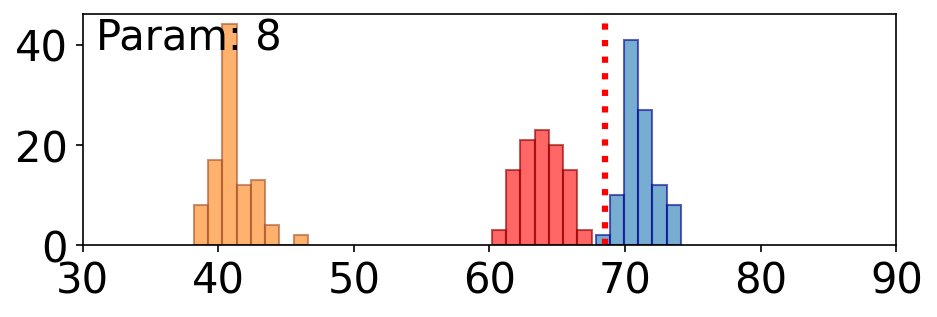}}
    
            \subfloat{\centering
            \makebox[10pt][l]{\raisebox{7pt}{\rotatebox[origin=l]{90}{\tiny Jitter}}}
            \includegraphics[width=1.03\linewidth]{figs/figs/hsqr/hist_dist_HSQR_color_jitter_1.png}
            \includegraphics[width=1.03\linewidth]{figs/figs/hsqr/hist_dist_HSQR_color_jitter_2.png}
            \includegraphics[width=1.03\linewidth]{figs/figs/hsqr/hist_dist_HSQR_color_jitter_4.png}
            \includegraphics[width=1.03\linewidth]{figs/figs/hsqr/hist_dist_HSQR_color_jitter_6.png}
            \includegraphics[width=1.03\linewidth]{figs/figs/hsqr/hist_dist_HSQR_color_jitter_8.png}}
    
            \subfloat{\centering
            \makebox[10pt][l]{\raisebox{7pt}{\rotatebox[origin=l]{90}{\tiny JPEG}}}
            \includegraphics[width=1.03\linewidth]{figs/figs/hsqr/hist_dist_HSQR_jpeg_80.png}
            \includegraphics[width=1.03\linewidth]{figs/figs/hsqr/hist_dist_HSQR_jpeg_40.png}
            \includegraphics[width=1.03\linewidth]{figs/figs/hsqr/hist_dist_HSQR_jpeg_25.png}
            \includegraphics[width=1.03\linewidth]{figs/figs/hsqr/hist_dist_HSQR_jpeg_5.png}
            \includegraphics[width=1.03\linewidth]{figs/figs/hsqr/hist_dist_HSQR_jpeg_1.png}}
    
            \subfloat{\centering
            \makebox[10pt][l]{\raisebox{3pt}{\rotatebox[origin=l]{90}{\tiny Rotation}}}
            \includegraphics[width=1.03\linewidth]{figs/figs/hsqr/hist_dist_HSQR_rotation_5.png}
            \includegraphics[width=1.03\linewidth]{figs/figs/hsqr/hist_dist_HSQR_rotation_15.png}
            \includegraphics[width=1.03\linewidth]{figs/figs/hsqr/hist_dist_HSQR_rotation_30.png}
            \includegraphics[width=1.03\linewidth]{figs/figs/hsqr/hist_dist_HSQR_rotation_45.png}
            \includegraphics[width=1.03\linewidth]{figs/figs/hsqr/hist_dist_HSQR_rotation_75.png}}
    
            \subfloat{\centering
            \makebox[10pt][l]{\raisebox{3pt}{\rotatebox[origin=l]{90}{\tiny Cropping}}}
            \includegraphics[width=1.03\linewidth]{figs/figs/hsqr/hist_dist_HSQR_cropping_0.9.png}
            \includegraphics[width=1.03\linewidth]{figs/figs/hsqr/hist_dist_HSQR_cropping_0.75.png}
            \includegraphics[width=1.03\linewidth]{figs/figs/hsqr/hist_dist_HSQR_cropping_0.6.png}
            \includegraphics[width=1.03\linewidth]{figs/figs/hsqr/hist_dist_HSQR_cropping_0.4.png}
            \includegraphics[width=1.03\linewidth]{figs/figs/hsqr/hist_dist_HSQR_cropping_0.2.png}}
    
            \subfloat{\centering
            \makebox[10pt][l]{\raisebox{3pt}{\rotatebox[origin=l]{90}{\tiny Elastic}}}
            \includegraphics[width=1.03\linewidth]{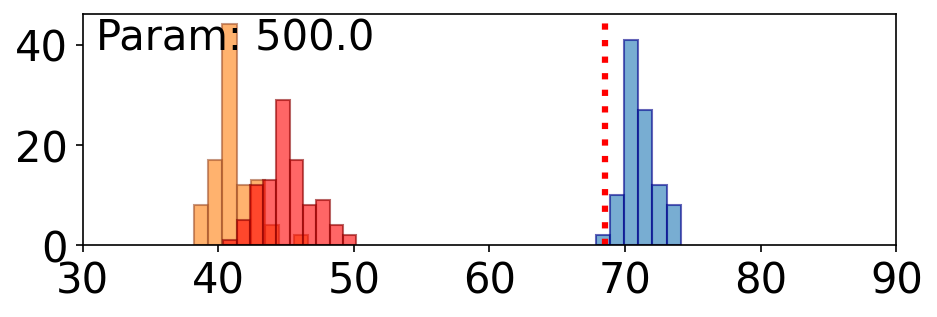}
            \includegraphics[width=1.03\linewidth]{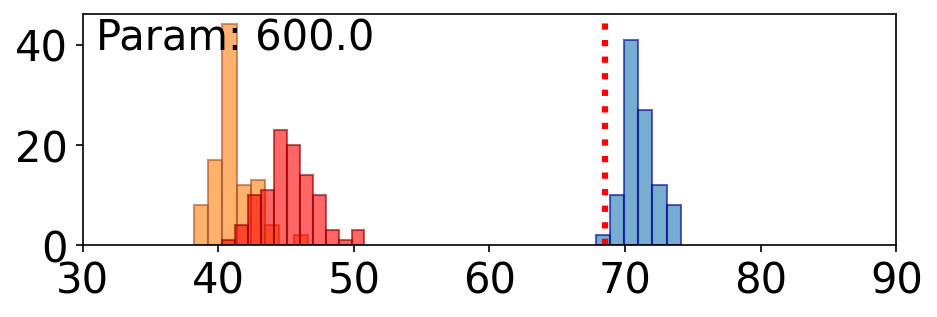}
            \includegraphics[width=1.03\linewidth]{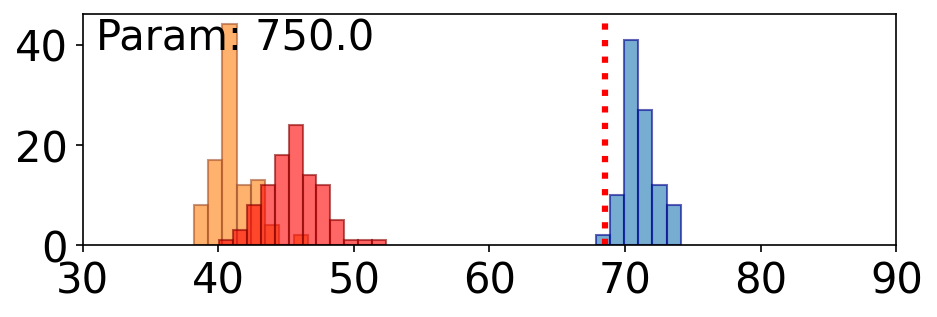}
            \includegraphics[width=1.03\linewidth]{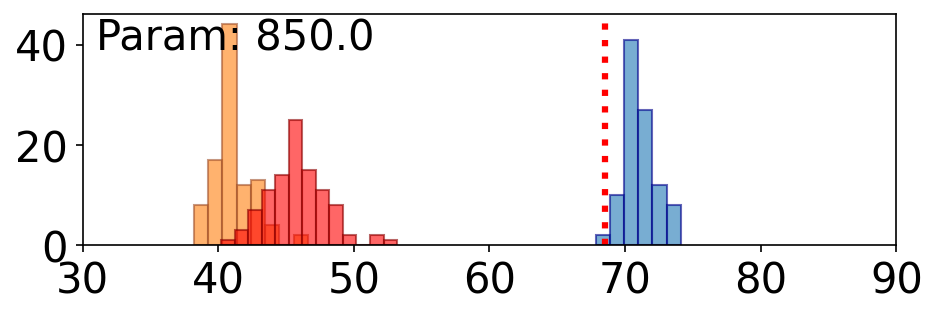}
            \includegraphics[width=1.03\linewidth]{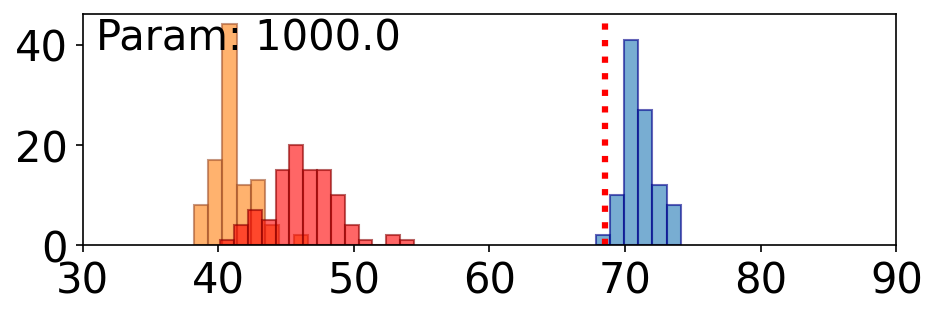}}
    
            \subfloat{\centering
            \makebox[10pt][l]{\raisebox{3pt}{\rotatebox[origin=l]{90}{\scalebox{.4}{Perspective}}}}
            \includegraphics[width=1.03\linewidth]{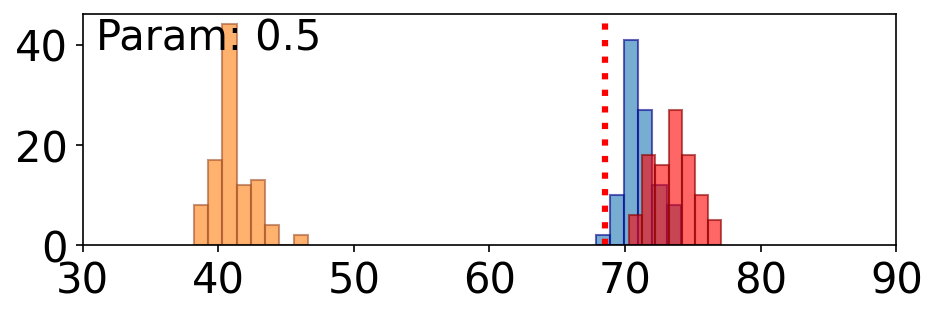}
            \includegraphics[width=1.03\linewidth]{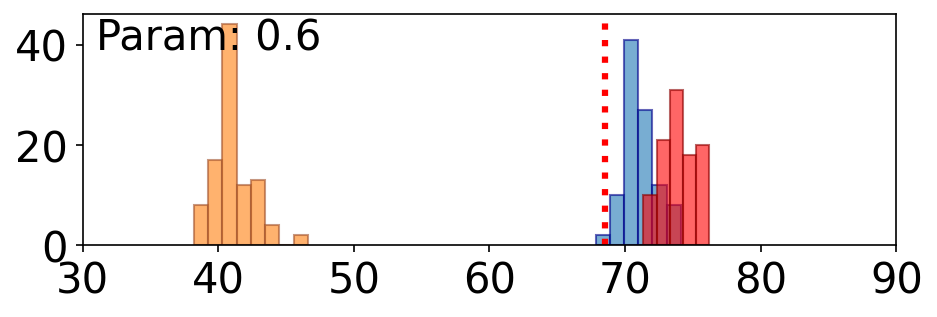}
            \includegraphics[width=1.03\linewidth]{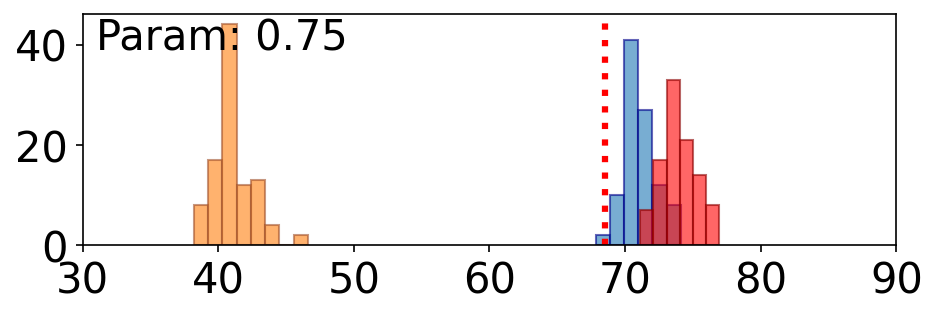}
            \includegraphics[width=1.03\linewidth]{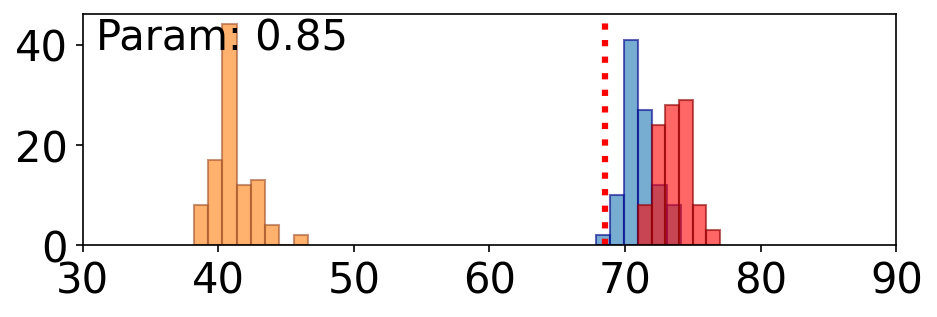}
            \includegraphics[width=1.03\linewidth]{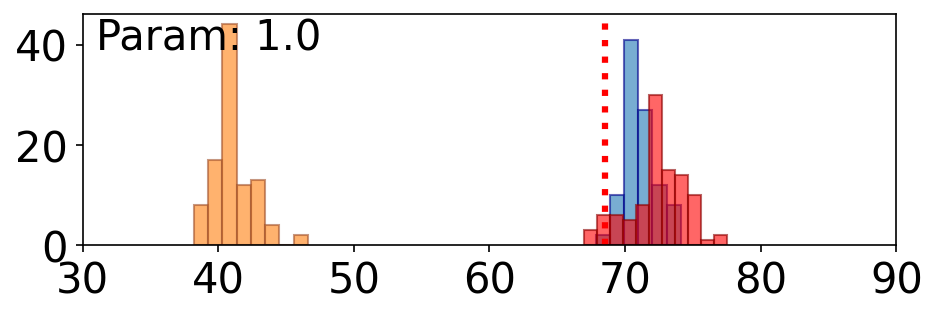}}
    
            \subfloat{\centering
            \makebox[10pt][l]{\raisebox{3pt}{\rotatebox[origin=l]{90}{\tiny Shear}}}
            \includegraphics[width=1.03\linewidth]{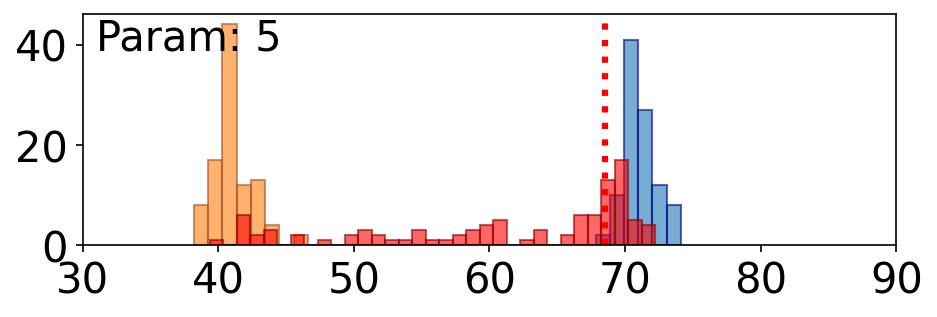}
            \includegraphics[width=1.03\linewidth]{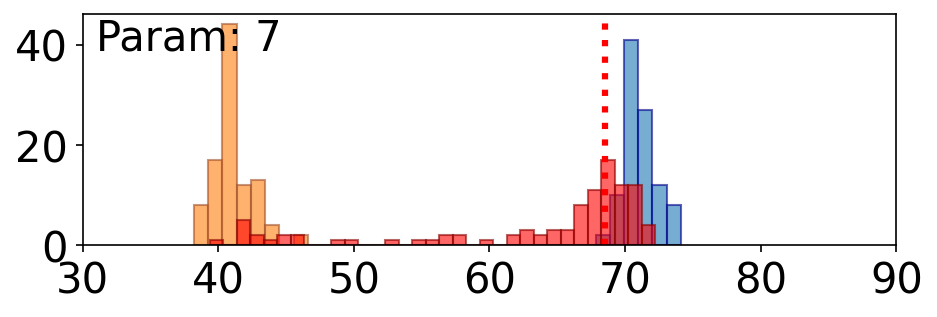}
            \includegraphics[width=1.03\linewidth]{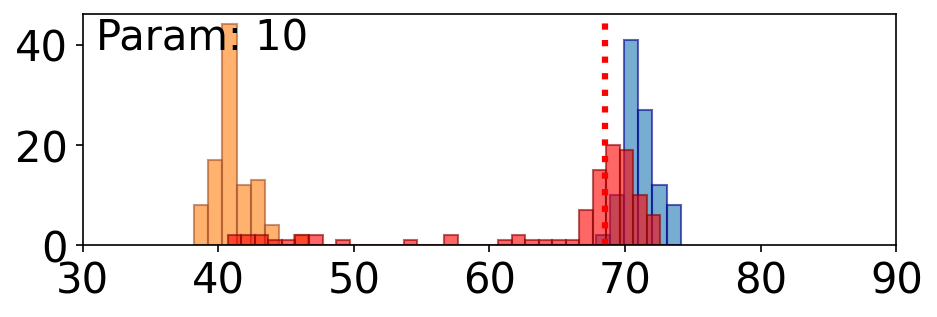}
            \includegraphics[width=1.03\linewidth]{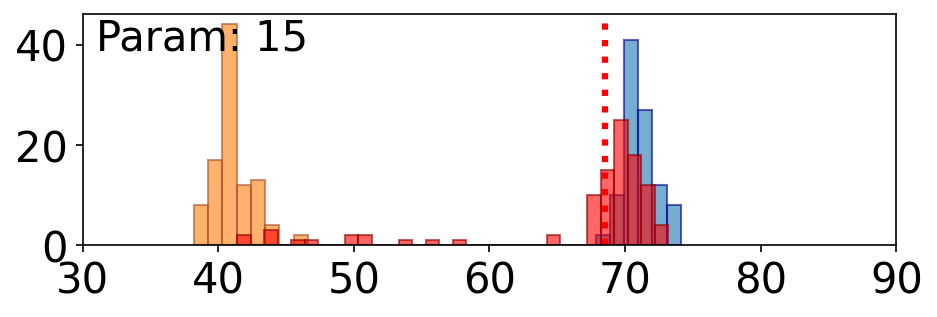}
            \includegraphics[width=1.03\linewidth]{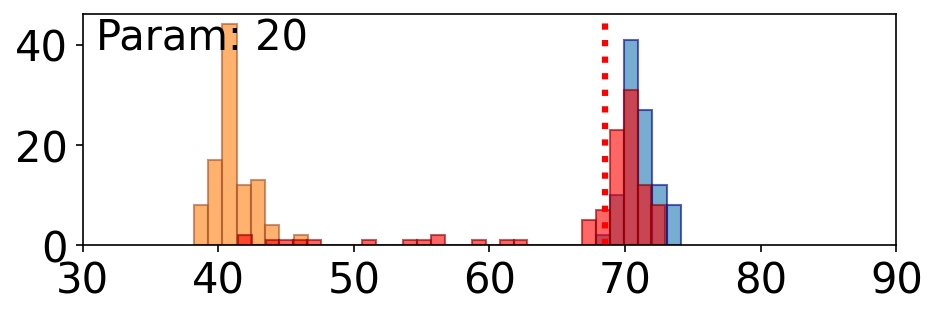}}

            \subfloat{\centering
            \makebox[10pt][l]{\raisebox{1pt}{\rotatebox[origin=l]{90}{\scalebox{.4}{Regen-VAE}}}}
            \includegraphics[width=1.03\linewidth]{figs/figs/hsqr/hist_dist_HSQR_regen_vae_5.png}
            \includegraphics[width=1.03\linewidth]{figs/figs/hsqr/hist_dist_HSQR_regen_vae_4.png}
            \includegraphics[width=1.03\linewidth]{figs/figs/hsqr/hist_dist_HSQR_regen_vae_3.png}
            \includegraphics[width=1.03\linewidth]{figs/figs/hsqr/hist_dist_HSQR_regen_vae_2.png}
            \includegraphics[width=1.03\linewidth]{figs/figs/hsqr/hist_dist_HSQR_regen_vae_1.png}}
    
            \subfloat{\centering
            \makebox[10pt][l]{\raisebox{0pt}{\rotatebox[origin=l]{90}{\scalebox{.4}{Regen-Diff}}}}
            \includegraphics[width=1.03\linewidth]{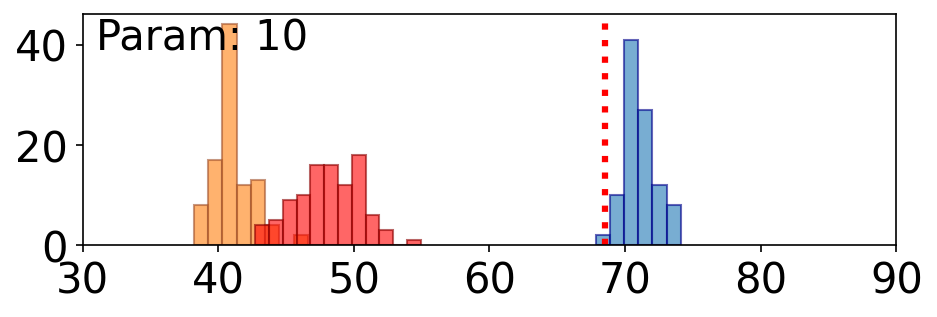}
            \includegraphics[width=1.03\linewidth]{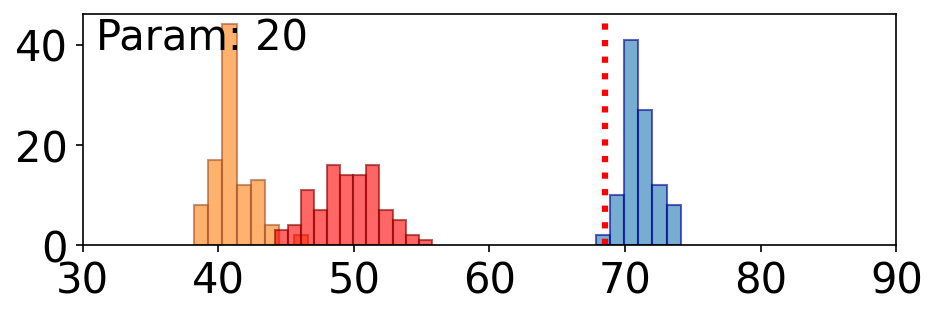}
            \includegraphics[width=1.03\linewidth]{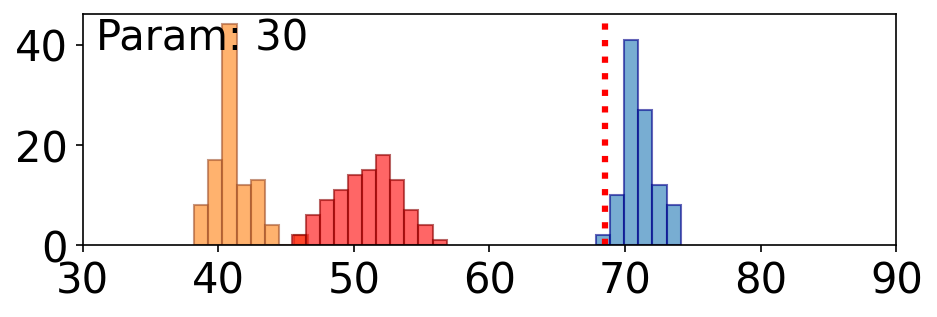}
            \includegraphics[width=1.03\linewidth]{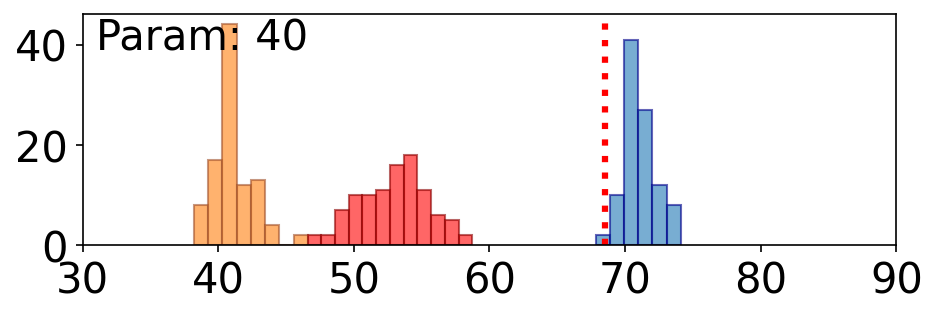}
            \includegraphics[width=1.03\linewidth]{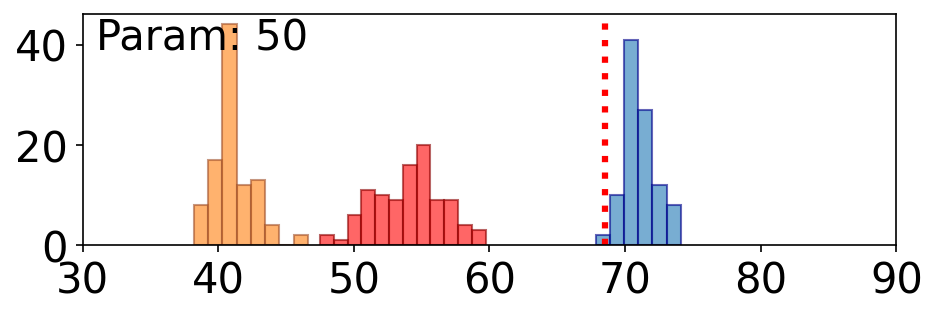}}
        \end{subfigure}
        
    \caption{The effects of different perturbations on $\|\mathbf{w}-\hat{\mathbf{w}}^{\prime}\|$. From left to right, the perturbations are set with different parameters (Param). 
    Watermark: HSQR. Thresholds: 68.52 (1e-2 FPR).}
    \label{fig: HSQR attack shifting}
\end{figure}

\begin{figure}[t]
        \begin{subfigure}[c]{0.185\textwidth}
            \centering
            \subfloat{\centering
            \makebox[10pt][l]{\raisebox{5pt}{\rotatebox[origin=l]{90}{\tiny Blurring}}}
            \includegraphics[width=1.03\linewidth]{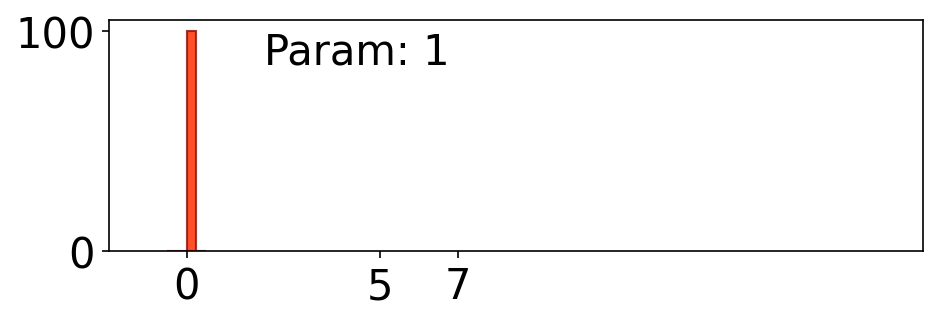}
            \includegraphics[width=1.03\linewidth]{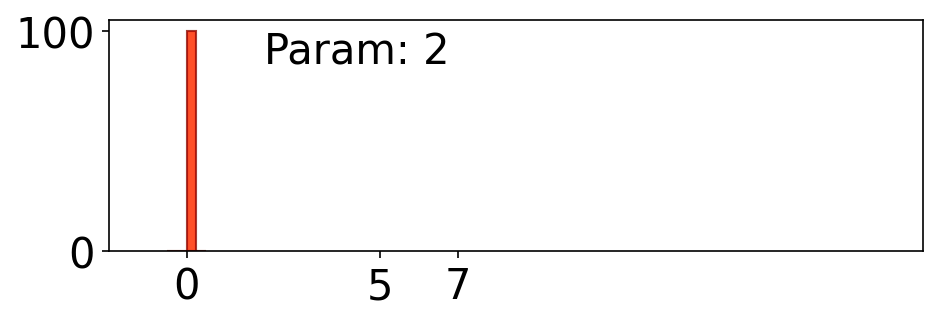}
            \includegraphics[width=1.03\linewidth]{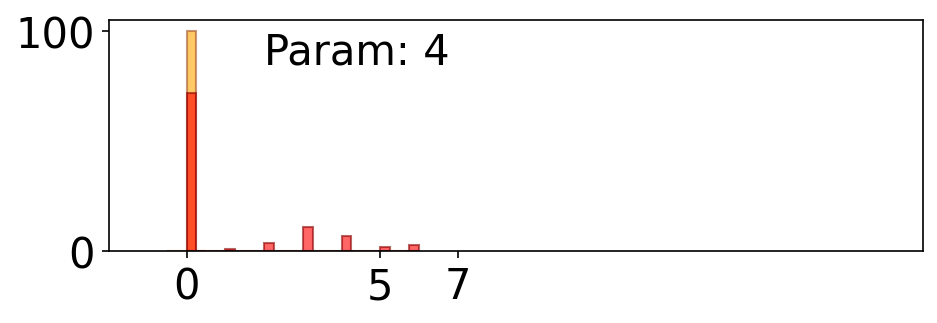}
            \includegraphics[width=1.03\linewidth]{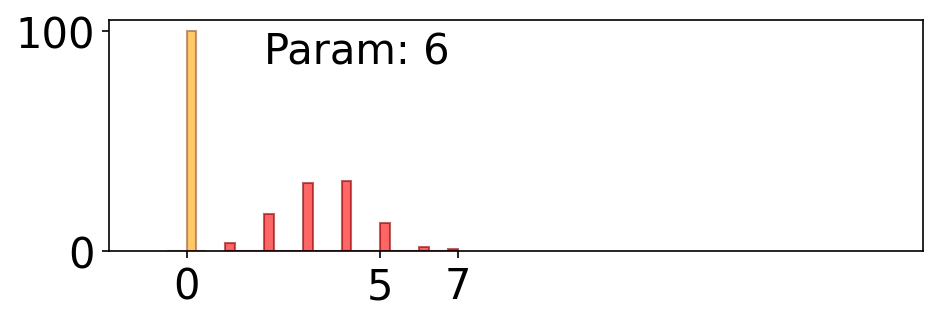}
            \includegraphics[width=1.03\linewidth]{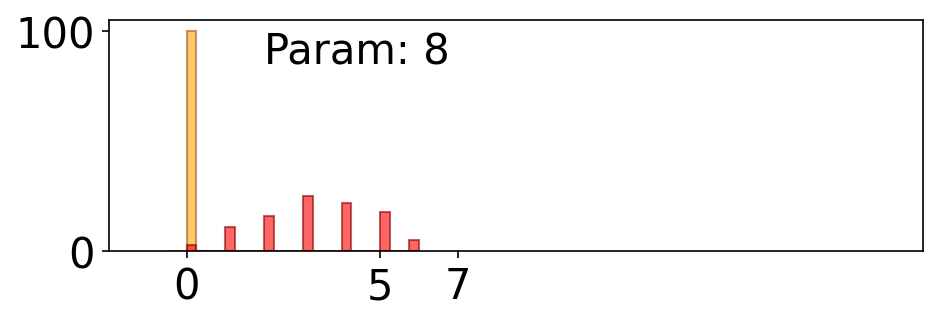}}
    
            \subfloat{\centering
            \makebox[10pt][l]{\raisebox{7pt}{\rotatebox[origin=l]{90}{\tiny Jitter}}}
            \includegraphics[width=1.03\linewidth]{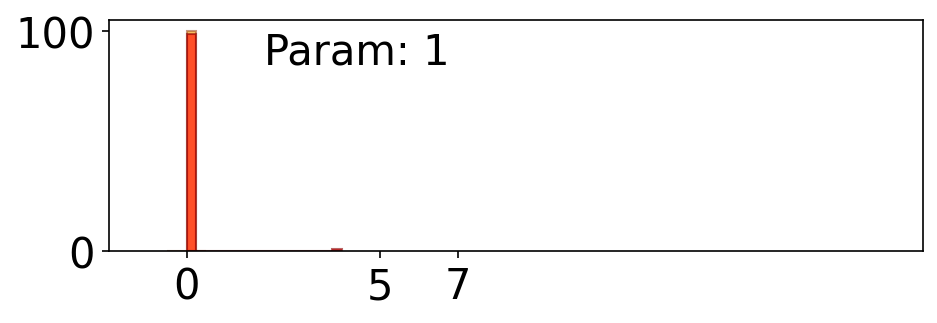}
            \includegraphics[width=1.03\linewidth]{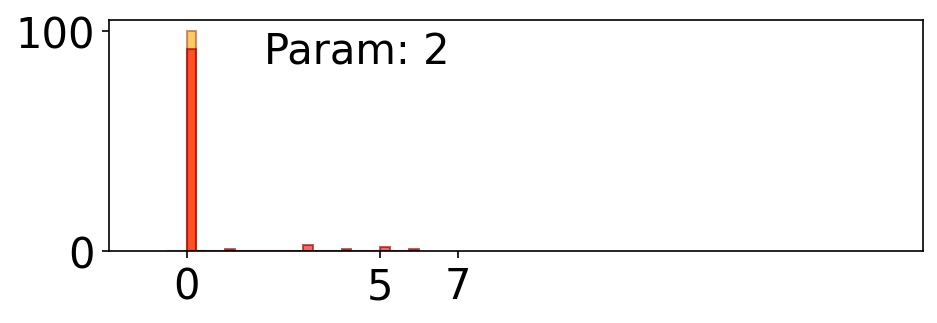}
            \includegraphics[width=1.03\linewidth]{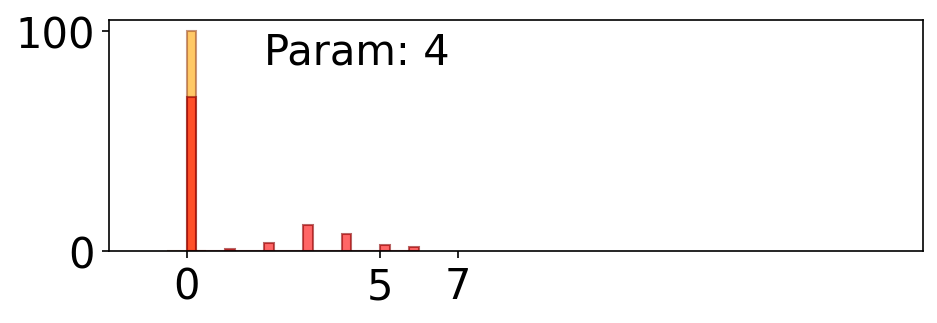}
            \includegraphics[width=1.03\linewidth]{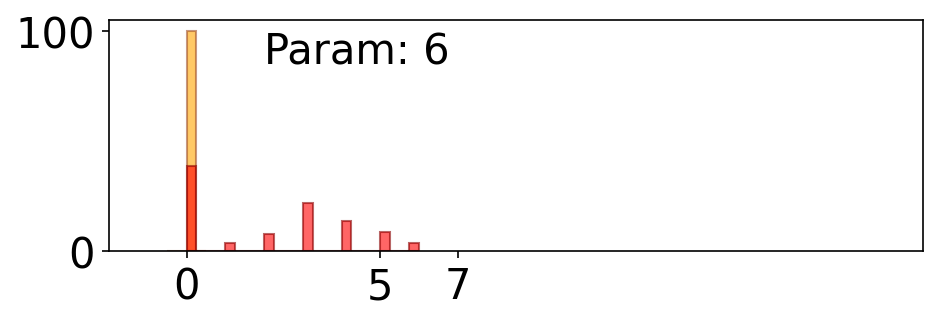}
            \includegraphics[width=1.03\linewidth]{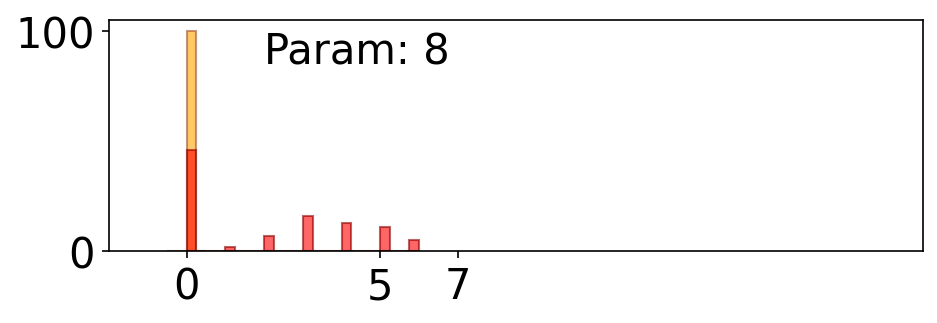}}
    
            \subfloat{\centering
            \makebox[10pt][l]{\raisebox{7pt}{\rotatebox[origin=l]{90}{\tiny JPEG}}}
            \includegraphics[width=1.03\linewidth]{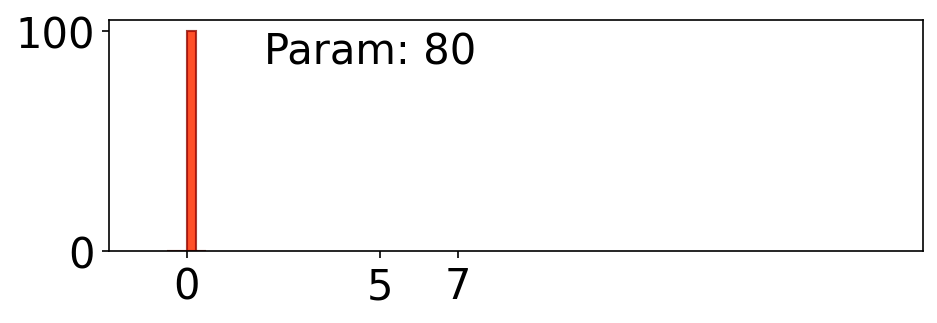}
            \includegraphics[width=1.03\linewidth]{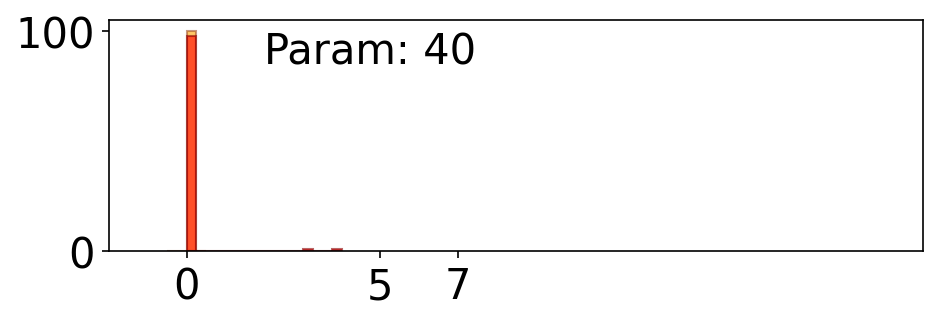}
            \includegraphics[width=1.03\linewidth]{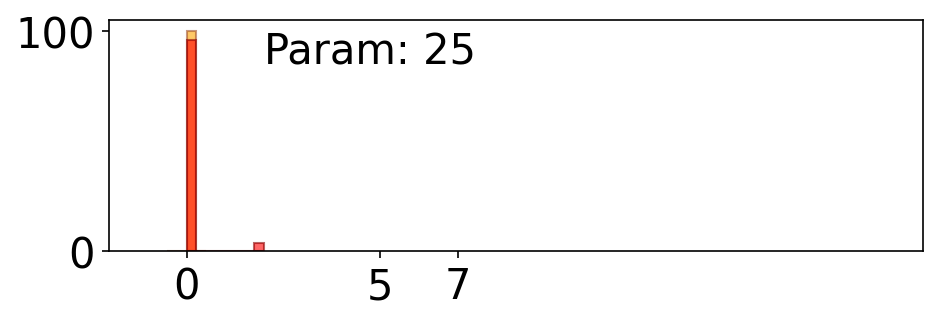}
            \includegraphics[width=1.03\linewidth]{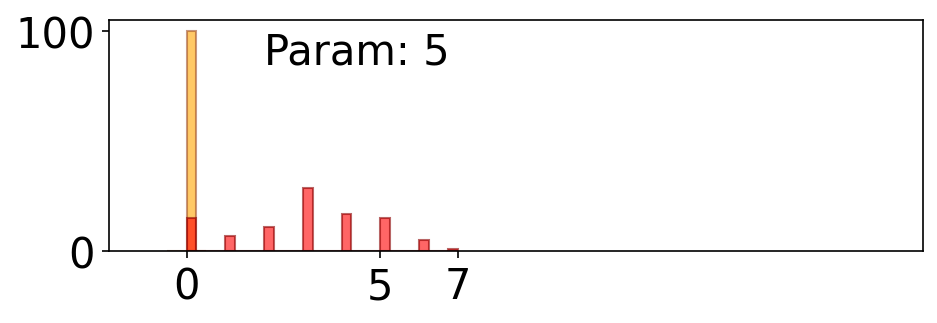}
            \includegraphics[width=1.03\linewidth]{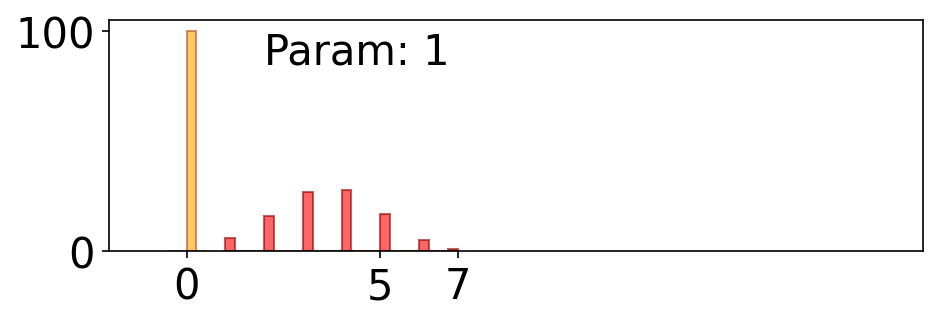}}
    
            \subfloat{\centering
            \makebox[10pt][l]{\raisebox{3pt}{\rotatebox[origin=l]{90}{\tiny Rotation}}}
            \includegraphics[width=1.03\linewidth]{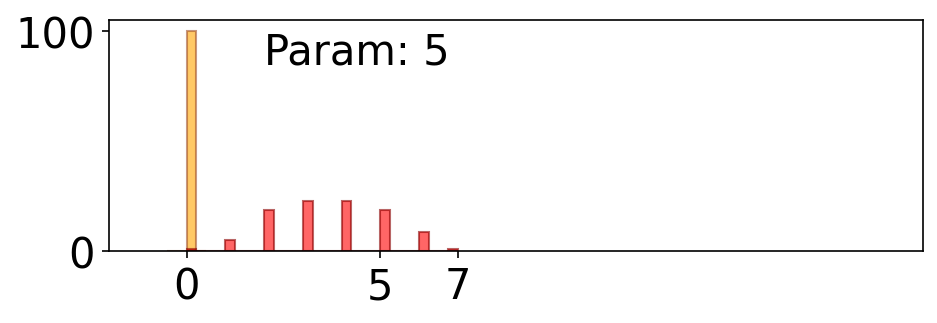}
            \includegraphics[width=1.03\linewidth]{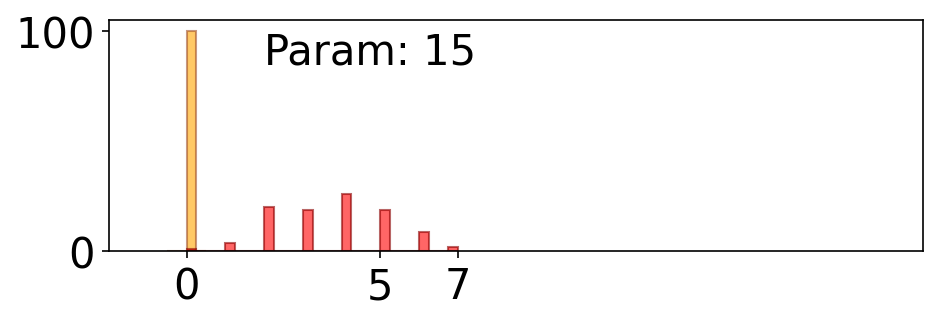}
            \includegraphics[width=1.03\linewidth]{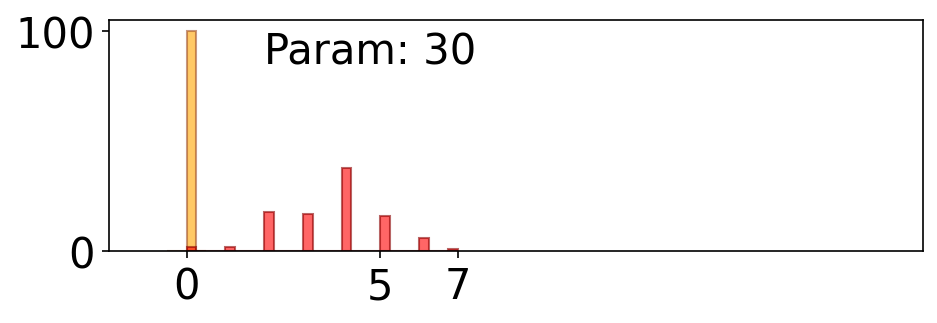}
            \includegraphics[width=1.03\linewidth]{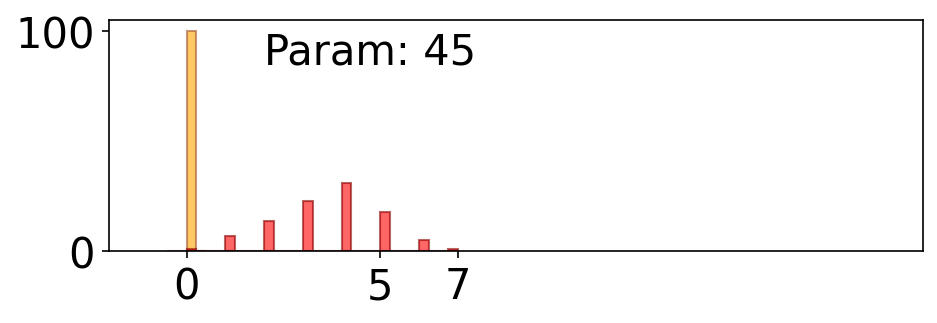}
            \includegraphics[width=1.03\linewidth]{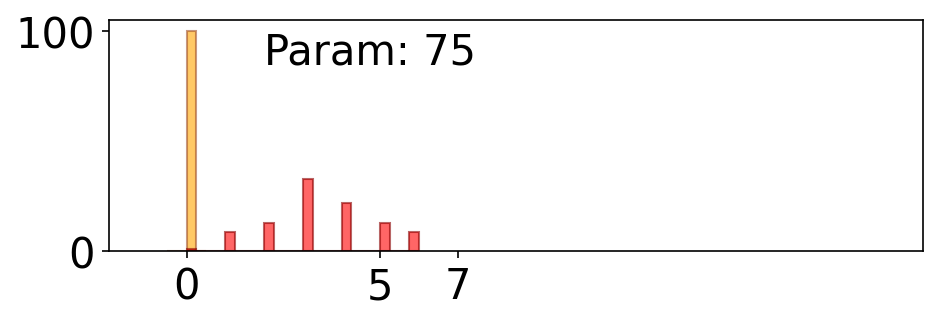}}
    
            \subfloat{\centering
            \makebox[10pt][l]{\raisebox{3pt}{\rotatebox[origin=l]{90}{\tiny Cropping}}}
            \includegraphics[width=1.03\linewidth]{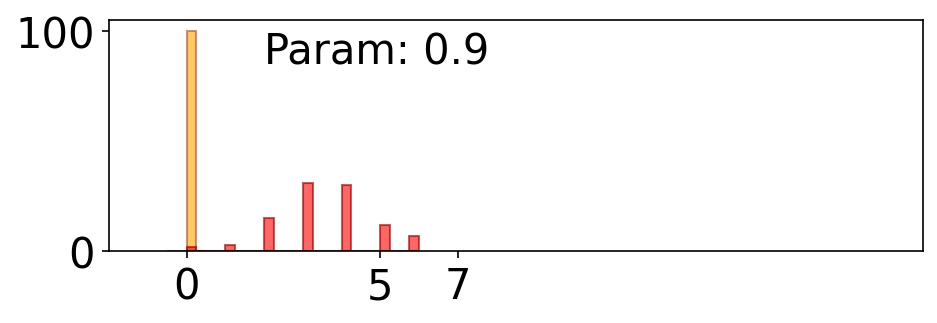}
            \includegraphics[width=1.03\linewidth]{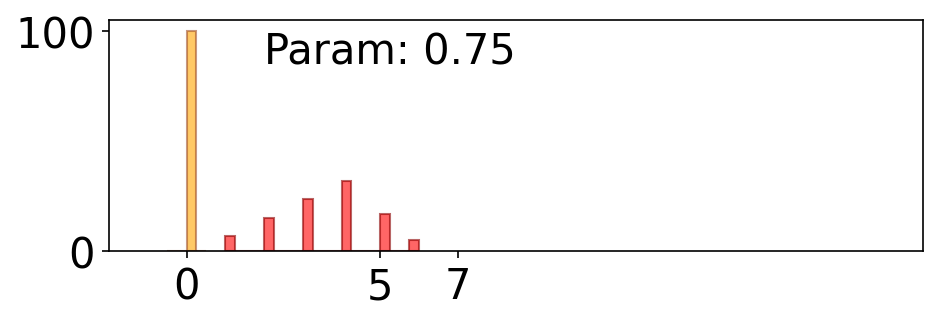}
            \includegraphics[width=1.03\linewidth]{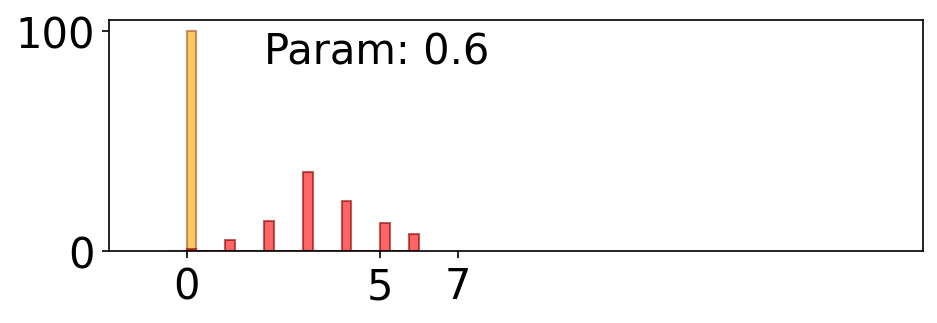}
            \includegraphics[width=1.03\linewidth]{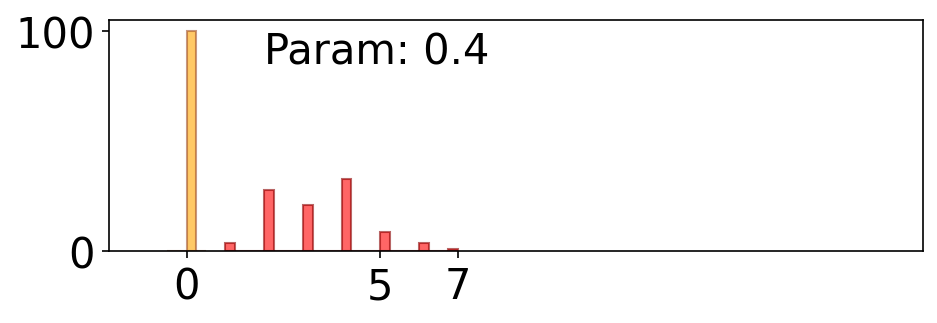}
            \includegraphics[width=1.03\linewidth]{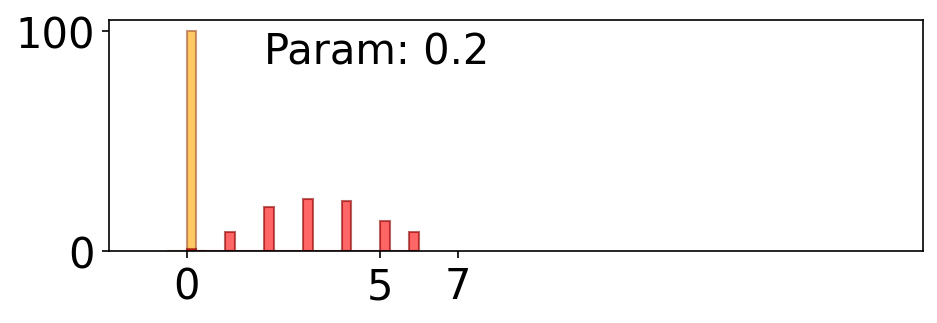}}
    
            \subfloat{\centering
            \makebox[10pt][l]{\raisebox{3pt}{\rotatebox[origin=l]{90}{\tiny Elastic}}}
            \includegraphics[width=1.03\linewidth]{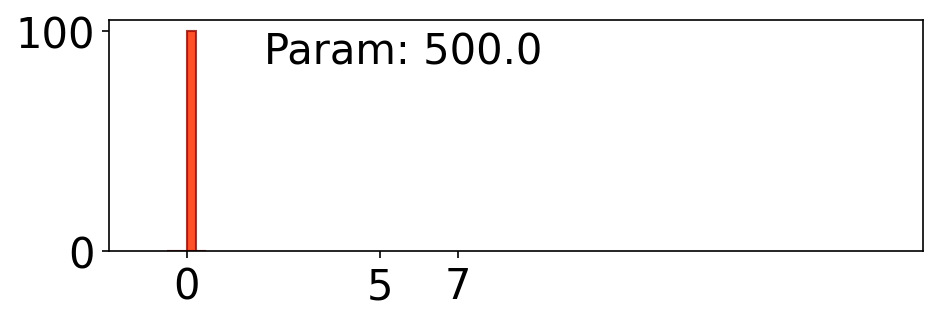}
            \includegraphics[width=1.03\linewidth]{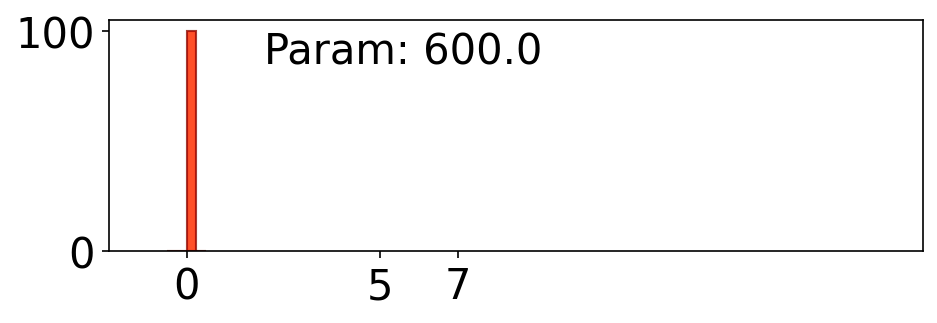}
            \includegraphics[width=1.03\linewidth]{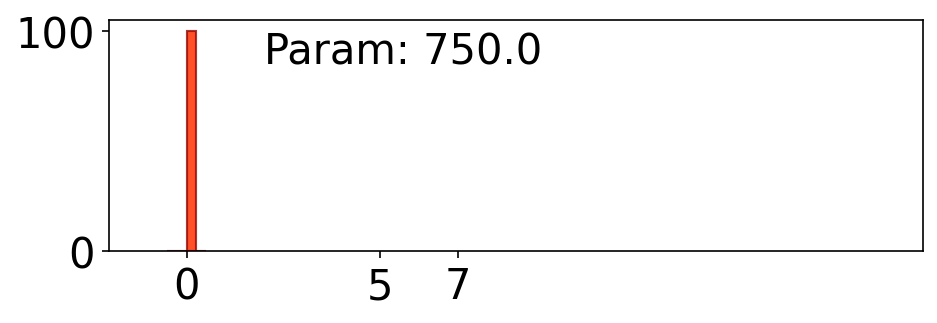}
            \includegraphics[width=1.03\linewidth]{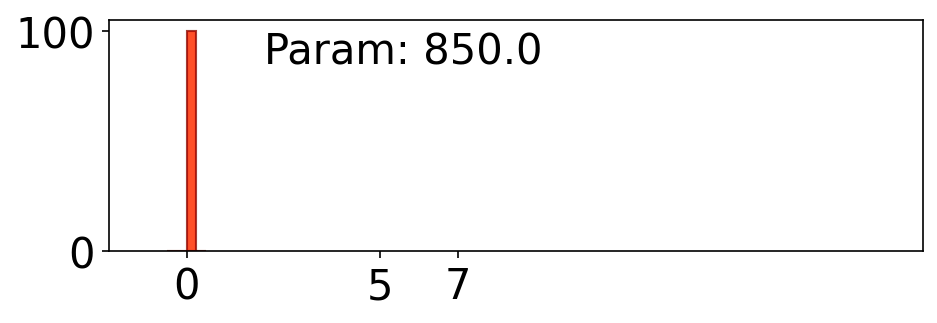}
            \includegraphics[width=1.03\linewidth]{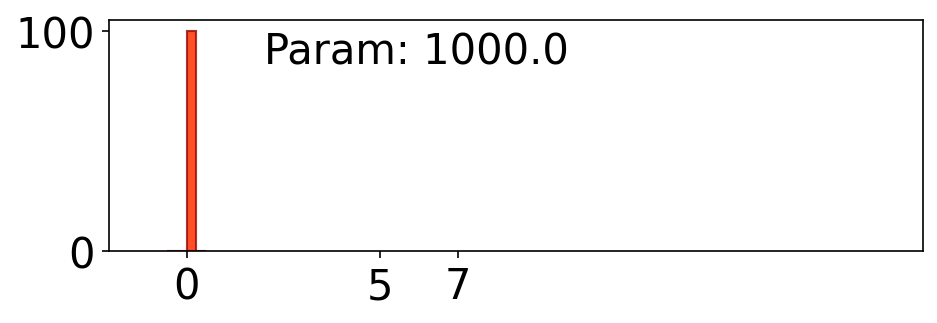}}
    
            \subfloat{\centering
            \makebox[10pt][l]{\raisebox{3pt}{\rotatebox[origin=l]{90}{\scalebox{.4}{Perspective}}}}
            \includegraphics[width=1.03\linewidth]{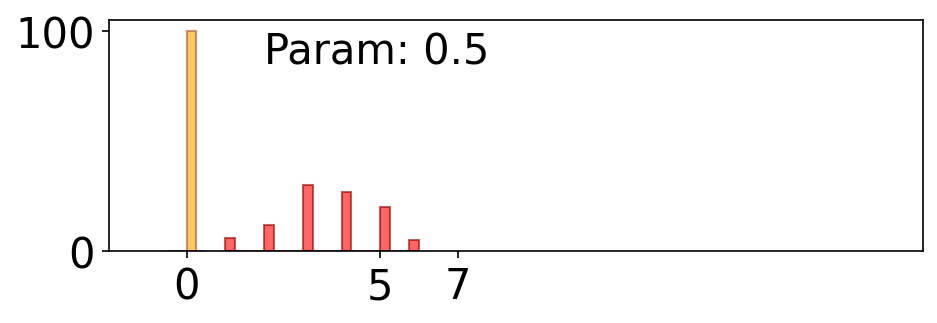}
            \includegraphics[width=1.03\linewidth]{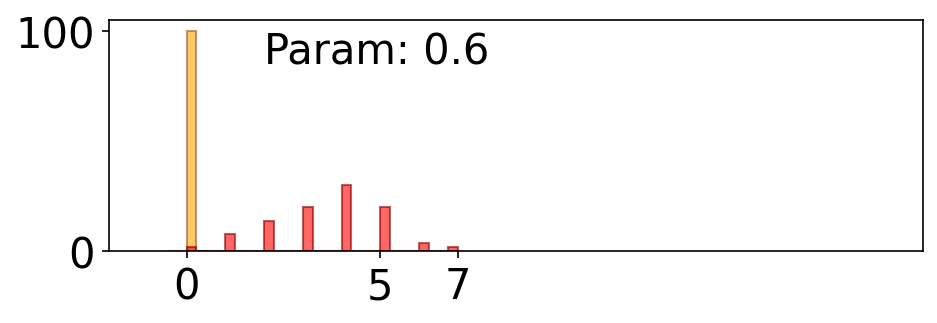}
            \includegraphics[width=1.03\linewidth]{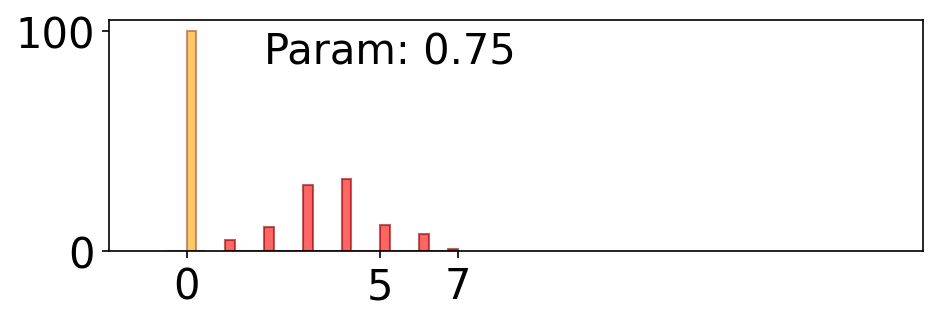}
            \includegraphics[width=1.03\linewidth]{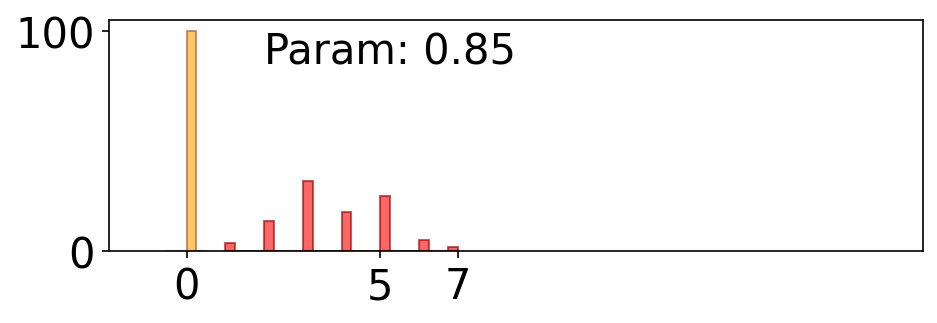}
            \includegraphics[width=1.03\linewidth]{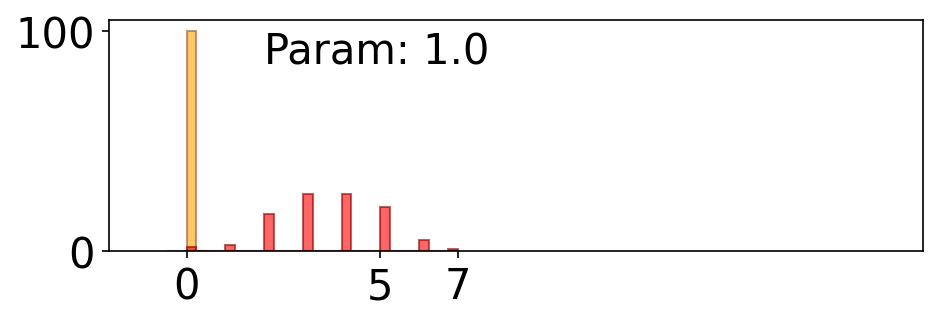}}
    
            \subfloat{\centering
            \makebox[10pt][l]{\raisebox{3pt}{\rotatebox[origin=l]{90}{\tiny Shear}}}
            \includegraphics[width=1.03\linewidth]{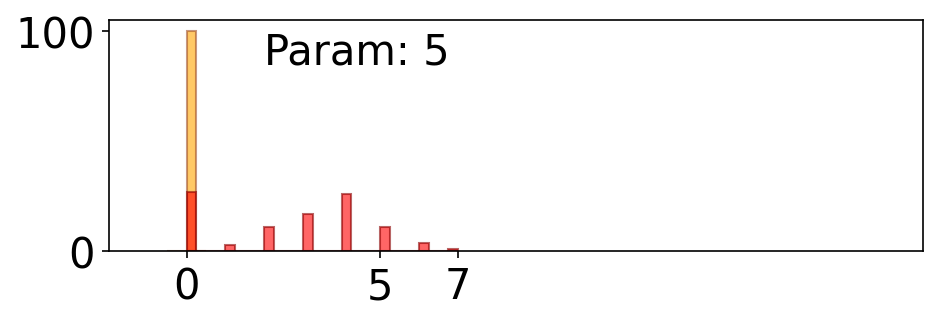}
            \includegraphics[width=1.03\linewidth]{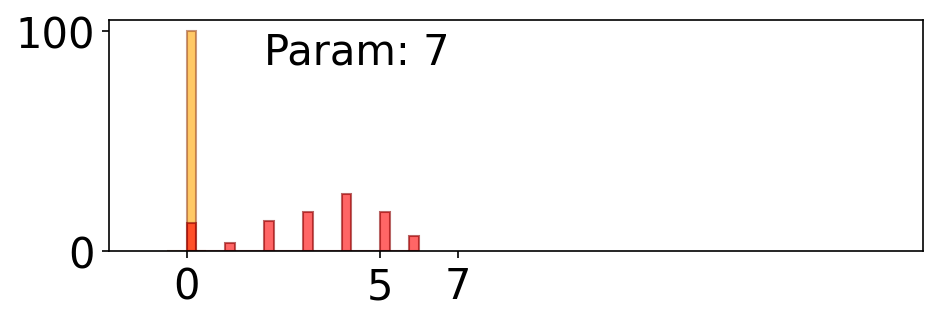}
            \includegraphics[width=1.03\linewidth]{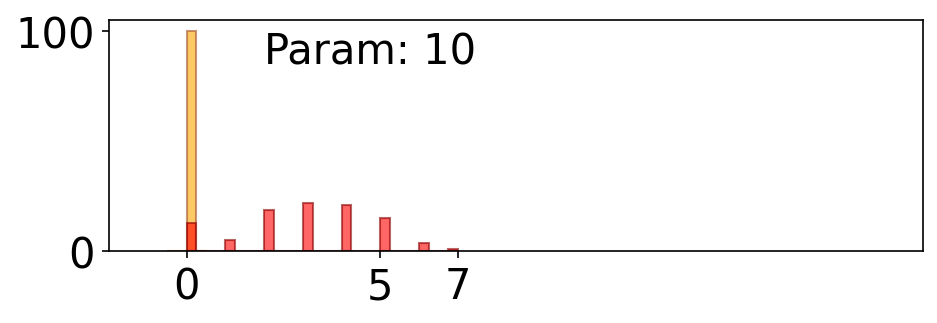}
            \includegraphics[width=1.03\linewidth]{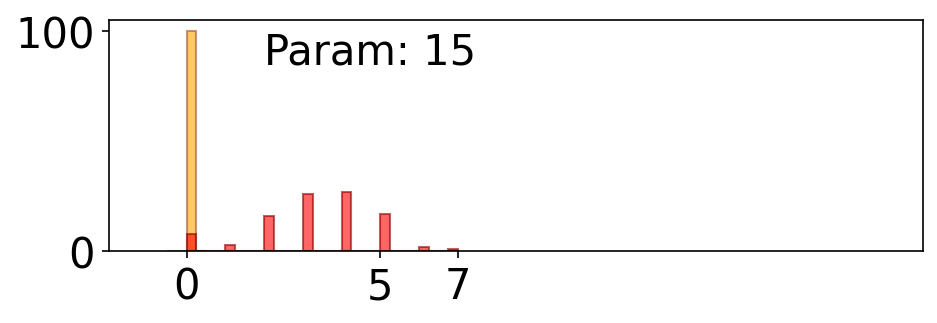}
            \includegraphics[width=1.03\linewidth]{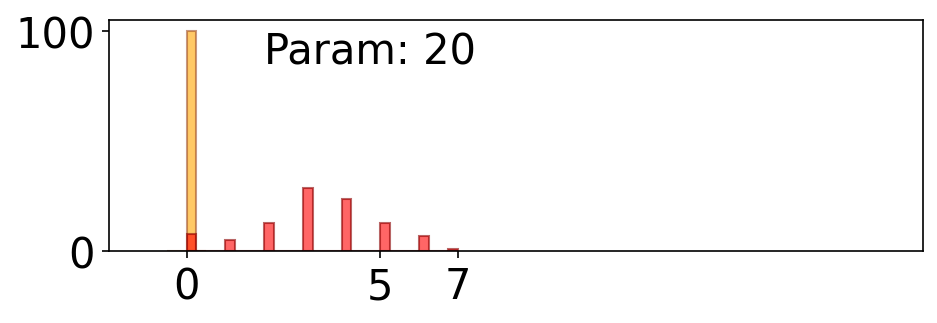}}

            \subfloat{\centering
            \makebox[10pt][l]{\raisebox{1pt}{\rotatebox[origin=l]{90}{\scalebox{.4}{Regen-VAE}}}}
            \includegraphics[width=1.03\linewidth]{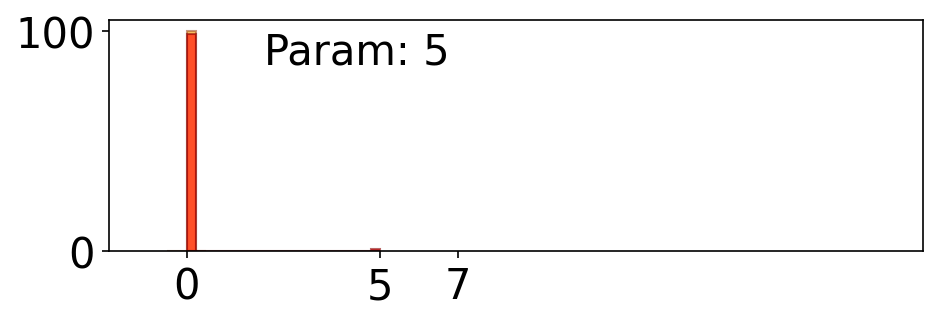}
            \includegraphics[width=1.03\linewidth]{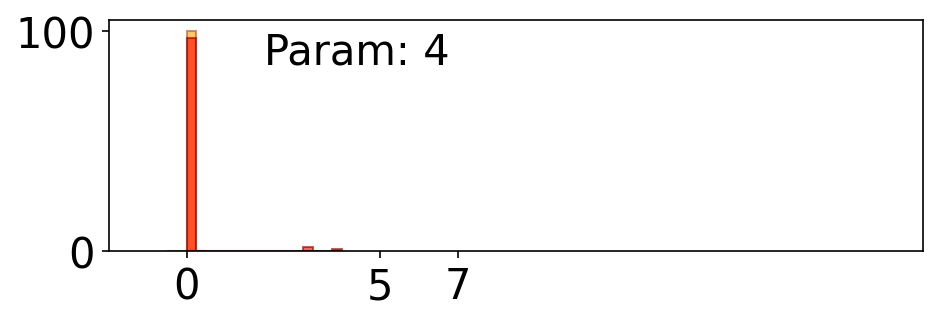}
            \includegraphics[width=1.03\linewidth]{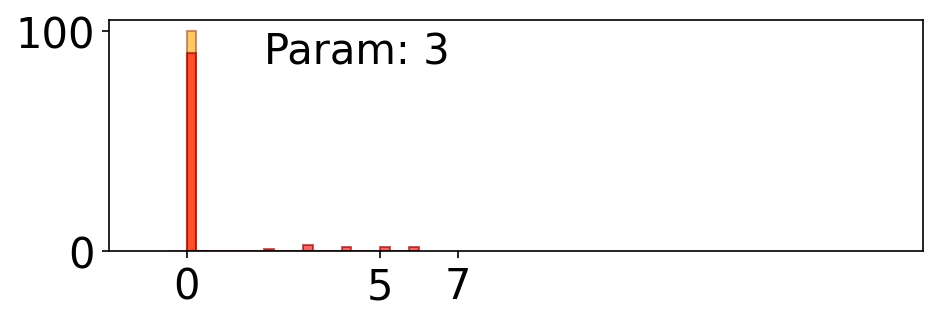}
            \includegraphics[width=1.03\linewidth]{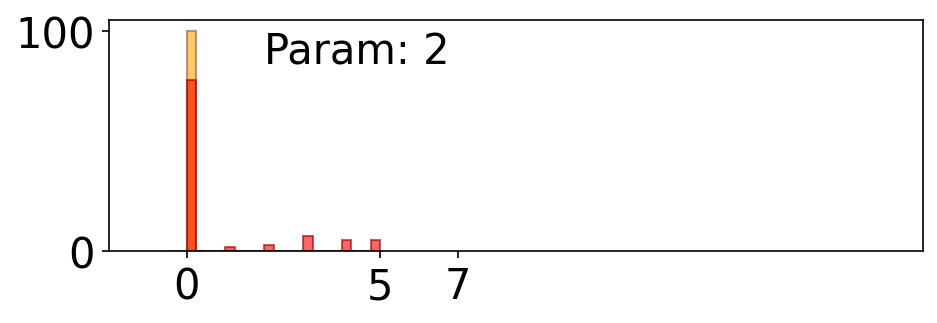}
            \includegraphics[width=1.03\linewidth]{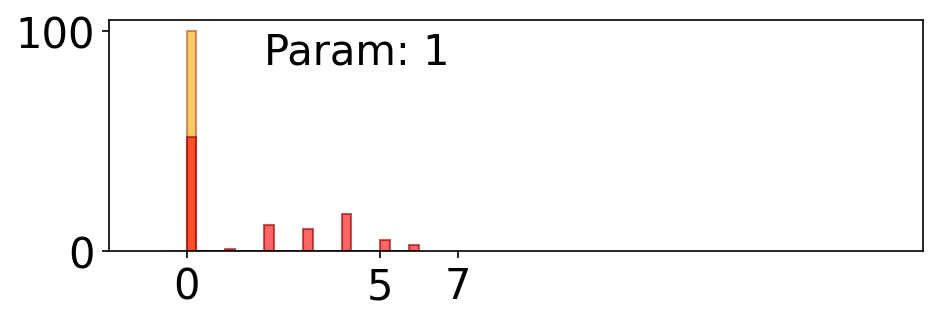}}
    
            \subfloat{\centering
            \makebox[10pt][l]{\raisebox{0pt}{\rotatebox[origin=l]{90}{\scalebox{.4}{Regen-Diff}}}}
            \includegraphics[width=1.03\linewidth]{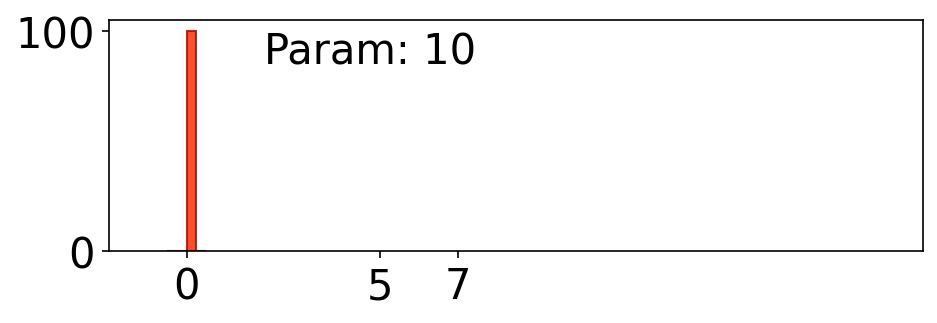}
            \includegraphics[width=1.03\linewidth]{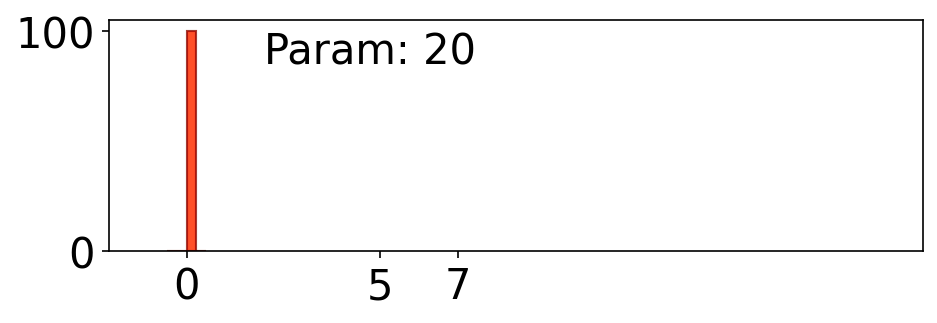}
            \includegraphics[width=1.03\linewidth]{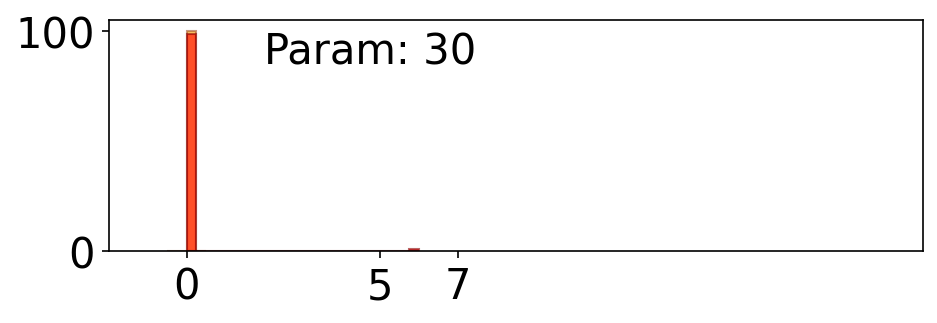}
            \includegraphics[width=1.03\linewidth]{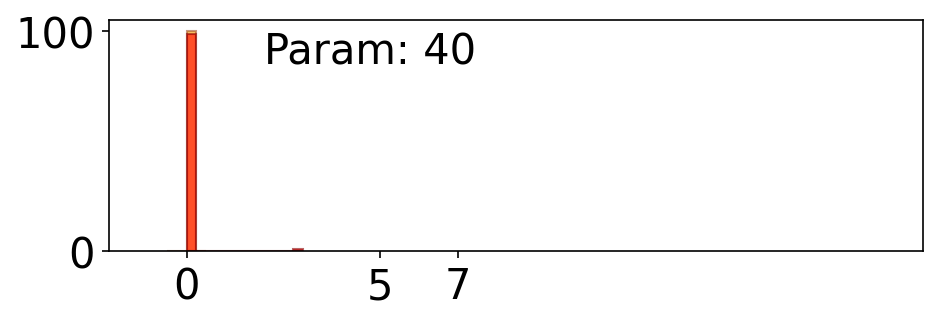}
            \includegraphics[width=1.03\linewidth]{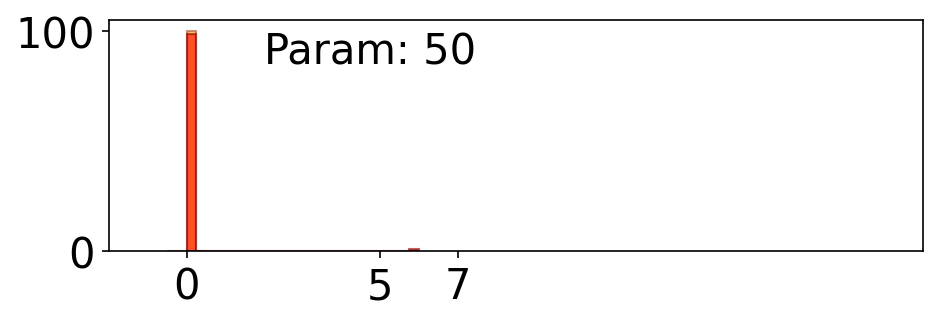}}
        \end{subfigure}
        
    \caption{The effects of different perturbations on $\|\mathbf{w}-\hat{\mathbf{w}}^{\prime}\|_H$. From left to right, the perturbations are set with different parameters (Param). 
    Watermark: PRCW decoding component (1e-2 FPR). In this case, the number of bits is $\left\lceil\log_2(1/10^{-2})\right\rceil=7$.}
    \label{fig: PRCW decoding 1e-2fpr attack shifting}
\end{figure}

\begin{figure}[t]
        \begin{subfigure}[c]{0.185\textwidth}
            \centering
            \subfloat{\centering
            \makebox[10pt][l]{\raisebox{5pt}{\rotatebox[origin=l]{90}{\tiny Blurring}}}
            \includegraphics[width=1.03\linewidth]{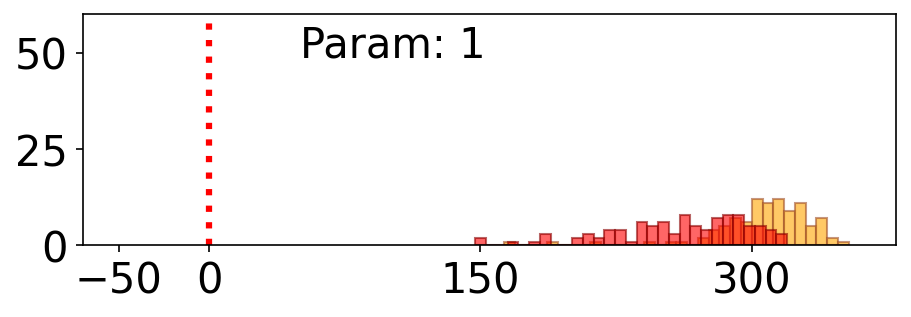}
            \includegraphics[width=1.03\linewidth]{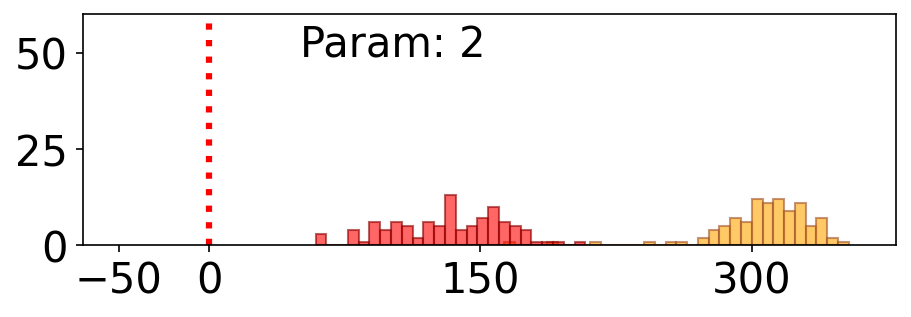}
            \includegraphics[width=1.03\linewidth]{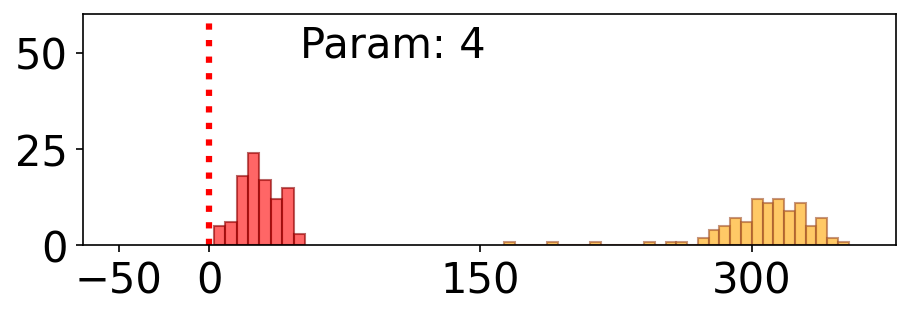}
            \includegraphics[width=1.03\linewidth]{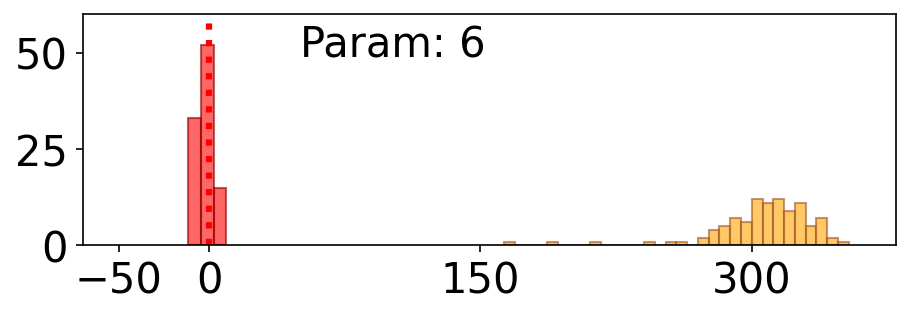}
            \includegraphics[width=1.03\linewidth]{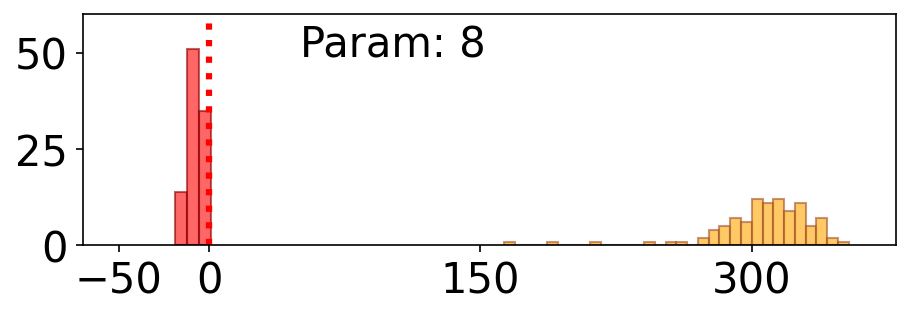}}
    
            \subfloat{\centering
            \makebox[10pt][l]{\raisebox{7pt}{\rotatebox[origin=l]{90}{\tiny Jitter}}}
            \includegraphics[width=1.03\linewidth]{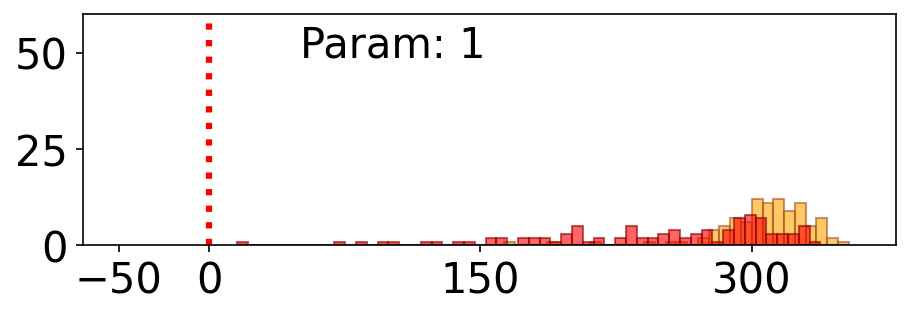}
            \includegraphics[width=1.03\linewidth]{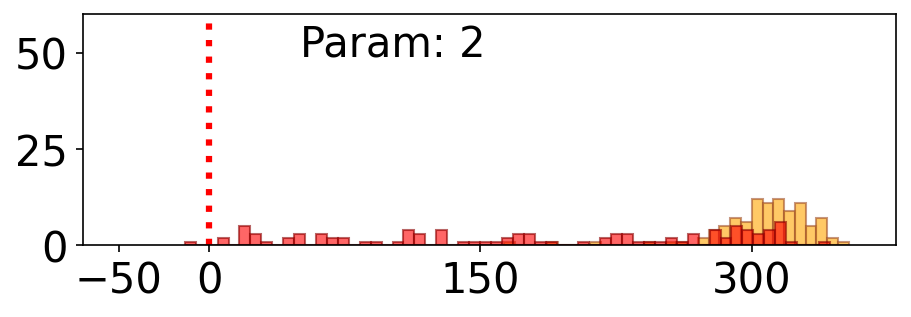}
            \includegraphics[width=1.03\linewidth]{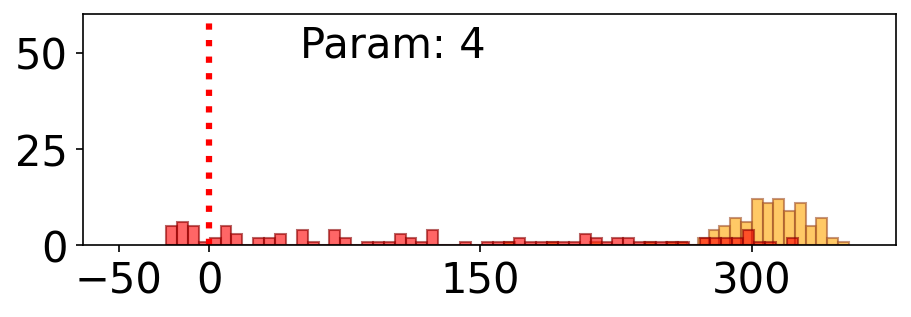}
            \includegraphics[width=1.03\linewidth]{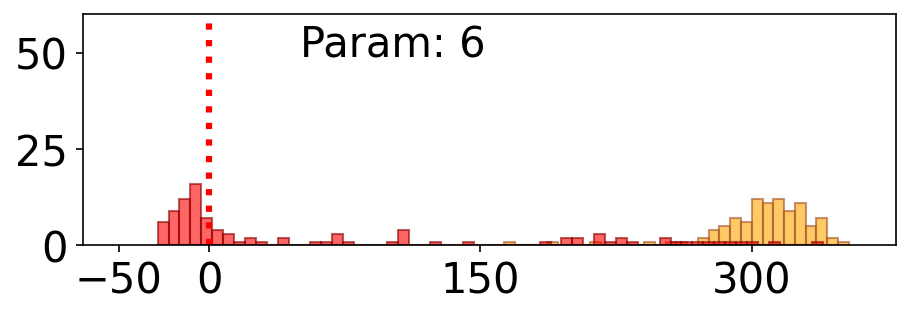}
            \includegraphics[width=1.03\linewidth]{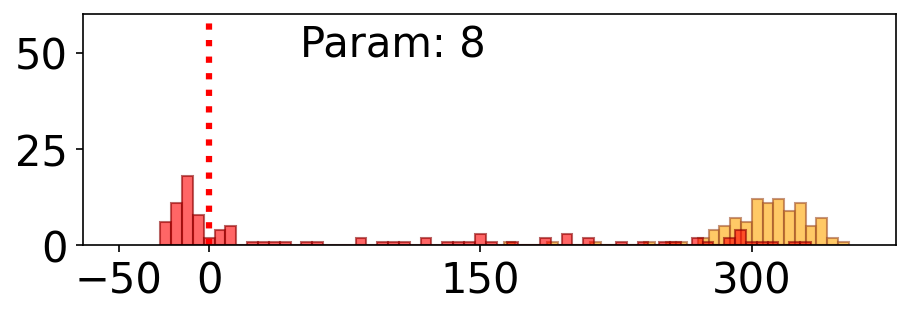}}
    
            \subfloat{\centering
            \makebox[10pt][l]{\raisebox{7pt}{\rotatebox[origin=l]{90}{\tiny JPEG}}}
            \includegraphics[width=1.03\linewidth]{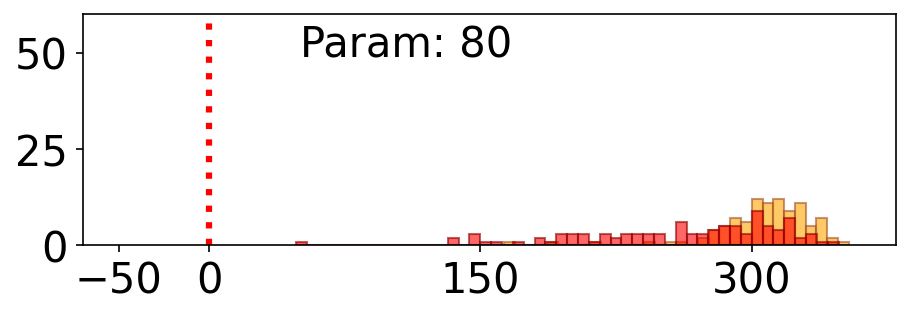}
            \includegraphics[width=1.03\linewidth]{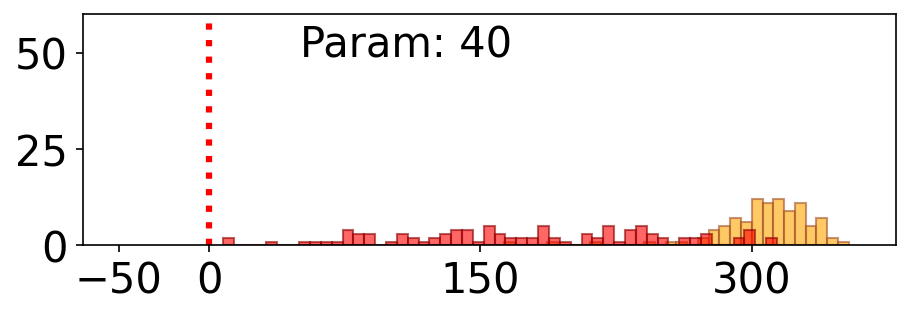}
            \includegraphics[width=1.03\linewidth]{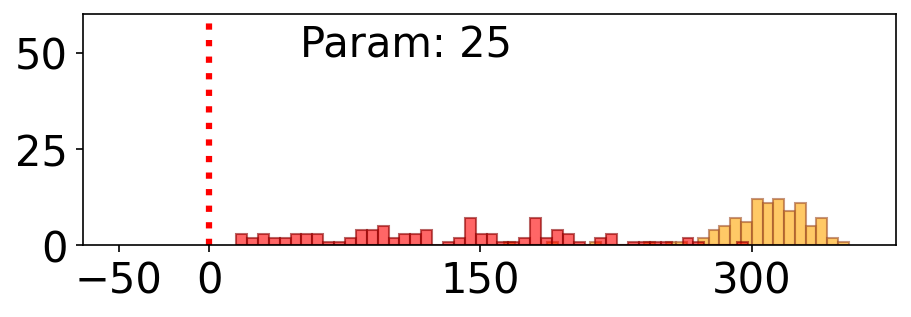}
            \includegraphics[width=1.03\linewidth]{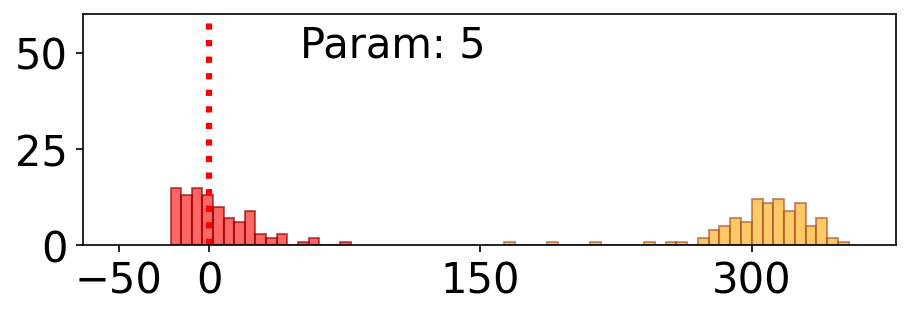}
            \includegraphics[width=1.03\linewidth]{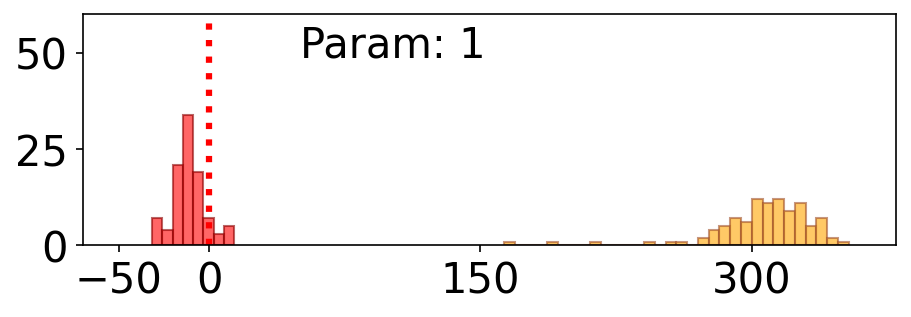}}
    
            \subfloat{\centering
            \makebox[10pt][l]{\raisebox{3pt}{\rotatebox[origin=l]{90}{\tiny Rotation}}}
            \includegraphics[width=1.03\linewidth]{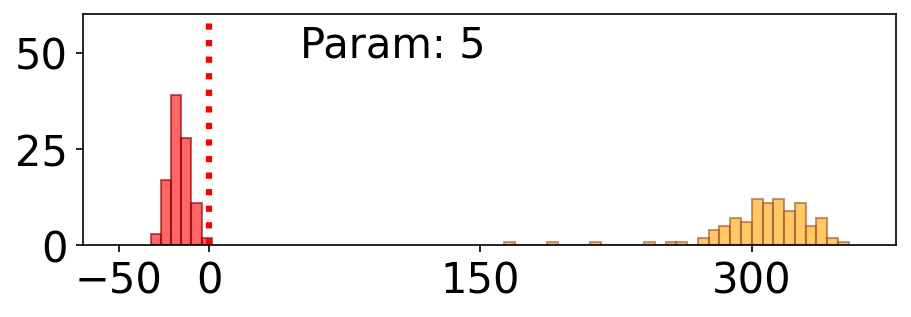}
            \includegraphics[width=1.03\linewidth]{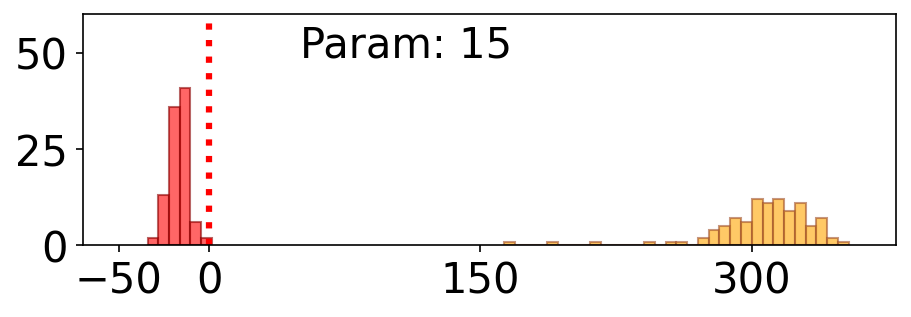}
            \includegraphics[width=1.03\linewidth]{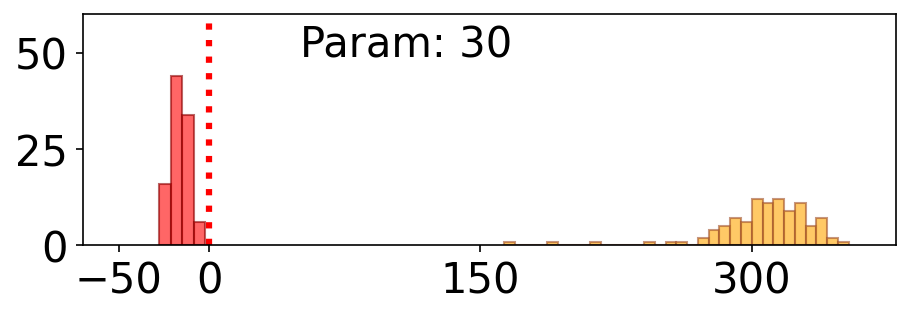}
            \includegraphics[width=1.03\linewidth]{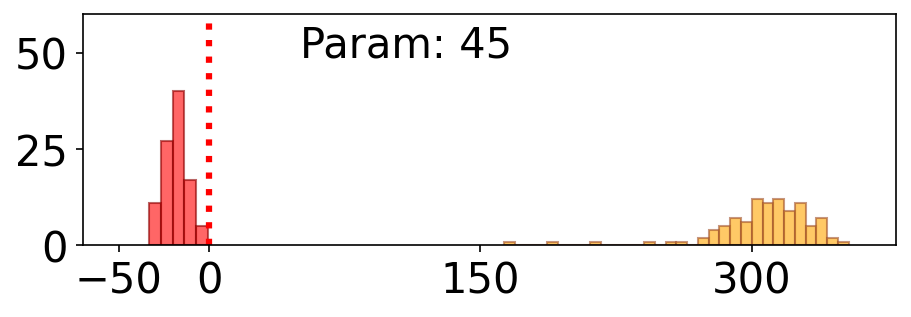}
            \includegraphics[width=1.03\linewidth]{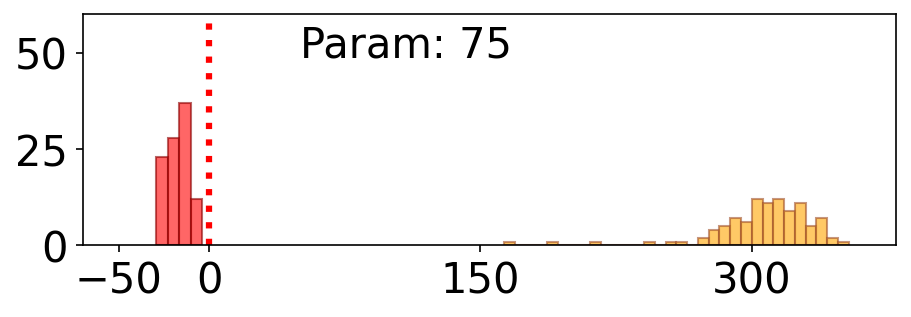}}
    
            \subfloat{\centering
            \makebox[10pt][l]{\raisebox{3pt}{\rotatebox[origin=l]{90}{\tiny Cropping}}}
            \includegraphics[width=1.03\linewidth]{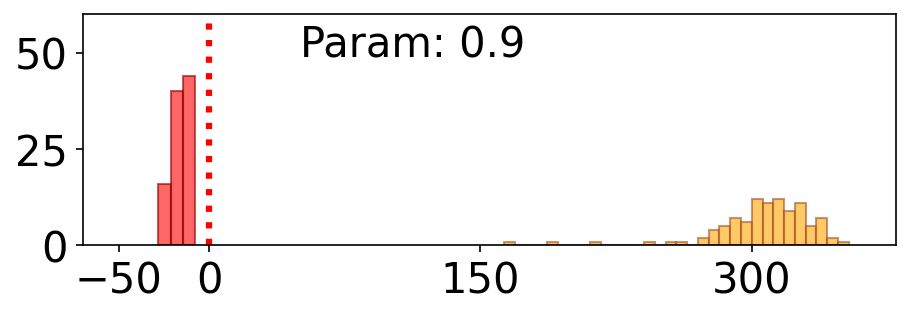}
            \includegraphics[width=1.03\linewidth]{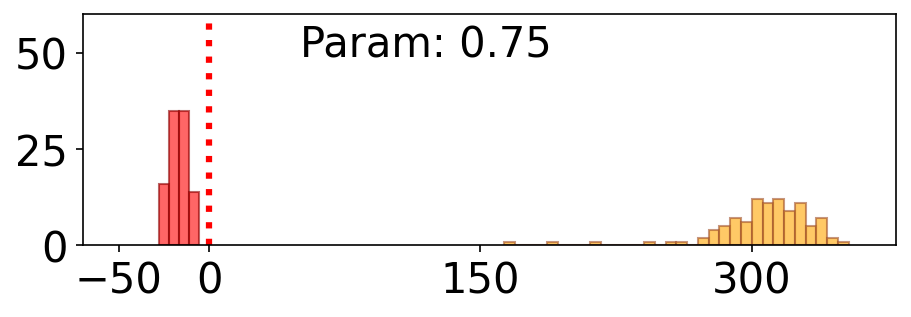}
            \includegraphics[width=1.03\linewidth]{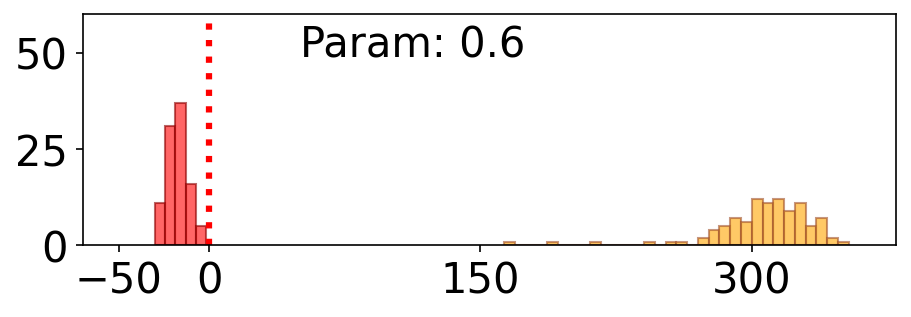}
            \includegraphics[width=1.03\linewidth]{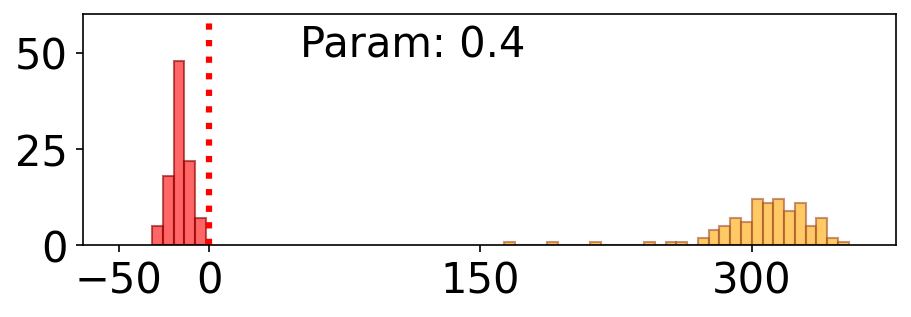}
            \includegraphics[width=1.03\linewidth]{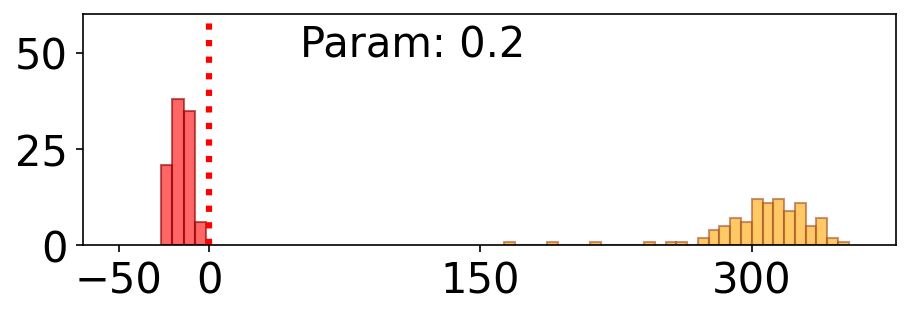}}
    
            \subfloat{\centering
            \makebox[10pt][l]{\raisebox{3pt}{\rotatebox[origin=l]{90}{\tiny Elastic}}}
            \includegraphics[width=1.03\linewidth]{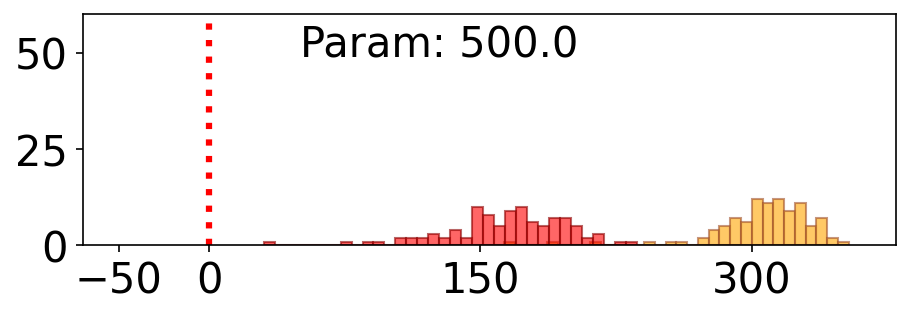}
            \includegraphics[width=1.03\linewidth]{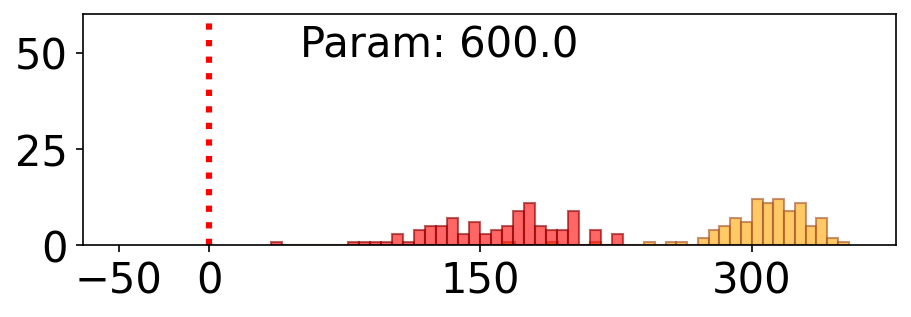}
            \includegraphics[width=1.03\linewidth]{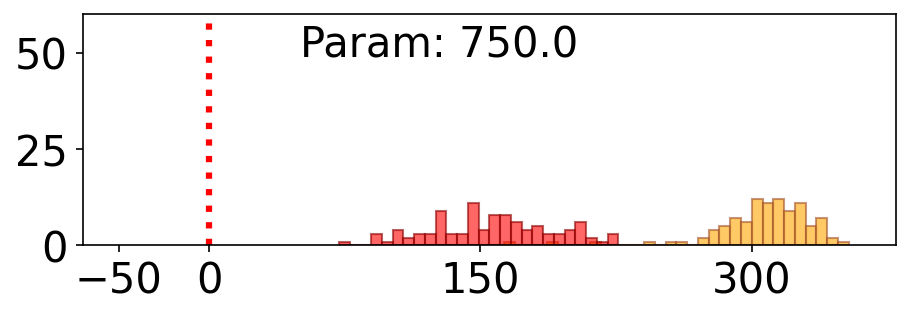}
            \includegraphics[width=1.03\linewidth]{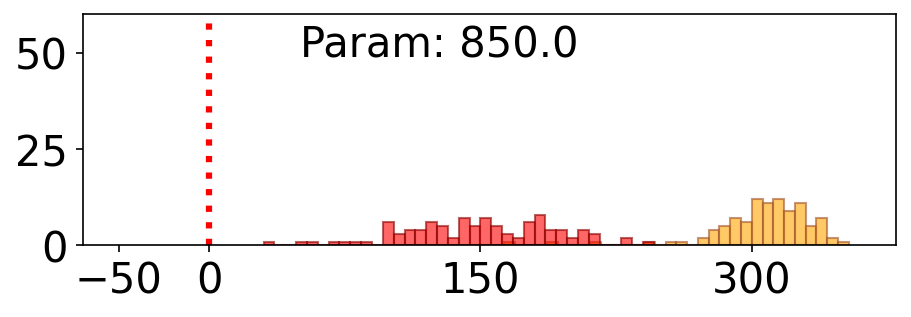}
            \includegraphics[width=1.03\linewidth]{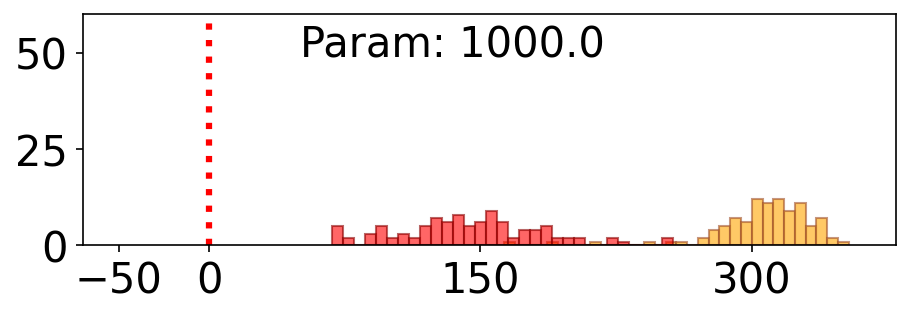}}
    
            \subfloat{\centering
            \makebox[10pt][l]{\raisebox{3pt}{\rotatebox[origin=l]{90}{\scalebox{.4}{Perspective}}}}
            \includegraphics[width=1.03\linewidth]{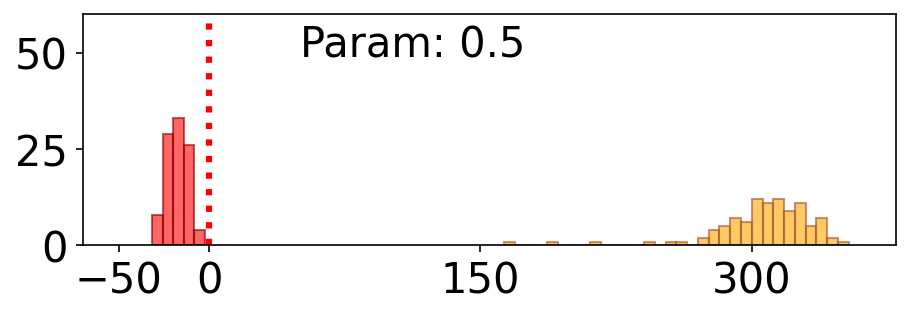}
            \includegraphics[width=1.03\linewidth]{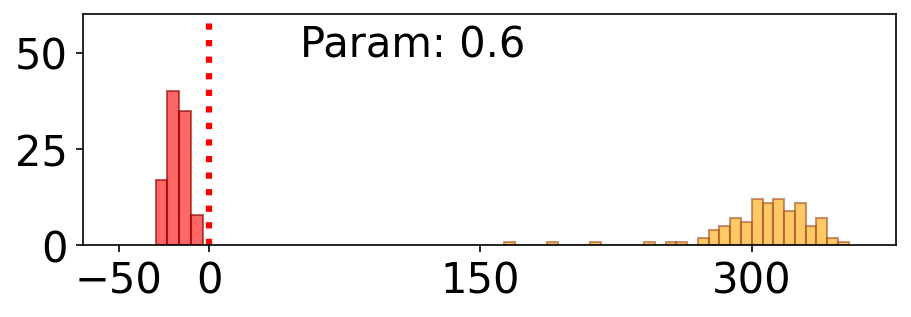}
            \includegraphics[width=1.03\linewidth]{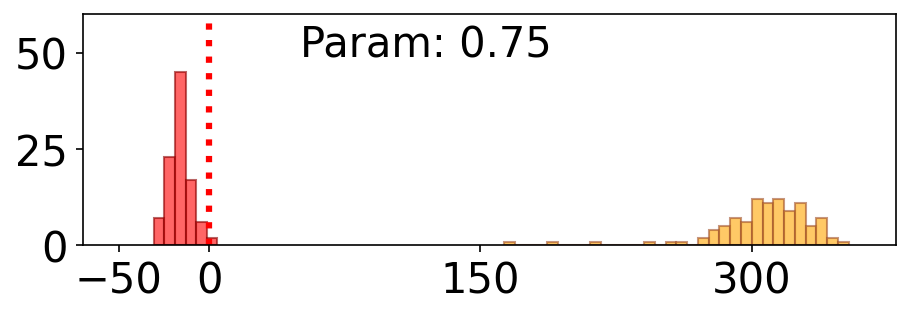}
            \includegraphics[width=1.03\linewidth]{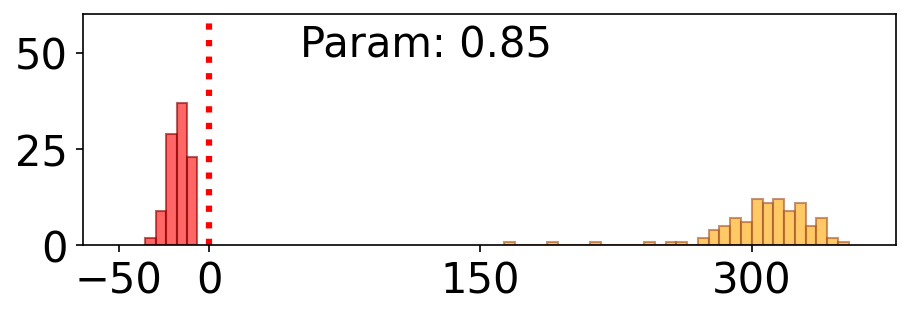}
            \includegraphics[width=1.03\linewidth]{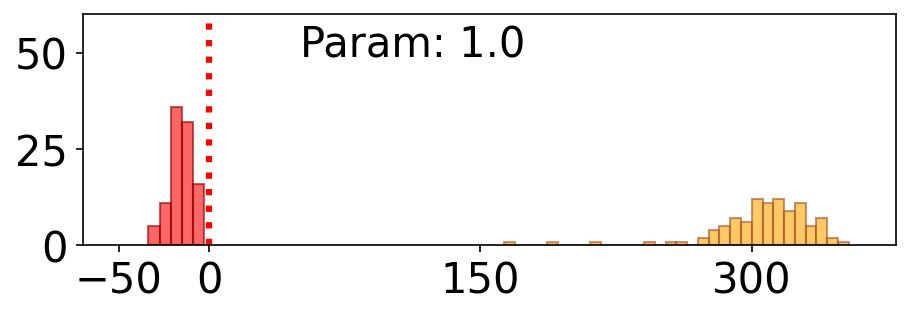}}
    
            \subfloat{\centering
            \makebox[10pt][l]{\raisebox{3pt}{\rotatebox[origin=l]{90}{\tiny Shear}}}
            \includegraphics[width=1.03\linewidth]{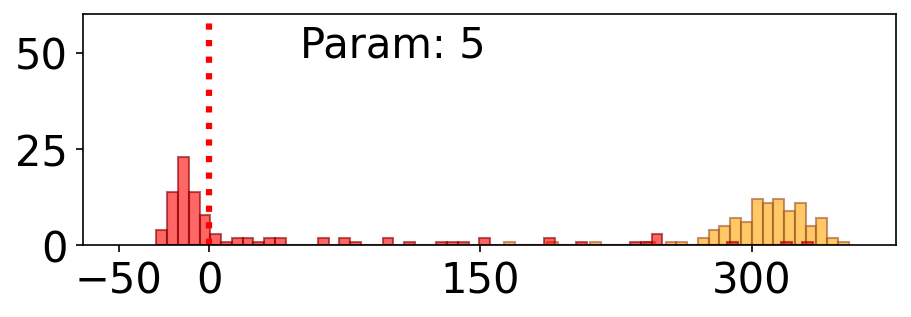}
            \includegraphics[width=1.03\linewidth]{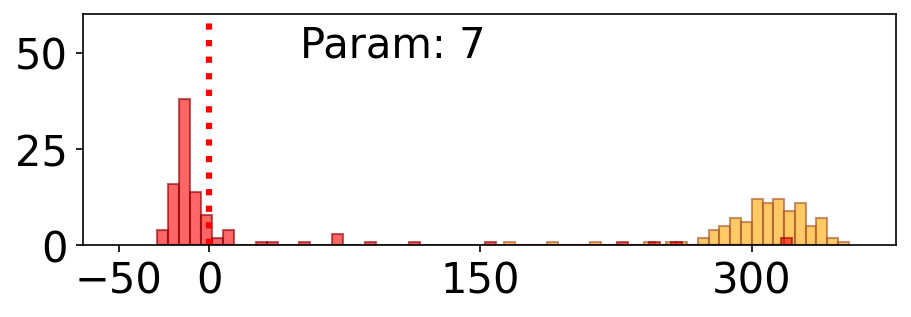}
            \includegraphics[width=1.03\linewidth]{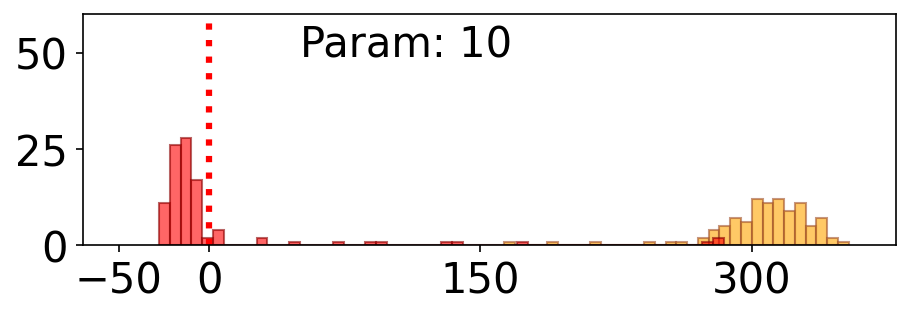}
            \includegraphics[width=1.03\linewidth]{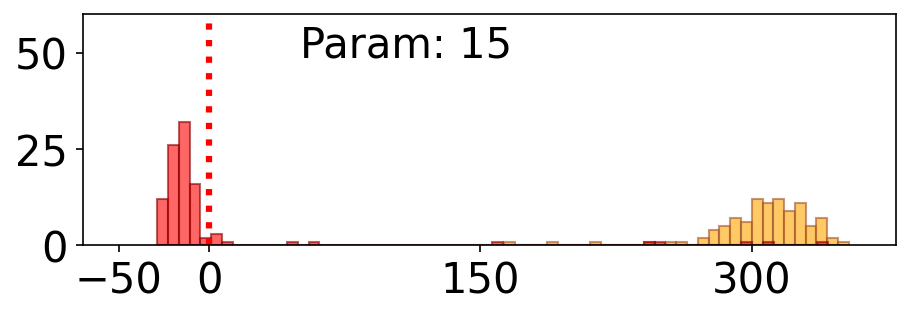}
            \includegraphics[width=1.03\linewidth]{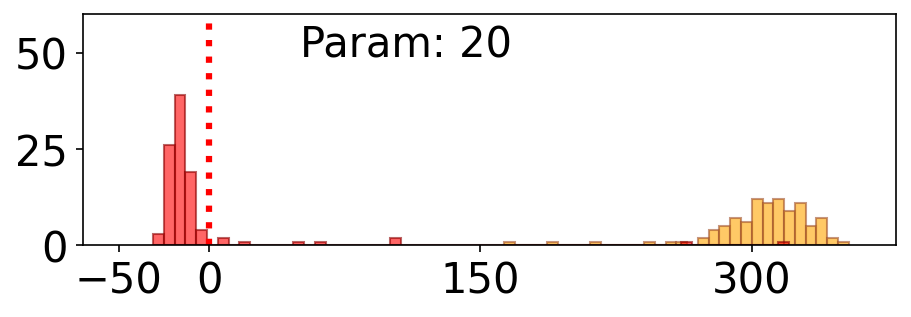}}

            \subfloat{\centering
            \makebox[10pt][l]{\raisebox{1pt}{\rotatebox[origin=l]{90}{\scalebox{.4}{Regen-VAE}}}}
            \includegraphics[width=1.03\linewidth]{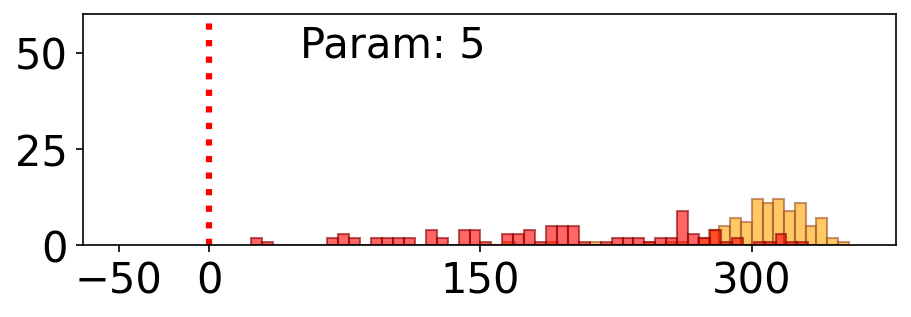}
            \includegraphics[width=1.03\linewidth]{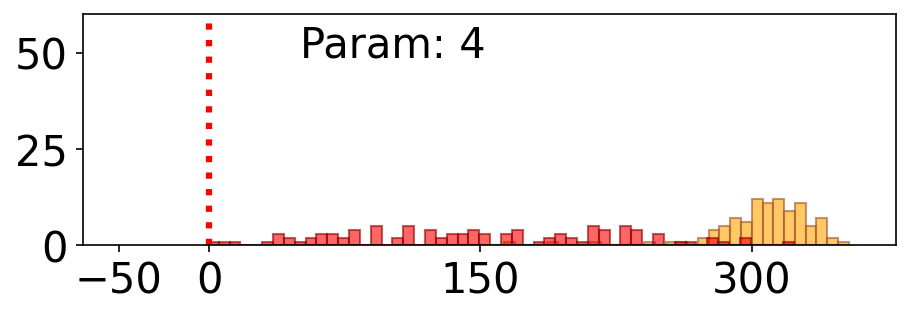}
            \includegraphics[width=1.03\linewidth]{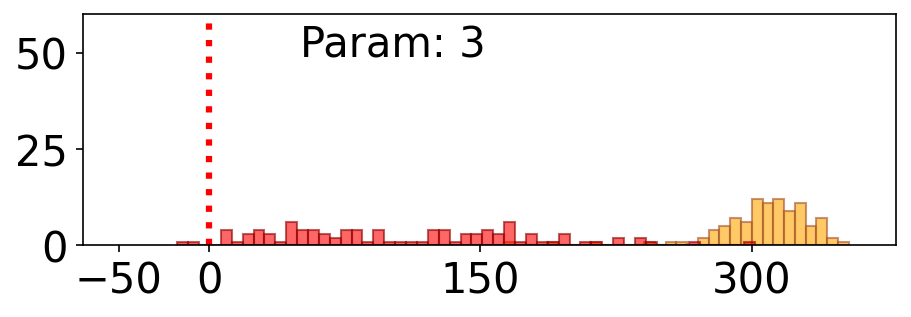}
            \includegraphics[width=1.03\linewidth]{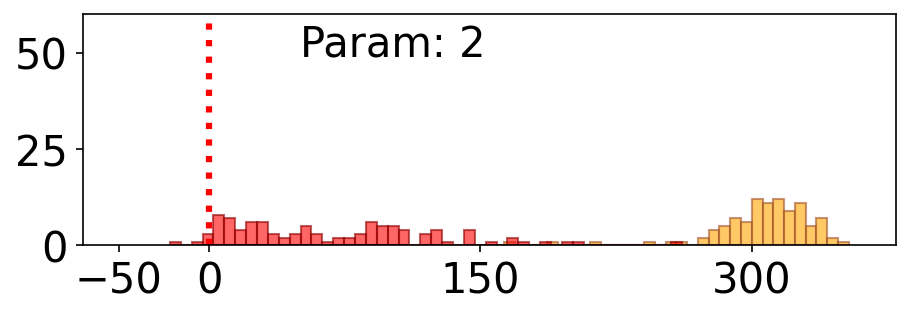}
            \includegraphics[width=1.03\linewidth]{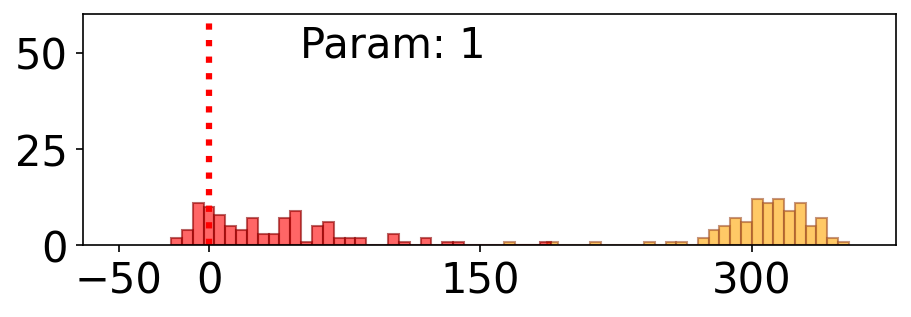}}
    
            \subfloat{\centering
            \makebox[10pt][l]{\raisebox{0pt}{\rotatebox[origin=l]{90}{\scalebox{.4}{Regen-Diff}}}}
            \includegraphics[width=1.03\linewidth]{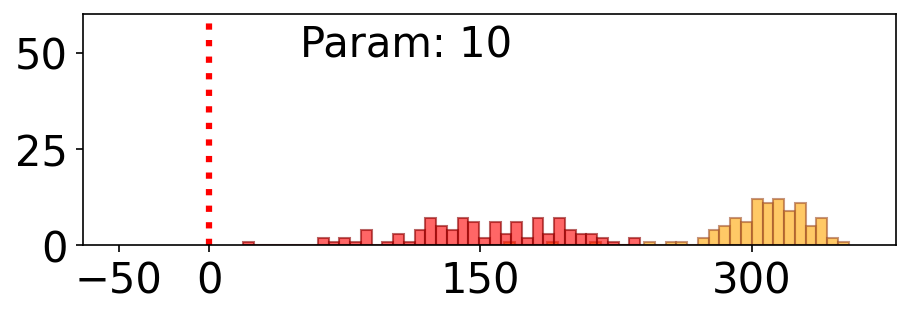}
            \includegraphics[width=1.03\linewidth]{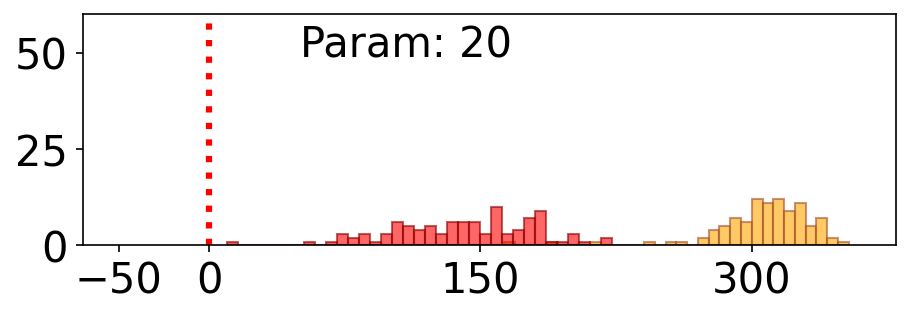}
            \includegraphics[width=1.03\linewidth]{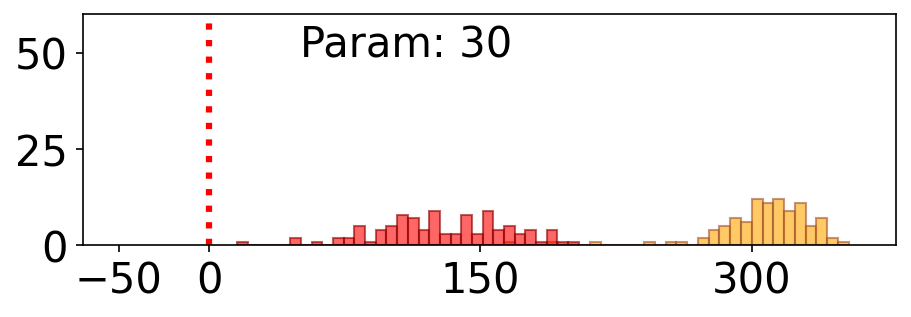}
            \includegraphics[width=1.03\linewidth]{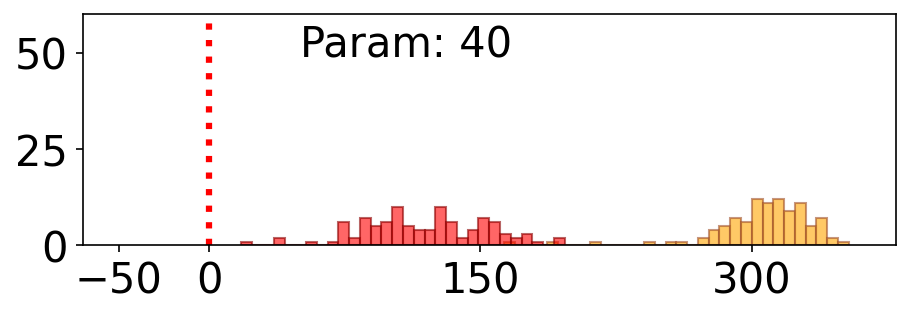}
            \includegraphics[width=1.03\linewidth]{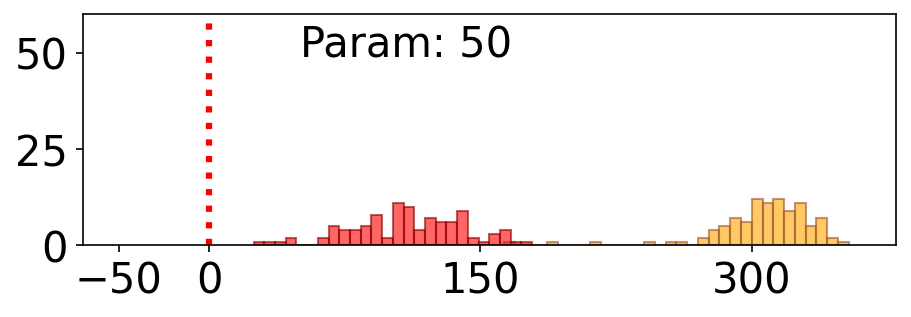}}
        \end{subfigure}
        
    \caption{The effects of different perturbations by Eq.~(\ref{eq: PRCW Alog 3 L 5}). From left to right, the perturbations are set with different parameters (Param). 
    Watermark: PRCW detection component (1e-2 FPR).}
    \label{fig: PRCW detection 1e-2fpr attack shifting}
\end{figure}

\begin{figure}[t]
        \begin{subfigure}[c]{0.185\textwidth}
            \centering
            \subfloat{\centering
            \makebox[10pt][l]{\raisebox{5pt}{\rotatebox[origin=l]{90}{\tiny Blurring}}}
            \includegraphics[width=1.03\linewidth]{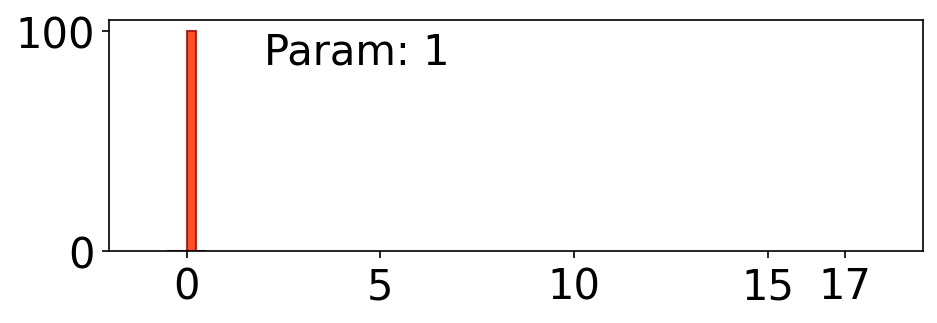}
            \includegraphics[width=1.03\linewidth]{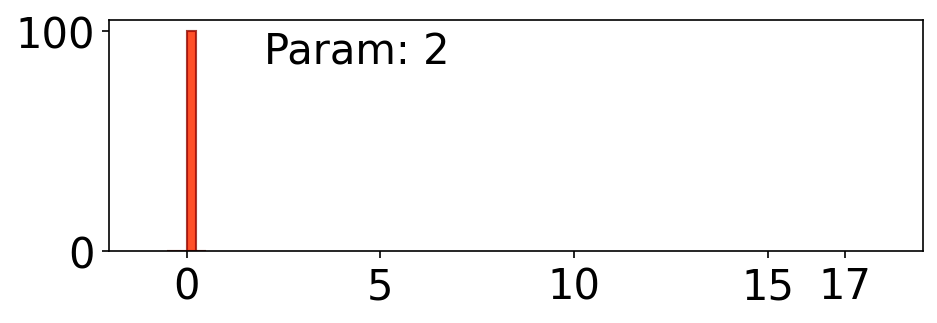}
            \includegraphics[width=1.03\linewidth]{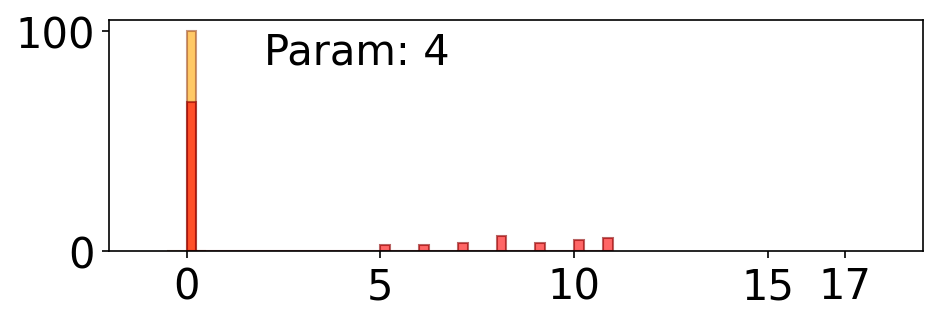}
            \includegraphics[width=1.03\linewidth]{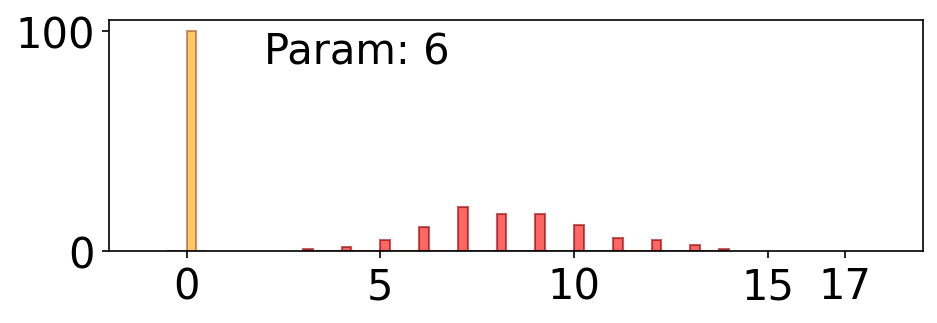}
            \includegraphics[width=1.03\linewidth]{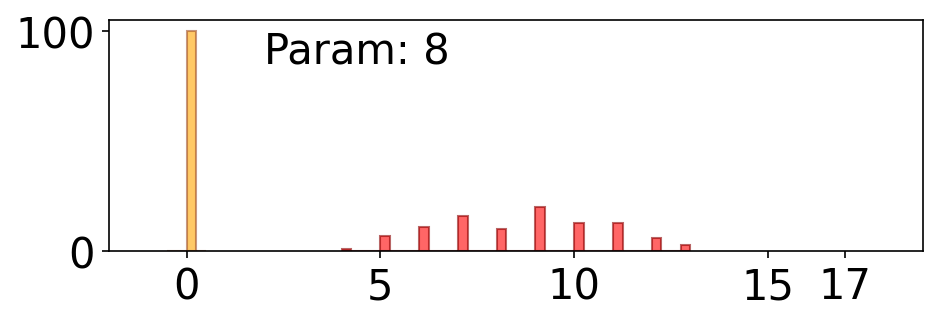}}
    
            \subfloat{\centering
            \makebox[10pt][l]{\raisebox{7pt}{\rotatebox[origin=l]{90}{\tiny Jitter}}}
            \includegraphics[width=1.03\linewidth]{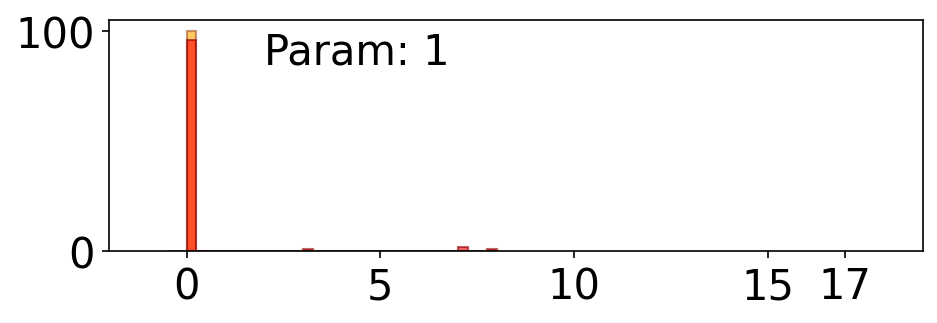}
            \includegraphics[width=1.03\linewidth]{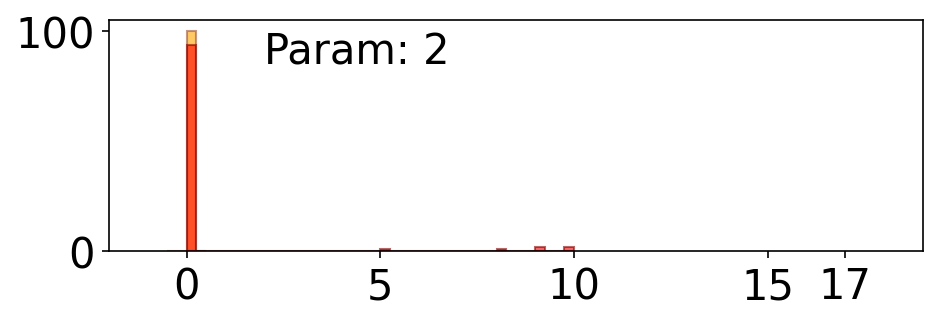}
            \includegraphics[width=1.03\linewidth]{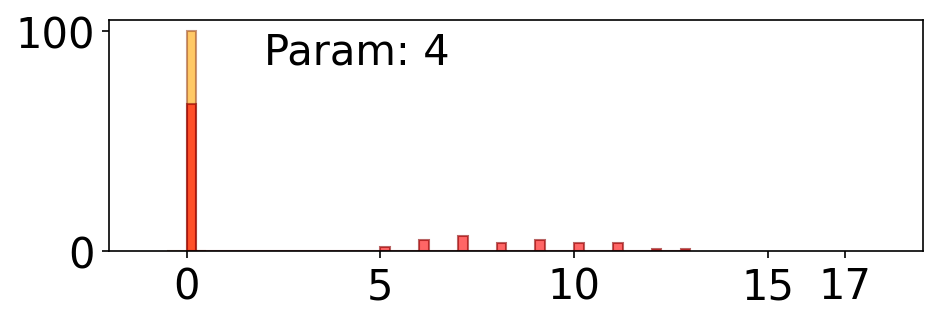}
            \includegraphics[width=1.03\linewidth]{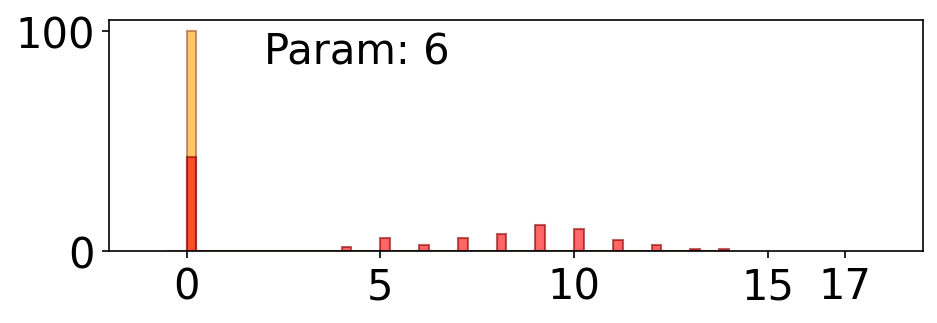}
            \includegraphics[width=1.03\linewidth]{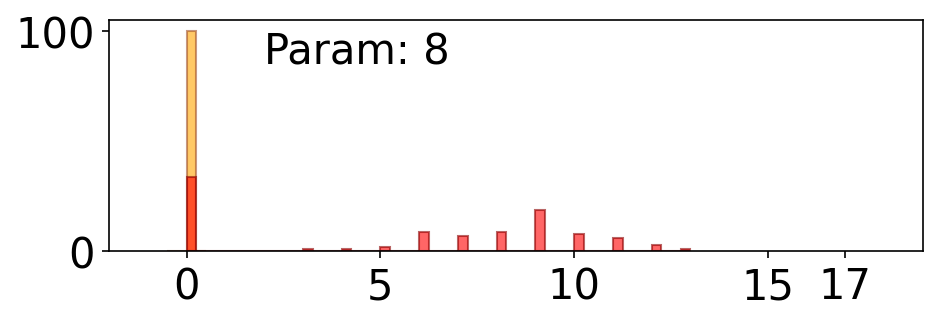}}
    
            \subfloat{\centering
            \makebox[10pt][l]{\raisebox{7pt}{\rotatebox[origin=l]{90}{\tiny JPEG}}}
            \includegraphics[width=1.03\linewidth]{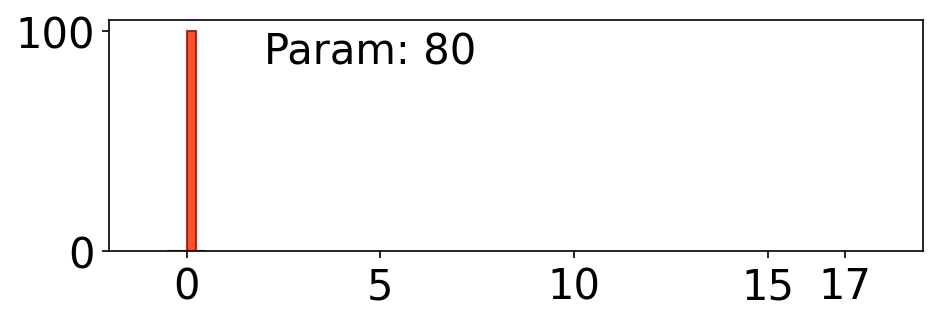}
            \includegraphics[width=1.03\linewidth]{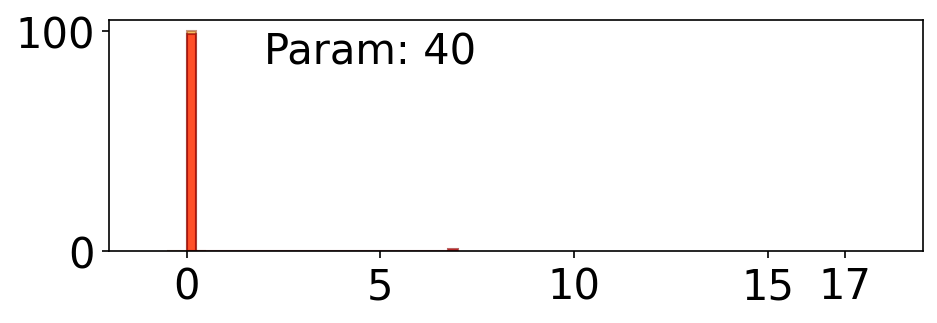}
            \includegraphics[width=1.03\linewidth]{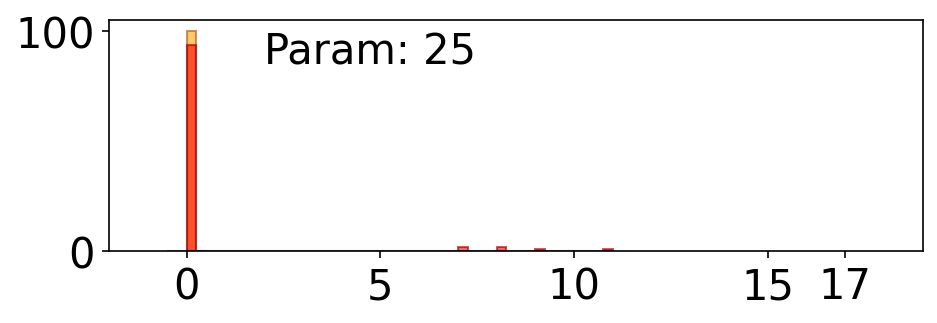}
            \includegraphics[width=1.03\linewidth]{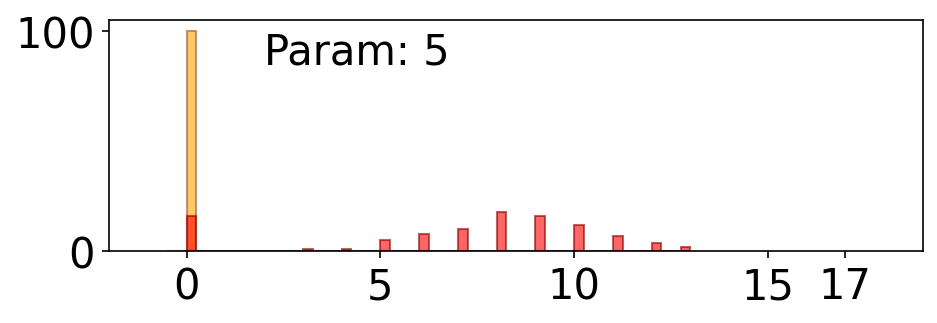}
            \includegraphics[width=1.03\linewidth]{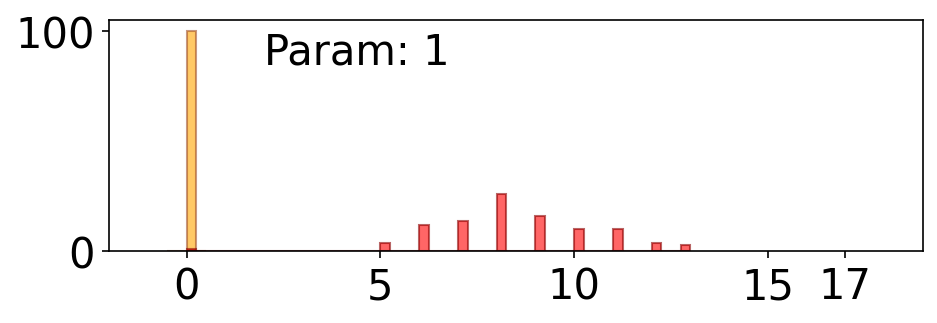}}
    
            \subfloat{\centering
            \makebox[10pt][l]{\raisebox{3pt}{\rotatebox[origin=l]{90}{\tiny Rotation}}}
            \includegraphics[width=1.03\linewidth]{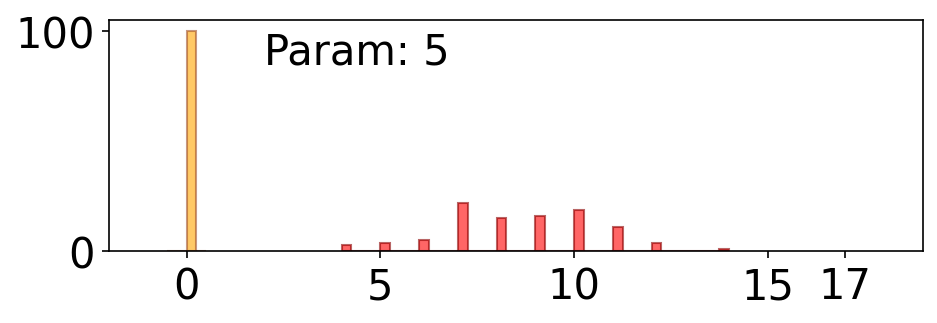}
            \includegraphics[width=1.03\linewidth]{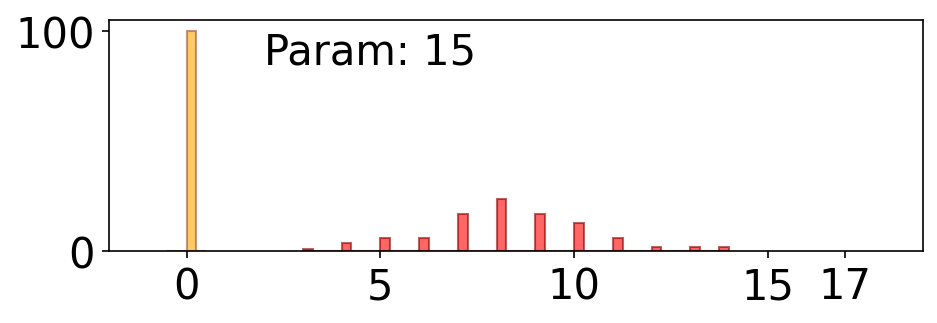}
            \includegraphics[width=1.03\linewidth]{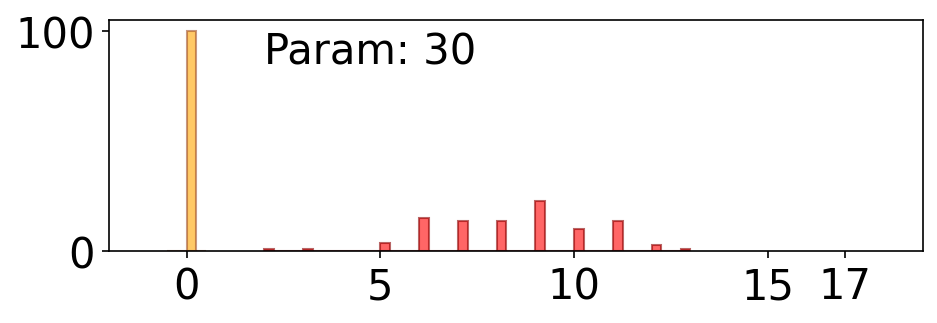}
            \includegraphics[width=1.03\linewidth]{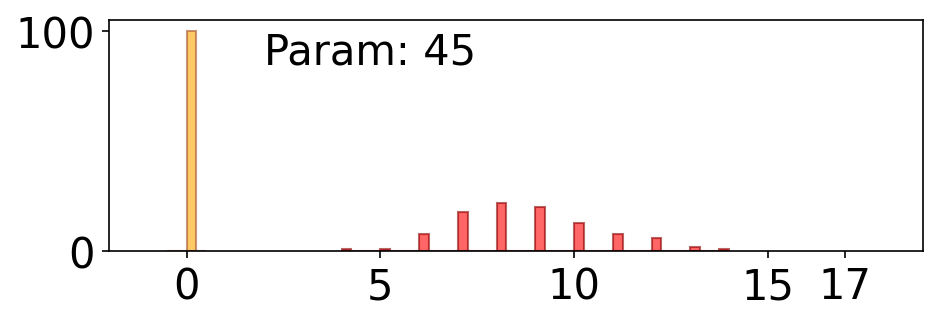}
            \includegraphics[width=1.03\linewidth]{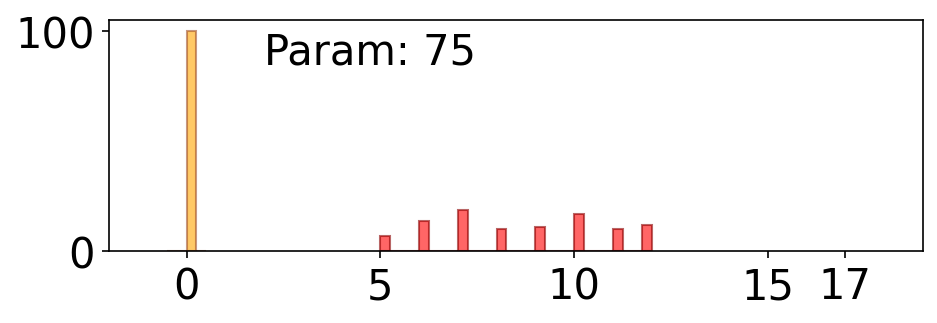}}
    
            \subfloat{\centering
            \makebox[10pt][l]{\raisebox{3pt}{\rotatebox[origin=l]{90}{\tiny Cropping}}}
            \includegraphics[width=1.03\linewidth]{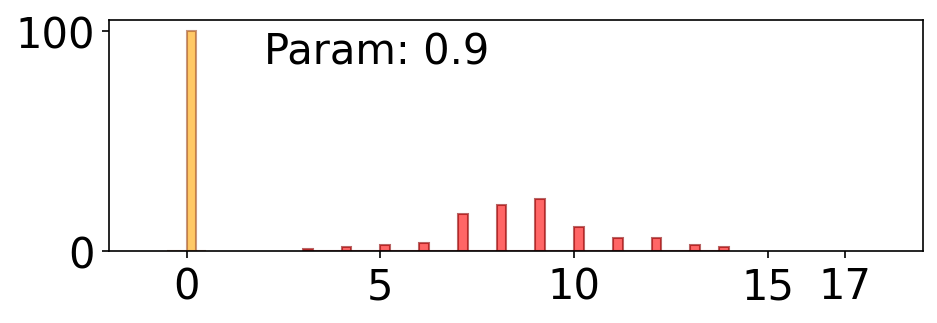}
            \includegraphics[width=1.03\linewidth]{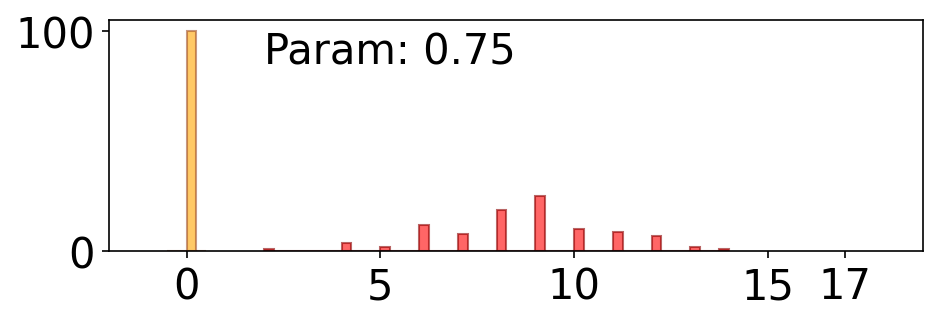}
            \includegraphics[width=1.03\linewidth]{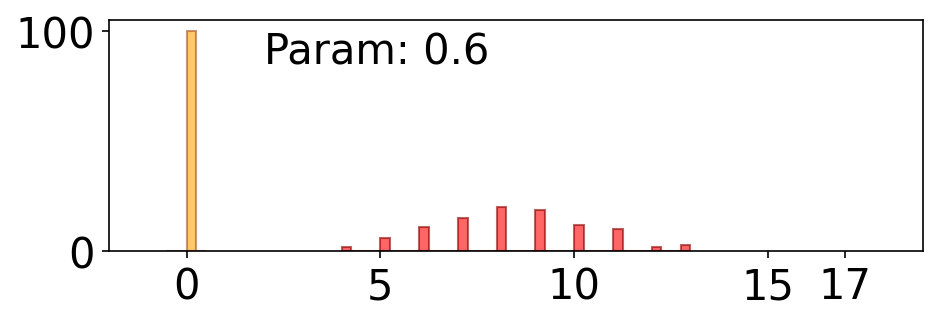}
            \includegraphics[width=1.03\linewidth]{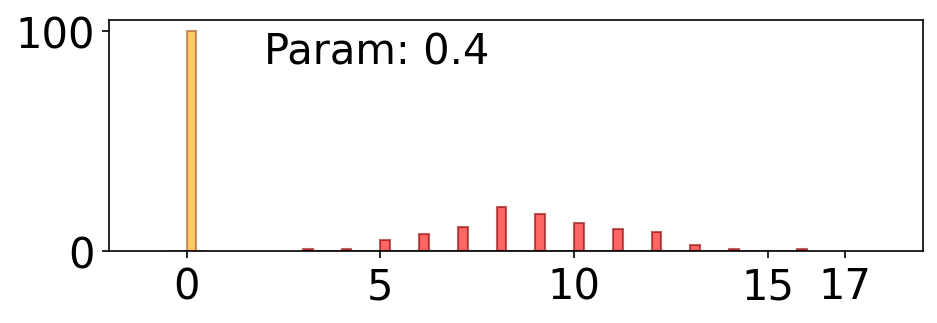}
            \includegraphics[width=1.03\linewidth]{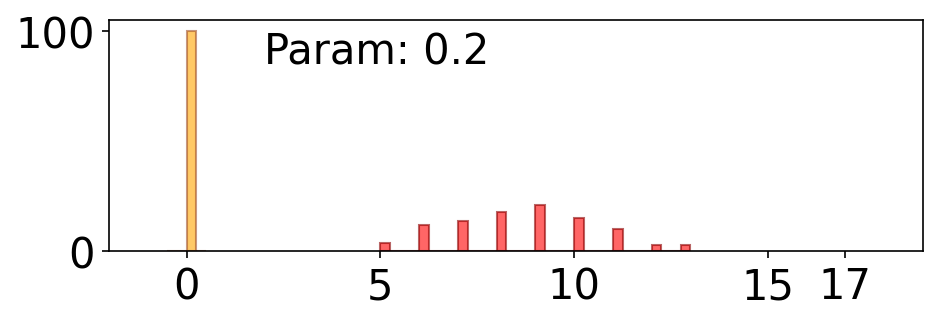}}
    
            \subfloat{\centering
            \makebox[10pt][l]{\raisebox{3pt}{\rotatebox[origin=l]{90}{\tiny Elastic}}}
            \includegraphics[width=1.03\linewidth]{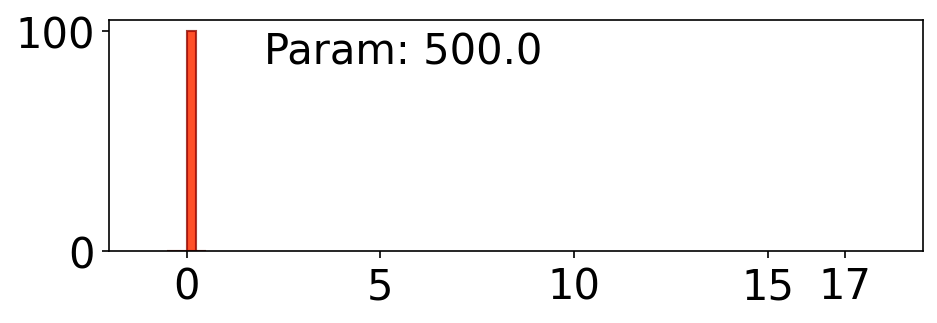}
            \includegraphics[width=1.03\linewidth]{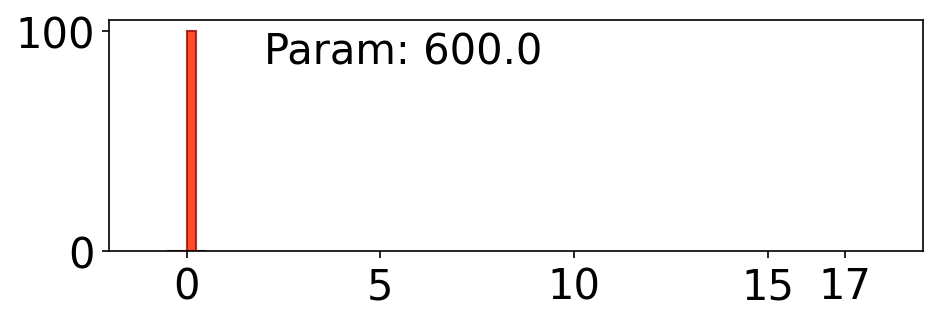}
            \includegraphics[width=1.03\linewidth]{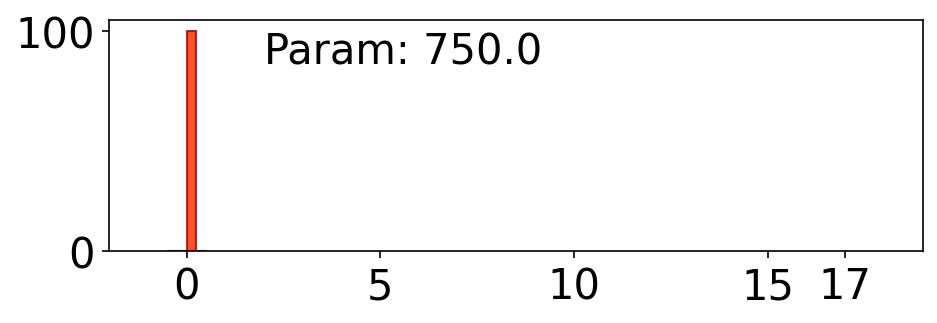}
            \includegraphics[width=1.03\linewidth]{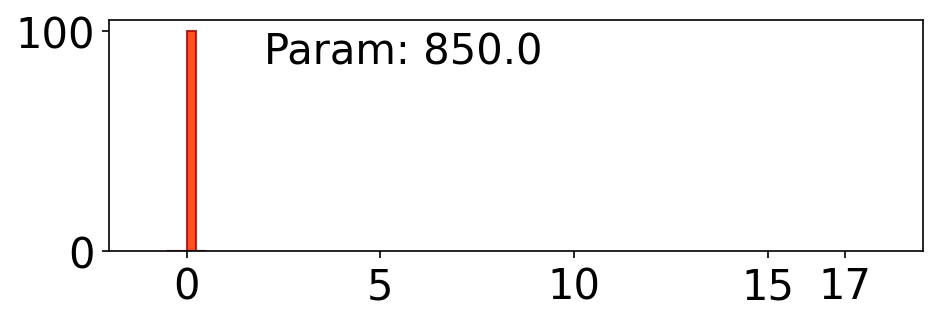}
            \includegraphics[width=1.03\linewidth]{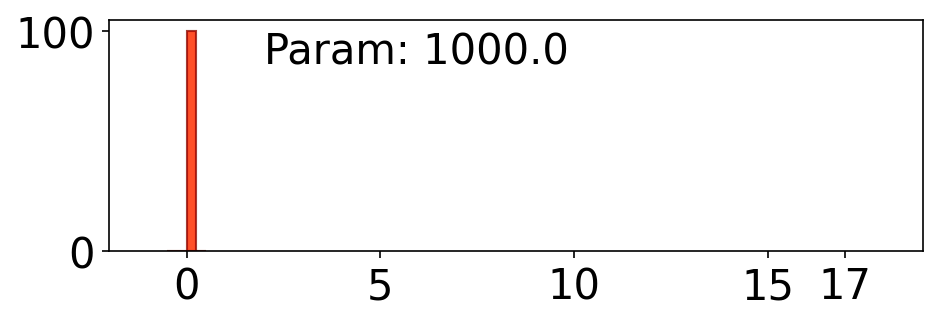}}
    
            \subfloat{\centering
            \makebox[10pt][l]{\raisebox{3pt}{\rotatebox[origin=l]{90}{\scalebox{.4}{Perspective}}}}
            \includegraphics[width=1.03\linewidth]{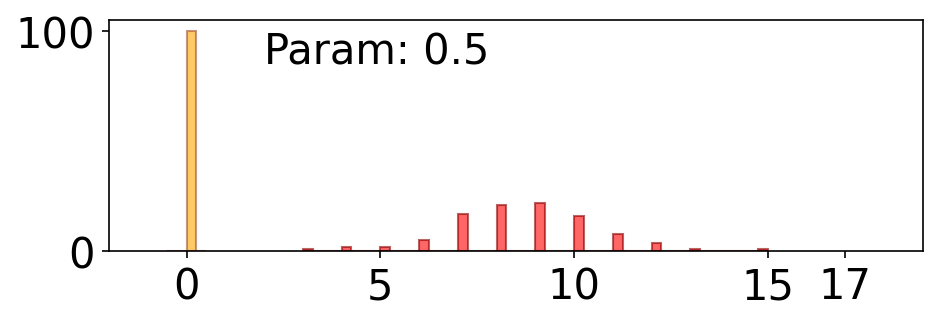}
            \includegraphics[width=1.03\linewidth]{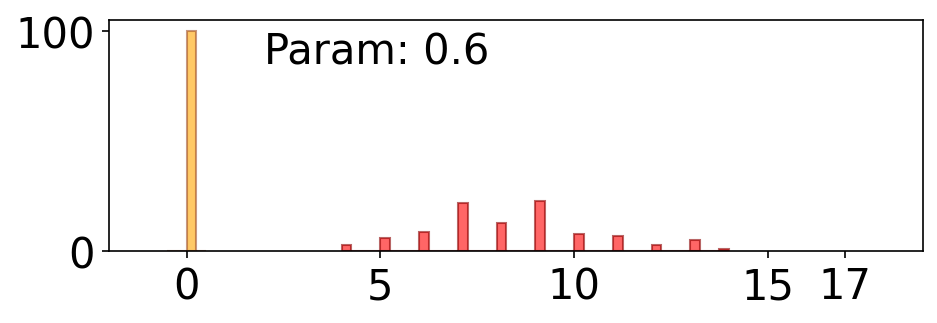}
            \includegraphics[width=1.03\linewidth]{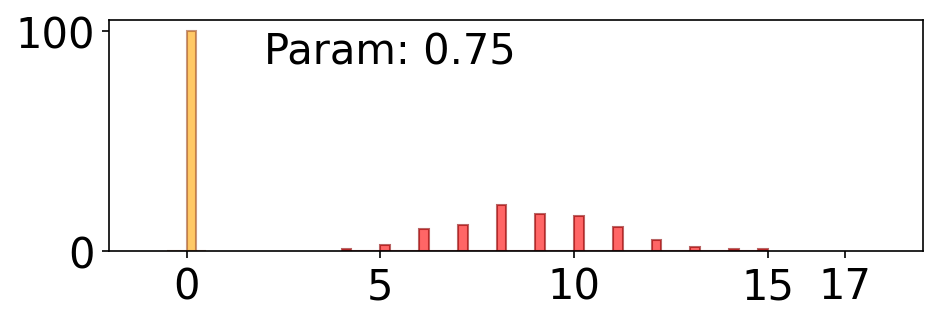}
            \includegraphics[width=1.03\linewidth]{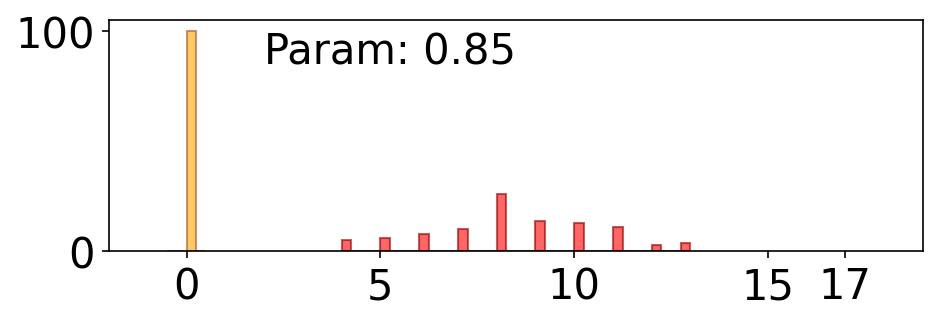}
            \includegraphics[width=1.03\linewidth]{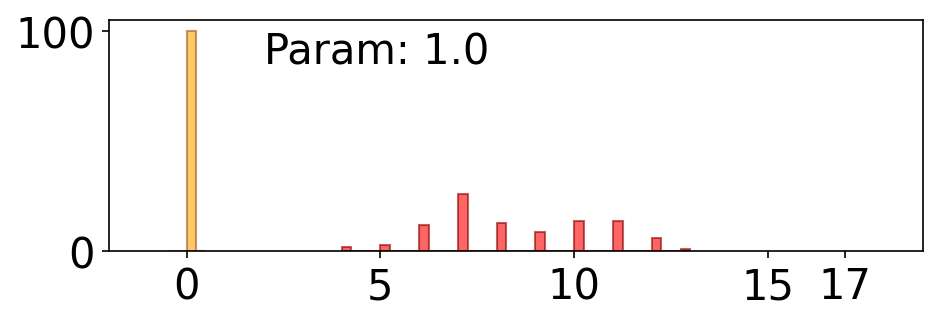}}
    
            \subfloat{\centering
            \makebox[10pt][l]{\raisebox{3pt}{\rotatebox[origin=l]{90}{\tiny Shear}}}
            \includegraphics[width=1.03\linewidth]{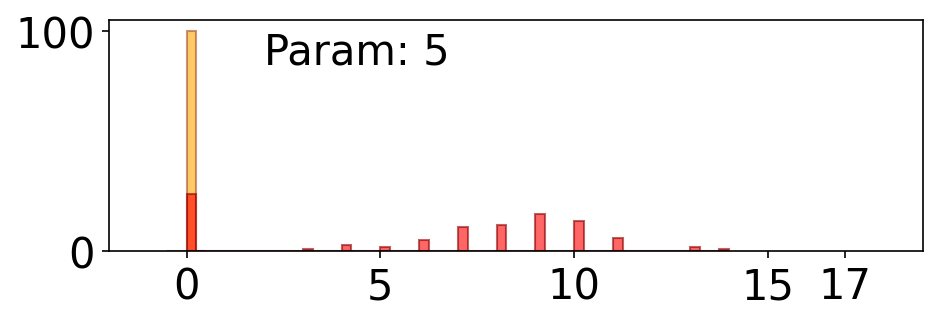}
            \includegraphics[width=1.03\linewidth]{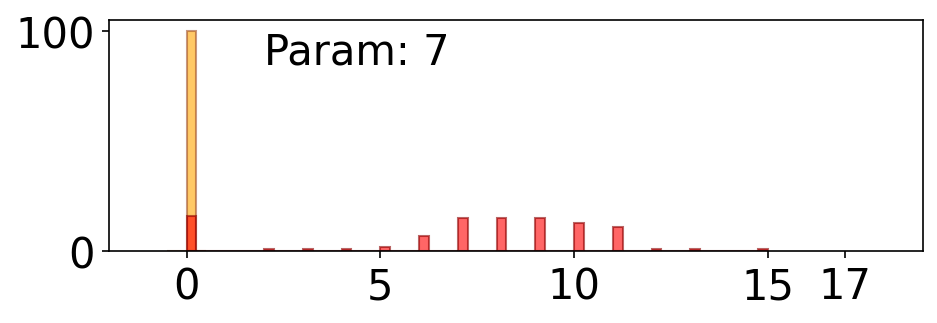}
            \includegraphics[width=1.03\linewidth]{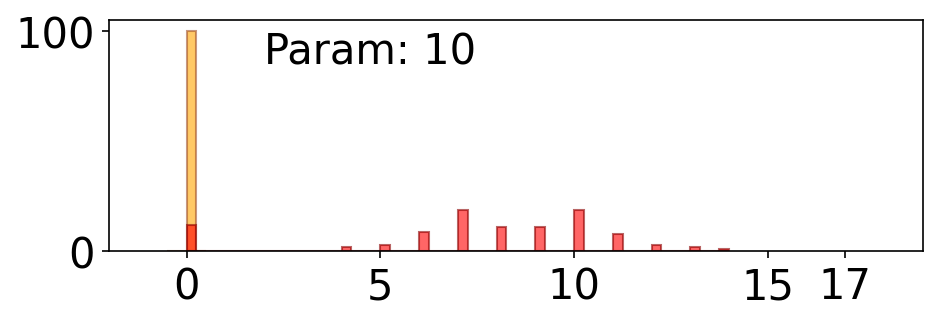}
            \includegraphics[width=1.03\linewidth]{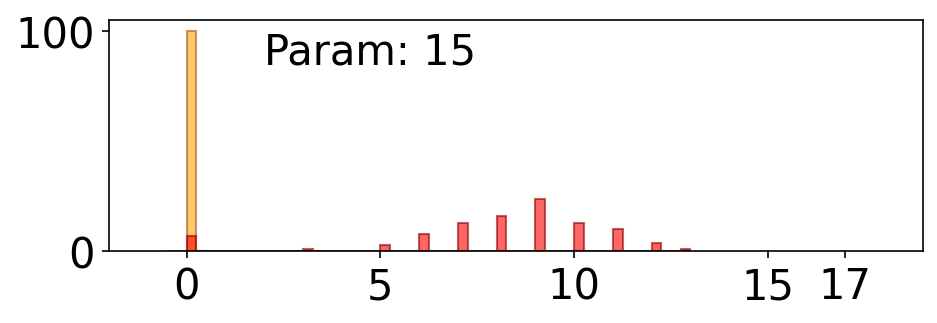}
            \includegraphics[width=1.03\linewidth]{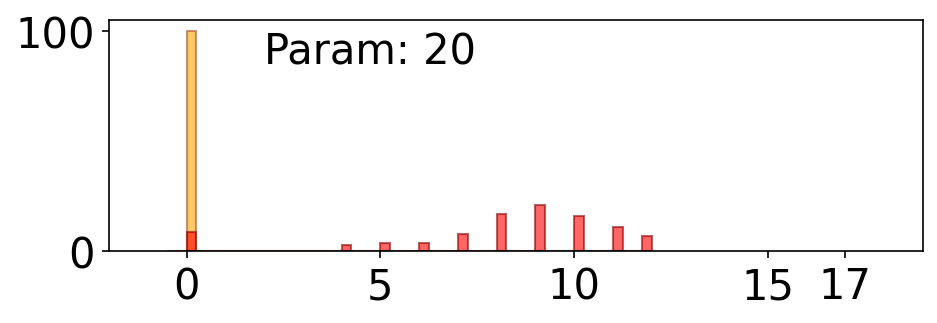}}

            \subfloat{\centering
            \makebox[10pt][l]{\raisebox{1pt}{\rotatebox[origin=l]{90}{\scalebox{.4}{Regen-VAE}}}}
            \includegraphics[width=1.03\linewidth]{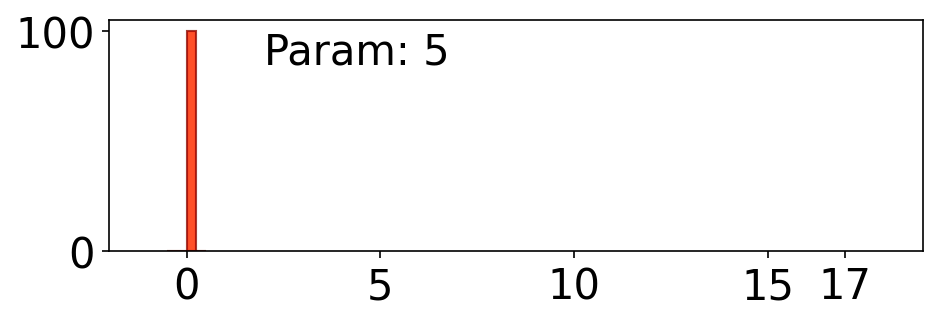}
            \includegraphics[width=1.03\linewidth]{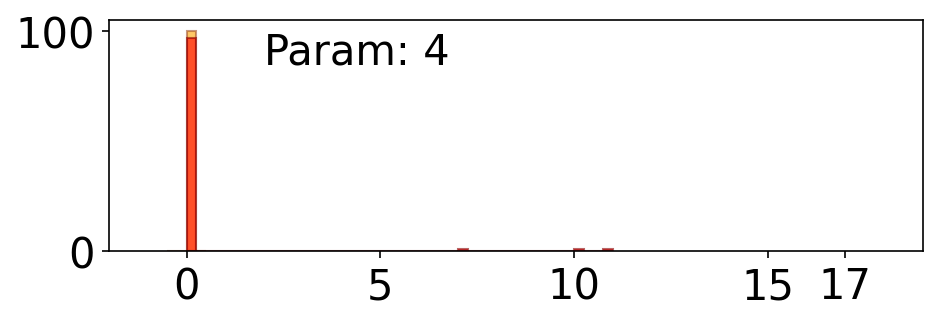}
            \includegraphics[width=1.03\linewidth]{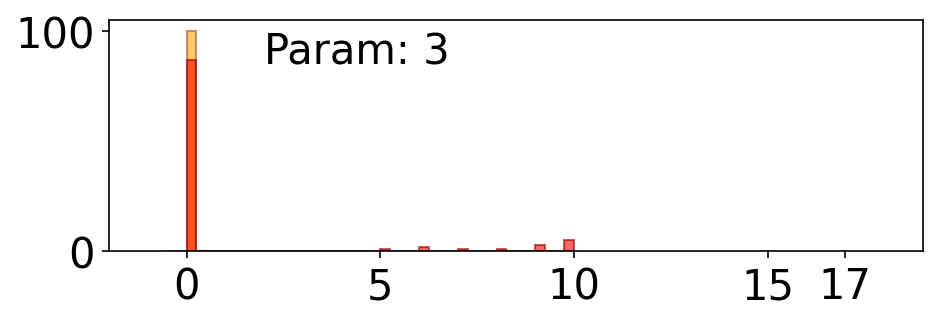}
            \includegraphics[width=1.03\linewidth]{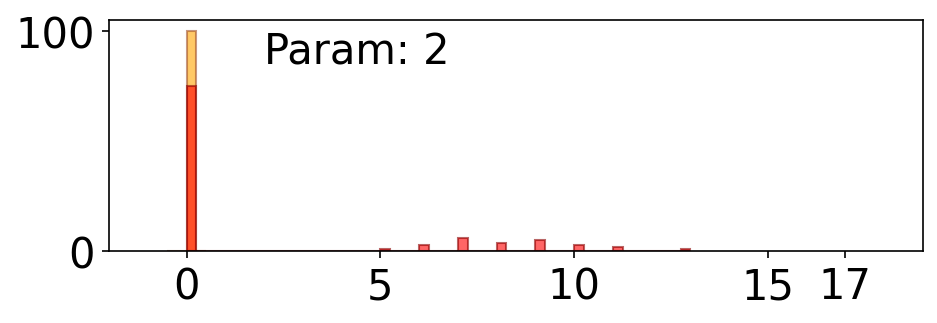}
            \includegraphics[width=1.03\linewidth]{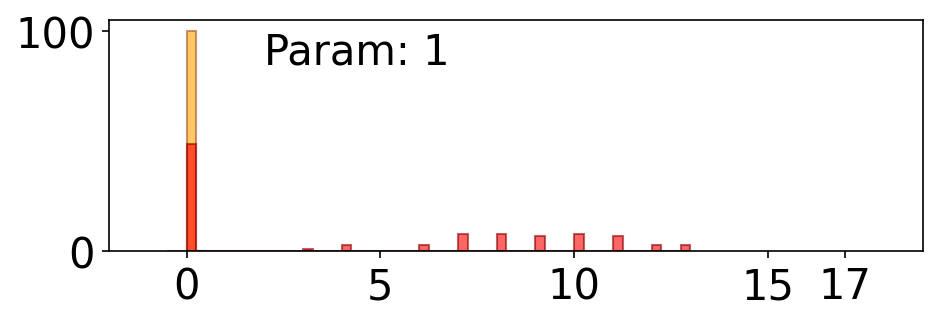}}
    
            \subfloat{\centering
            \makebox[10pt][l]{\raisebox{0pt}{\rotatebox[origin=l]{90}{\scalebox{.4}{Regen-Diff}}}}
            \includegraphics[width=1.03\linewidth]{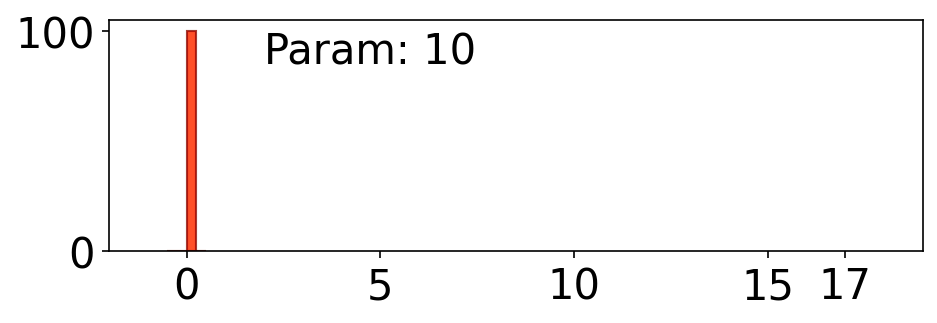}
            \includegraphics[width=1.03\linewidth]{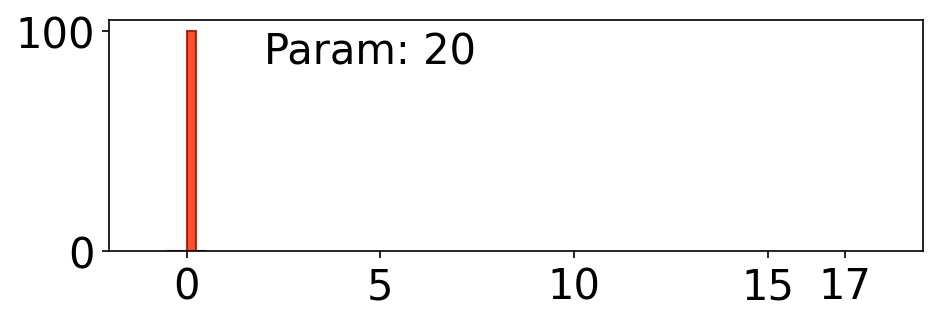}
            \includegraphics[width=1.03\linewidth]{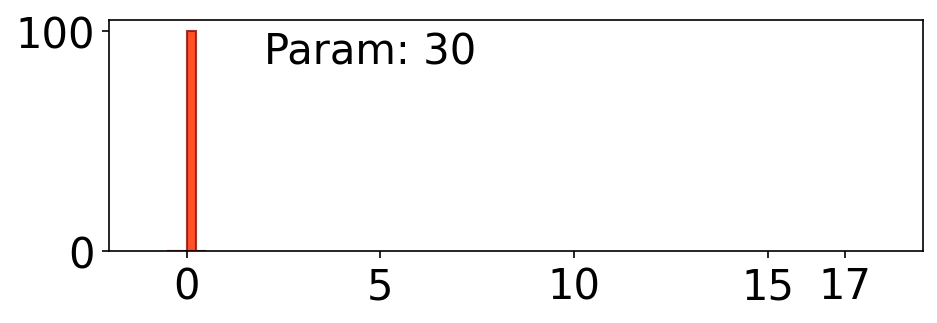}
            \includegraphics[width=1.03\linewidth]{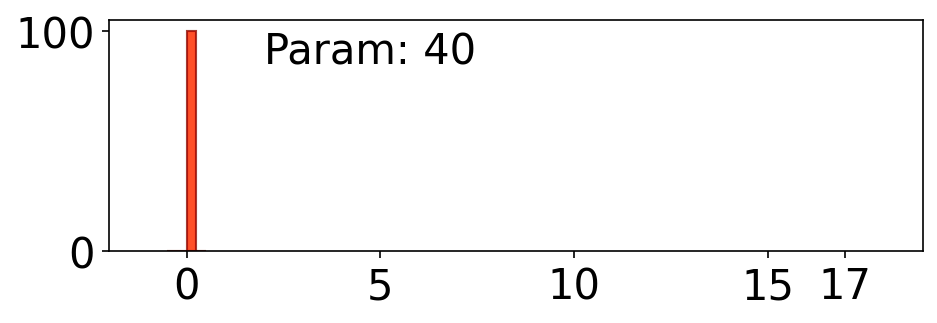}
            \includegraphics[width=1.03\linewidth]{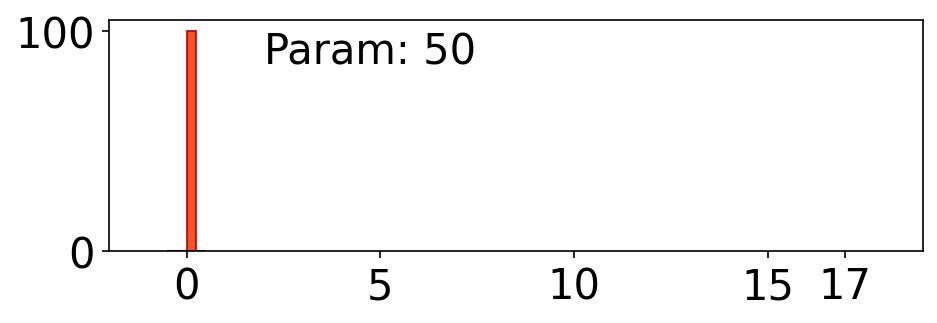}}
        \end{subfigure}
        
    \caption{The effects of different perturbations on $\|\mathbf{w}-\hat{\mathbf{w}}^{\prime}\|_H$. From left to right, the perturbations are set with different parameters (Param). 
    Watermark: PRCW decoding component (1e-5 FPR). In this case, the number of bits is $\left\lceil\log_2(1/10^{-5})\right\rceil=17$.}
    \label{fig: PRCW decoding 1e-5fpr attack shifting}
\end{figure}

\begin{figure}[t]
        \begin{subfigure}[c]{0.185\textwidth}
            \centering
            \subfloat{\centering
            \makebox[10pt][l]{\raisebox{5pt}{\rotatebox[origin=l]{90}{\tiny Blurring}}}
            \includegraphics[width=1.03\linewidth]{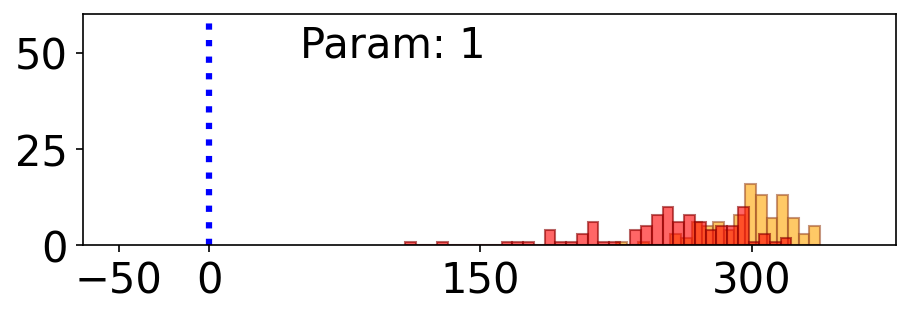}
            \includegraphics[width=1.03\linewidth]{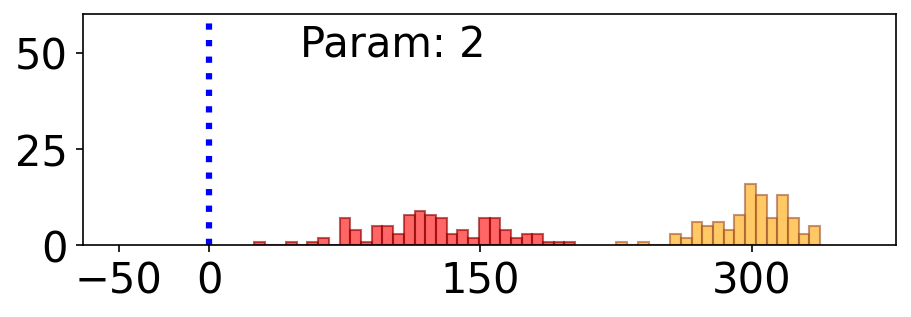}
            \includegraphics[width=1.03\linewidth]{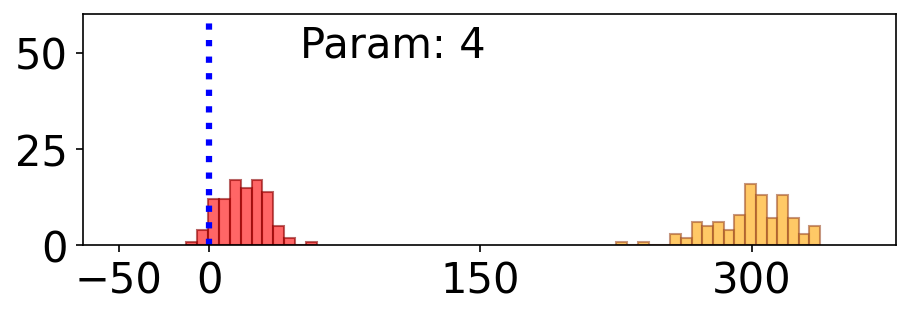}
            \includegraphics[width=1.03\linewidth]{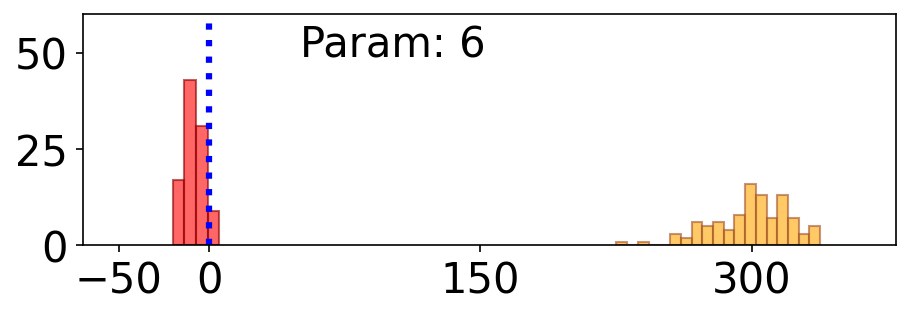}
            \includegraphics[width=1.03\linewidth]{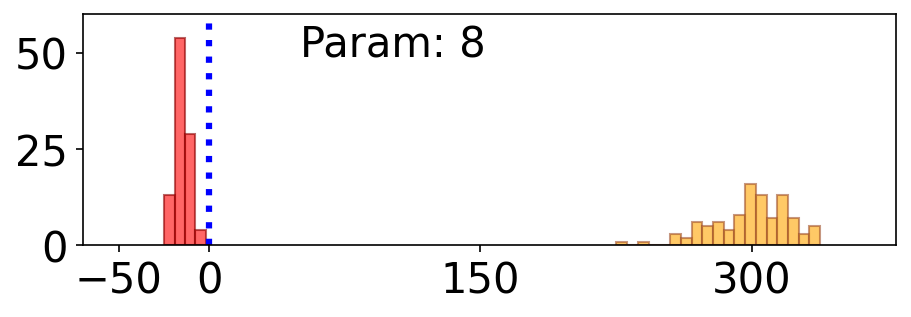}}
    
            \subfloat{\centering
            \makebox[10pt][l]{\raisebox{7pt}{\rotatebox[origin=l]{90}{\tiny Jitter}}}
            \includegraphics[width=1.03\linewidth]{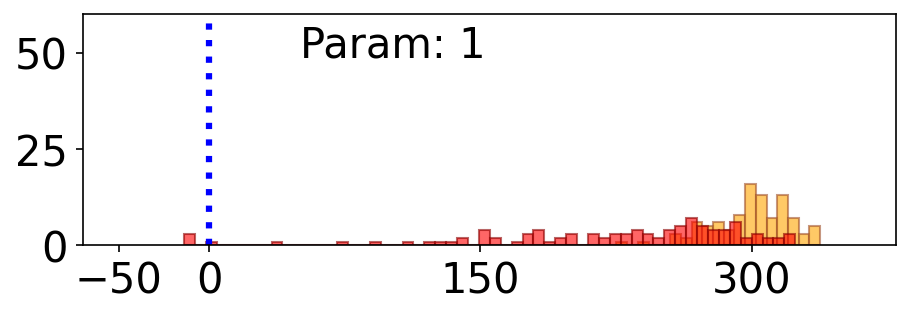}
            \includegraphics[width=1.03\linewidth]{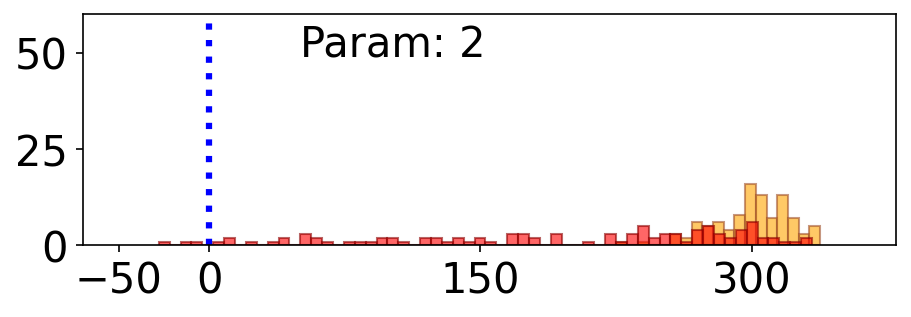}
            \includegraphics[width=1.03\linewidth]{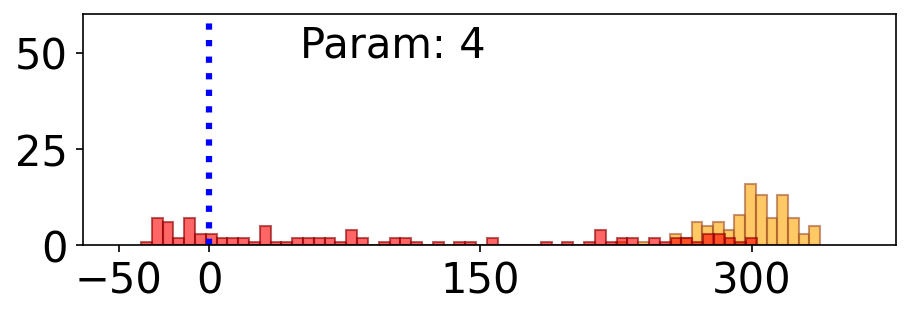}
            \includegraphics[width=1.03\linewidth]{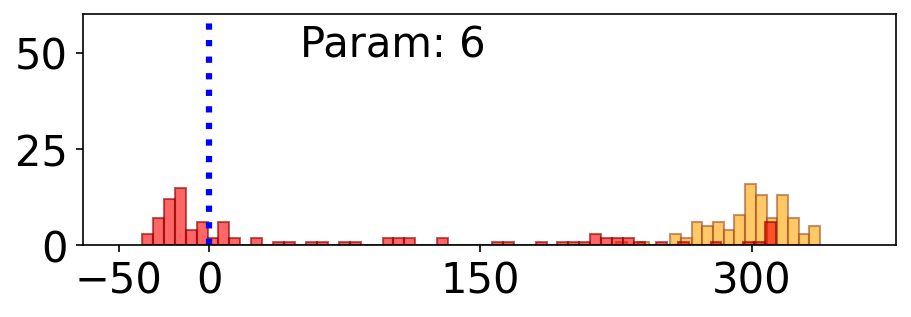}
            \includegraphics[width=1.03\linewidth]{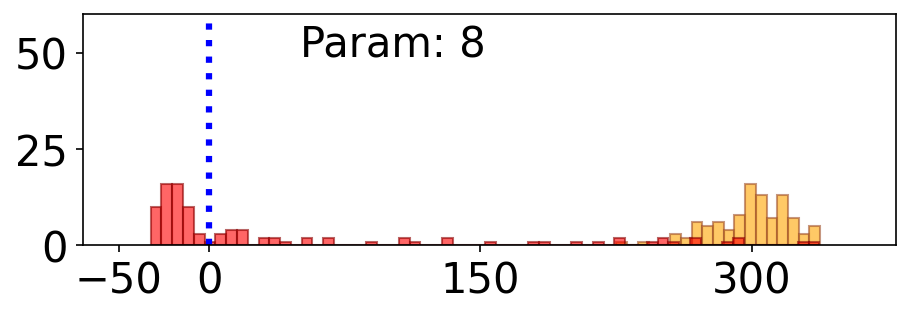}}
    
            \subfloat{\centering
            \makebox[10pt][l]{\raisebox{7pt}{\rotatebox[origin=l]{90}{\tiny JPEG}}}
            \includegraphics[width=1.03\linewidth]{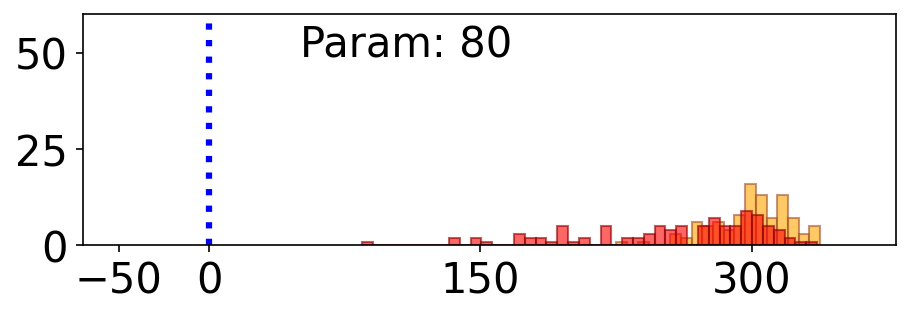}
            \includegraphics[width=1.03\linewidth]{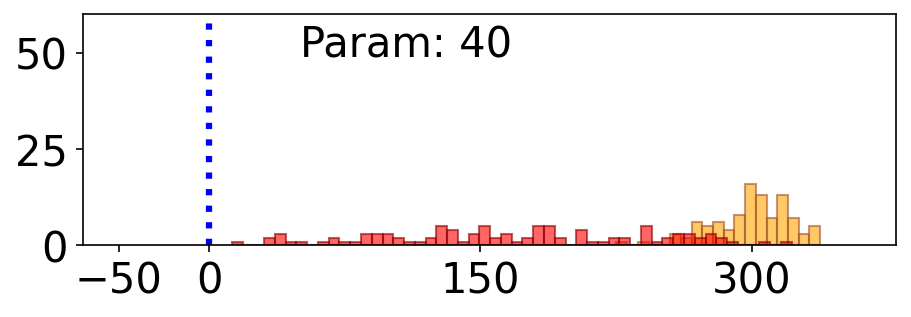}
            \includegraphics[width=1.03\linewidth]{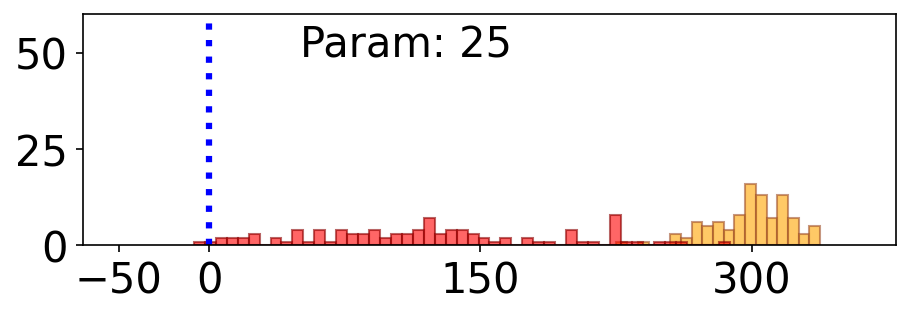}
            \includegraphics[width=1.03\linewidth]{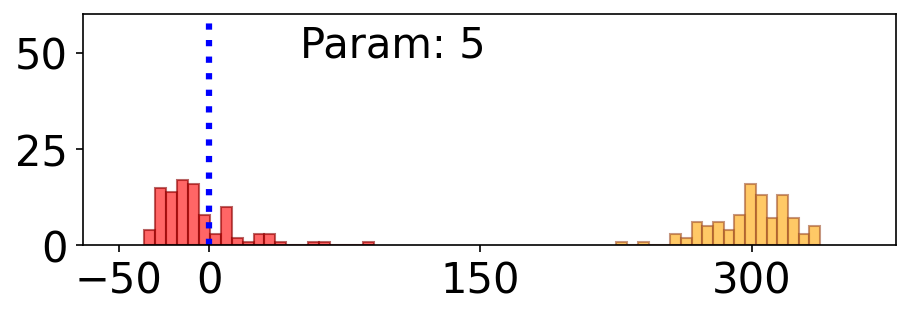}
            \includegraphics[width=1.03\linewidth]{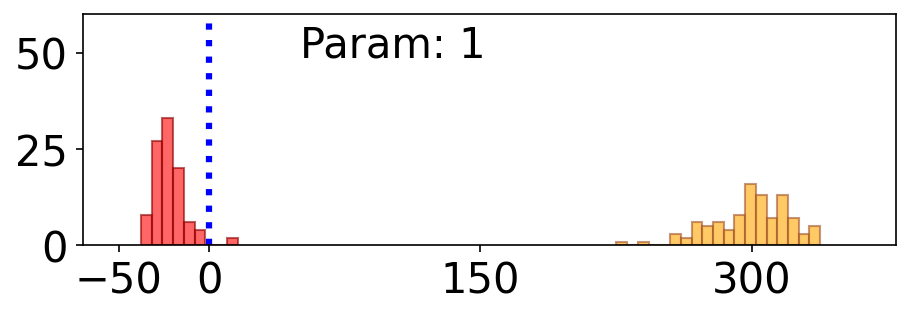}}
    
            \subfloat{\centering
            \makebox[10pt][l]{\raisebox{3pt}{\rotatebox[origin=l]{90}{\tiny Rotation}}}
            \includegraphics[width=1.03\linewidth]{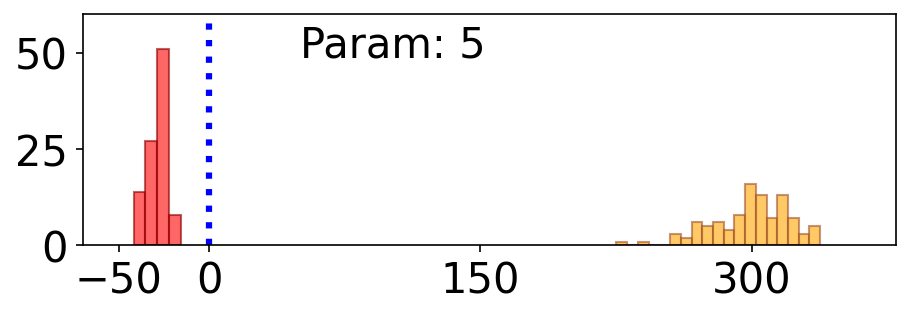}
            \includegraphics[width=1.03\linewidth]{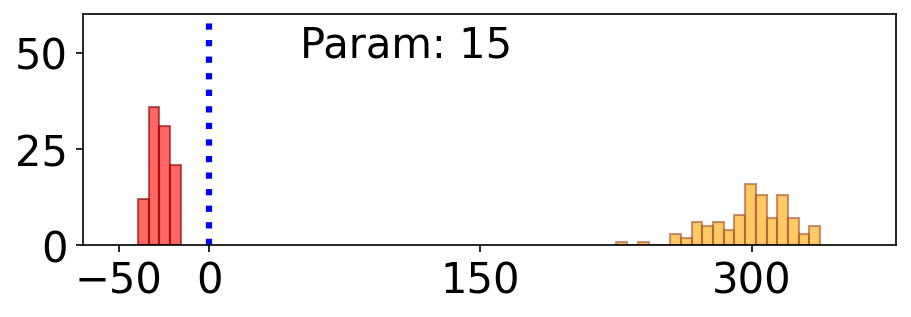}
            \includegraphics[width=1.03\linewidth]{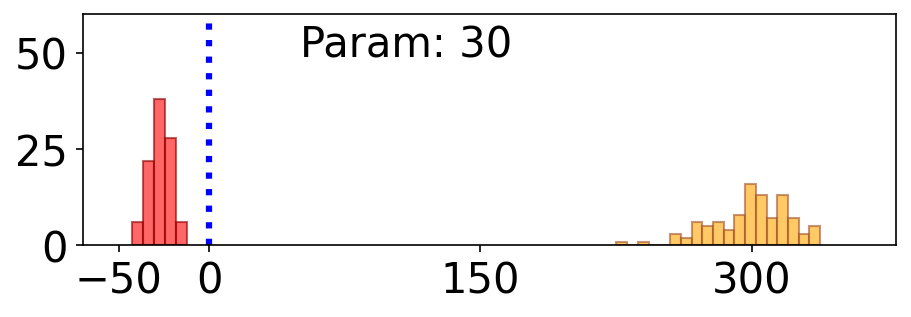}
            \includegraphics[width=1.03\linewidth]{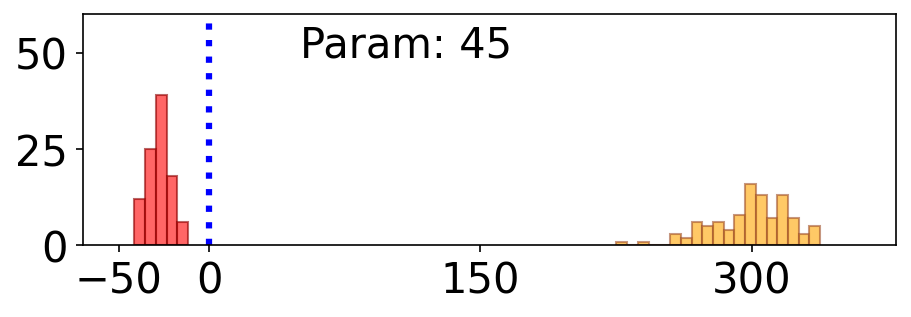}
            \includegraphics[width=1.03\linewidth]{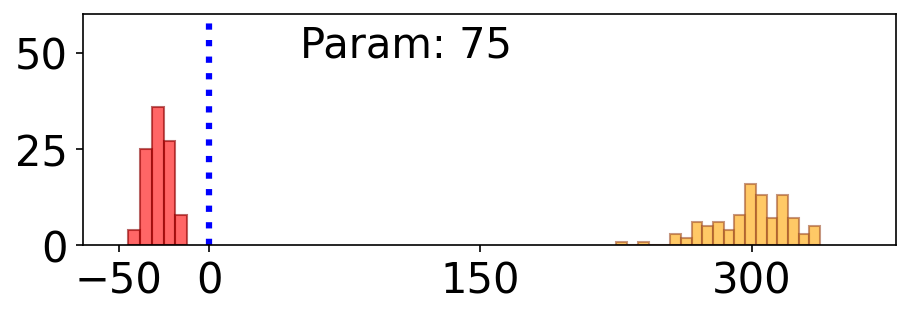}}
    
            \subfloat{\centering
            \makebox[10pt][l]{\raisebox{3pt}{\rotatebox[origin=l]{90}{\tiny Cropping}}}
            \includegraphics[width=1.03\linewidth]{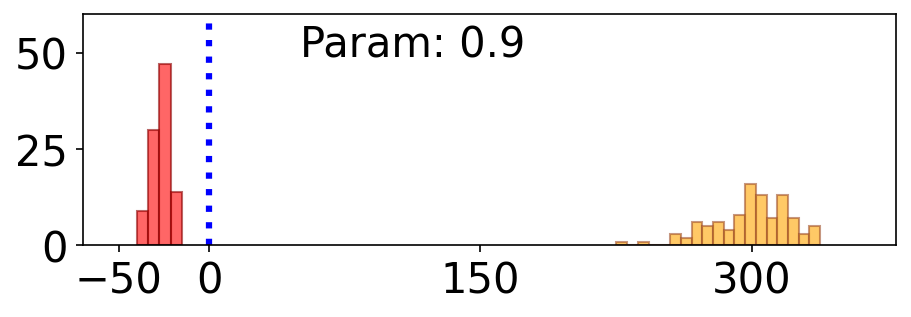}
            \includegraphics[width=1.03\linewidth]{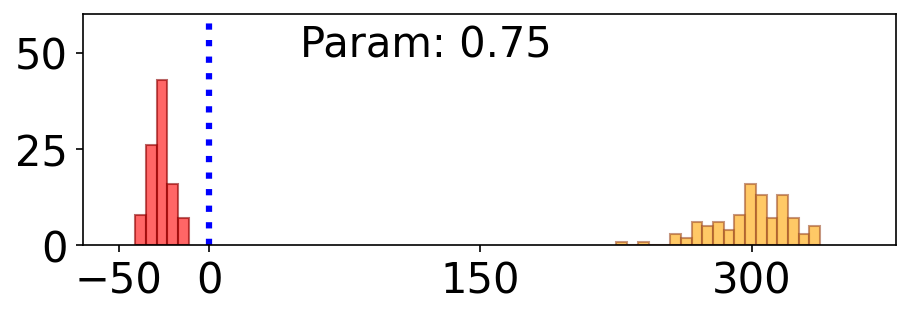}
            \includegraphics[width=1.03\linewidth]{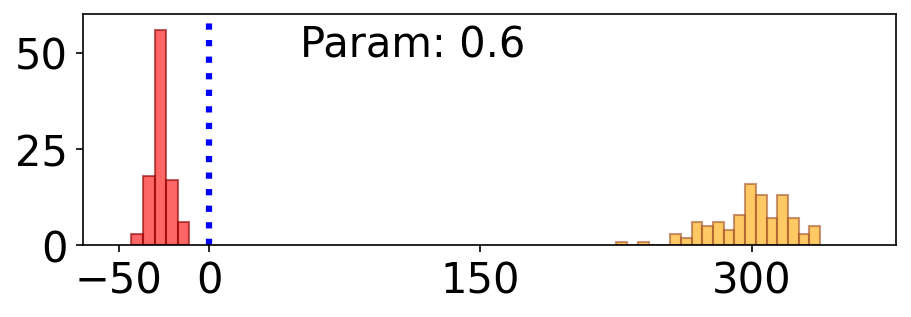}
            \includegraphics[width=1.03\linewidth]{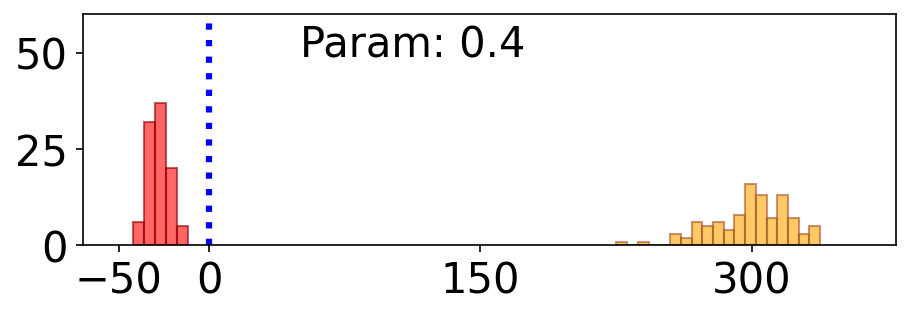}
            \includegraphics[width=1.03\linewidth]{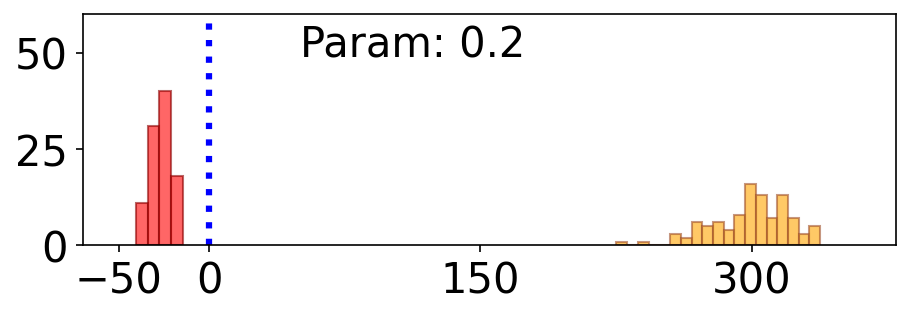}}
    
            \subfloat{\centering
            \makebox[10pt][l]{\raisebox{3pt}{\rotatebox[origin=l]{90}{\tiny Elastic}}}
            \includegraphics[width=1.03\linewidth]{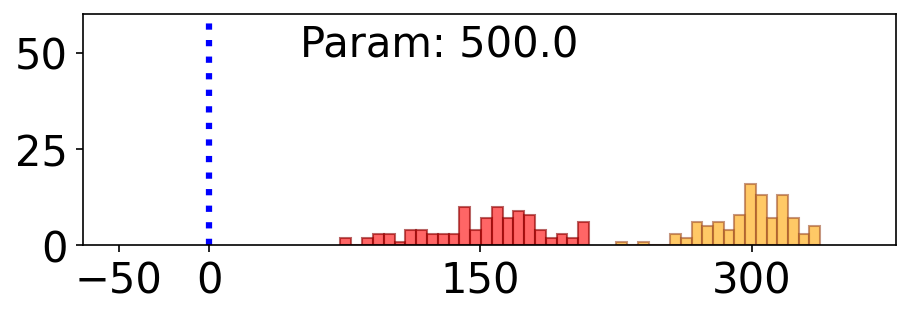}
            \includegraphics[width=1.03\linewidth]{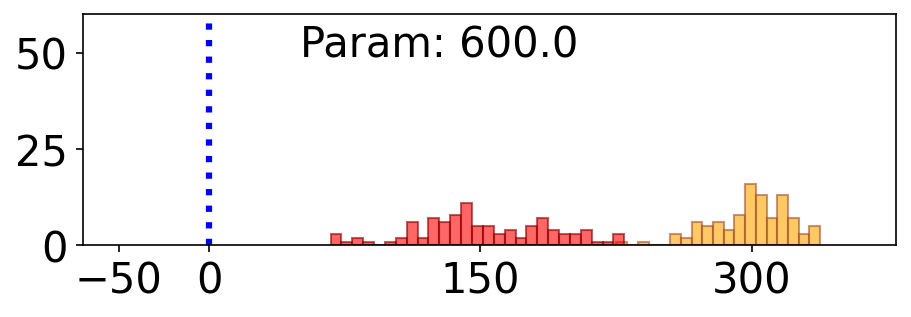}
            \includegraphics[width=1.03\linewidth]{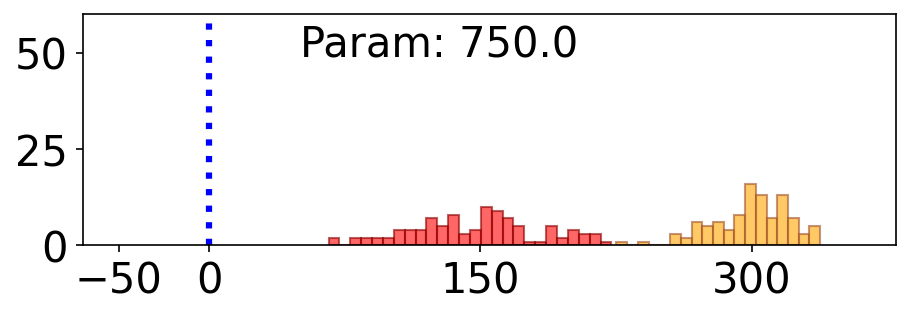}
            \includegraphics[width=1.03\linewidth]{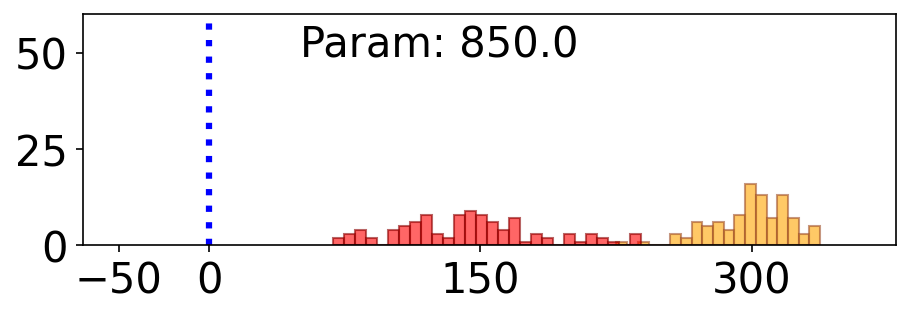}
            \includegraphics[width=1.03\linewidth]{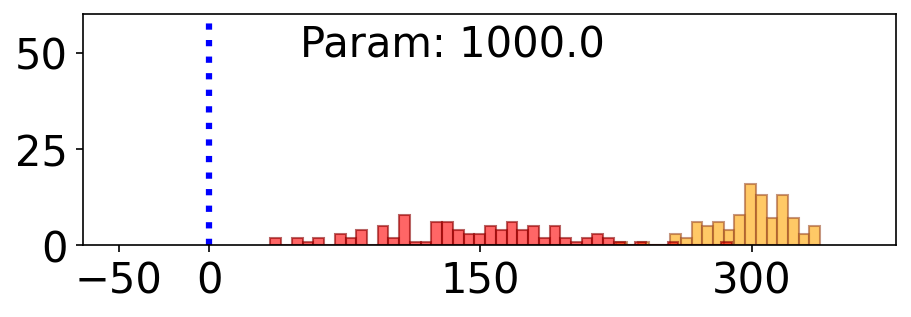}}
    
            \subfloat{\centering
            \makebox[10pt][l]{\raisebox{3pt}{\rotatebox[origin=l]{90}{\scalebox{.4}{Perspective}}}}
            \includegraphics[width=1.03\linewidth]{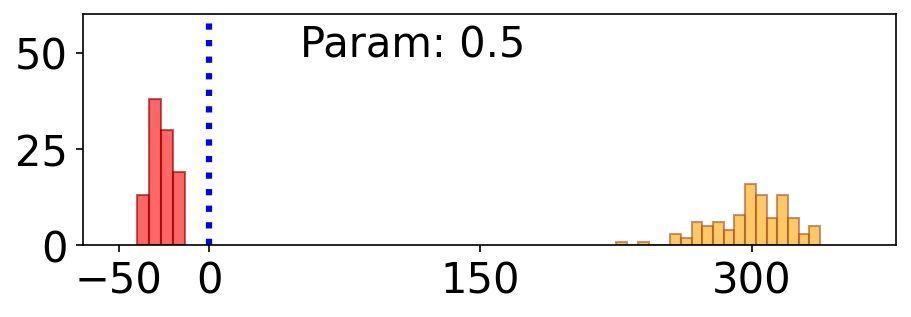}
            \includegraphics[width=1.03\linewidth]{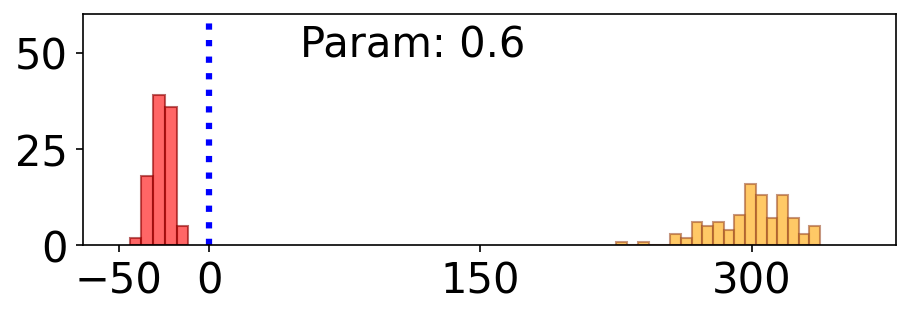}
            \includegraphics[width=1.03\linewidth]{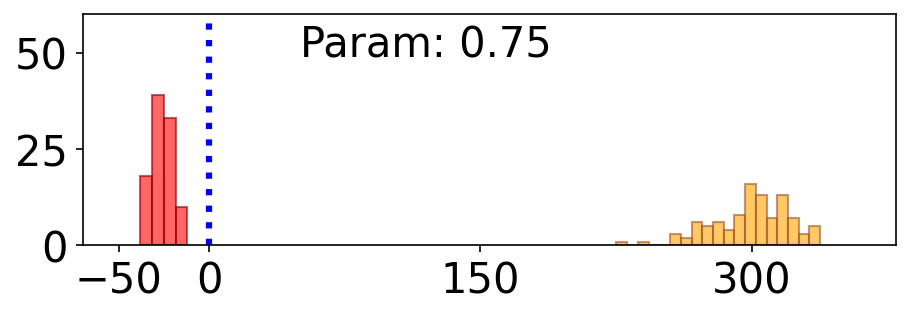}
            \includegraphics[width=1.03\linewidth]{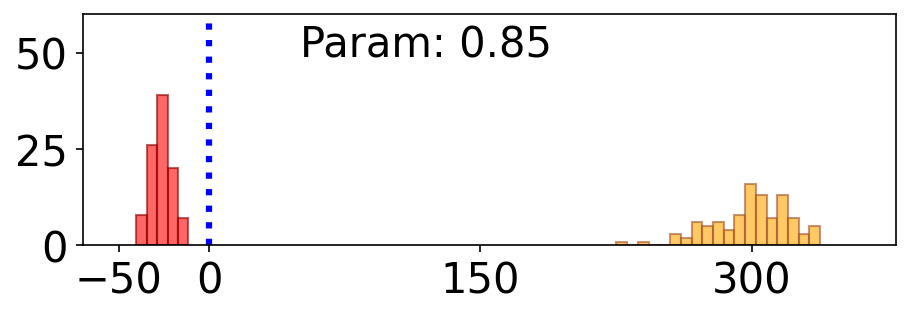}
            \includegraphics[width=1.03\linewidth]{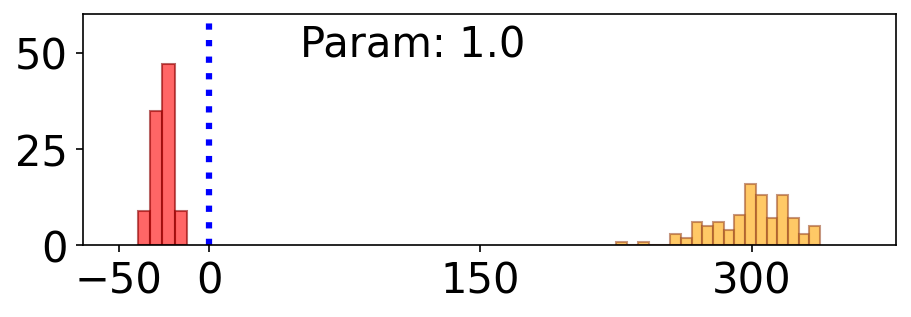}}
    
            \subfloat{\centering
            \makebox[10pt][l]{\raisebox{3pt}{\rotatebox[origin=l]{90}{\tiny Shear}}}
            \includegraphics[width=1.03\linewidth]{figs/figs/prcw/hist_detection_fpr_0.01_shear_5.png}
            \includegraphics[width=1.03\linewidth]{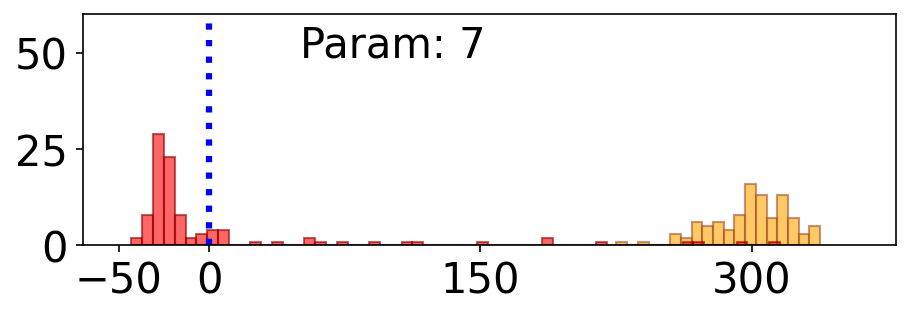}
            \includegraphics[width=1.03\linewidth]{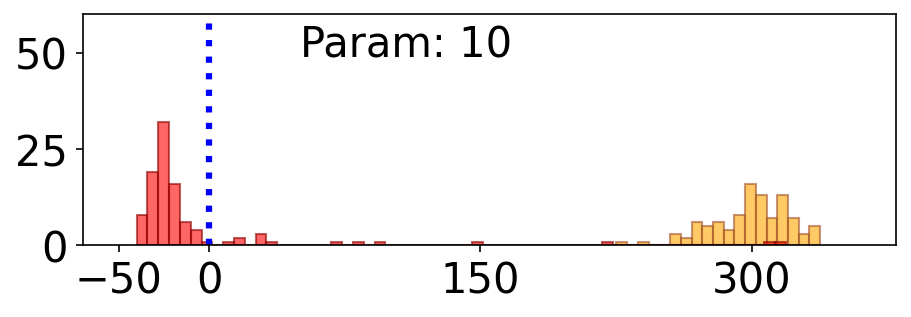}
            \includegraphics[width=1.03\linewidth]{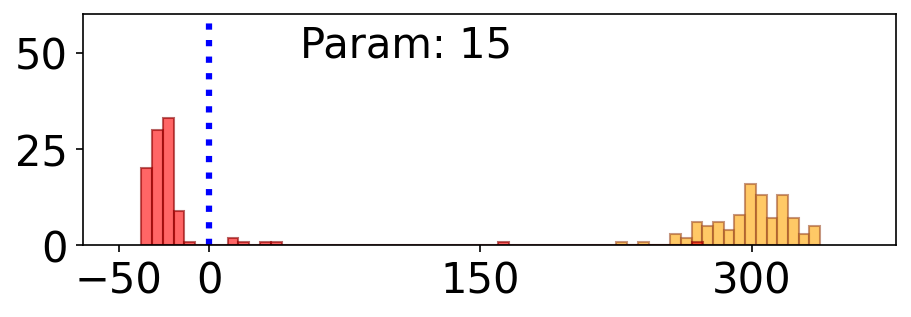}
            \includegraphics[width=1.03\linewidth]{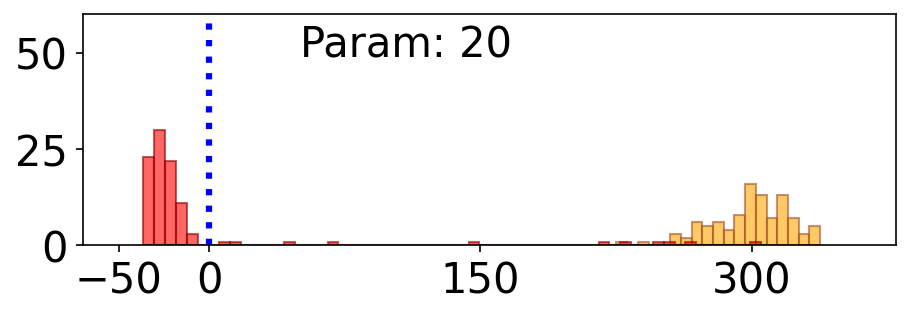}}

            \subfloat{\centering
            \makebox[10pt][l]{\raisebox{1pt}{\rotatebox[origin=l]{90}{\scalebox{.4}{Regen-VAE}}}}
            \includegraphics[width=1.03\linewidth]{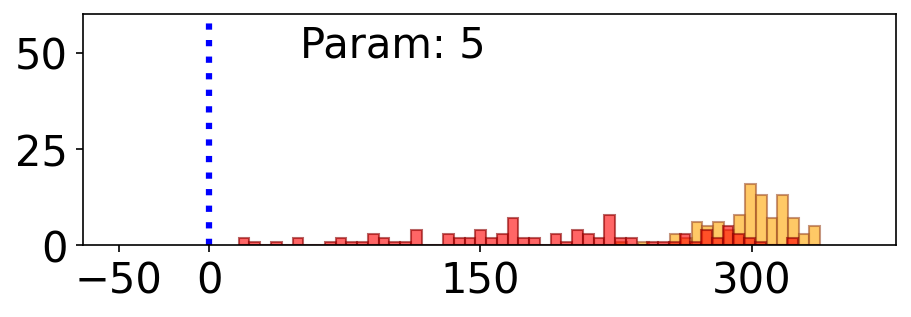}
            \includegraphics[width=1.03\linewidth]{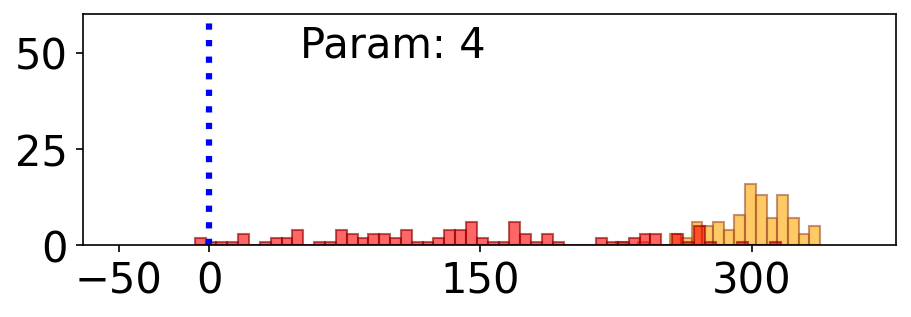}
            \includegraphics[width=1.03\linewidth]{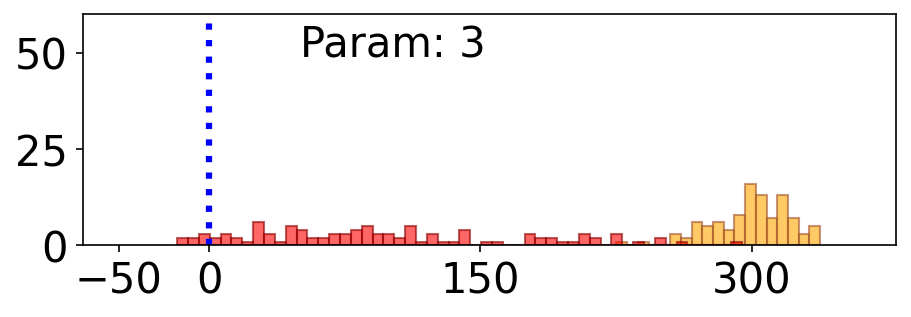}
            \includegraphics[width=1.03\linewidth]{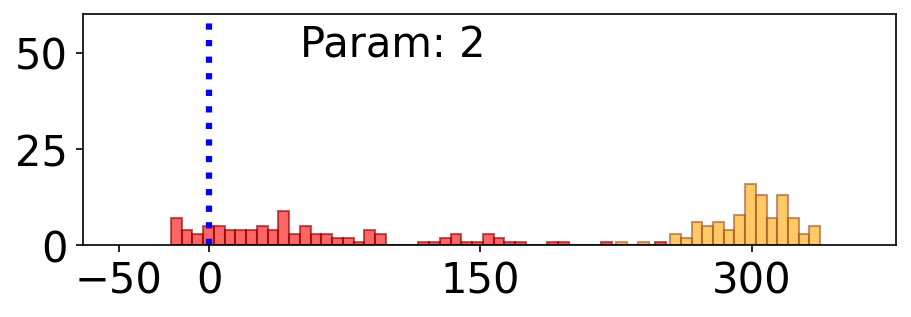}
            \includegraphics[width=1.03\linewidth]{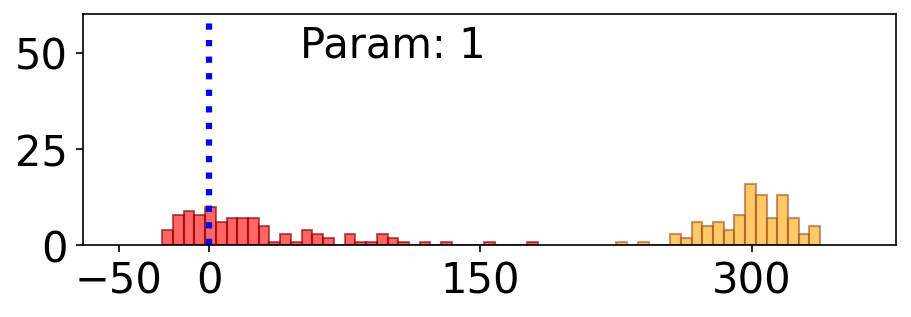}}
    
            \subfloat{\centering
            \makebox[10pt][l]{\raisebox{0pt}{\rotatebox[origin=l]{90}{\scalebox{.4}{Regen-Diff}}}}
            \includegraphics[width=1.03\linewidth]{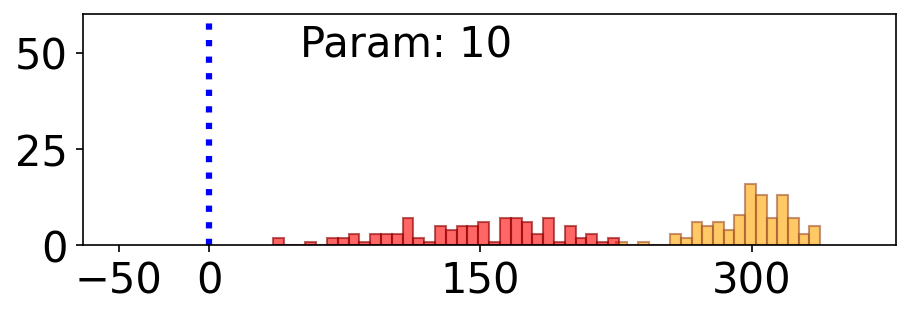}
            \includegraphics[width=1.03\linewidth]{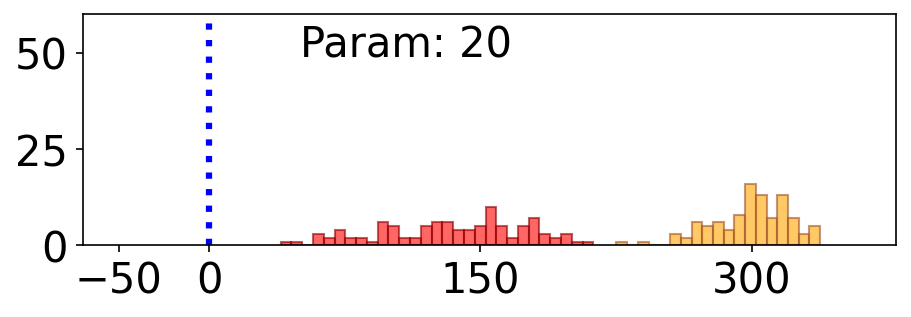}
            \includegraphics[width=1.03\linewidth]{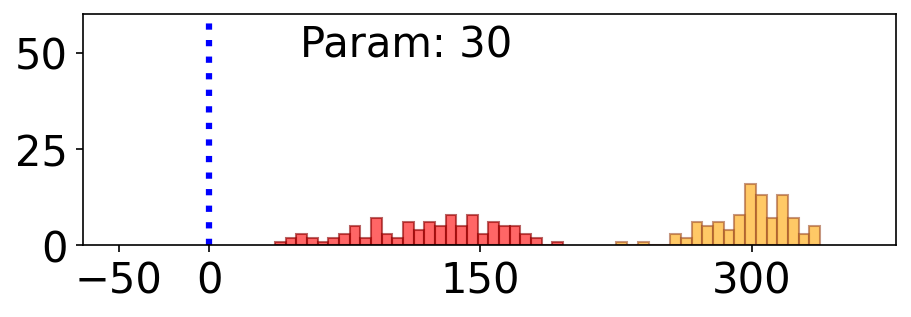}
            \includegraphics[width=1.03\linewidth]{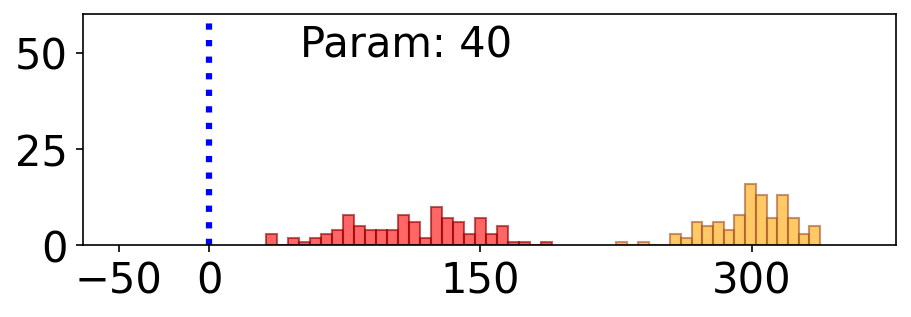}
            \includegraphics[width=1.03\linewidth]{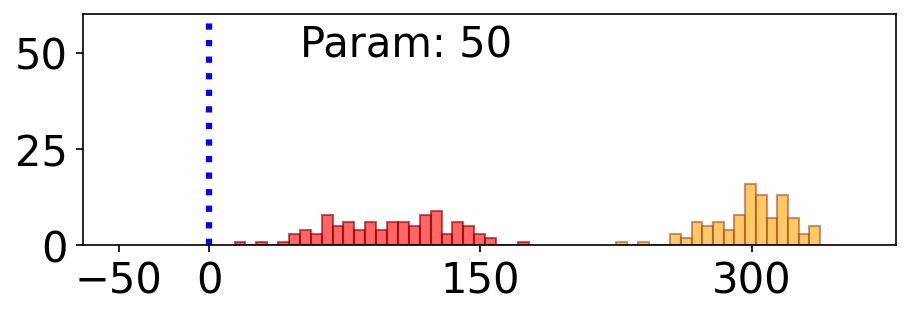}}
        \end{subfigure}
        
    \caption{The effects of different perturbations by Eq.~(\ref{eq: PRCW Alog 3 L 5}). From left to right, the perturbations are set with different parameters (Param). 
    Watermark: PRCW detection component (1e-5 FPR).}
    \label{fig: PRCW detection 1e-5fpr attack shifting}
\end{figure}

\begin{table}[t]
  \centering
  \small
  \caption{Effects of different number of inference time steps on $\|\mathbf{w}-\hat{\mathbf{w}}^{\prime}\|$. The number of inference time steps were set to 10, 30, 50 (default), 70, and 90. We present the mean and the maximum values of a set of $\|\mathbf{w}-\hat{\mathbf{w}}^{\prime}\|$.}\label{Table: numerical error}
  \begin{tabular}{|c|p{1.05cm}p{1.05cm}|p{1.05cm}p{1.05cm}|p{1.05cm}p{1.05cm}|p{1.05cm}p{1.05cm}|}
    \toprule
    WM Method & \multicolumn{2}{c|}{Tree-Ring} & \multicolumn{2}{c|}{RingID} & \multicolumn{2}{c|}{HSQR} & \multicolumn{2}{c|}{HSTR} \\
    \midrule[2pt]
     & \multicolumn{1}{c}{mean} & \multicolumn{1}{c|}{max} & \multicolumn{1}{c}{mean} & \multicolumn{1}{c|}{max} & \multicolumn{1}{c}{mean} & \multicolumn{1}{c|}{max} & \multicolumn{1}{c}{mean} & \multicolumn{1}{c|}{max} \\
    \hline
    $t=10$ & \multicolumn{1}{c}{55.36} & \multicolumn{1}{c|}{62.38} & \multicolumn{1}{c}{29.34} & \multicolumn{1}{c|}{45.37} & \multicolumn{1}{c}{43.17} & \multicolumn{1}{c|}{48.91} & \multicolumn{1}{c}{19.94} & \multicolumn{1}{c|}{25.93} \\
    $t=30$ & \multicolumn{1}{c}{53.96} & \multicolumn{1}{c|}{60.62} & \multicolumn{1}{c}{26.31} & \multicolumn{1}{c|}{45.52} & \multicolumn{1}{c}{41.69} & \multicolumn{1}{c|}{47.76} & \multicolumn{1}{c}{17.48} & \multicolumn{1}{c|}{24.95} \\
    $t=50$ & \multicolumn{1}{c}{53.72} & \multicolumn{1}{c|}{60.53} & \multicolumn{1}{c}{25.63} & \multicolumn{1}{c|}{45.91} & \multicolumn{1}{c}{41.33} & \multicolumn{1}{c|}{47.53} & \multicolumn{1}{c}{16.94} & \multicolumn{1}{c|}{24.86} \\
    $t=70$ & \multicolumn{1}{c}{53.58} & \multicolumn{1}{c|}{60.28} & \multicolumn{1}{c}{25.36} & \multicolumn{1}{c|}{44.42} & \multicolumn{1}{c}{41.30} & \multicolumn{1}{c|}{47.75} & \multicolumn{1}{c}{16.68} & \multicolumn{1}{c|}{23.88} \\
    $t=90$ & \multicolumn{1}{c}{53.52} & \multicolumn{1}{c|}{60.16} & \multicolumn{1}{c}{25.03} & \multicolumn{1}{c|}{43.66} & \multicolumn{1}{c}{41.19} & \multicolumn{1}{c|}{47.75} & \multicolumn{1}{c}{16.55} & \multicolumn{1}{c|}{23.54} \\ 
    \bottomrule
  \end{tabular}
\end{table}

\begin{figure}[t]
    \begin{subfigure}[c]{0.185\textwidth}
        \centering\subfloat{\centering\makebox[10pt][l]{\raisebox{5pt}{\rotatebox[origin=l]{90}{\tiny Tree-Ring}}}
        \includegraphics[width=\figWidth\linewidth]{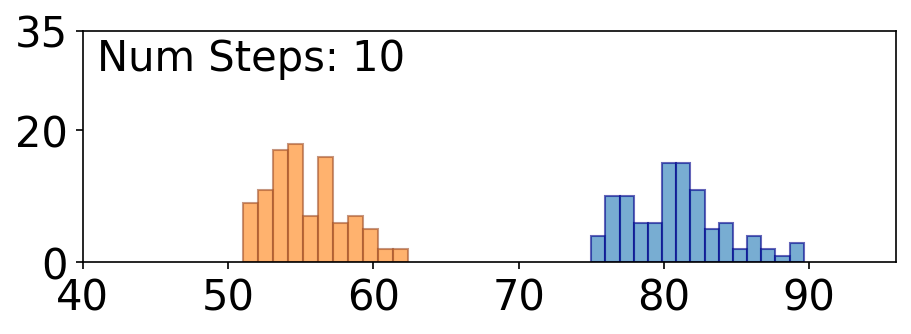}
        \includegraphics[width=\figWidth\linewidth]{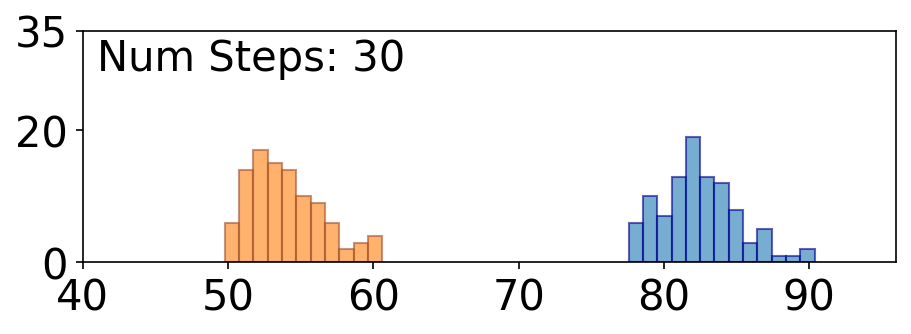}
        \includegraphics[width=\figWidth\linewidth]{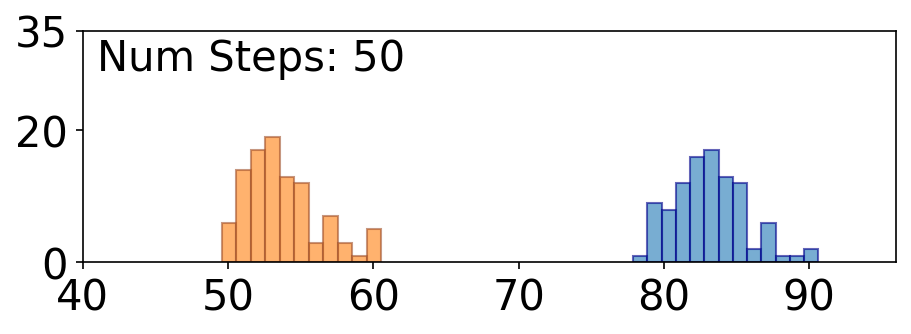}
        \includegraphics[width=\figWidth\linewidth]{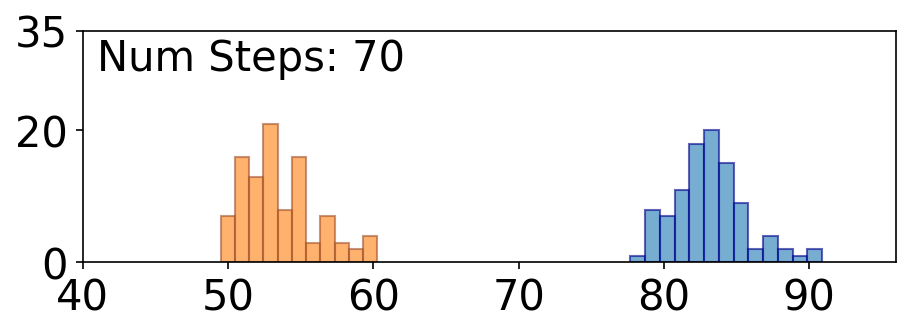}
        \includegraphics[width=\figWidth\linewidth]{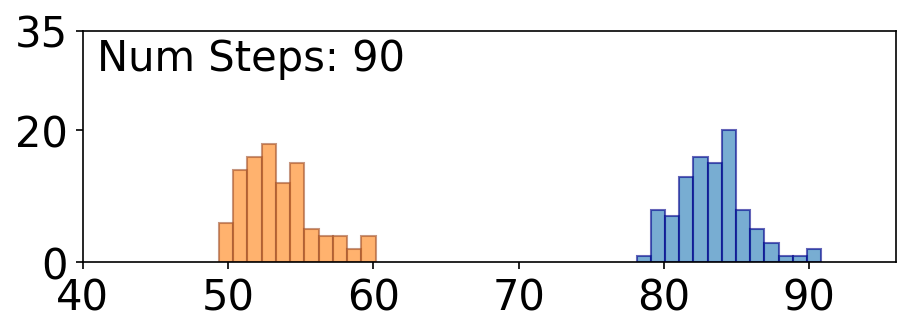}}

        \subfloat{\centering\makebox[10pt][l]{\raisebox{5pt}{\rotatebox[origin=l]{90}{\tiny RingID}}}
        \includegraphics[width=\figWidth\linewidth]{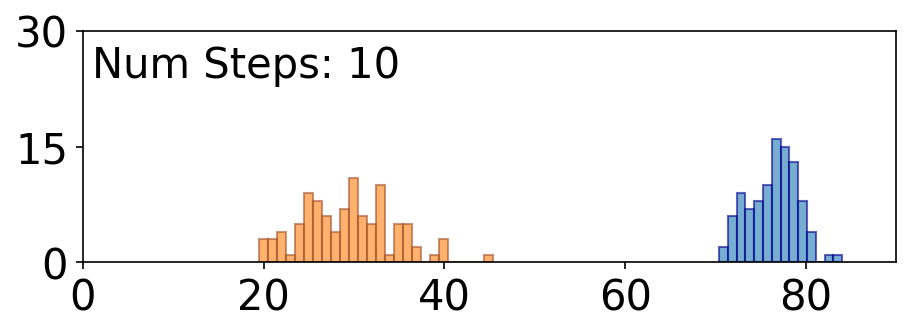}
        \includegraphics[width=\figWidth\linewidth]{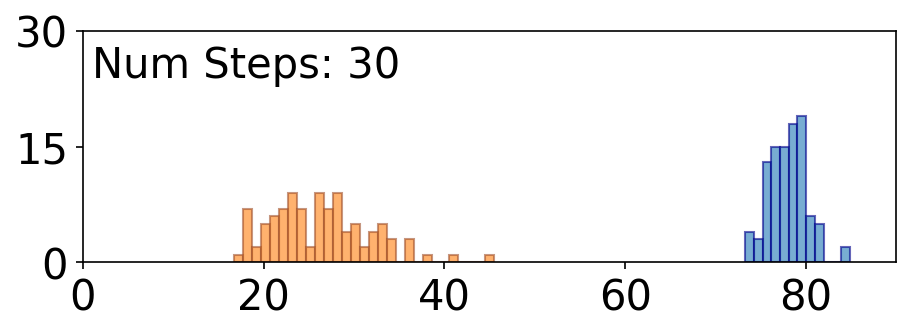}
        \includegraphics[width=\figWidth\linewidth]{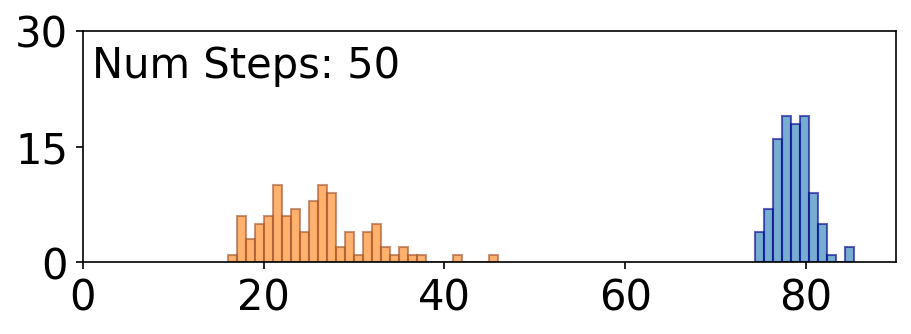}
        \includegraphics[width=\figWidth\linewidth]{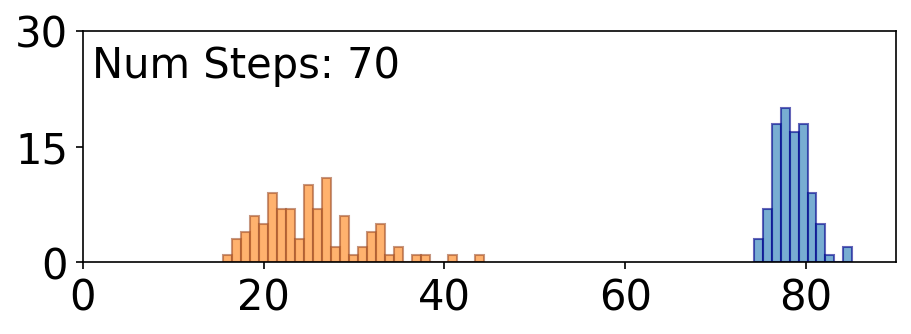}
        \includegraphics[width=\figWidth\linewidth]{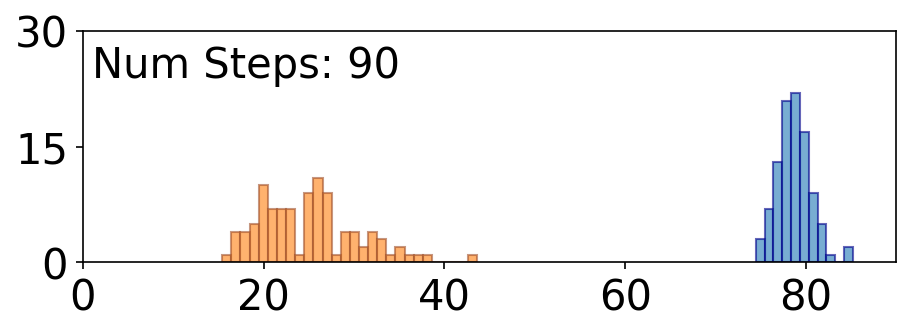}}

        \subfloat{\centering\makebox[10pt][l]{\raisebox{3pt}{\rotatebox[origin=l]{90}{\tiny HSQR}}}
        \includegraphics[width=\figWidth\linewidth]{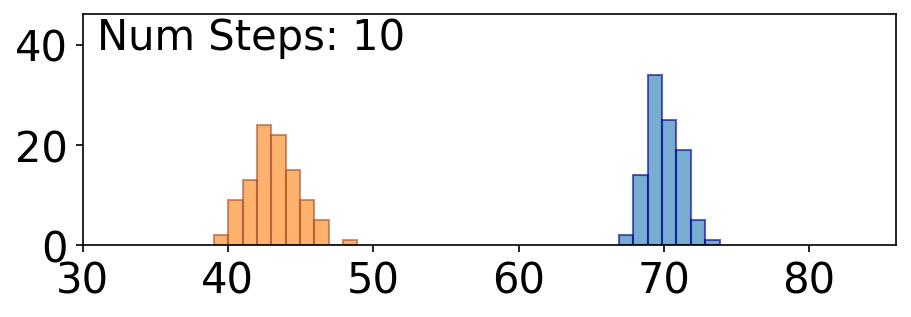}
        \includegraphics[width=\figWidth\linewidth]{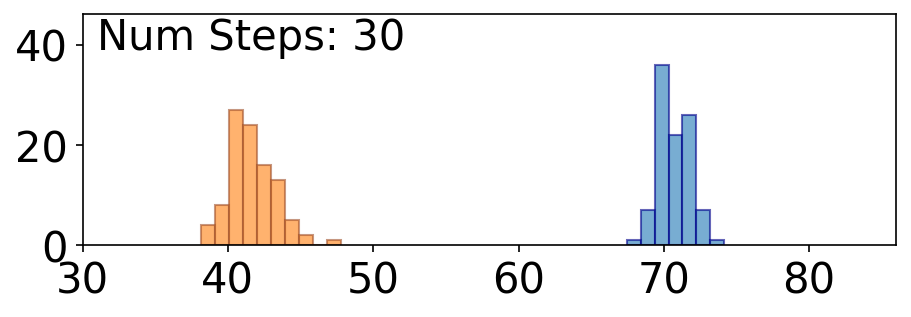}
        \includegraphics[width=\figWidth\linewidth]{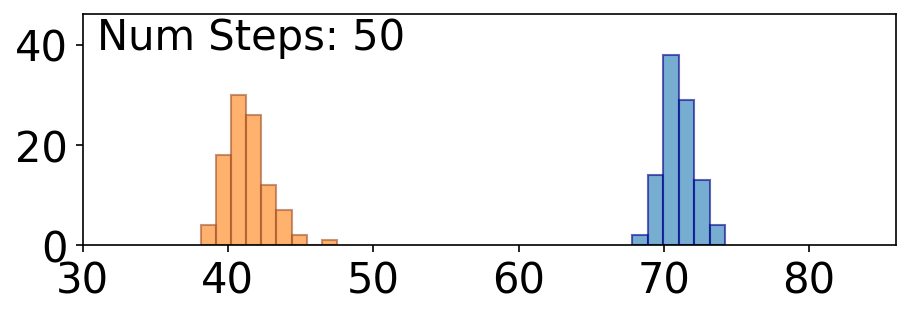}
        \includegraphics[width=\figWidth\linewidth]{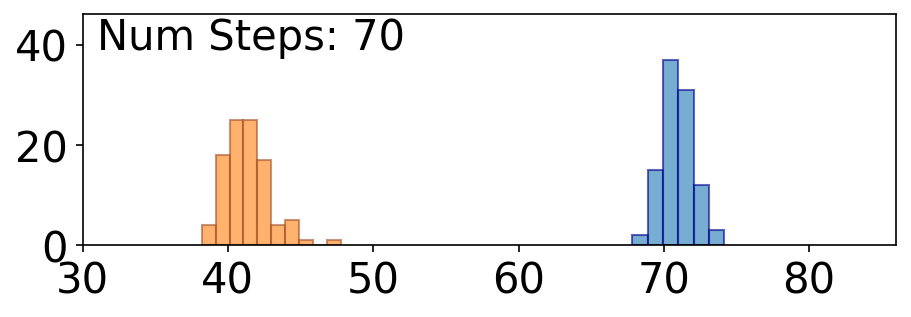}
        \includegraphics[width=\figWidth\linewidth]{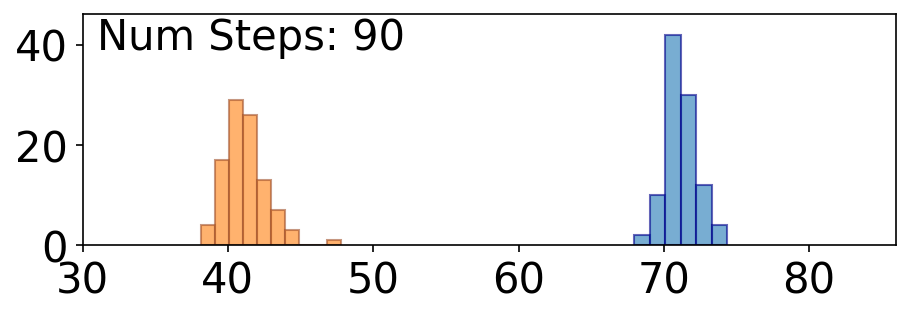}}

        \subfloat{\centering\makebox[10pt][l]{\raisebox{3pt}{\rotatebox[origin=l]{90}{\tiny HSTR}}}
        \includegraphics[width=\figWidth\linewidth]{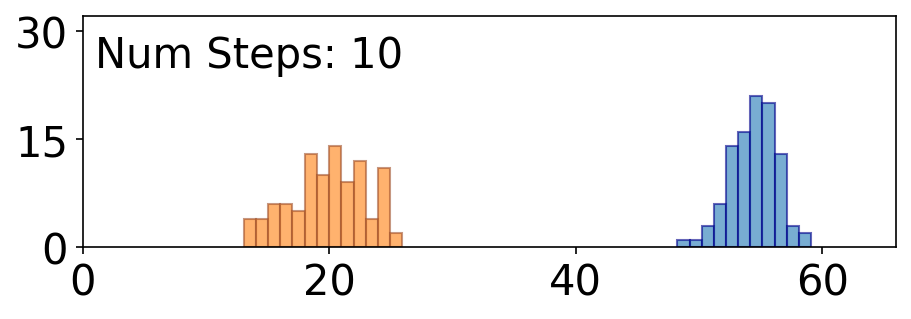}
        \includegraphics[width=\figWidth\linewidth]{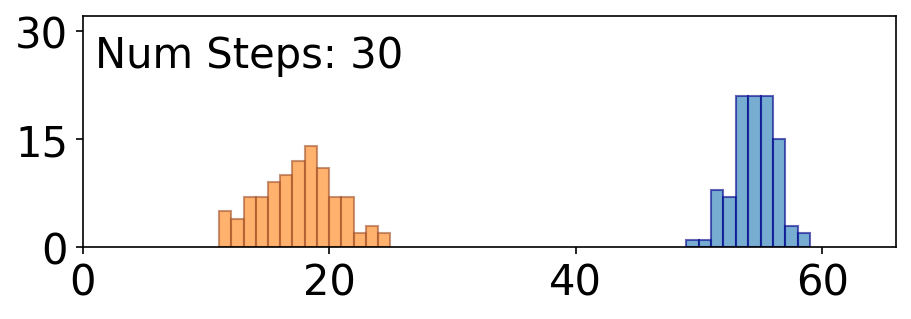}
        \includegraphics[width=\figWidth\linewidth]{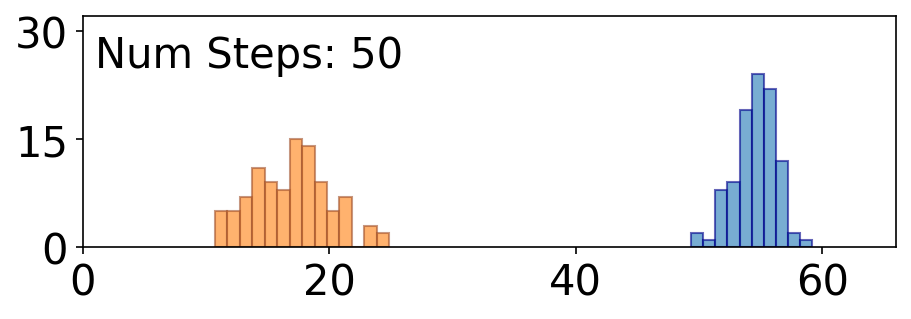}
        \includegraphics[width=\figWidth\linewidth]{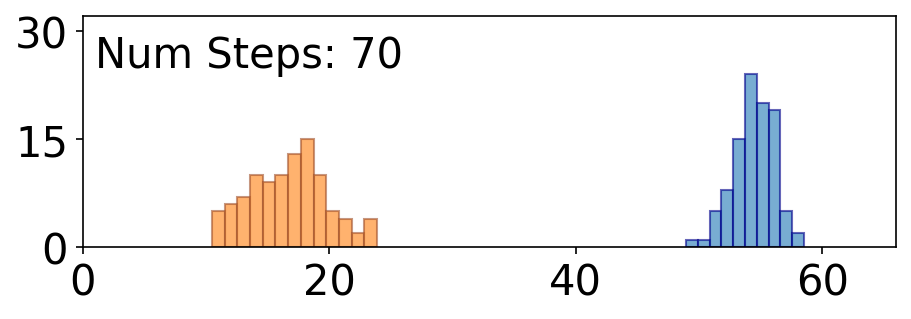}
        \includegraphics[width=\figWidth\linewidth]{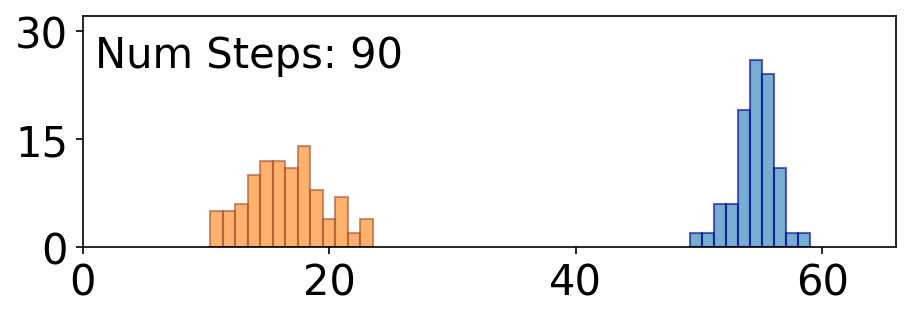}}
    \end{subfigure}
    \caption{Effects of different number of inference time steps on $\|\mathbf{w}-\hat{\mathbf{w}}^{\prime}\|$. From left to right, the number of inference time steps (Num Steps) were set to 10, 30, 50 (default), 70, and 90.}\label{fig: numerical error}
\end{figure}

\section{Effects of $\|\mathbf{w}\|$ on Detection Distance}\label{Sec: w norm increase on RID HSTR}
        We examined the effect of $\|\mathbf{w}\|$ on watermark detection, as discussed in Lemma~\ref{lem: max bound} and Sec.~\ref{dist shift}. We used RingID~\cite{ci2024ringid} and HSTR~\cite{lee2025semantic} as representative examples, with the corresponding results presented in Figs.~\ref{fig: RingID attack shifting subfig} and~\ref{fig: RingID 2 128 attack shifting subfig}, and Figs.~\ref{fig: HSTR 1 1 attack shifting} and~\ref{fig: HSTR 2 2 attack shifting} in Sec.~\ref{exps: ringid} of Appendix. According to Eq.~(\ref{max perturb bound}), we can expand the upper bound by increasing $\|\mathbf{w}\|$. The increase in $\|\mathbf{w}\|$ can be done by setting $\alpha$ in RingID~\cite{ci2024ringid} from $64$ to $128$, and multiplying the pattern injected in channel 0 by $2$. Comparing Fig.~\ref{fig: RingID attack shifting subfig} and Fig.~\ref{fig: RingID 2 128 attack shifting subfig}, we see that the distance between the orange and blue distributions increases in Fig.~\ref{fig: RingID 2 128 attack shifting subfig}. Such a gap can accommodate estimated watermarks that corresponding images are perturbed in the pixel space, and thereby create a buffer to withstand perturbations, especially when we use the threshold (red) of 1e-2 FPR. This modification can be done on HSTR. We also provide the experiments that we only set $\alpha$ to 128 without multiplying the pattern injected in channel 0 by $2$. See Fig.~\ref{fig: RingID 128 attack shifting} in Sec.~\ref{exps: ringid} of Appendix. 
        For more thorough attacks, we used Stirmark~\cite{petitcolas1998attacks} to examine the robustness of Tree-Ring and RingID in Sec.~\ref{stirmark exp} of Appendix.

        \begin{figure}[t]
            \begin{subfigure}[c]{0.185\textwidth}
                \centering\subfloat{\centering\makebox[10pt][l]{\raisebox{7pt}{\rotatebox[origin=l]{90}{\tiny Jitter}}}
                \includegraphics[width=\figWidth\linewidth]{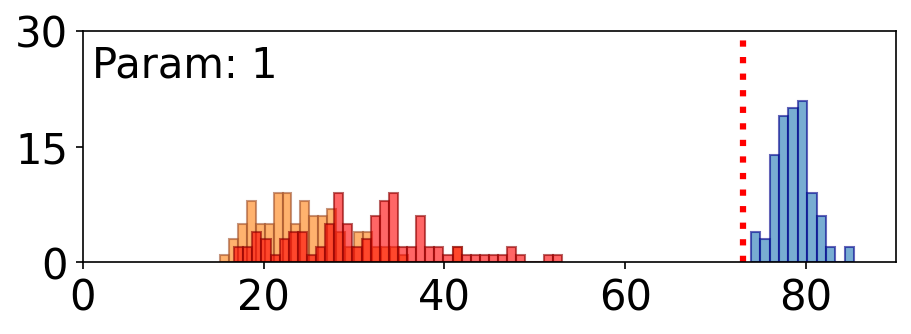}
                \includegraphics[width=\figWidth\linewidth]{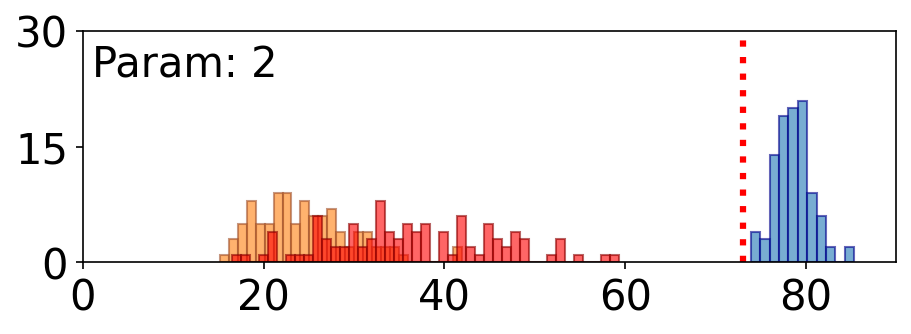}
                \includegraphics[width=\figWidth\linewidth]{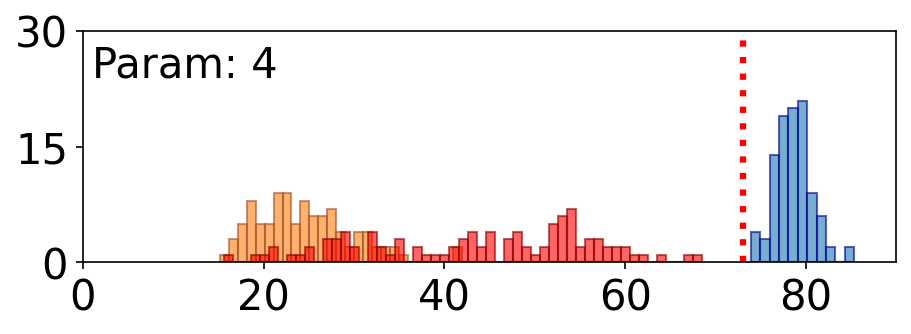}
                \includegraphics[width=\figWidth\linewidth]{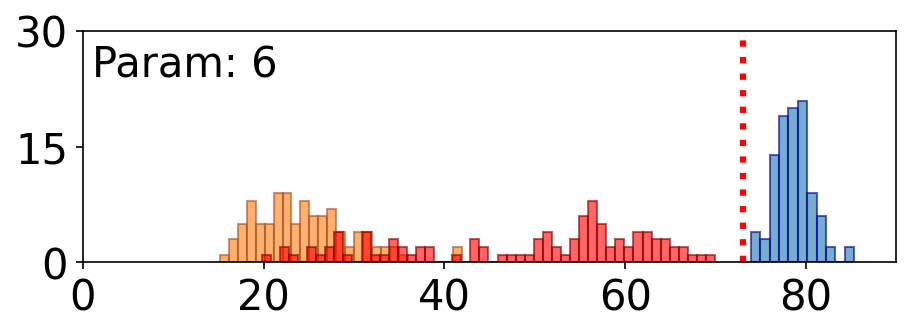}
                \includegraphics[width=\figWidth\linewidth]{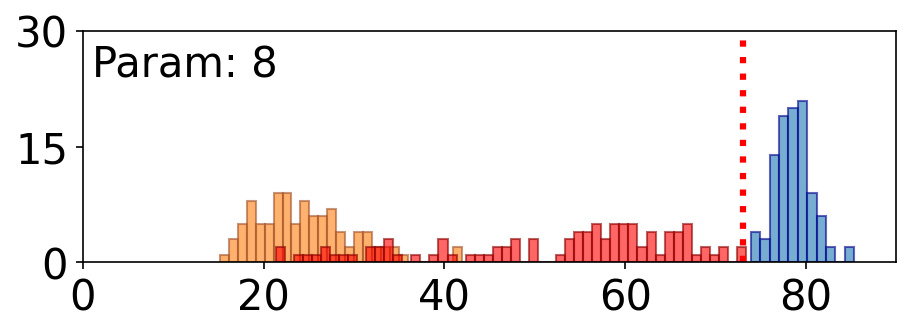}}
    
                \subfloat{\centering\makebox[10pt][l]{\raisebox{7pt}{\rotatebox[origin=l]{90}{\tiny JPEG}}}
                \includegraphics[width=\figWidth\linewidth]{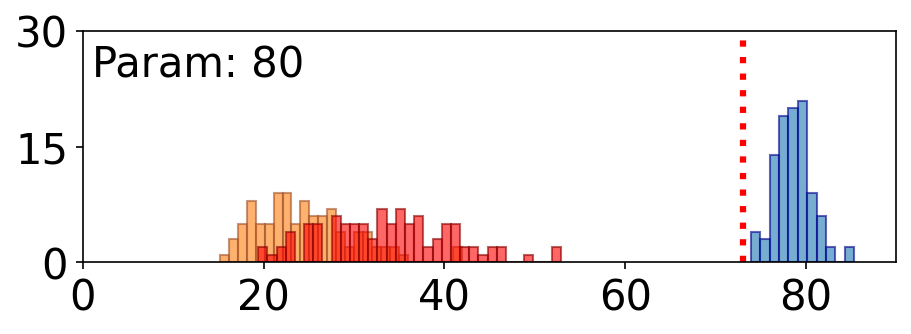}
                \includegraphics[width=\figWidth\linewidth]{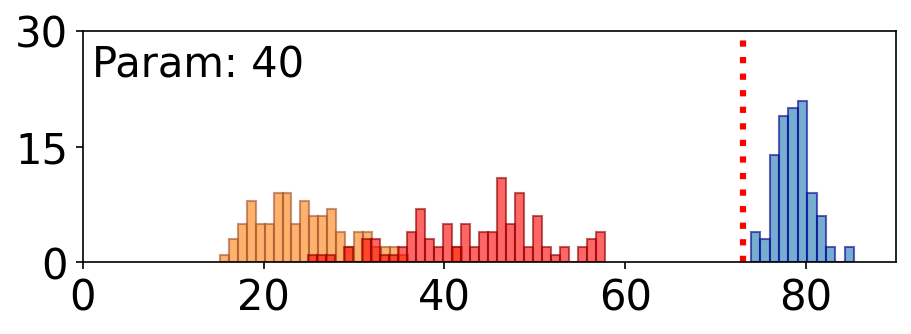}
                \includegraphics[width=\figWidth\linewidth]{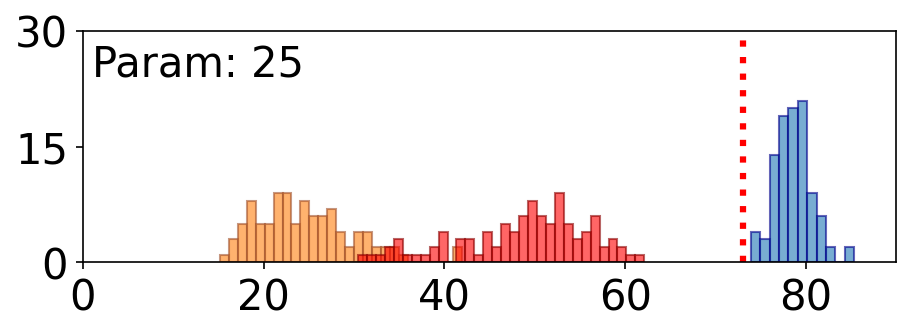}
                \includegraphics[width=\figWidth\linewidth]{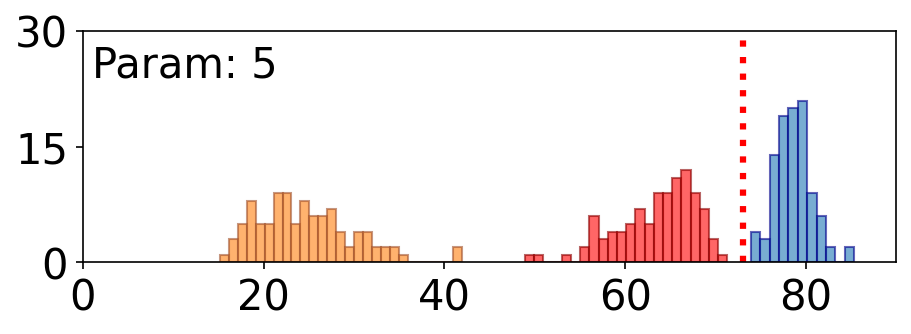}
                \includegraphics[width=\figWidth\linewidth]{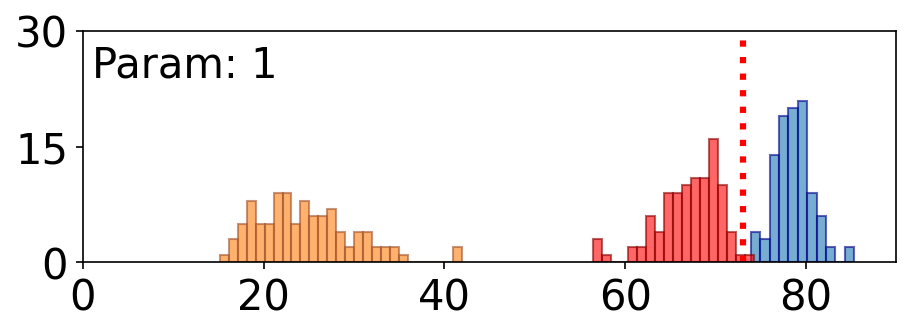}}
    
                \subfloat{\centering\makebox[10pt][l]{\raisebox{3pt}{\rotatebox[origin=l]{90}{\tiny Rotation}}}
                \includegraphics[width=\figWidth\linewidth]{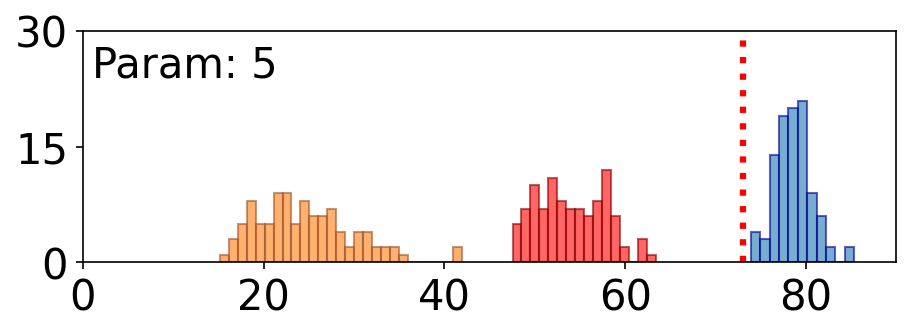}
                \includegraphics[width=\figWidth\linewidth]{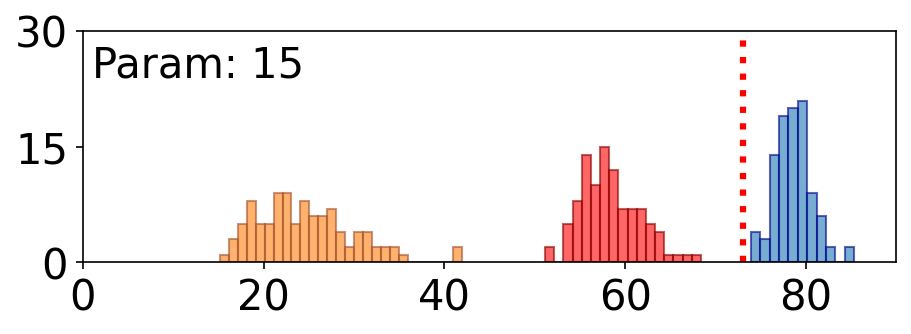}
                \includegraphics[width=\figWidth\linewidth]{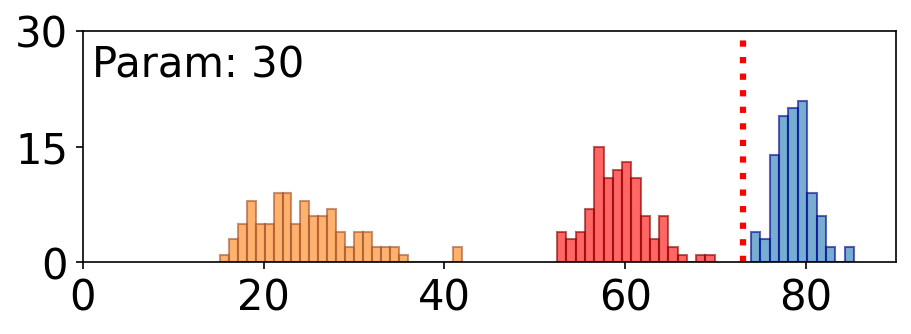}
                \includegraphics[width=\figWidth\linewidth]{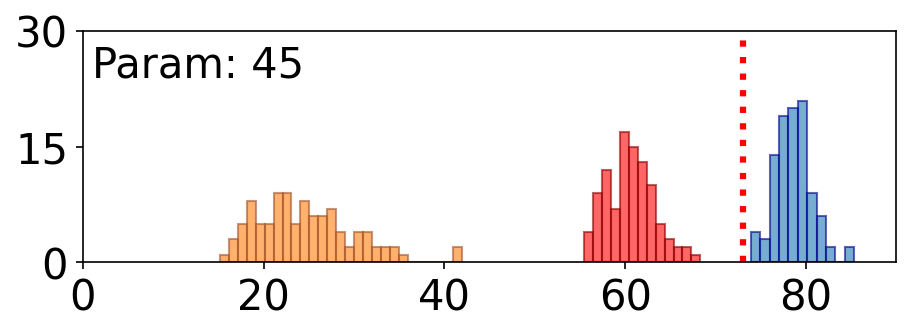}
                \includegraphics[width=\figWidth\linewidth]{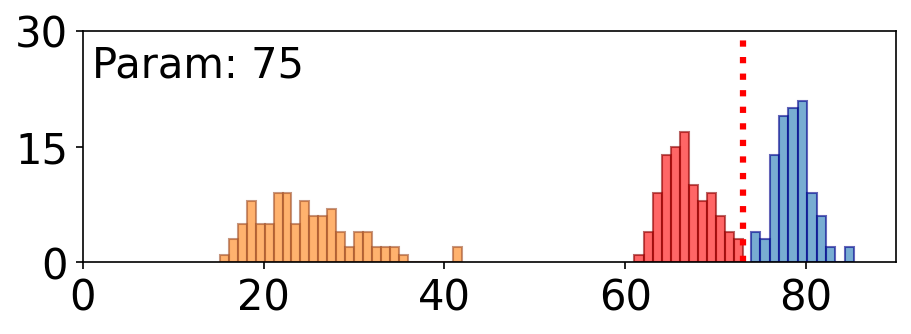}}
    
                \subfloat{\centering\makebox[10pt][l]{\raisebox{3pt}{\rotatebox[origin=l]{90}{\tiny Cropping}}}
                \includegraphics[width=\figWidth\linewidth]{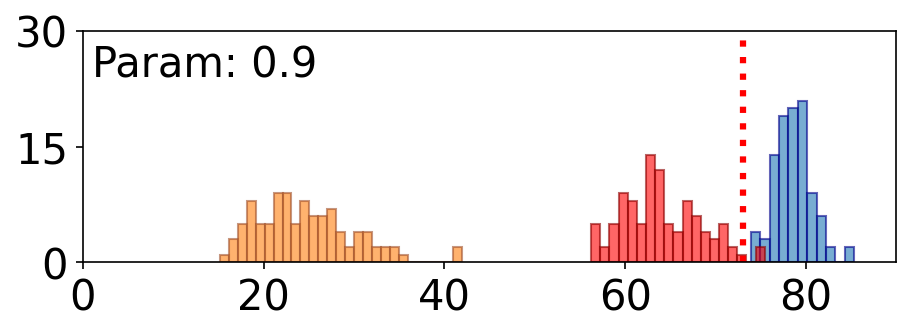}
                \includegraphics[width=\figWidth\linewidth]{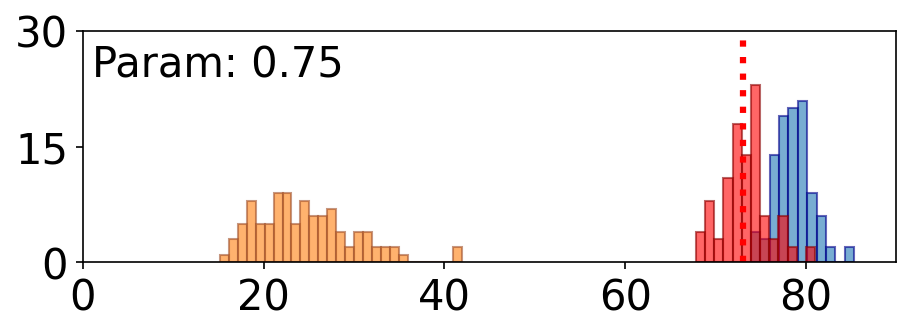}
                \includegraphics[width=\figWidth\linewidth]{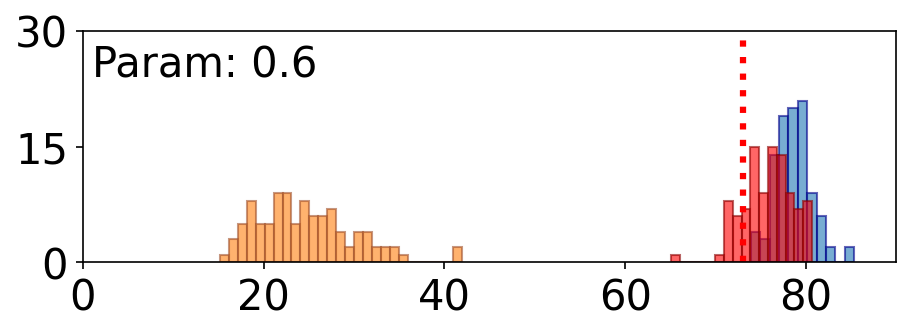}
                \includegraphics[width=\figWidth\linewidth]{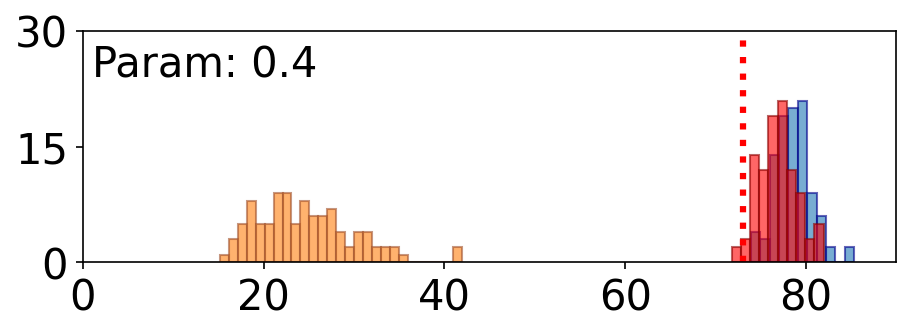}
                \includegraphics[width=\figWidth\linewidth]{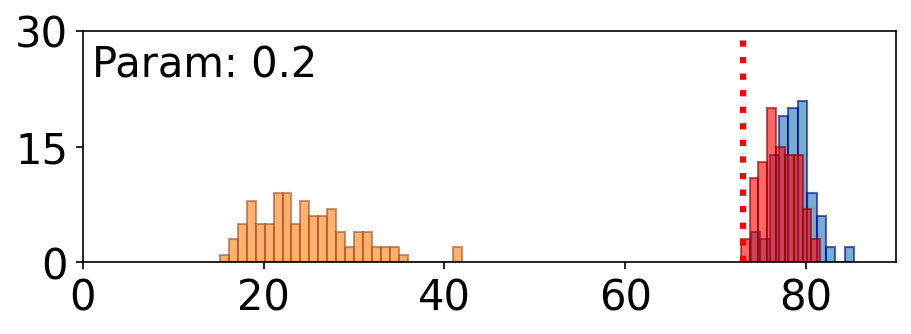}}

                \subfloat{\centering\makebox[10pt][l]{\raisebox{0pt}{\rotatebox[origin=l]{90}{\tiny Regen-VAE}}}
                \includegraphics[width=\figWidth\linewidth]{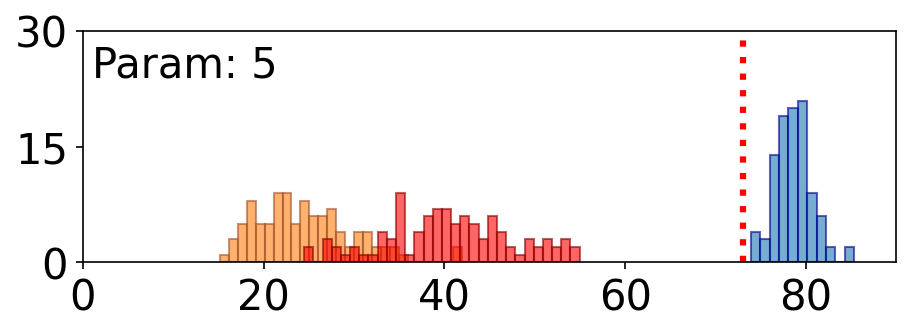}
                \includegraphics[width=\figWidth\linewidth]{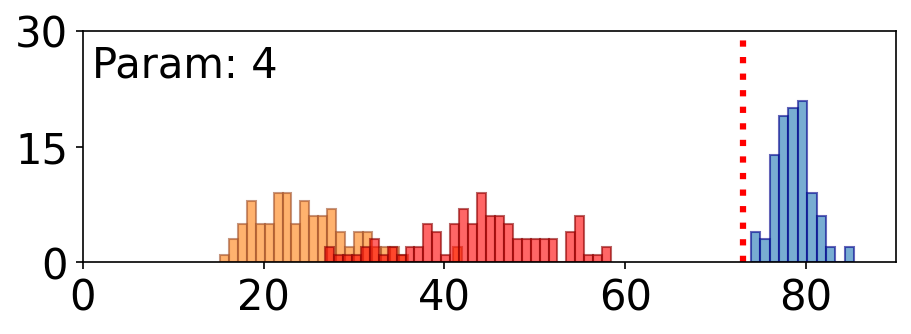}
                \includegraphics[width=\figWidth\linewidth]{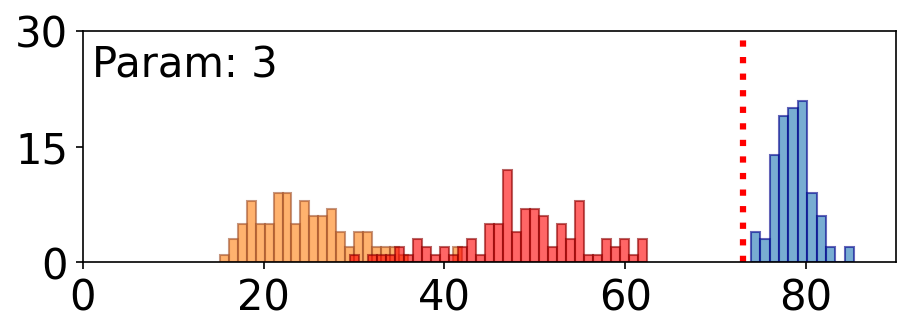}
                \includegraphics[width=\figWidth\linewidth]{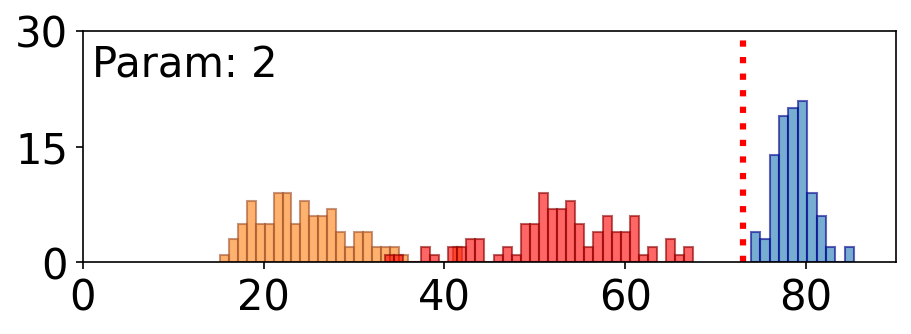}
                \includegraphics[width=\figWidth\linewidth]{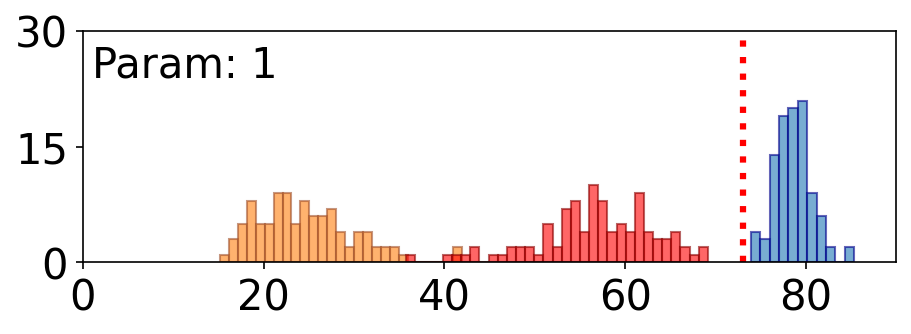}}
            \end{subfigure}
            \caption{Effects of different perturbations on $\|\mathbf{w}-\hat{\mathbf{w}}^{\prime}\|$. Watermark: RingID. Thresholds: 73.0 (1e-2 FPR).}\label{fig: RingID attack shifting subfig}
        \end{figure}

        \begin{figure}[t]
            \begin{subfigure}[c]{0.185\textwidth}
                \centering\subfloat{\centering\makebox[10pt][l]{\raisebox{7pt}{\rotatebox[origin=l]{90}{\tiny Jitter}}}
                \includegraphics[width=\figWidth\linewidth]{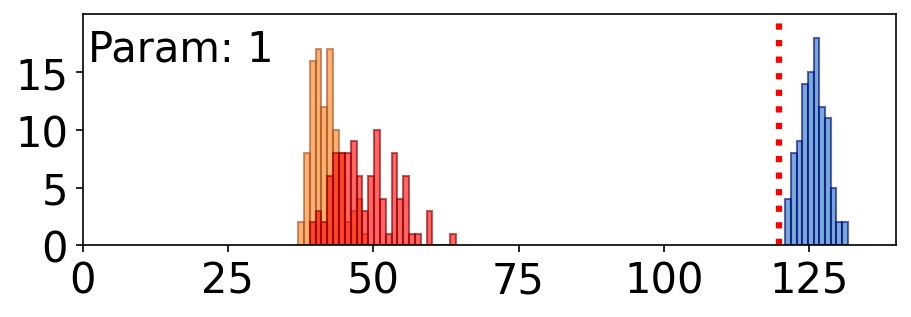}
                \includegraphics[width=\figWidth\linewidth]{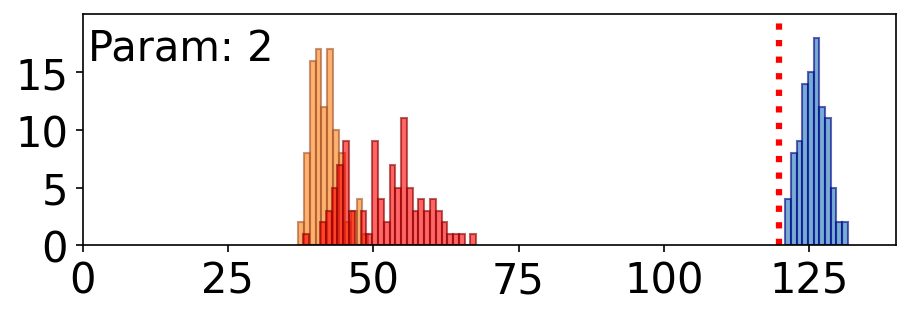}
                \includegraphics[width=\figWidth\linewidth]{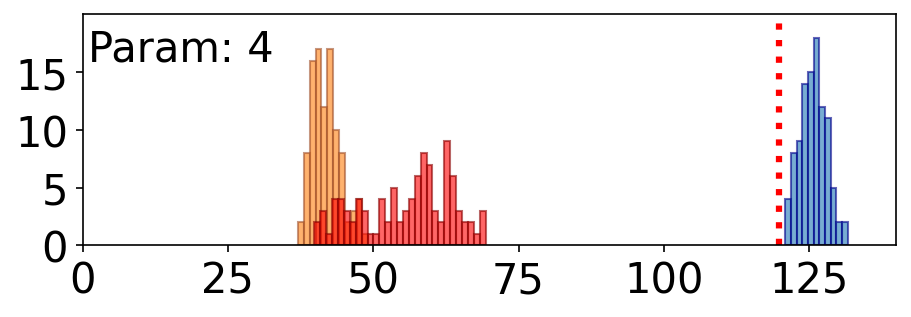}
                \includegraphics[width=\figWidth\linewidth]{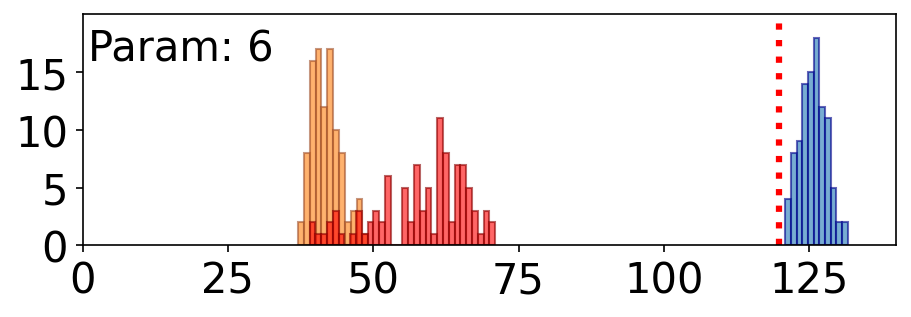}
                \includegraphics[width=\figWidth\linewidth]{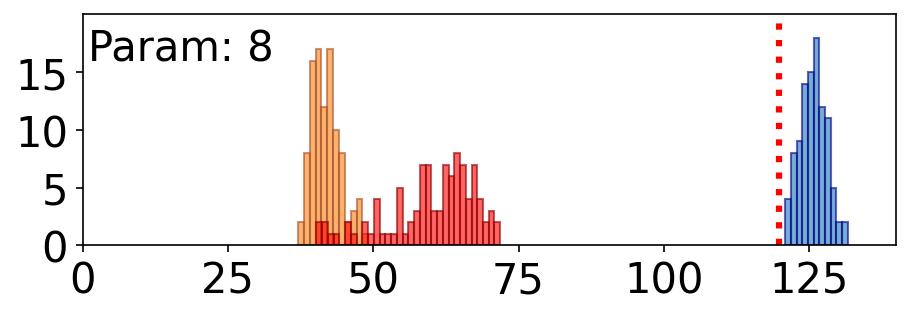}}
    
                \subfloat{\centering\makebox[10pt][l]{\raisebox{7pt}{\rotatebox[origin=l]{90}{\tiny JPEG}}}
                \includegraphics[width=\figWidth\linewidth]{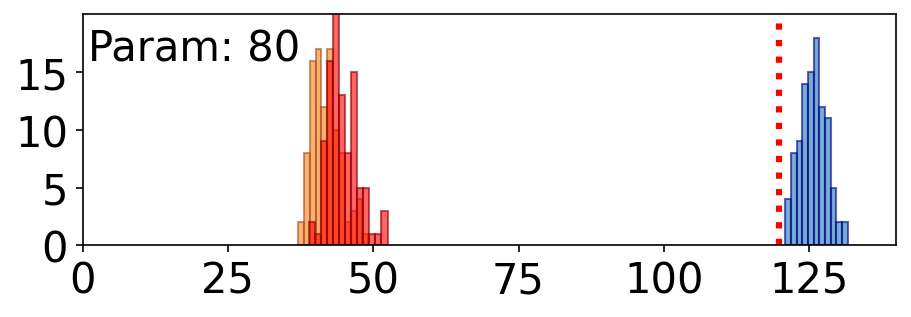}
                \includegraphics[width=\figWidth\linewidth]{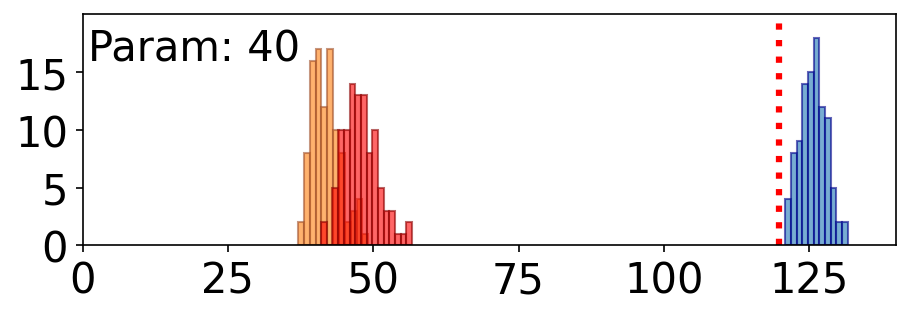}
                \includegraphics[width=\figWidth\linewidth]{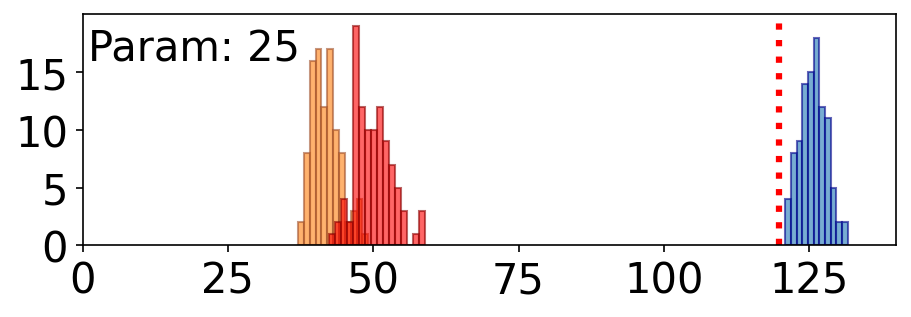}
                \includegraphics[width=\figWidth\linewidth]{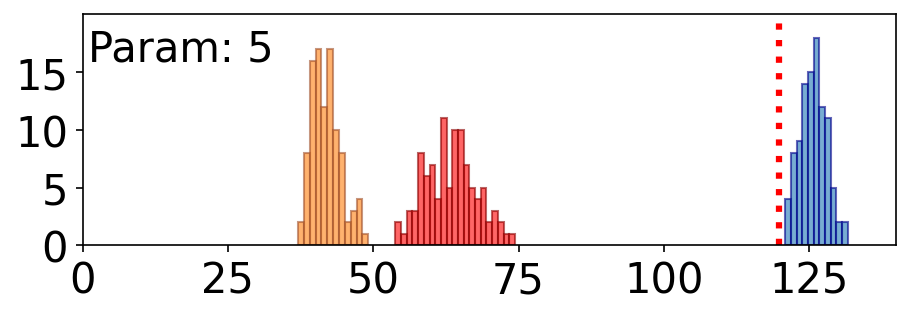}
                \includegraphics[width=\figWidth\linewidth]{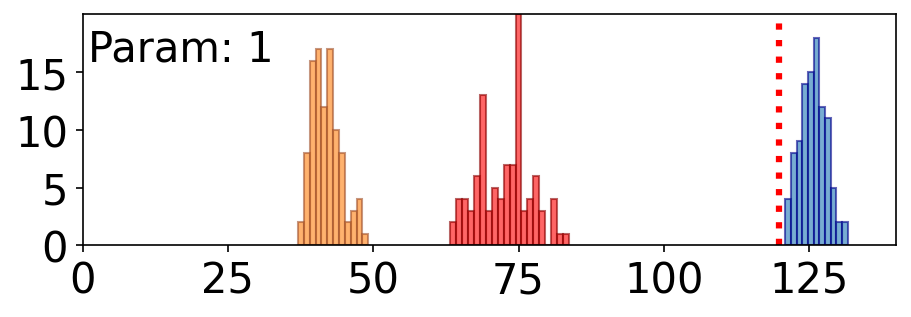}}
    
                \subfloat{\centering\makebox[10pt][l]{\raisebox{3pt}{\rotatebox[origin=l]{90}{\tiny Rotation}}}
                \includegraphics[width=\figWidth\linewidth]{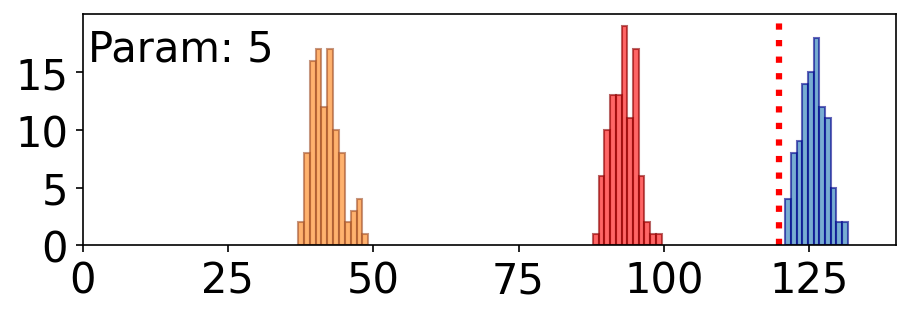}
                \includegraphics[width=\figWidth\linewidth]{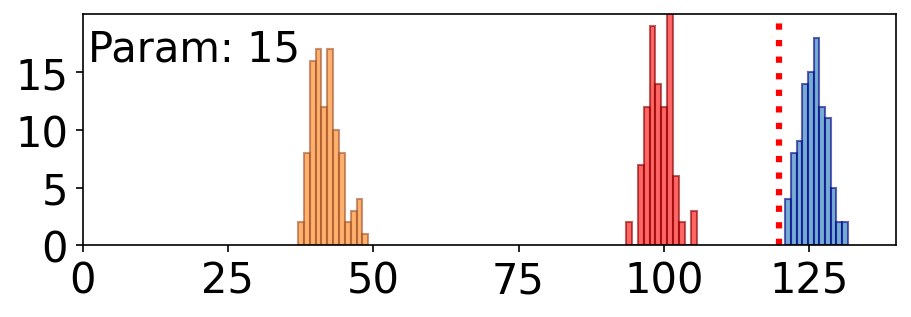}
                \includegraphics[width=\figWidth\linewidth]{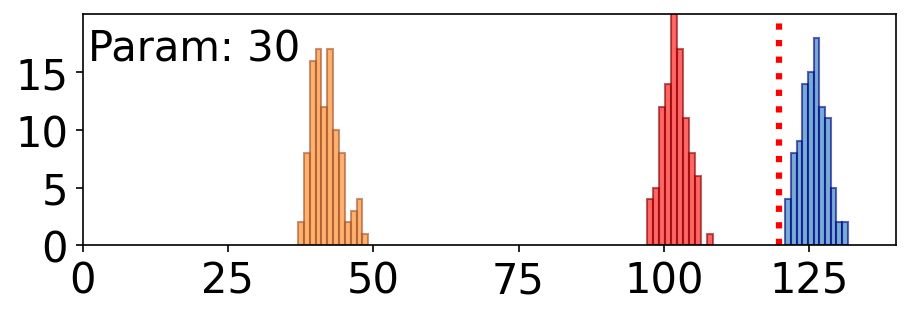}
                \includegraphics[width=\figWidth\linewidth]{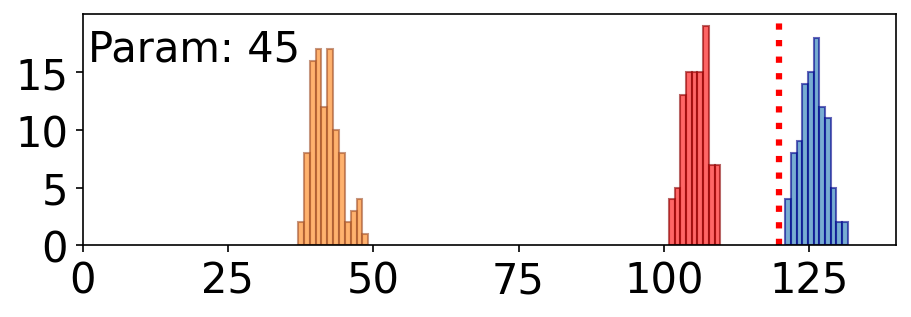}
                \includegraphics[width=\figWidth\linewidth]{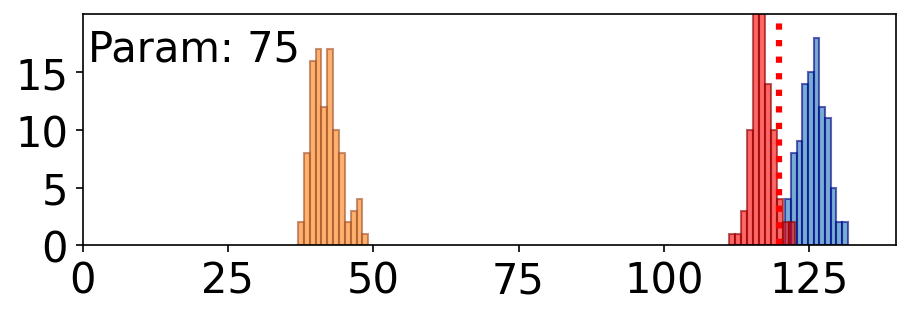}}
    
                \subfloat{\centering\makebox[10pt][l]{\raisebox{3pt}{\rotatebox[origin=l]{90}{\tiny Cropping}}}
                \includegraphics[width=\figWidth\linewidth]{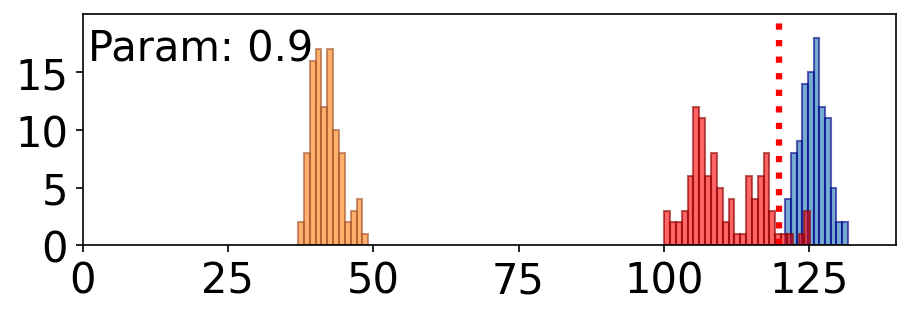}
                \includegraphics[width=\figWidth\linewidth]{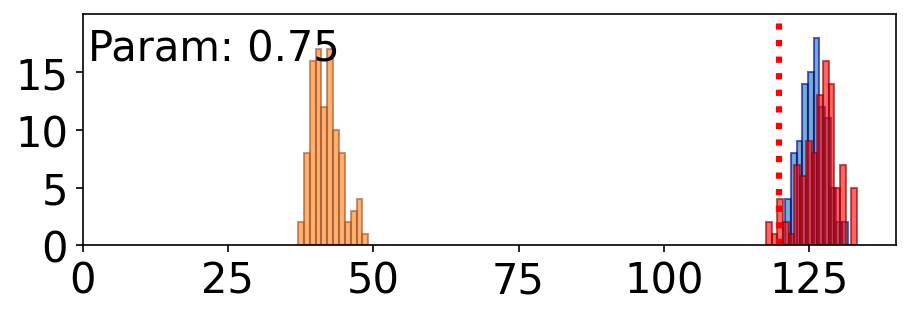}
                \includegraphics[width=\figWidth\linewidth]{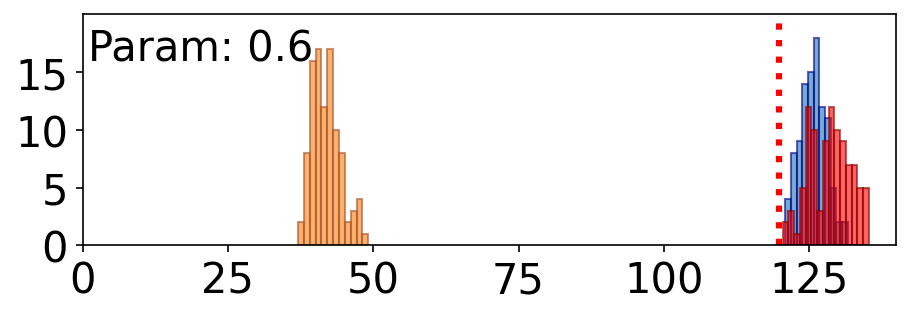}
                \includegraphics[width=\figWidth\linewidth]{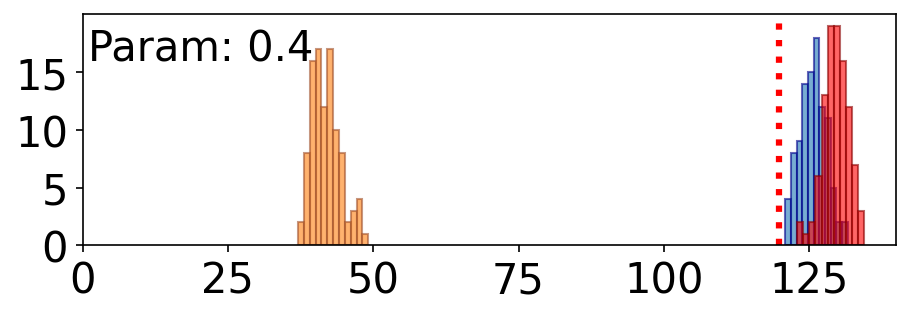}
                \includegraphics[width=\figWidth\linewidth]{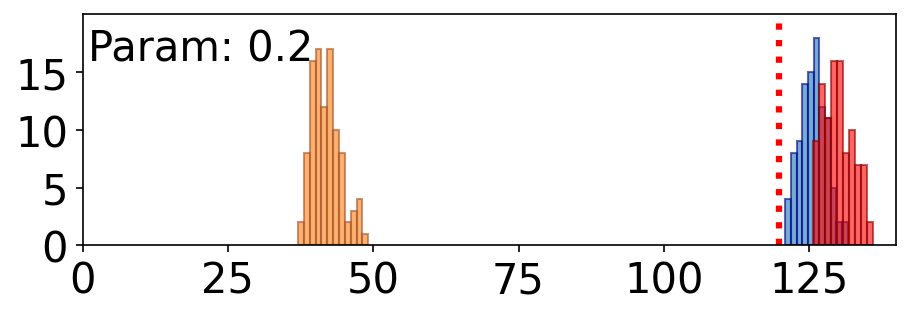}}

                \subfloat{\centering\makebox[10pt][l]{\raisebox{0pt}{\rotatebox[origin=l]{90}{\tiny Regen-VAE}}}
                \includegraphics[width=\figWidth\linewidth]{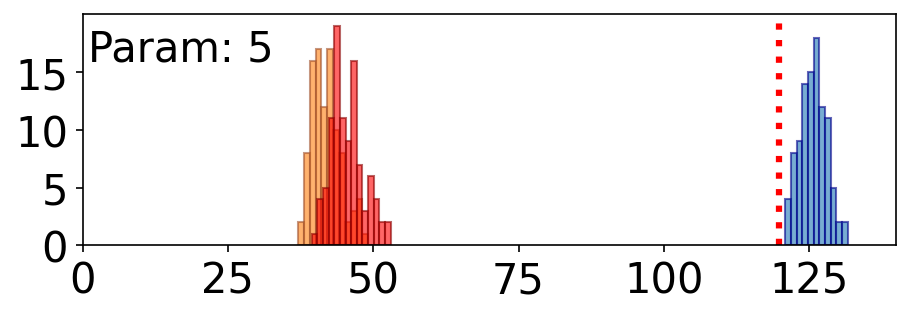}
                \includegraphics[width=\figWidth\linewidth]{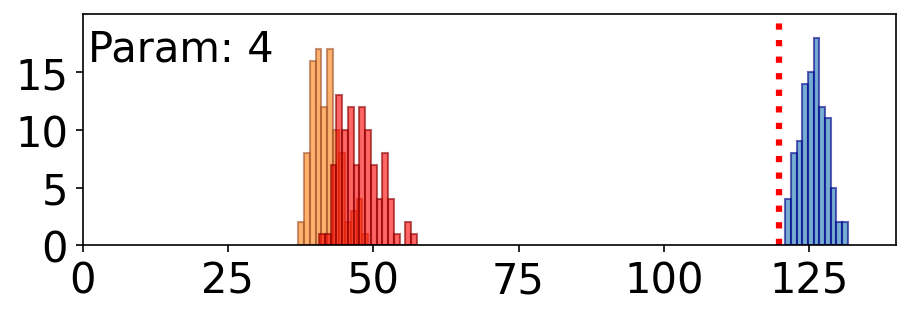}
                \includegraphics[width=\figWidth\linewidth]{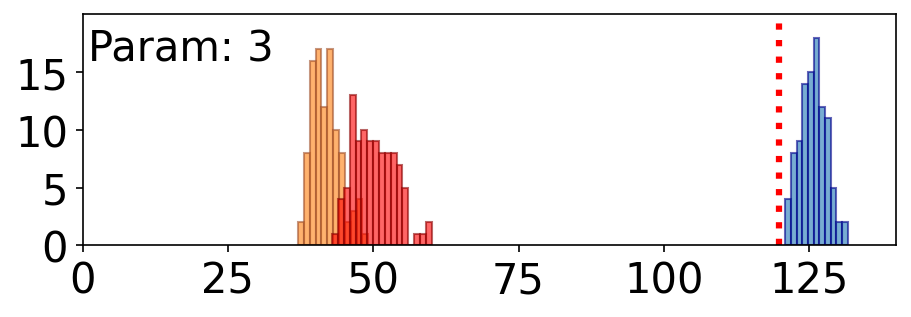}
                \includegraphics[width=\figWidth\linewidth]{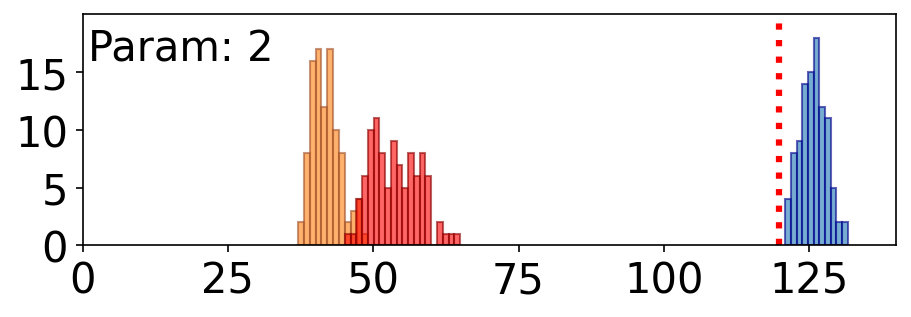}
                \includegraphics[width=\figWidth\linewidth]{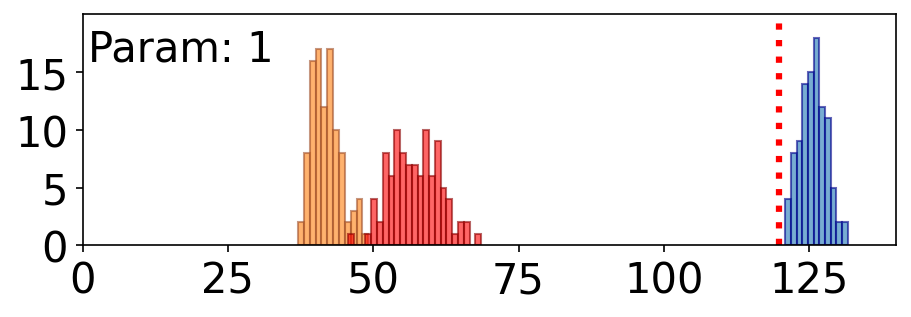}}
            \end{subfigure}
            \caption{Effects of different perturbations on $\|\mathbf{w}-\hat{\mathbf{w}}^{\prime}\|$. Watermark: RingID with channel 0 pattern multiplied by $2$ and $\alpha=128$. Thresholds: 119.85 (1e-2 FPR).}\label{fig: RingID 2 128 attack shifting subfig}
        \end{figure}

\subsection{Complete experiments of the effect of $\|\mathbf{w}\|$ on Detection Distance}\label{exps: ringid}
In this section, we provide experiments on the effects of increasing $\|\mathbf{w}\|$ on robustness of watermark, which we consider RingID and HSTR as representative watermarking schemes. 
For notational simplicity, RingID-$m$-$\alpha$ denotes a configuration where the pattern on channel 0 is scaled by a factor of $m$ (default: $m=1$) and the pattern on channel 3 is governed by the parameter $\alpha$ (default: $\alpha=64$). Similarly, HSTR-$m$-$m$ indicates that the patterns on both channels 0 and 3 are multiplied by $m$ (default: $m=1$).
Please refer to Figs.~\ref{fig: RingID attack shifting},~\ref{fig: RingID 2 128 attack shifting},~\ref{fig: RingID 128 attack shifting},~\ref{fig: HSTR 1 1 attack shifting}, and~\ref{fig: HSTR 2 2 attack shifting}. For FID score results, please refer to Table~\ref{tab: FID}.

\begin{table}[t]
\centering
\small

\begin{tabular}{@{}cccccc@{}}

\toprule


WM Methods &  FID ($\dto$)  \\ \midrule

RingID & 45.31 \\


RingID-2-128 & 428.68 \\

HSTR & 45.32  \\

HSTR-2-2 & 187.77 \\ \bottomrule

\end{tabular}%

\caption{The FID score of watermark methods with different magnitude $\|\mathbf{w}\|$.}
    \label{tab: FID}

\end{table}

\begin{figure*}[t]
        \begin{subfigure}[c]{0.185\textwidth}
            \centering
            \subfloat{\centering
            \makebox[10pt][l]{\raisebox{5pt}{\rotatebox[origin=l]{90}{\tiny Blurring}}}
            \includegraphics[width=1.03\linewidth]{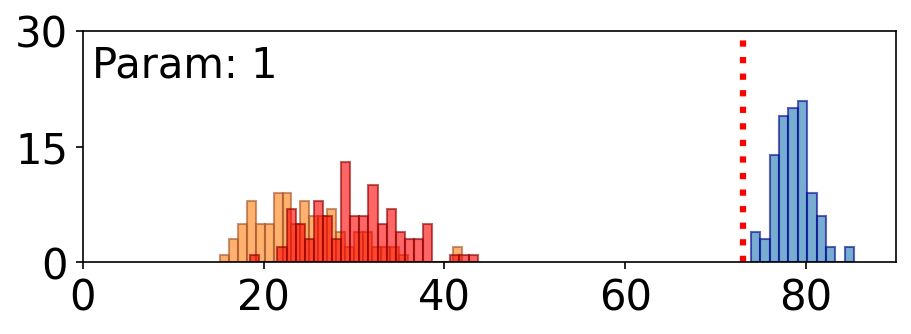}
            \includegraphics[width=1.03\linewidth]{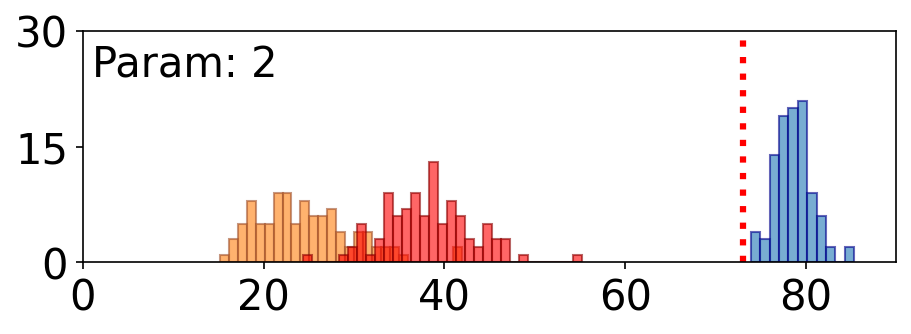}
            \includegraphics[width=1.03\linewidth]{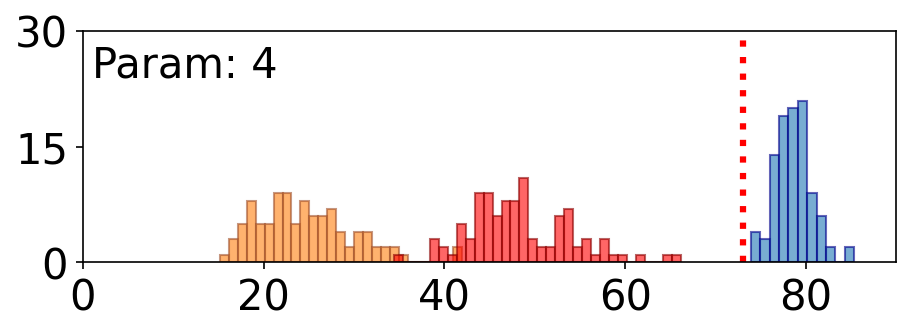}
            \includegraphics[width=1.03\linewidth]{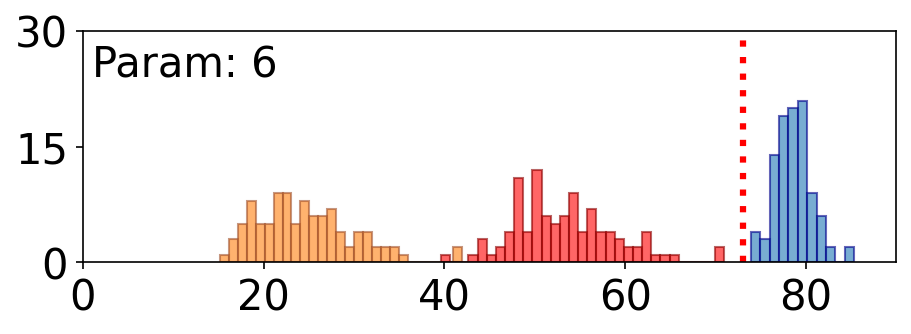}
            \includegraphics[width=1.03\linewidth]{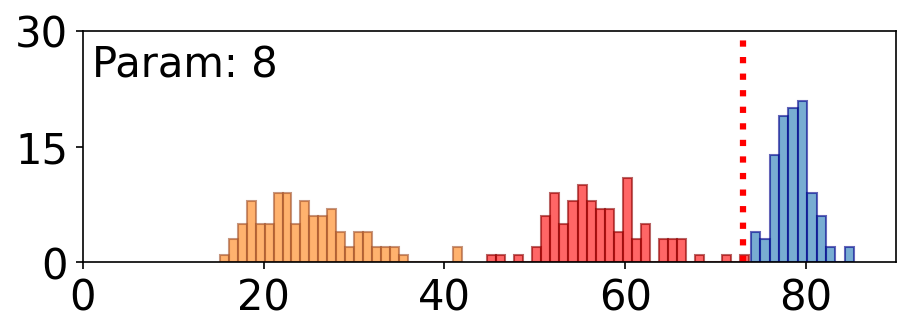}}
    
            \subfloat{\centering
            \makebox[10pt][l]{\raisebox{7pt}{\rotatebox[origin=l]{90}{\tiny Jitter}}}
            \includegraphics[width=1.03\linewidth]{figs/figs/ringid_64/hist_dist_ringid_64_color_jitter_1.png}
            \includegraphics[width=1.03\linewidth]{figs/figs/ringid_64/hist_dist_ringid_64_color_jitter_2.png}
            \includegraphics[width=1.03\linewidth]{figs/figs/ringid_64/hist_dist_ringid_64_color_jitter_4.png}
            \includegraphics[width=1.03\linewidth]{figs/figs/ringid_64/hist_dist_ringid_64_color_jitter_6.png}
            \includegraphics[width=1.03\linewidth]{figs/figs/ringid_64/hist_dist_ringid_64_color_jitter_8.png}}
    
            \subfloat{\centering
            \makebox[10pt][l]{\raisebox{7pt}{\rotatebox[origin=l]{90}{\tiny JPEG}}}
            \includegraphics[width=1.03\linewidth]{figs/figs/ringid_64/hist_dist_ringid_64_jpeg_80.png}
            \includegraphics[width=1.03\linewidth]{figs/figs/ringid_64/hist_dist_ringid_64_jpeg_40.png}
            \includegraphics[width=1.03\linewidth]{figs/figs/ringid_64/hist_dist_ringid_64_jpeg_25.png}
            \includegraphics[width=1.03\linewidth]{figs/figs/ringid_64/hist_dist_ringid_64_jpeg_5.png}
            \includegraphics[width=1.03\linewidth]{figs/figs/ringid_64/hist_dist_ringid_64_jpeg_1.png}}
    
            \subfloat{\centering
            \makebox[10pt][l]{\raisebox{3pt}{\rotatebox[origin=l]{90}{\tiny Rotation}}}
            \includegraphics[width=1.03\linewidth]{figs/figs/ringid_64/hist_dist_ringid_64_rotation_5.png}
            \includegraphics[width=1.03\linewidth]{figs/figs/ringid_64/hist_dist_ringid_64_rotation_15.png}
            \includegraphics[width=1.03\linewidth]{figs/figs/ringid_64/hist_dist_ringid_64_rotation_30.png}
            \includegraphics[width=1.03\linewidth]{figs/figs/ringid_64/hist_dist_ringid_64_rotation_45.png}
            \includegraphics[width=1.03\linewidth]{figs/figs/ringid_64/hist_dist_ringid_64_rotation_75.png}}
    
            \subfloat{\centering
            \makebox[10pt][l]{\raisebox{3pt}{\rotatebox[origin=l]{90}{\tiny Cropping}}}
            \includegraphics[width=1.03\linewidth]{figs/figs/ringid_64/hist_dist_ringid_64_cropping_0.9.png}
            \includegraphics[width=1.03\linewidth]{figs/figs/ringid_64/hist_dist_ringid_64_cropping_0.75.png}
            \includegraphics[width=1.03\linewidth]{figs/figs/ringid_64/hist_dist_ringid_64_cropping_0.6.png}
            \includegraphics[width=1.03\linewidth]{figs/figs/ringid_64/hist_dist_ringid_64_cropping_0.4.png}
            \includegraphics[width=1.03\linewidth]{figs/figs/ringid_64/hist_dist_ringid_64_cropping_0.2.png}}
    
            \subfloat{\centering
            \makebox[10pt][l]{\raisebox{3pt}{\rotatebox[origin=l]{90}{\tiny Elastic}}}
            \includegraphics[width=1.03\linewidth]{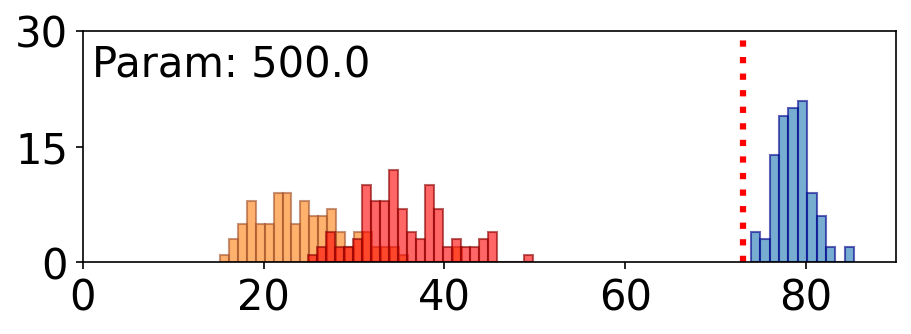}
            \includegraphics[width=1.03\linewidth]{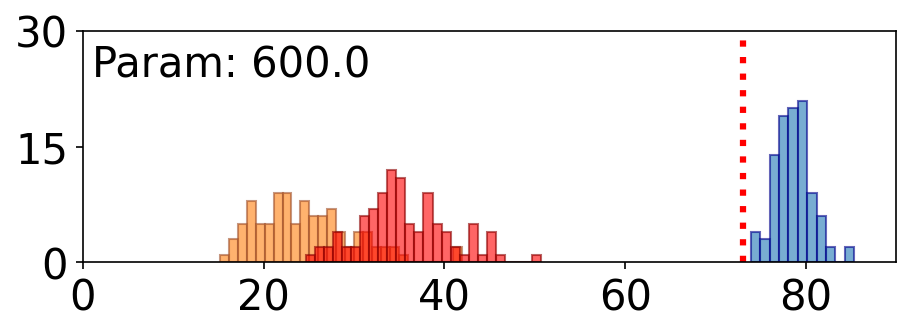}
            \includegraphics[width=1.03\linewidth]{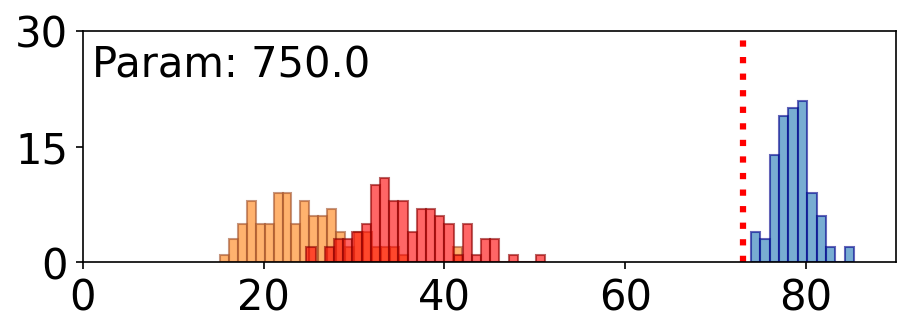}
            \includegraphics[width=1.03\linewidth]{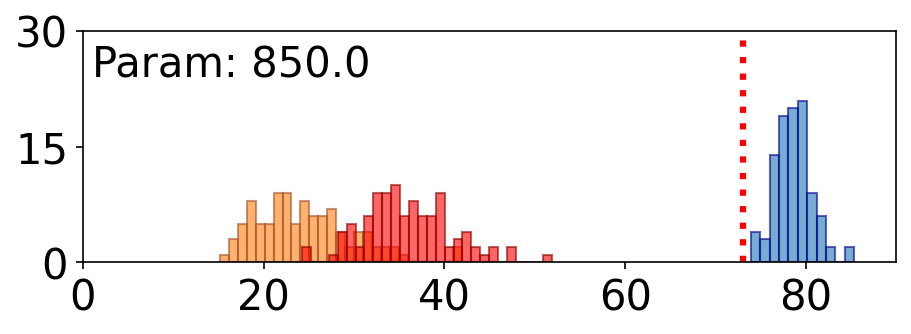}
            \includegraphics[width=1.03\linewidth]{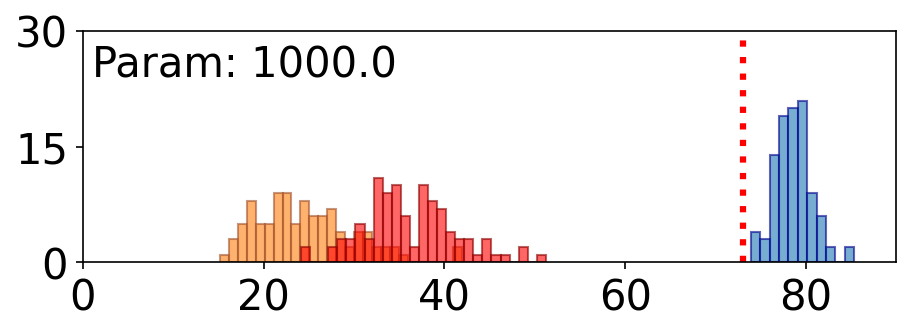}}
    
            \subfloat{\centering
            \makebox[10pt][l]{\raisebox{3pt}{\rotatebox[origin=l]{90}{\scalebox{.4}{Perspective}}}}
            \includegraphics[width=1.03\linewidth]{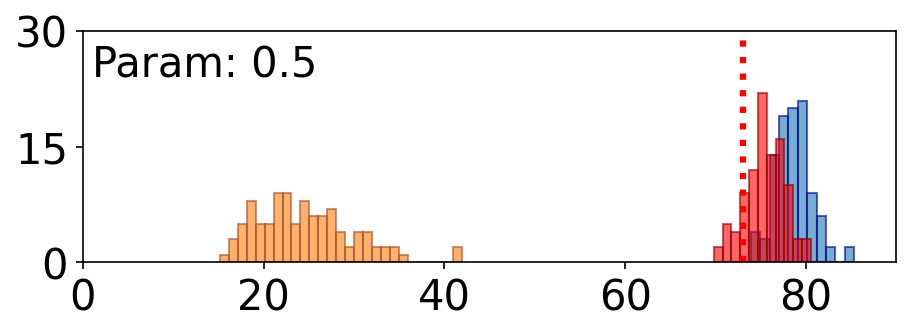}
            \includegraphics[width=1.03\linewidth]{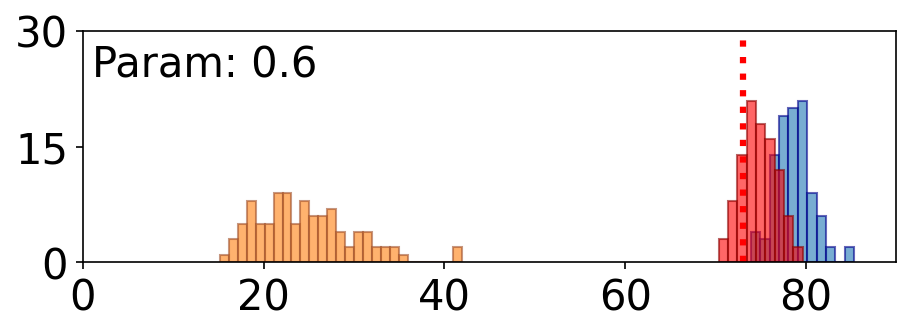}
            \includegraphics[width=1.03\linewidth]{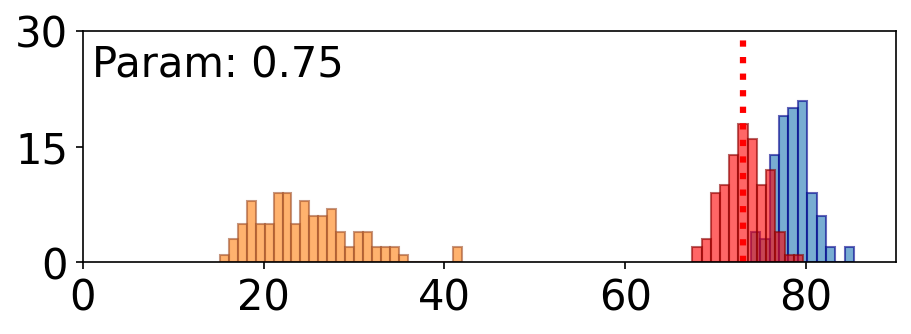}
            \includegraphics[width=1.03\linewidth]{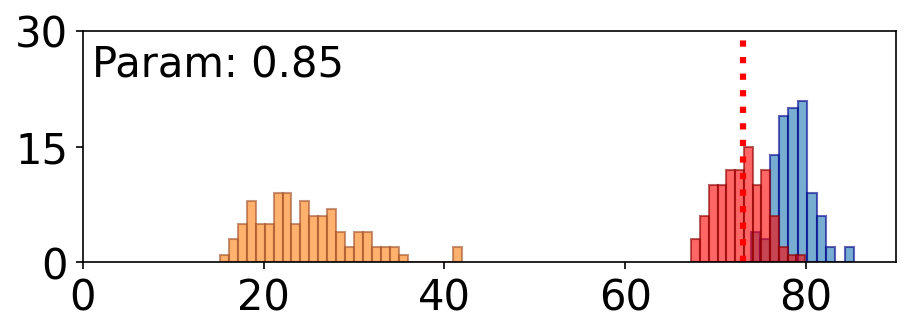}
            \includegraphics[width=1.03\linewidth]{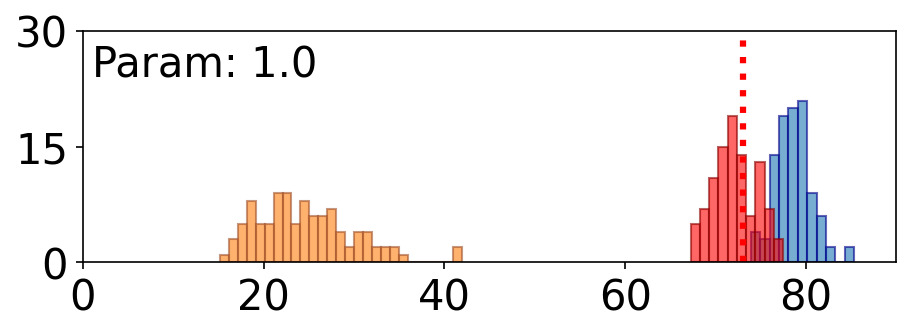}}
    
            \subfloat{\centering
            \makebox[10pt][l]{\raisebox{3pt}{\rotatebox[origin=l]{90}{\tiny Shear}}}
            \includegraphics[width=1.03\linewidth]{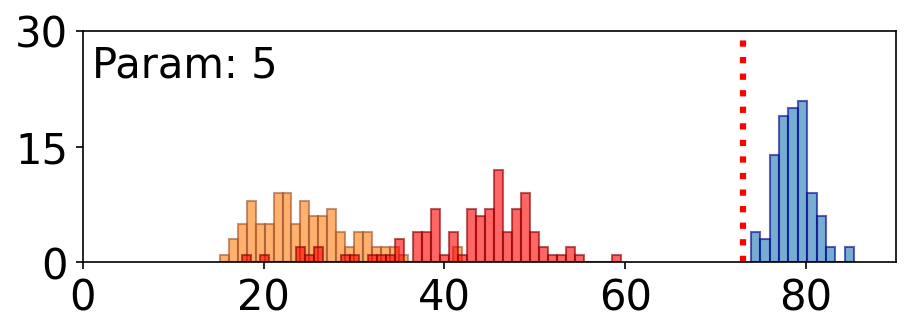}
            \includegraphics[width=1.03\linewidth]{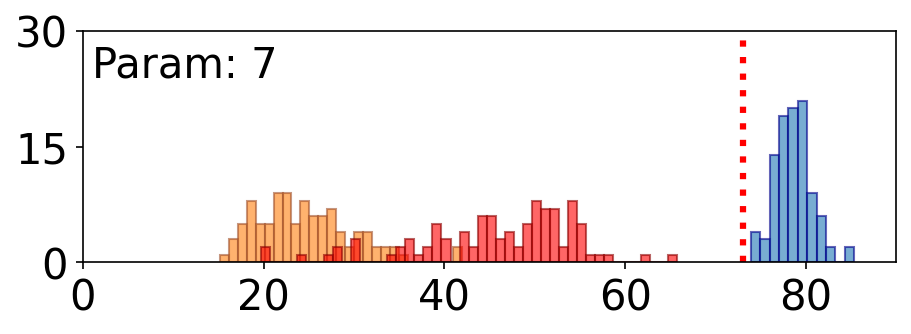}
            \includegraphics[width=1.03\linewidth]{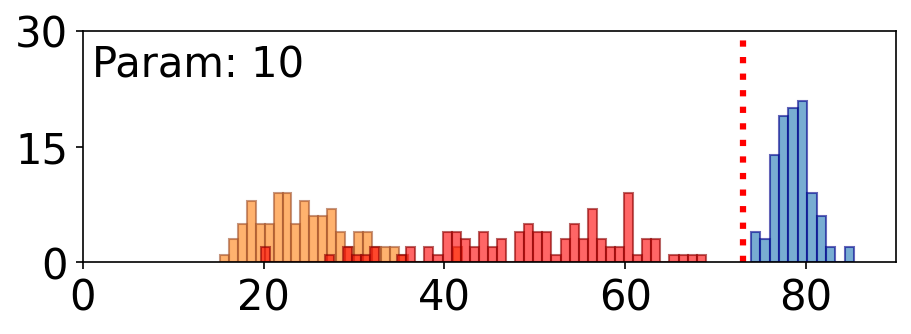}
            \includegraphics[width=1.03\linewidth]{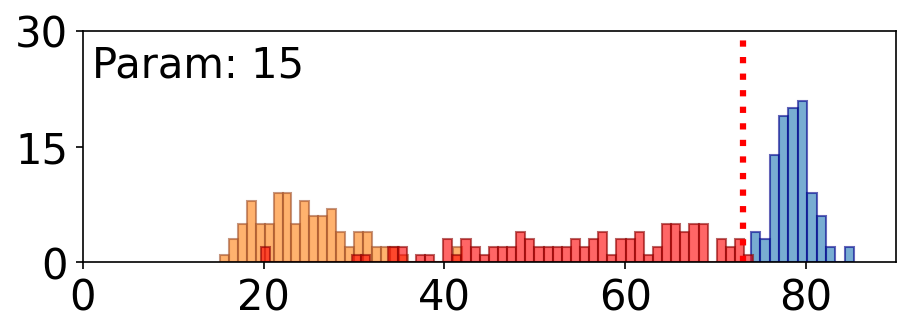}
            \includegraphics[width=1.03\linewidth]{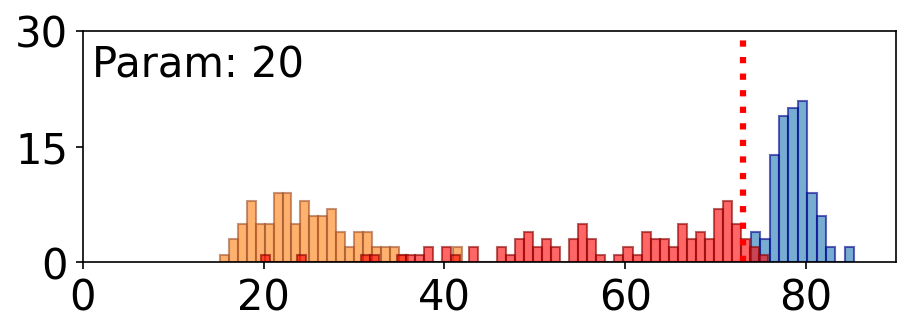}}

            \subfloat{\centering
            \makebox[10pt][l]{\raisebox{1pt}{\rotatebox[origin=l]{90}{\scalebox{.4}{Regen-VAE}}}}
            \includegraphics[width=1.03\linewidth]{figs/figs/ringid_64/hist_dist_ringid_64_regen_vae_5.png}
            \includegraphics[width=1.03\linewidth]{figs/figs/ringid_64/hist_dist_ringid_64_regen_vae_4.png}
            \includegraphics[width=1.03\linewidth]{figs/figs/ringid_64/hist_dist_ringid_64_regen_vae_3.png}
            \includegraphics[width=1.03\linewidth]{figs/figs/ringid_64/hist_dist_ringid_64_regen_vae_2.png}
            \includegraphics[width=1.03\linewidth]{figs/figs/ringid_64/hist_dist_ringid_64_regen_vae_1.png}}
    
            \subfloat{\centering
            \makebox[10pt][l]{\raisebox{0pt}{\rotatebox[origin=l]{90}{\scalebox{.4}{Regen-Diff}}}}
            \includegraphics[width=1.03\linewidth]{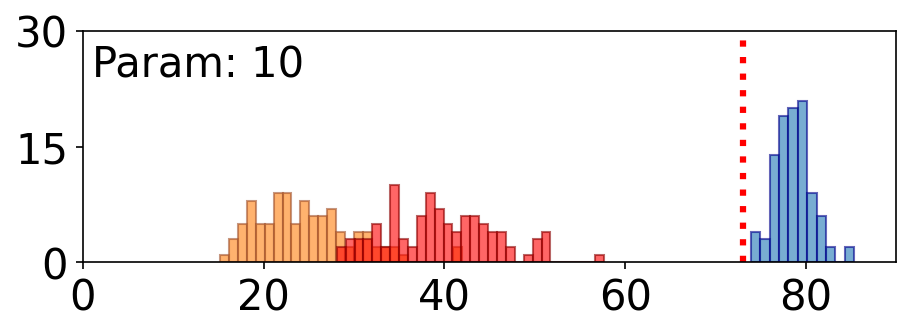}
            \includegraphics[width=1.03\linewidth]{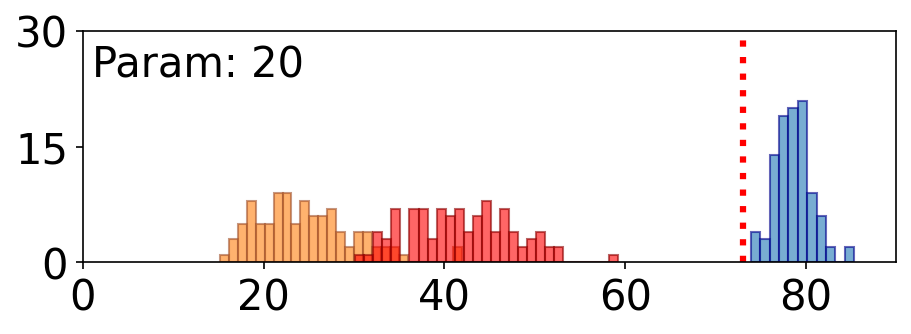}
            \includegraphics[width=1.03\linewidth]{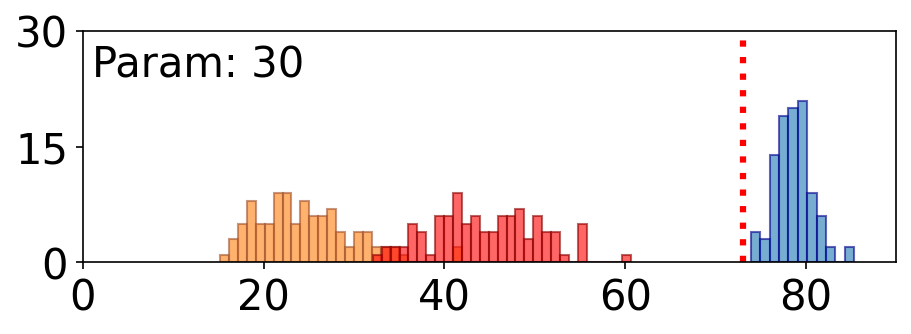}
            \includegraphics[width=1.03\linewidth]{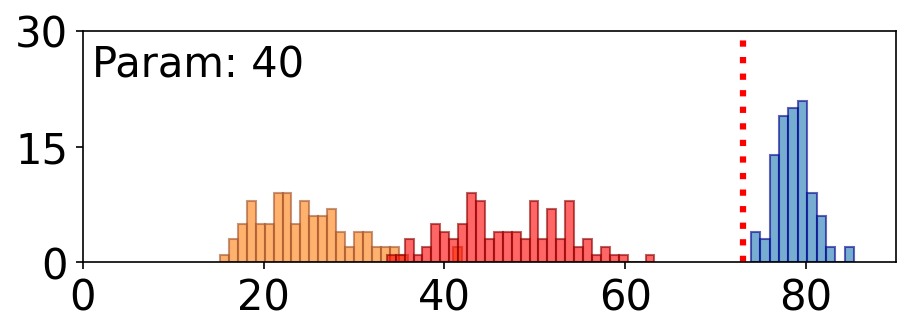}
            \includegraphics[width=1.03\linewidth]{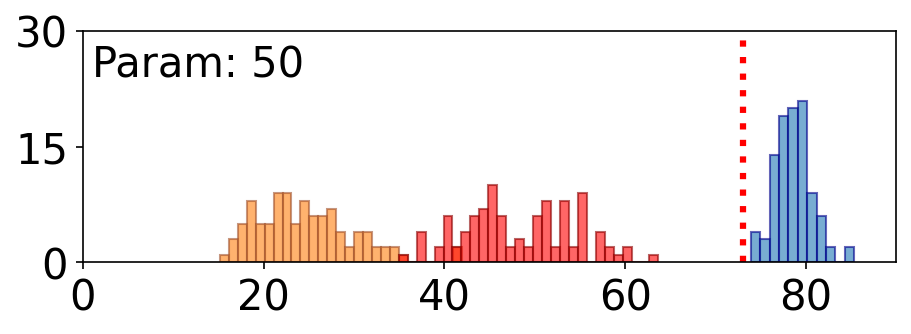}}
        \end{subfigure}
        
    \caption{The effects of different perturbations on $\|\mathbf{w}-\hat{\mathbf{w}}^{\prime}\|$. From left to right, the perturbations are set with different parameters (Param). 
    Watermark: RingID. Thresholds: 73.0 (1e-2 FPR).}
    \label{fig: RingID attack shifting}
\end{figure*}

\begin{figure*}[t]
        \begin{subfigure}[c]{0.185\textwidth}
            \centering
            \subfloat{\centering
            \makebox[10pt][l]{\raisebox{5pt}{\rotatebox[origin=l]{90}{\tiny Blurring}}}
            \includegraphics[width=1.03\linewidth]{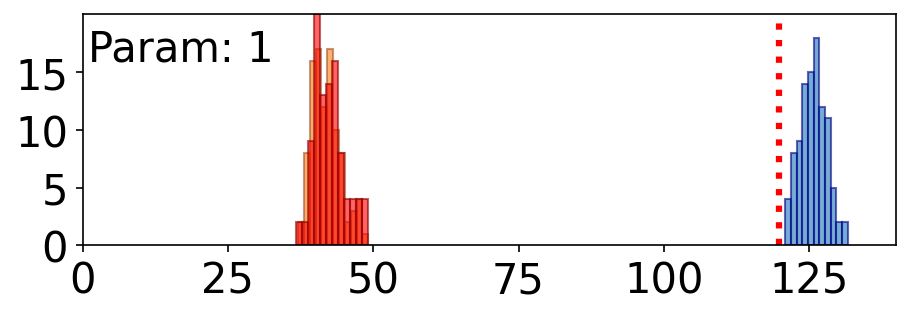}
            \includegraphics[width=1.03\linewidth]{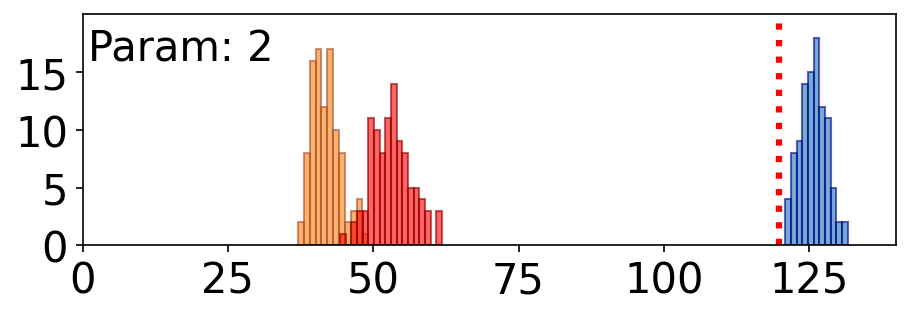}
            \includegraphics[width=1.03\linewidth]{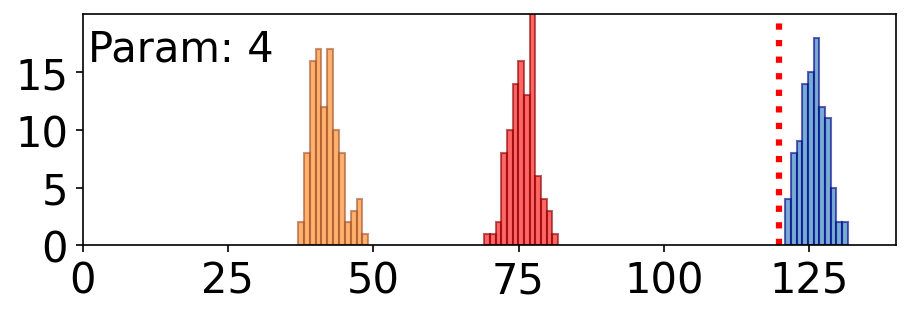}
            \includegraphics[width=1.03\linewidth]{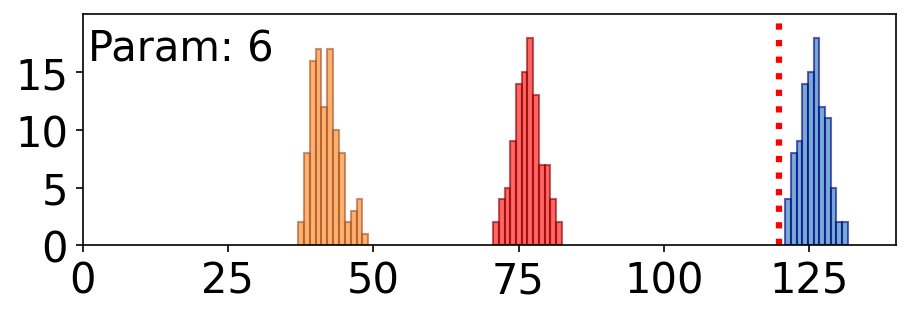}
            \includegraphics[width=1.03\linewidth]{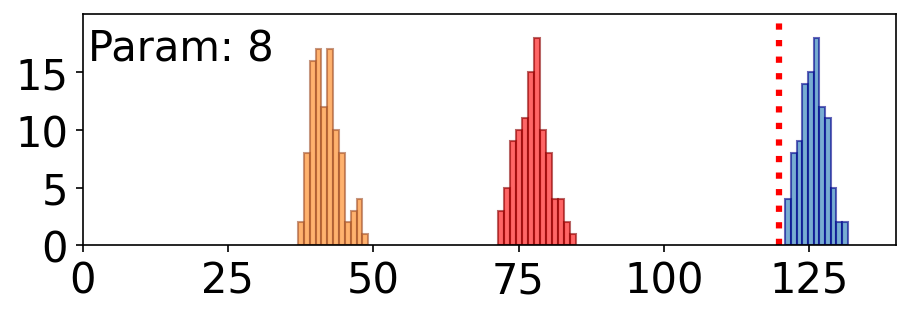}}
    
            \subfloat{\centering
            \makebox[10pt][l]{\raisebox{5pt}{\rotatebox[origin=l]{90}{\tiny Jitter}}}
            \includegraphics[width=1.03\linewidth]{figs/figs/ringid_2_128/hist_dist_ringid_128_color_jitter_1.png}
            \includegraphics[width=1.03\linewidth]{figs/figs/ringid_2_128/hist_dist_ringid_128_color_jitter_2.png}
            \includegraphics[width=1.03\linewidth]{figs/figs/ringid_2_128/hist_dist_ringid_128_color_jitter_4.png}
            \includegraphics[width=1.03\linewidth]{figs/figs/ringid_2_128/hist_dist_ringid_128_color_jitter_6.png}
            \includegraphics[width=1.03\linewidth]{figs/figs/ringid_2_128/hist_dist_ringid_128_color_jitter_8.png}}
    
            \subfloat{\centering
            \makebox[10pt][l]{\raisebox{5pt}{\rotatebox[origin=l]{90}{\tiny JPEG}}}
            \includegraphics[width=1.03\linewidth]{figs/figs/ringid_2_128/hist_dist_ringid_128_jpeg_80.png}
            \includegraphics[width=1.03\linewidth]{figs/figs/ringid_2_128/hist_dist_ringid_128_jpeg_40.png}
            \includegraphics[width=1.03\linewidth]{figs/figs/ringid_2_128/hist_dist_ringid_128_jpeg_25.png}
            \includegraphics[width=1.03\linewidth]{figs/figs/ringid_2_128/hist_dist_ringid_128_jpeg_5.png}
            \includegraphics[width=1.03\linewidth]{figs/figs/ringid_2_128/hist_dist_ringid_128_jpeg_1.png}}
    
            \subfloat{\centering
            \makebox[10pt][l]{\raisebox{5pt}{\rotatebox[origin=l]{90}{\tiny Rotation}}}
            \includegraphics[width=1.03\linewidth]{figs/figs/ringid_2_128/hist_dist_ringid_128_rotation_5.png}
            \includegraphics[width=1.03\linewidth]{figs/figs/ringid_2_128/hist_dist_ringid_128_rotation_15.png}
            \includegraphics[width=1.03\linewidth]{figs/figs/ringid_2_128/hist_dist_ringid_128_rotation_30.png}
            \includegraphics[width=1.03\linewidth]{figs/figs/ringid_2_128/hist_dist_ringid_128_rotation_45.png}
            \includegraphics[width=1.03\linewidth]{figs/figs/ringid_2_128/hist_dist_ringid_128_rotation_75.png}}
    
            \subfloat{\centering
            \makebox[10pt][l]{\raisebox{5pt}{\rotatebox[origin=l]{90}{\tiny Cropping}}}
            \includegraphics[width=1.03\linewidth]{figs/figs/ringid_2_128/hist_dist_ringid_128_cropping_0.9.png}
            \includegraphics[width=1.03\linewidth]{figs/figs/ringid_2_128/hist_dist_ringid_128_cropping_0.75.png}
            \includegraphics[width=1.03\linewidth]{figs/figs/ringid_2_128/hist_dist_ringid_128_cropping_0.6.png}
            \includegraphics[width=1.03\linewidth]{figs/figs/ringid_2_128/hist_dist_ringid_128_cropping_0.4.png}
            \includegraphics[width=1.03\linewidth]{figs/figs/ringid_2_128/hist_dist_ringid_128_cropping_0.2.png}}
    
            \subfloat{\centering
            \makebox[10pt][l]{\raisebox{3pt}{\rotatebox[origin=l]{90}{\tiny Elastic}}}
            \includegraphics[width=1.03\linewidth]{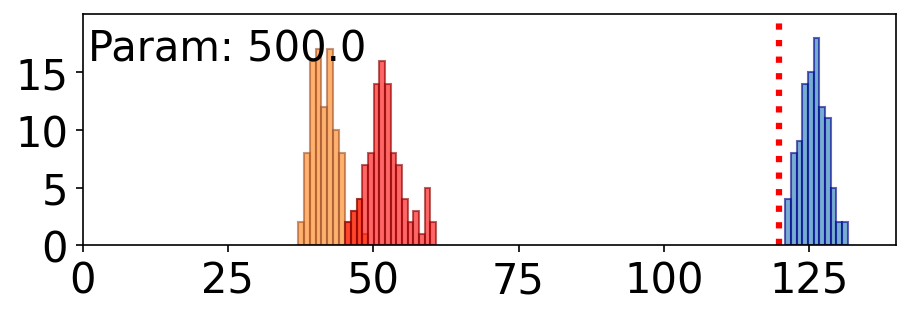}
            \includegraphics[width=1.03\linewidth]{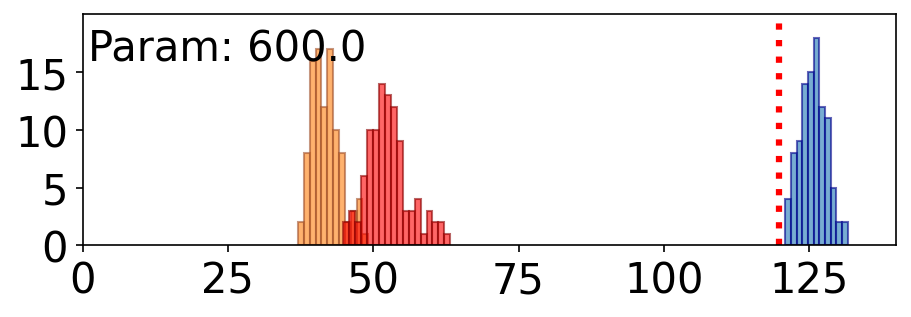}
            \includegraphics[width=1.03\linewidth]{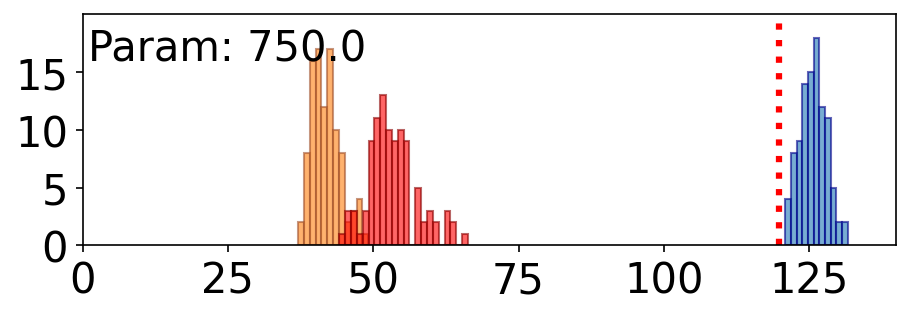}
            \includegraphics[width=1.03\linewidth]{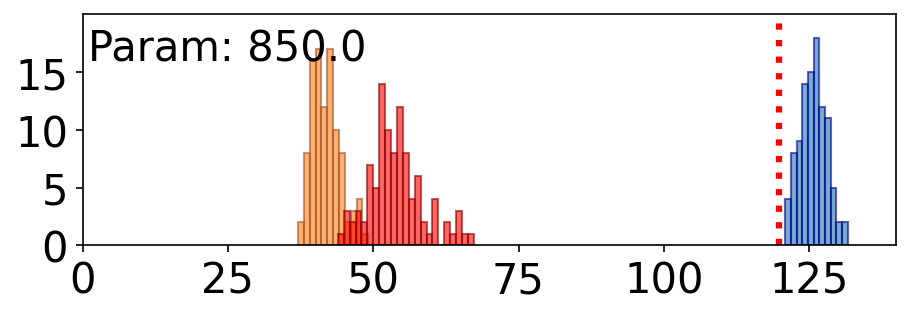}
            \includegraphics[width=1.03\linewidth]{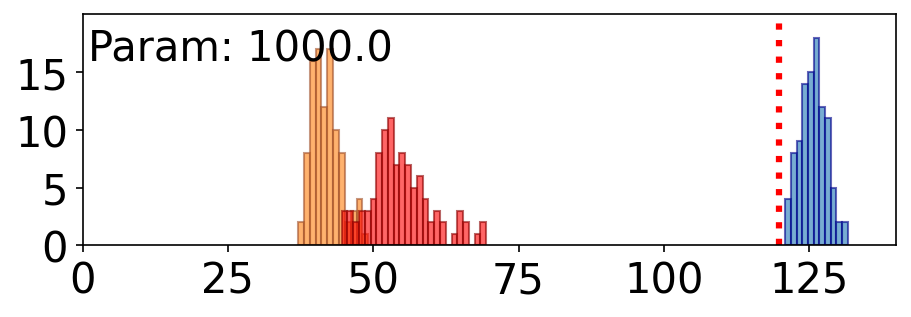}}
    
            \subfloat{\centering
            \makebox[10pt][l]{\raisebox{3pt}{\rotatebox[origin=l]{90}{\scalebox{.4}{Perspective}}}}
            \includegraphics[width=1.03\linewidth]{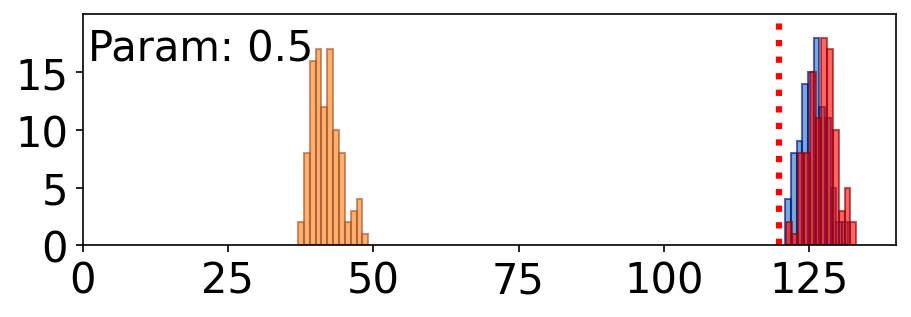}
            \includegraphics[width=1.03\linewidth]{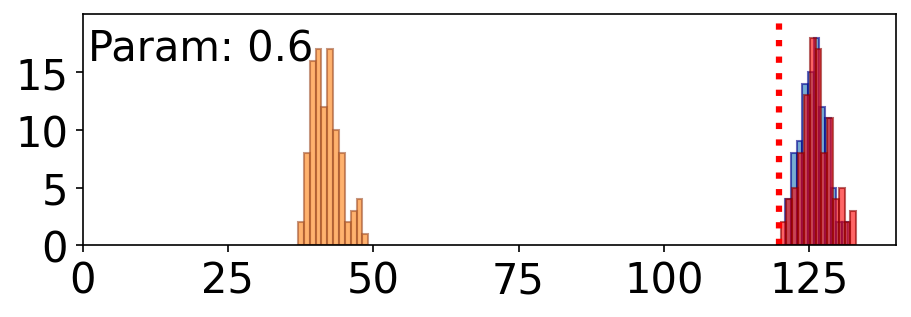}
            \includegraphics[width=1.03\linewidth]{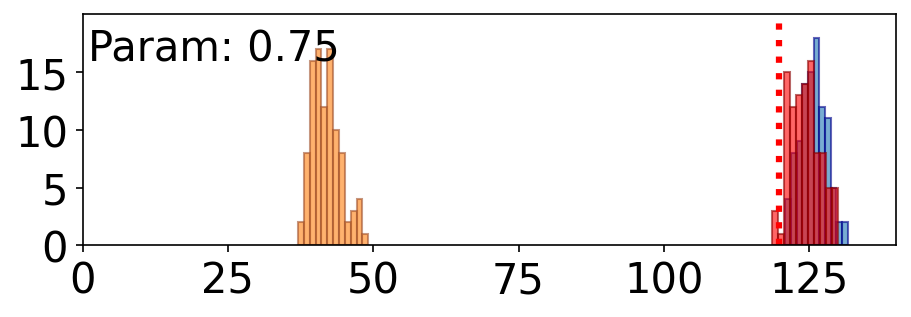}
            \includegraphics[width=1.03\linewidth]{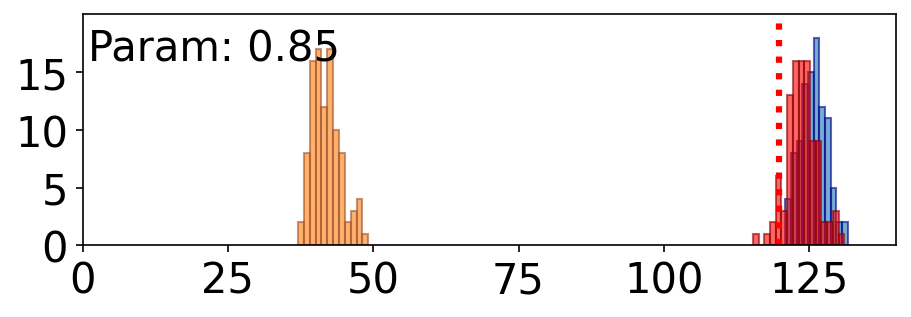}
            \includegraphics[width=1.03\linewidth]{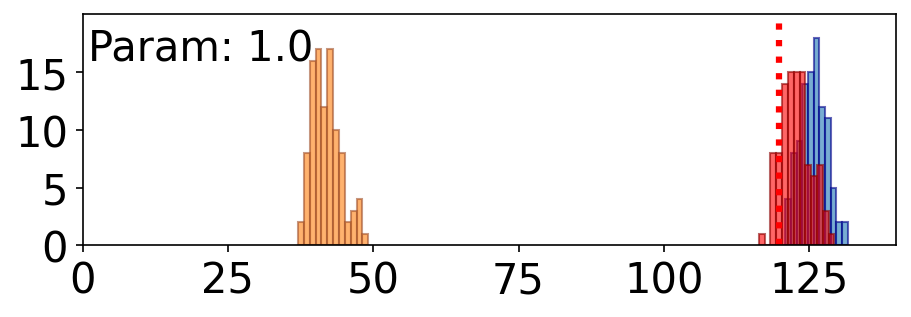}}
    
            \subfloat{\centering
            \makebox[10pt][l]{\raisebox{3pt}{\rotatebox[origin=l]{90}{\tiny Shear}}}
            \includegraphics[width=1.03\linewidth]{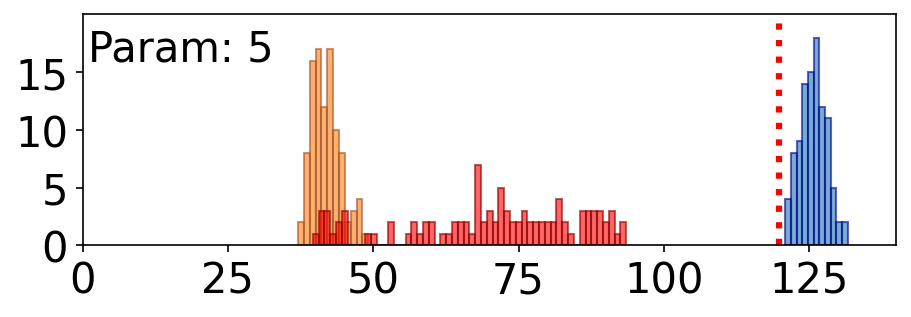}
            \includegraphics[width=1.03\linewidth]{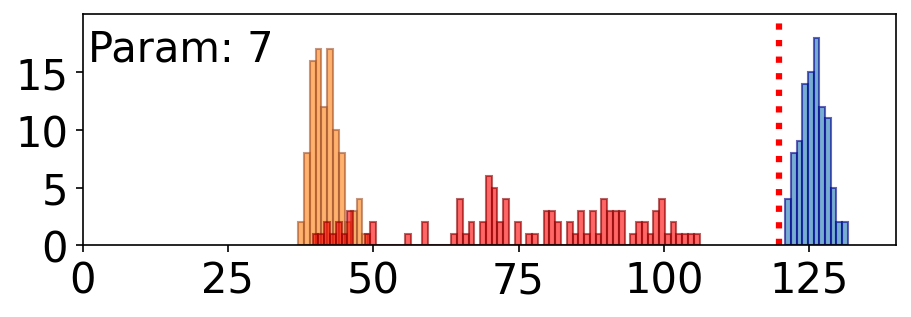}
            \includegraphics[width=1.03\linewidth]{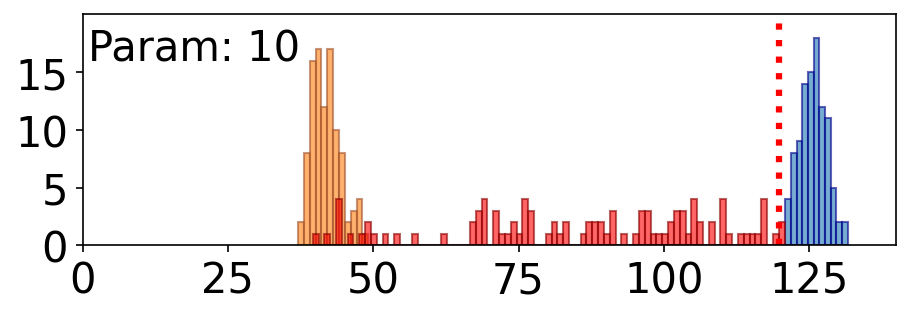}
            \includegraphics[width=1.03\linewidth]{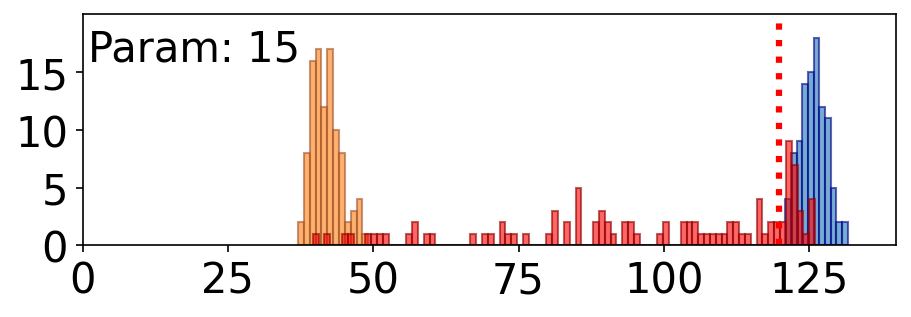}
            \includegraphics[width=1.03\linewidth]{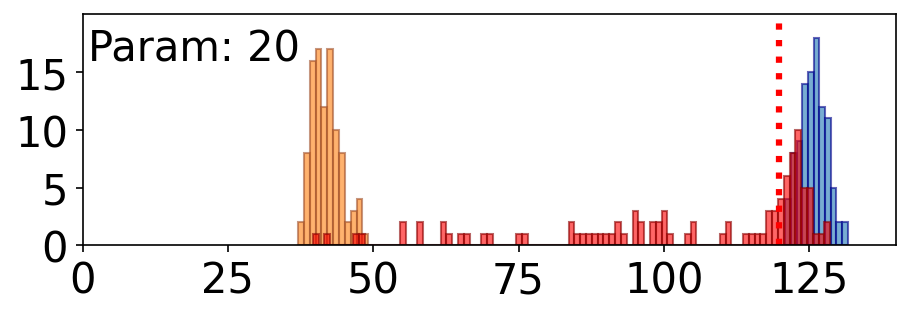}}

            \subfloat{\centering
            \makebox[10pt][l]{\raisebox{3pt}{\rotatebox[origin=l]{90}{\scalebox{.4}{Regen-VAE}}}}
            \includegraphics[width=1.03\linewidth]{figs/figs/ringid_2_128/hist_dist_ringid_128_regen_vae_5.png}
            \includegraphics[width=1.03\linewidth]{figs/figs/ringid_2_128/hist_dist_ringid_128_regen_vae_4.png}
            \includegraphics[width=1.03\linewidth]{figs/figs/ringid_2_128/hist_dist_ringid_128_regen_vae_3.png}
            \includegraphics[width=1.03\linewidth]{figs/figs/ringid_2_128/hist_dist_ringid_128_regen_vae_2.png}
            \includegraphics[width=1.03\linewidth]{figs/figs/ringid_2_128/hist_dist_ringid_128_regen_vae_1.png}}
    
            \subfloat{\centering
            \makebox[10pt][l]{\raisebox{0pt}{\rotatebox[origin=l]{90}{\scalebox{.4}{Regen-Diff}}}}
            \includegraphics[width=1.03\linewidth]{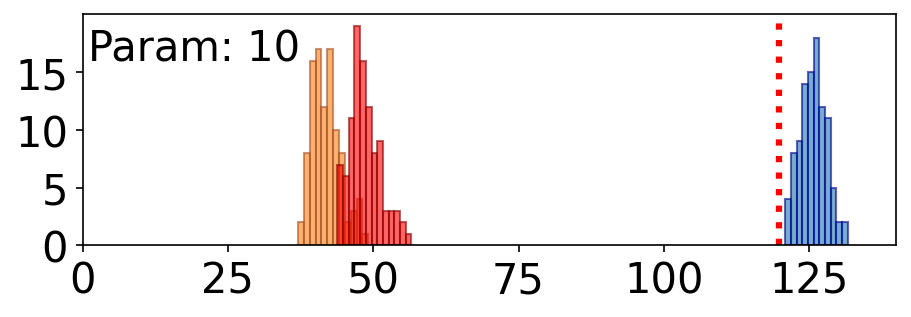}
            \includegraphics[width=1.03\linewidth]{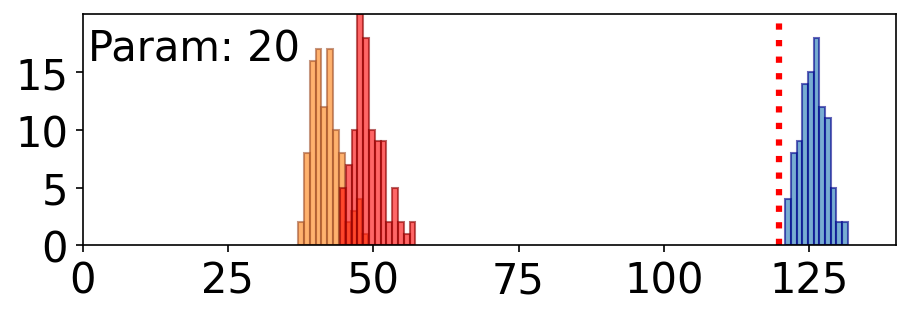}
            \includegraphics[width=1.03\linewidth]{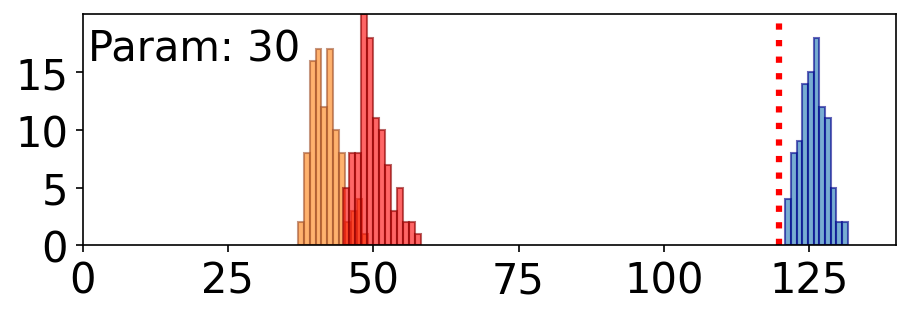}
            \includegraphics[width=1.03\linewidth]{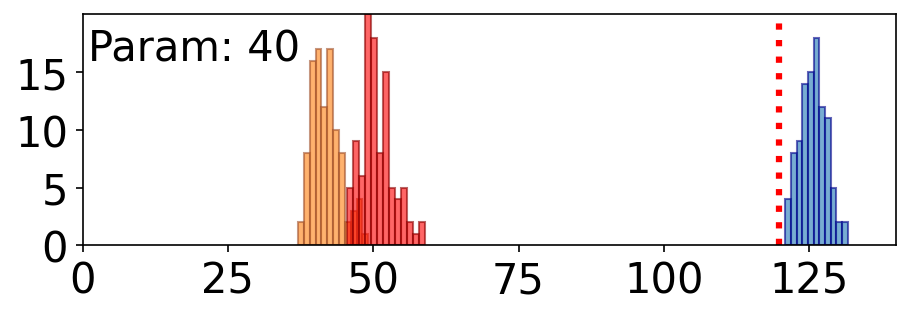}
            \includegraphics[width=1.03\linewidth]{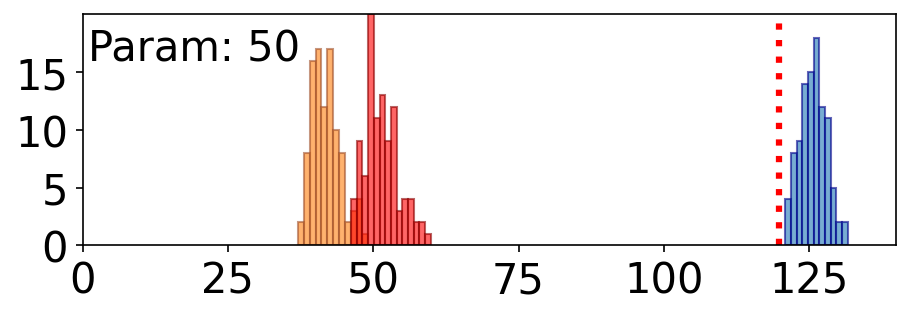}}
        \end{subfigure}
        
    \caption{The effects of different perturbations on $\|\mathbf{w}-\hat{\mathbf{w}}^{\prime}\|$. From left to right, the perturbations are set with different parameters (Param). 
    Watermark: RingID with channel 0 pattern multiplied by $2$ and $\alpha=128$ (RingID-2-128). Thresholds: 119.85 (1e-2 FPR).}
    \label{fig: RingID 2 128 attack shifting}
\end{figure*}

\begin{figure*}[t]
        \begin{subfigure}[c]{0.185\textwidth}
            \centering
            \subfloat{\centering
            \makebox[10pt][l]{\raisebox{5pt}{\rotatebox[origin=l]{90}{\tiny Blurring}}}
            \includegraphics[width=1.03\linewidth]{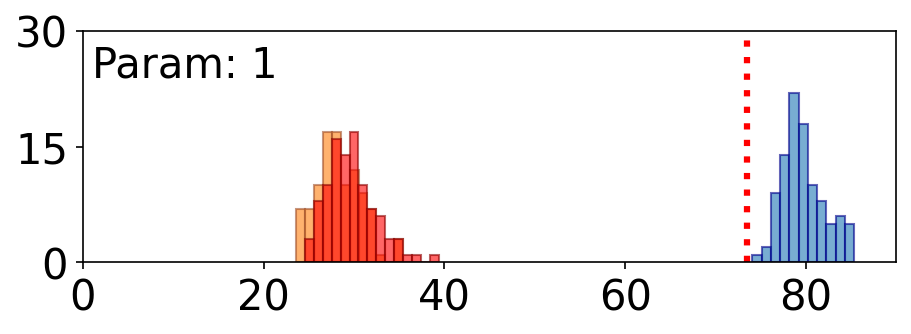}
            \includegraphics[width=1.03\linewidth]{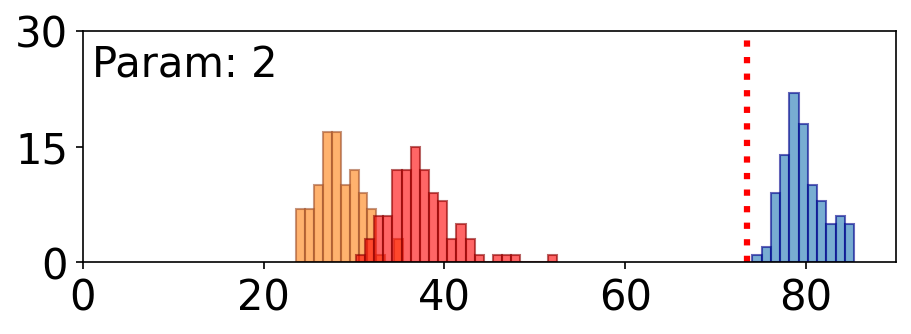}
            \includegraphics[width=1.03\linewidth]{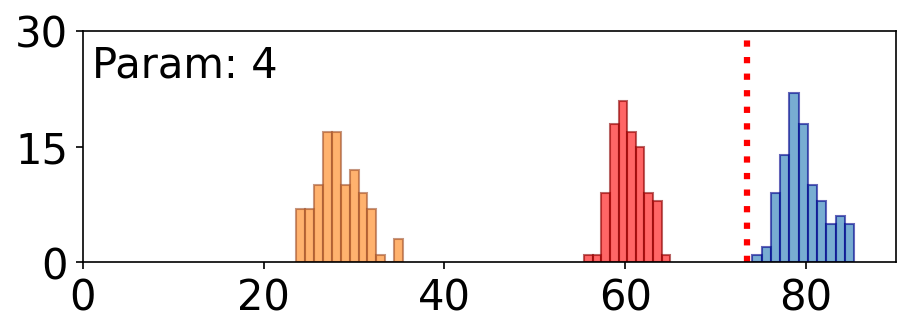}
            \includegraphics[width=1.03\linewidth]{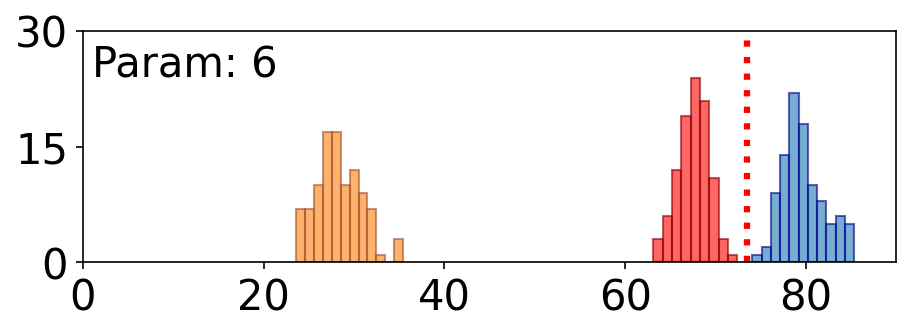}
            \includegraphics[width=1.03\linewidth]{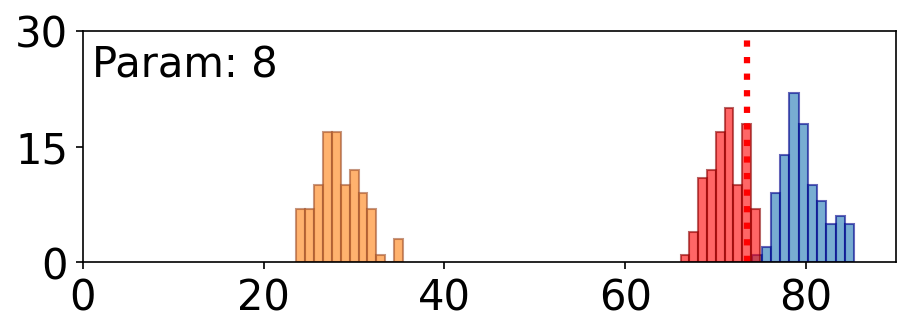}}
    
            \subfloat{\centering
            \makebox[10pt][l]{\raisebox{5pt}{\rotatebox[origin=l]{90}{\tiny Jitter}}}
            \includegraphics[width=1.03\linewidth]{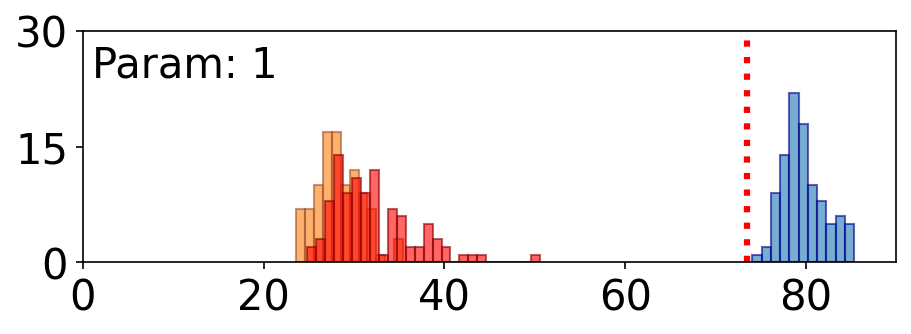}
            \includegraphics[width=1.03\linewidth]{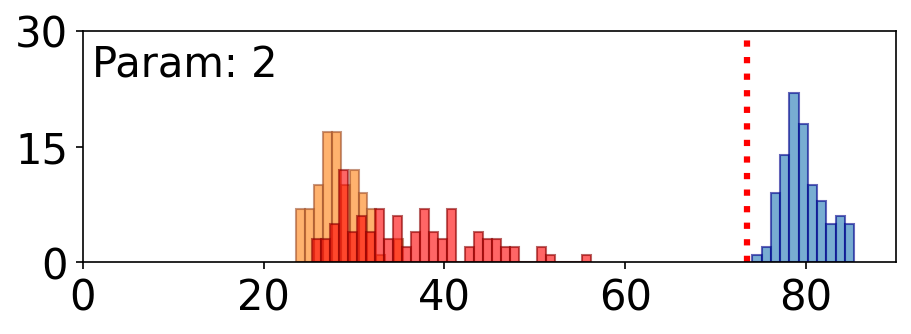}
            \includegraphics[width=1.03\linewidth]{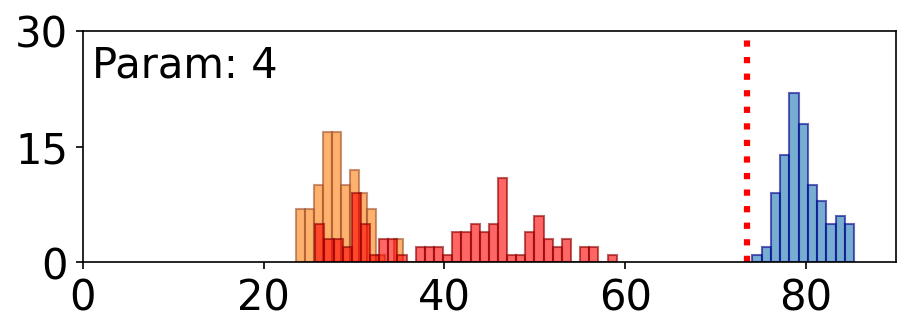}
            \includegraphics[width=1.03\linewidth]{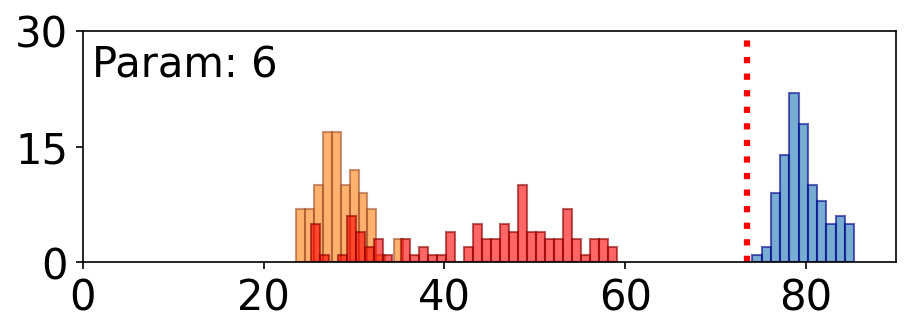}
            \includegraphics[width=1.03\linewidth]{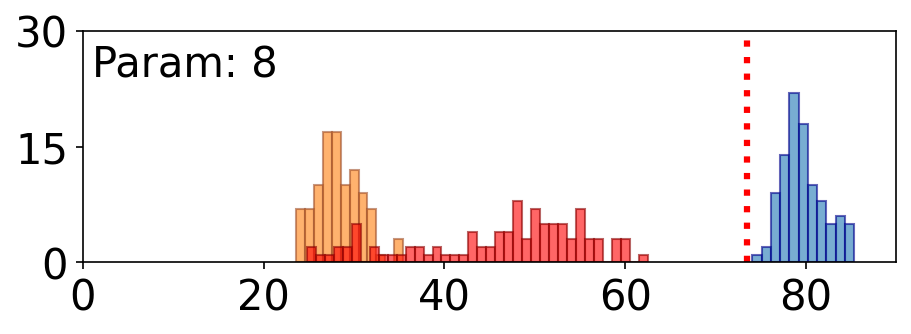}}
    
            \subfloat{\centering
            \makebox[10pt][l]{\raisebox{5pt}{\rotatebox[origin=l]{90}{\tiny JPEG}}}
            \includegraphics[width=1.03\linewidth]{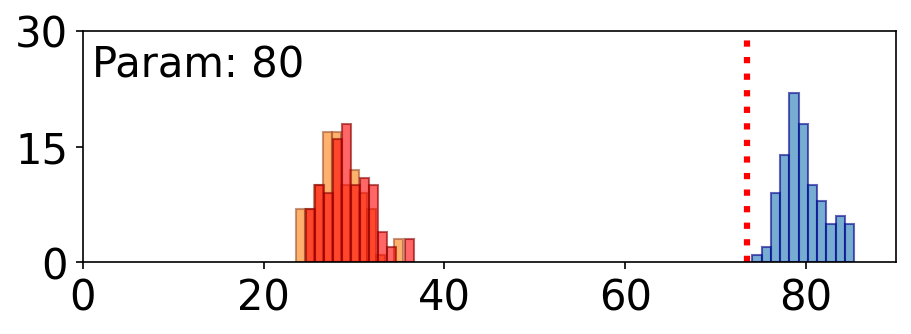}
            \includegraphics[width=1.03\linewidth]{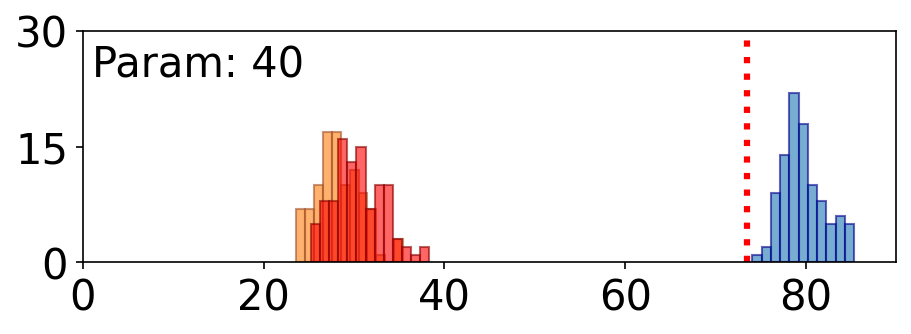}
            \includegraphics[width=1.03\linewidth]{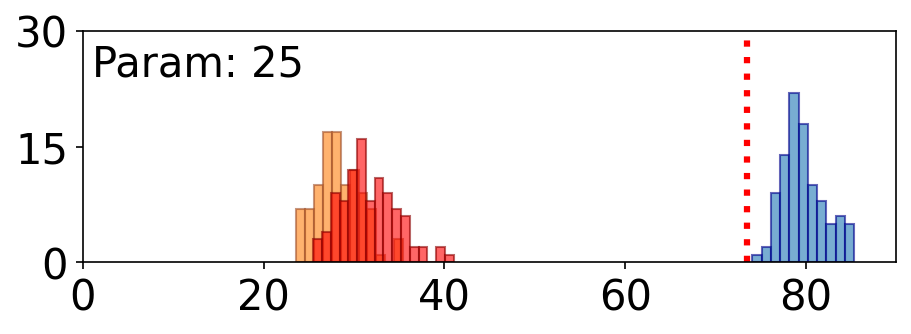}
            \includegraphics[width=1.03\linewidth]{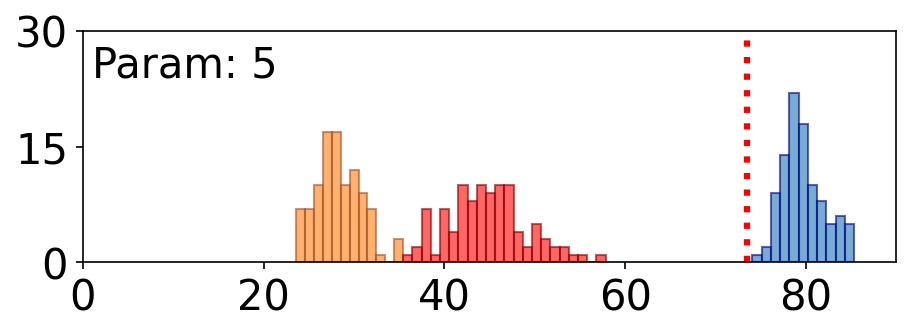}
            \includegraphics[width=1.03\linewidth]{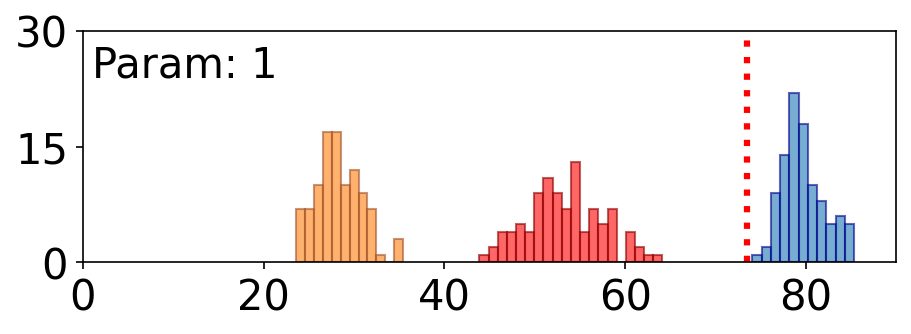}}
    
            \subfloat{\centering
            \makebox[10pt][l]{\raisebox{5pt}{\rotatebox[origin=l]{90}{\tiny Rotation}}}
            \includegraphics[width=1.03\linewidth]{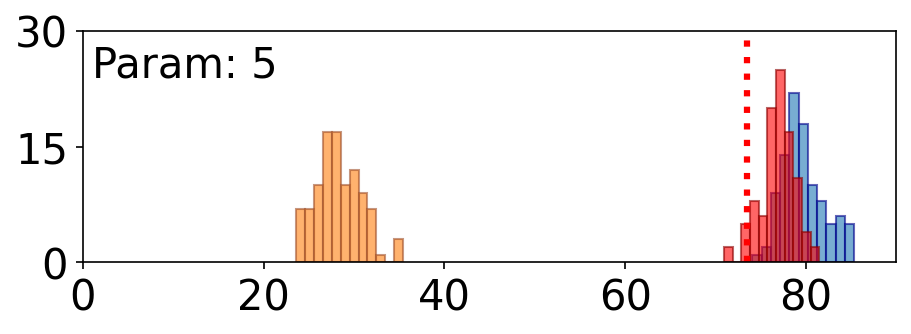}
            \includegraphics[width=1.03\linewidth]{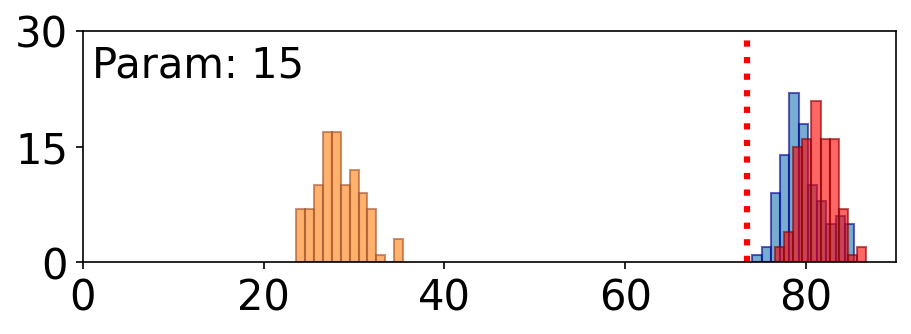}
            \includegraphics[width=1.03\linewidth]{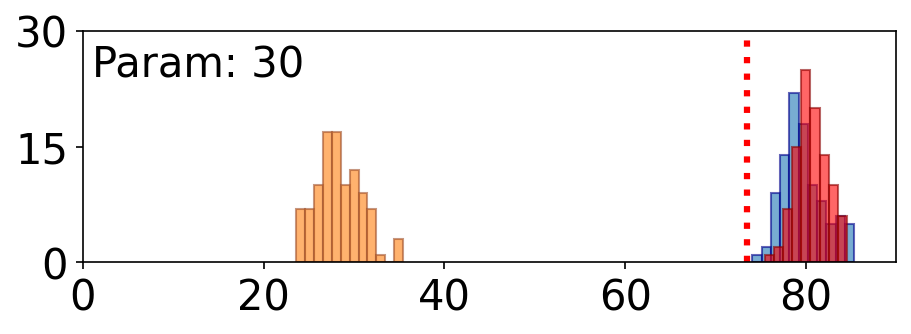}
            \includegraphics[width=1.03\linewidth]{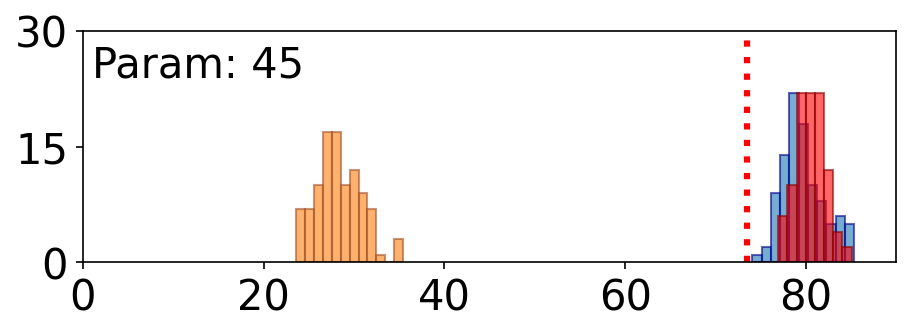}
            \includegraphics[width=1.03\linewidth]{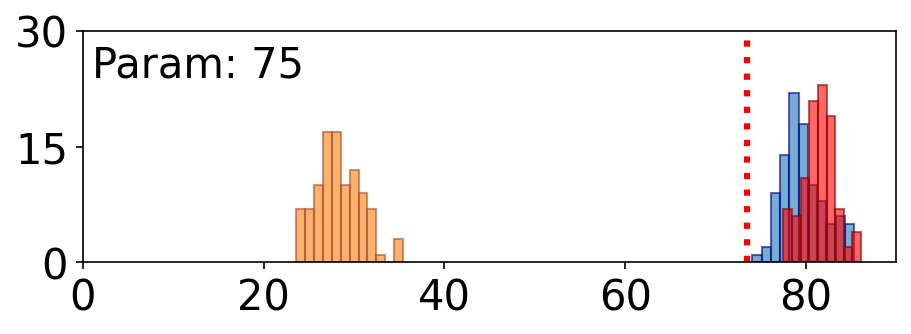}}
    
            \subfloat{\centering
            \makebox[10pt][l]{\raisebox{5pt}{\rotatebox[origin=l]{90}{\tiny Cropping}}}
            \includegraphics[width=1.03\linewidth]{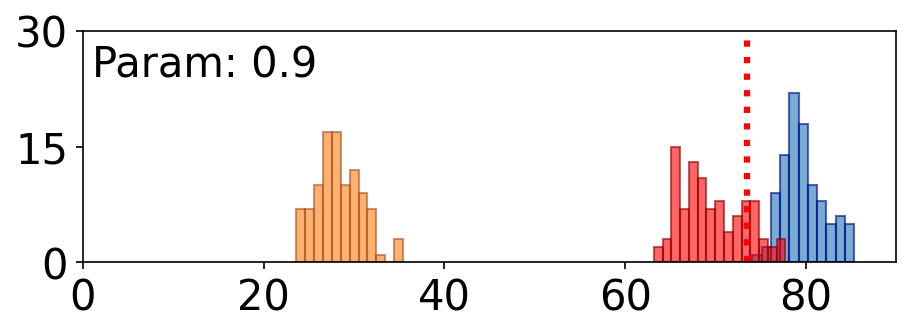}
            \includegraphics[width=1.03\linewidth]{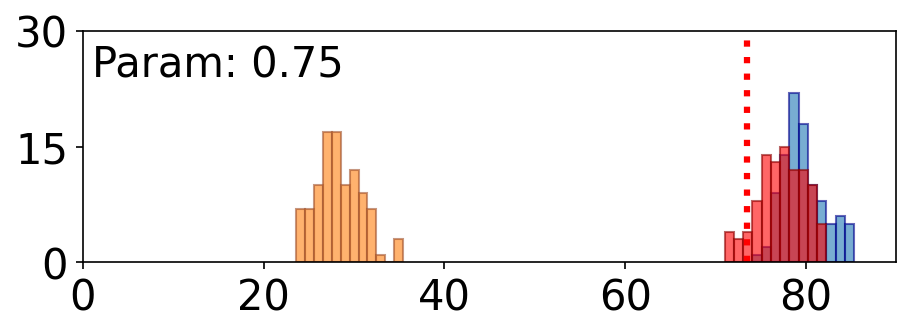}
            \includegraphics[width=1.03\linewidth]{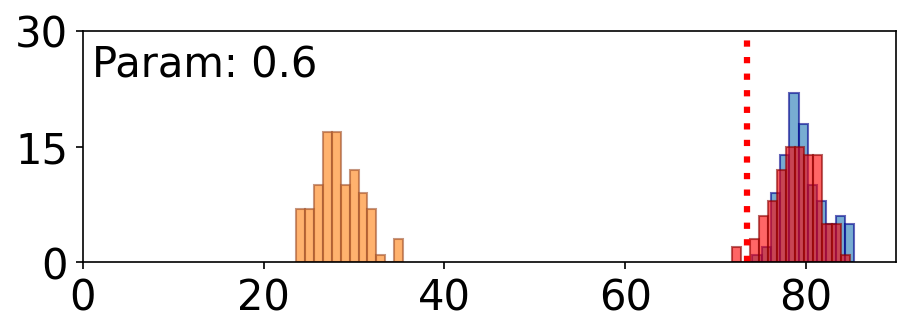}
            \includegraphics[width=1.03\linewidth]{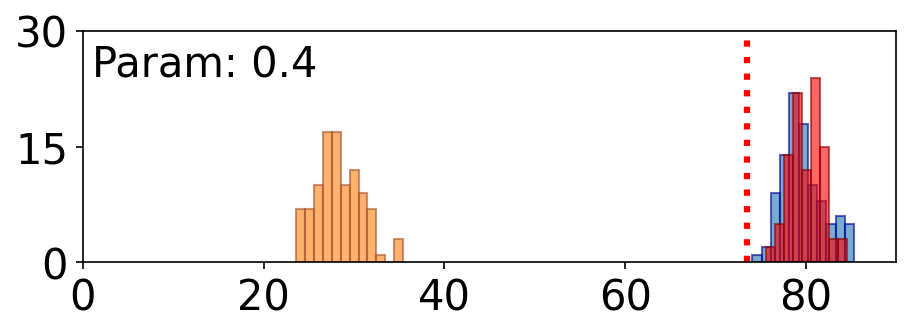}
            \includegraphics[width=1.03\linewidth]{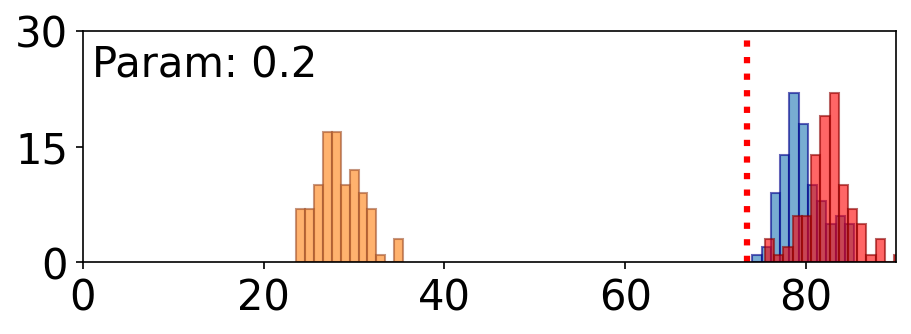}}
    
            \subfloat{\centering
            \makebox[10pt][l]{\raisebox{3pt}{\rotatebox[origin=l]{90}{\tiny Elastic}}}
            \includegraphics[width=1.03\linewidth]{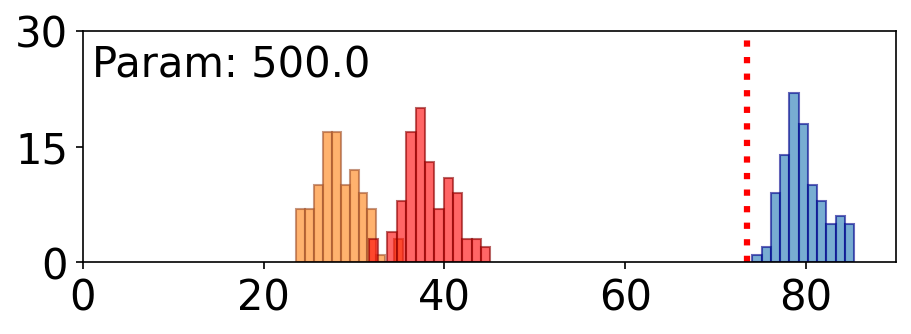}
            \includegraphics[width=1.03\linewidth]{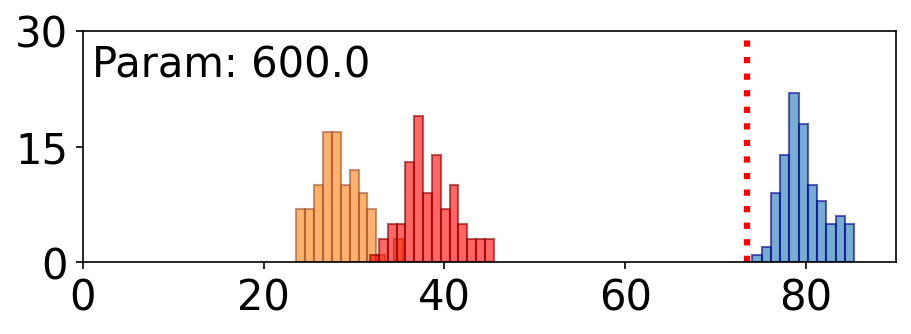}
            \includegraphics[width=1.03\linewidth]{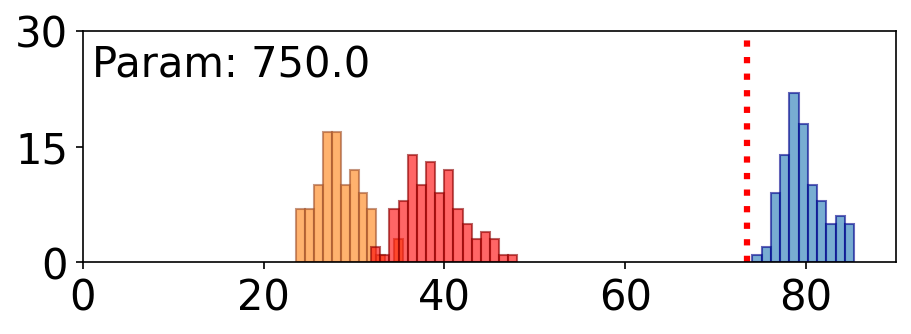}
            \includegraphics[width=1.03\linewidth]{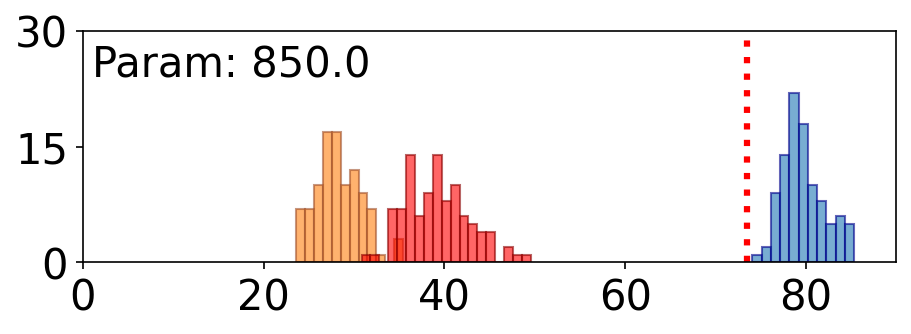}
            \includegraphics[width=1.03\linewidth]{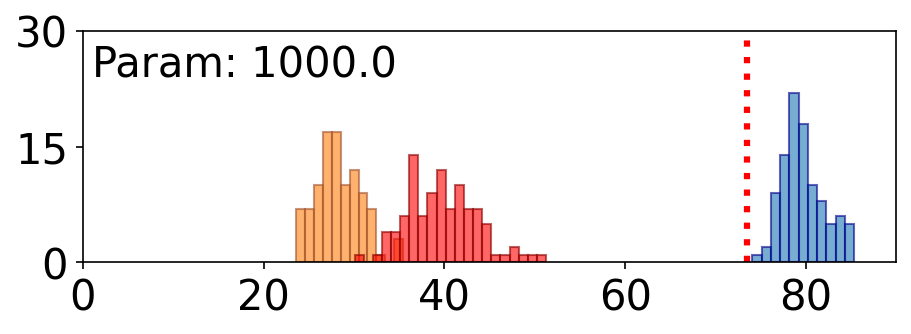}}
    
            \subfloat{\centering
            \makebox[10pt][l]{\raisebox{3pt}{\rotatebox[origin=l]{90}{\scalebox{.4}{Perspective}}}}
            \includegraphics[width=1.03\linewidth]{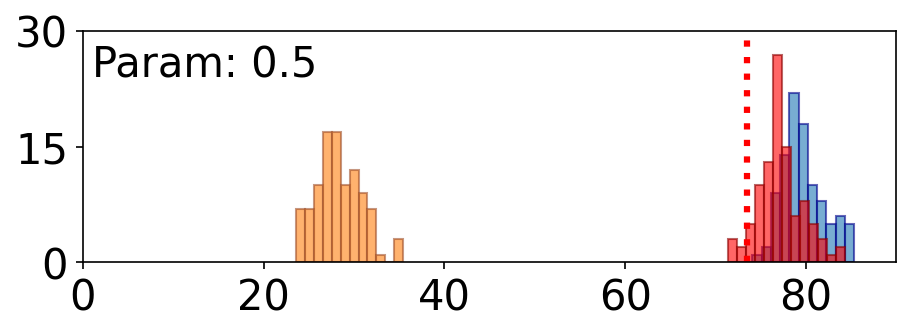}
            \includegraphics[width=1.03\linewidth]{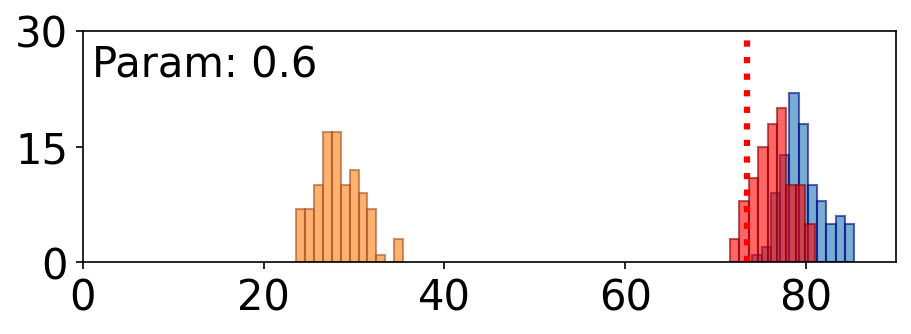}
            \includegraphics[width=1.03\linewidth]{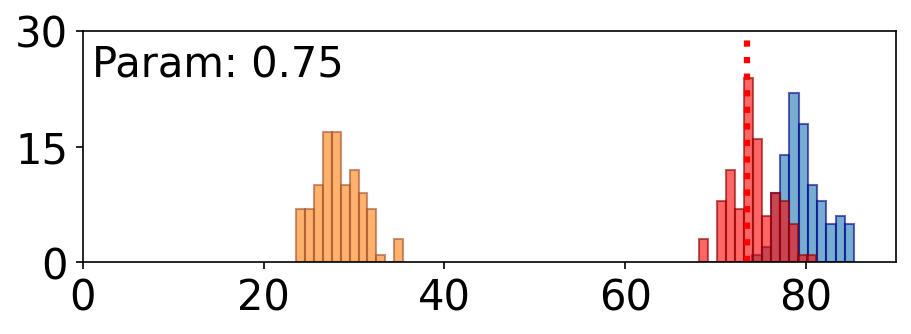}
            \includegraphics[width=1.03\linewidth]{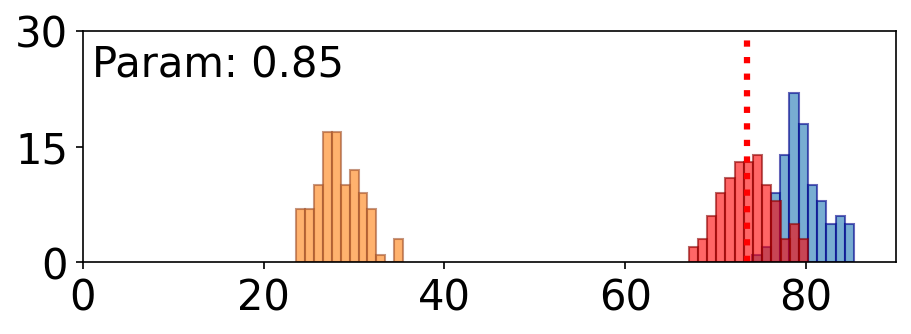}
            \includegraphics[width=1.03\linewidth]{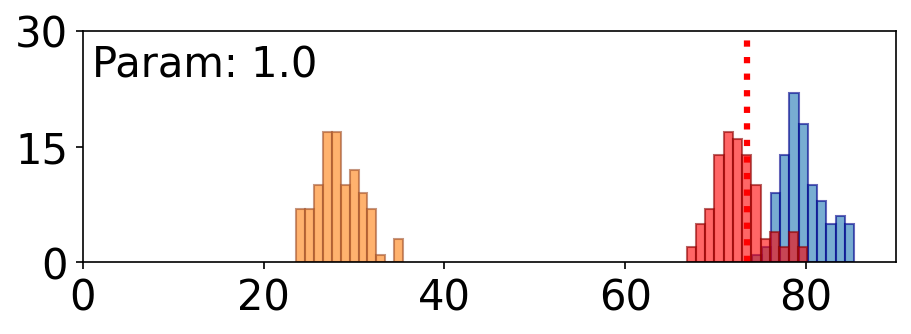}}
    
            \subfloat{\centering
            \makebox[10pt][l]{\raisebox{3pt}{\rotatebox[origin=l]{90}{\tiny Shear}}}
            \includegraphics[width=1.03\linewidth]{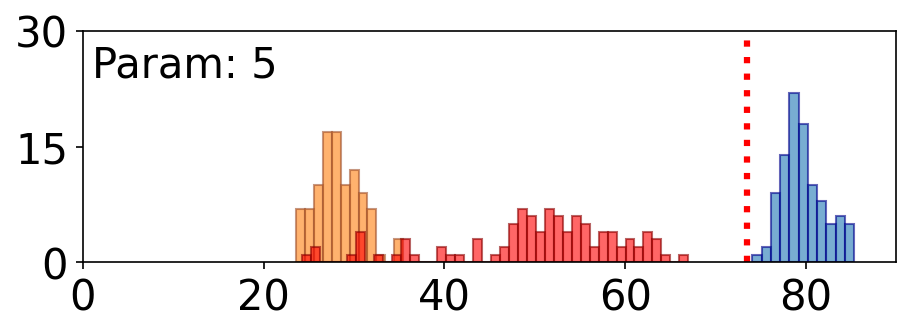}
            \includegraphics[width=1.03\linewidth]{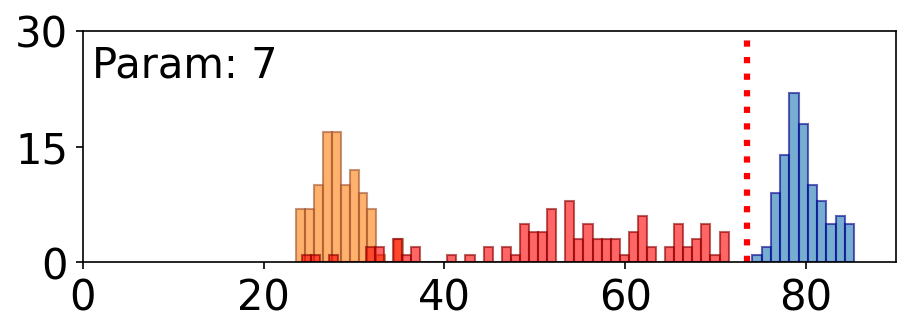}
            \includegraphics[width=1.03\linewidth]{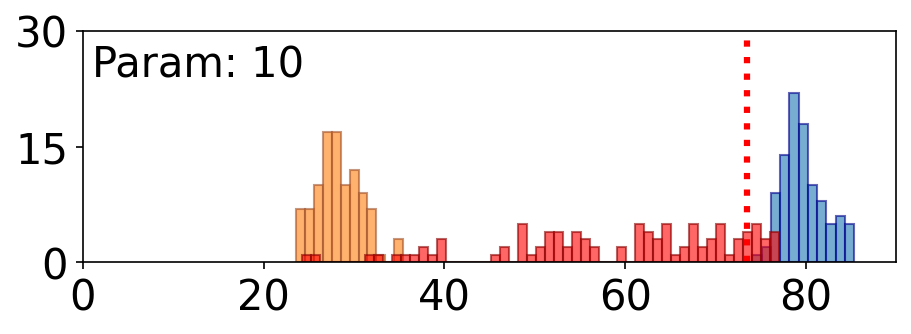}
            \includegraphics[width=1.03\linewidth]{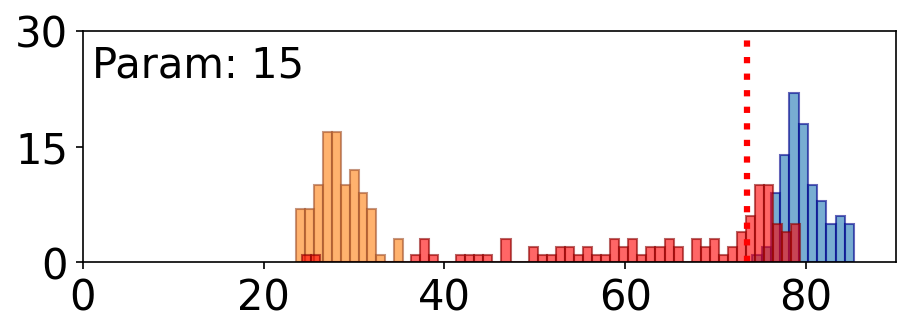}
            \includegraphics[width=1.03\linewidth]{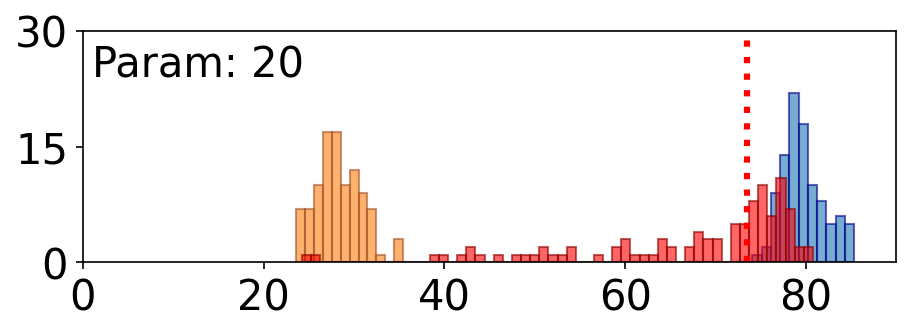}}

            \subfloat{\centering
            \makebox[10pt][l]{\raisebox{10pt}{\rotatebox[origin=l]{90}{\scalebox{.4}{Regen-VAE}}}}
            \includegraphics[width=1.03\linewidth]{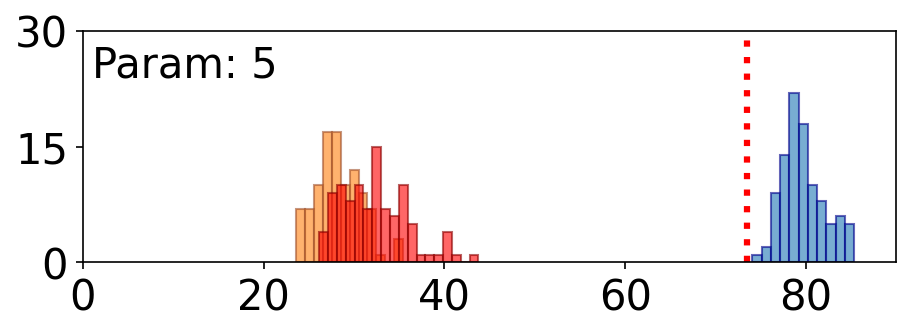}
            \includegraphics[width=1.03\linewidth]{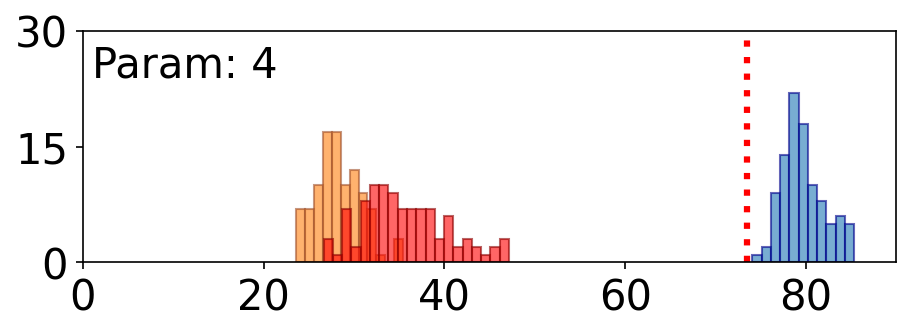}
            \includegraphics[width=1.03\linewidth]{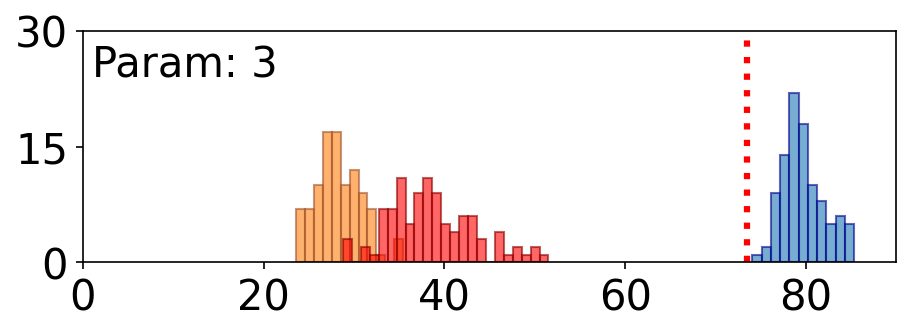}
            \includegraphics[width=1.03\linewidth]{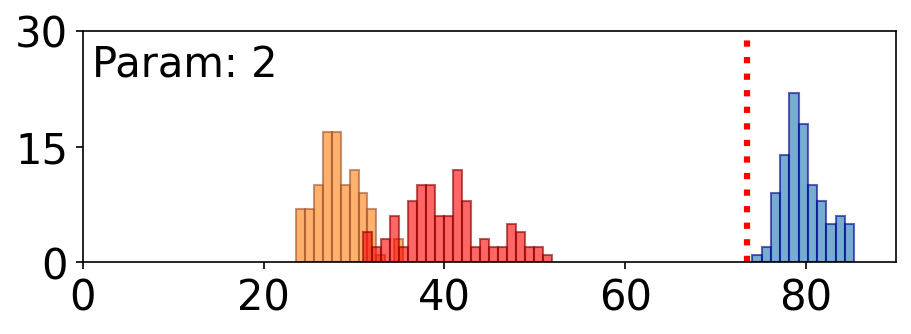}
            \includegraphics[width=1.03\linewidth]{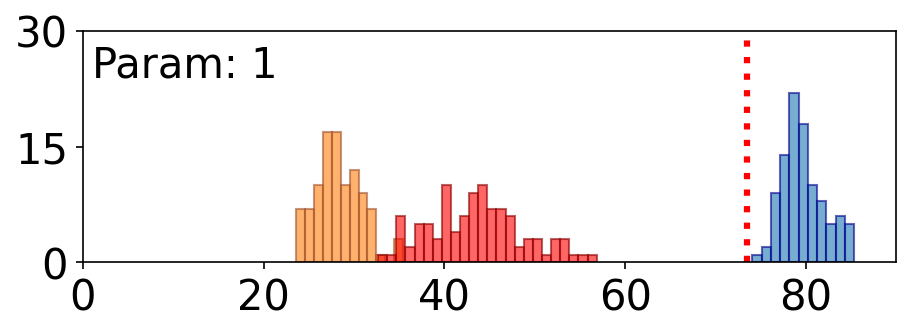}}
    
            \subfloat{\centering
            \makebox[10pt][l]{\raisebox{10pt}{\rotatebox[origin=l]{90}{\scalebox{.4}{Regen-Diff}}}}
            \includegraphics[width=1.03\linewidth]{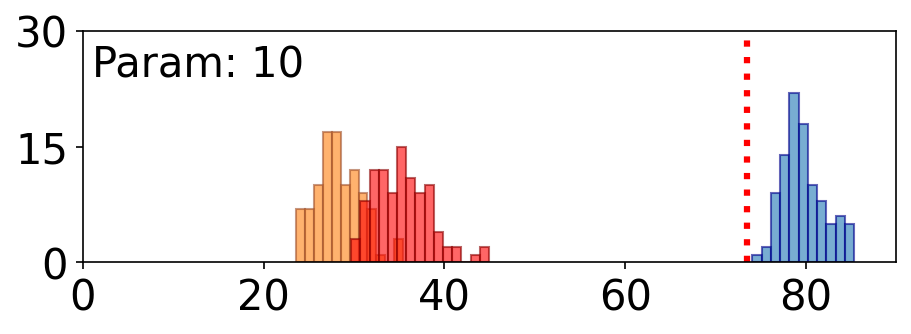}
            \includegraphics[width=1.03\linewidth]{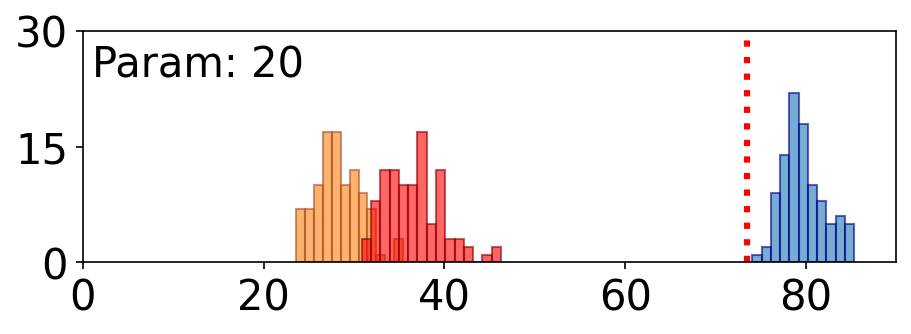}
            \includegraphics[width=1.03\linewidth]{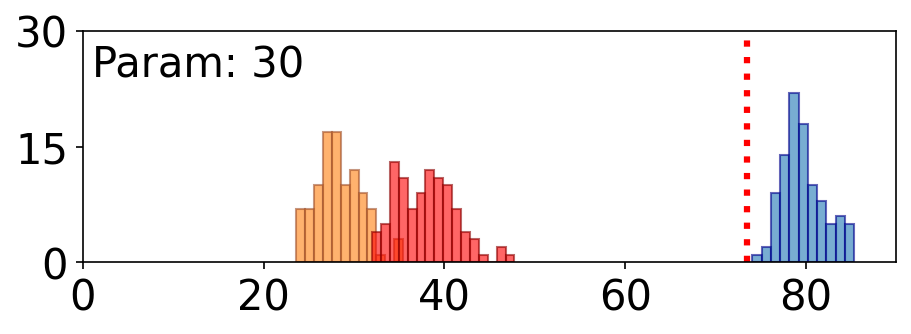}
            \includegraphics[width=1.03\linewidth]{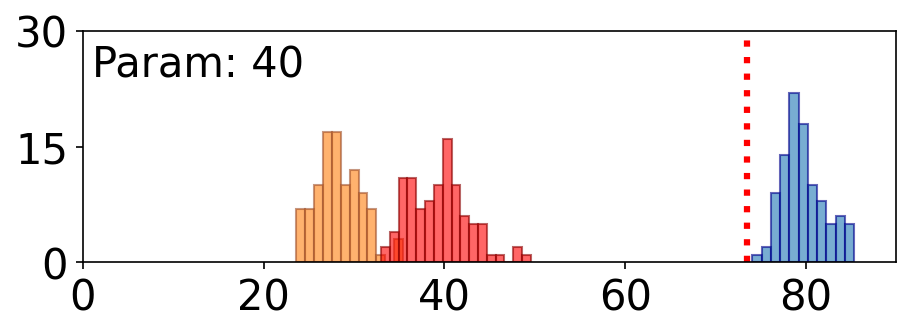}
            \includegraphics[width=1.03\linewidth]{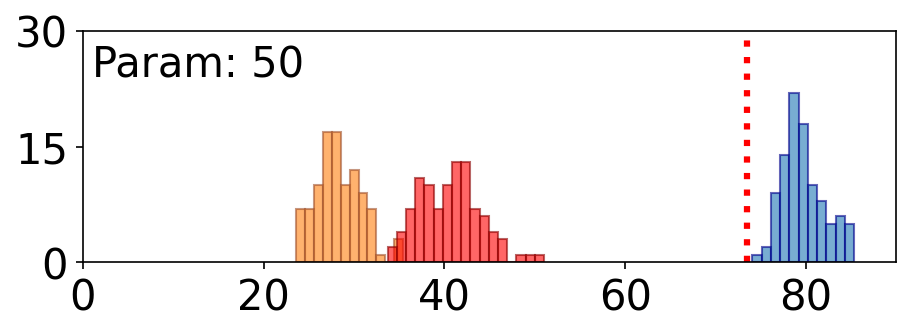}}
        \end{subfigure}
        
    \caption{The effects of different perturbations on $\|\mathbf{w}-\hat{\mathbf{w}}^{\prime}\|$. From left to right, the perturbations are set with different parameters (Param). 
    Watermark: RingID with $\alpha=128$  (RingID-1-128). Thresholds: 73.51 (1e-2 FPR).}
    \label{fig: RingID 128 attack shifting}
\end{figure*}

\begin{figure*}[t]
        \begin{subfigure}[c]{0.185\textwidth}
            \centering
            \subfloat{\centering
            \makebox[10pt][l]{\raisebox{5pt}{\rotatebox[origin=l]{90}{\tiny Blurring}}}
            \includegraphics[width=1.03\linewidth]{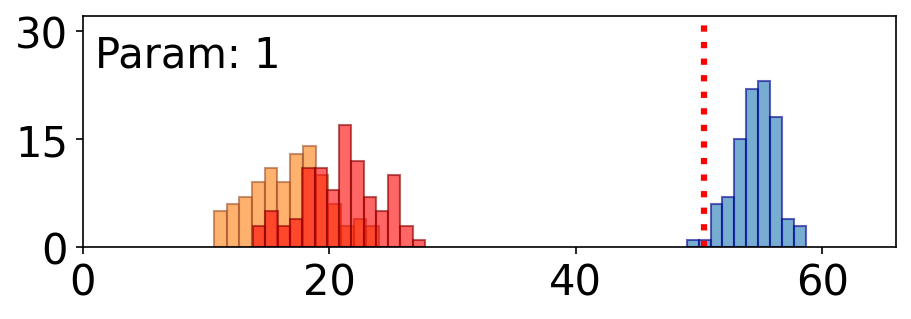}
            \includegraphics[width=1.03\linewidth]{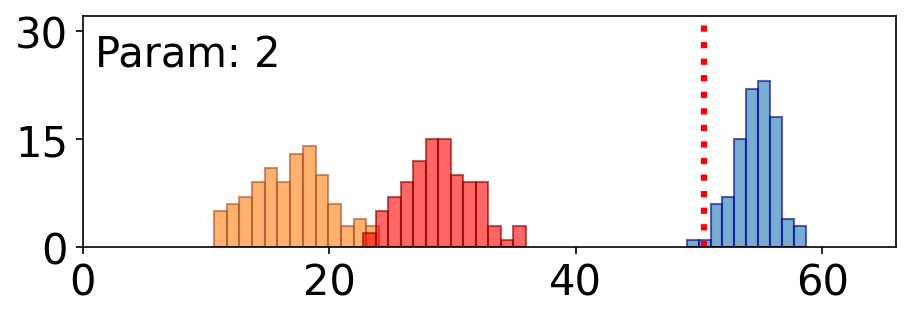}
            \includegraphics[width=1.03\linewidth]{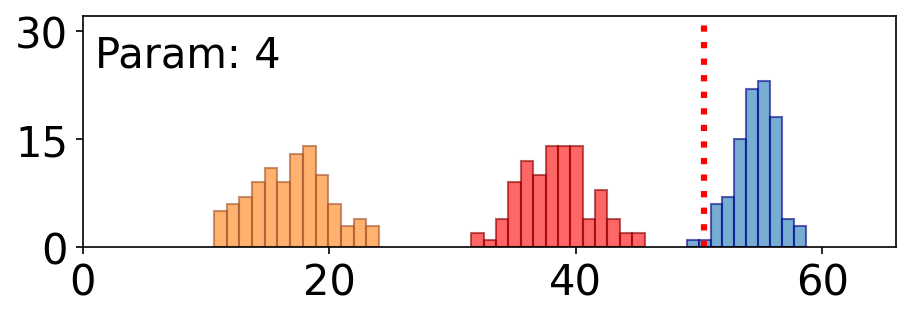}
            \includegraphics[width=1.03\linewidth]{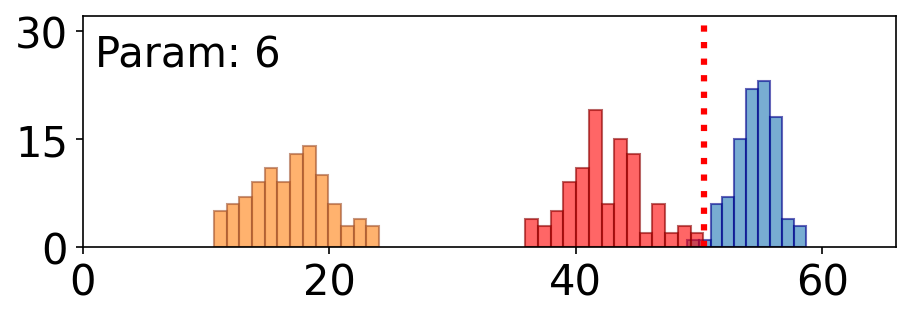}
            \includegraphics[width=1.03\linewidth]{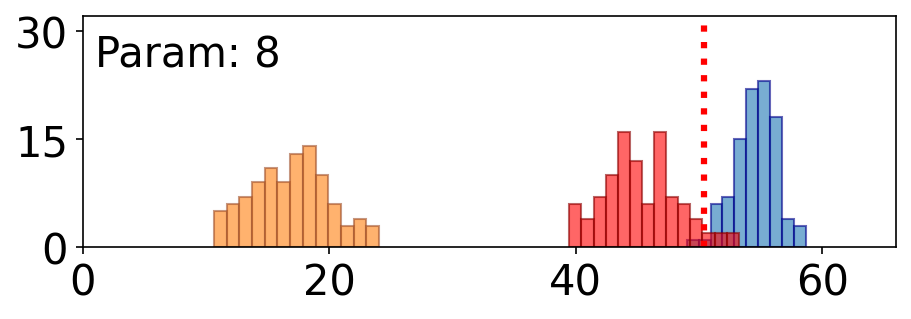}}
    
            \subfloat{\centering
            \makebox[10pt][l]{\raisebox{5pt}{\rotatebox[origin=l]{90}{\tiny Jitter}}}
            \includegraphics[width=1.03\linewidth]{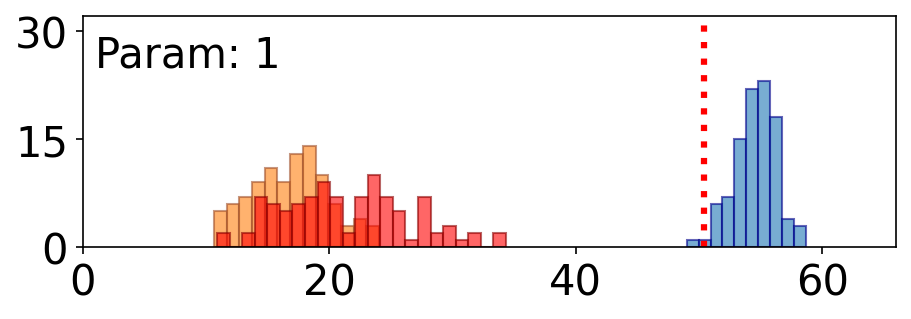}
            \includegraphics[width=1.03\linewidth]{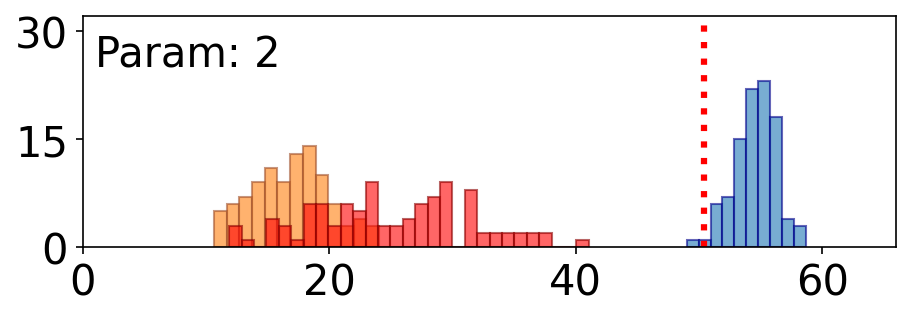}
            \includegraphics[width=1.03\linewidth]{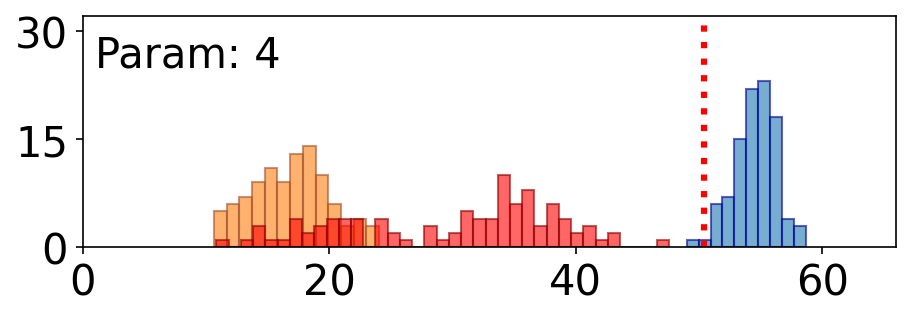}
            \includegraphics[width=1.03\linewidth]{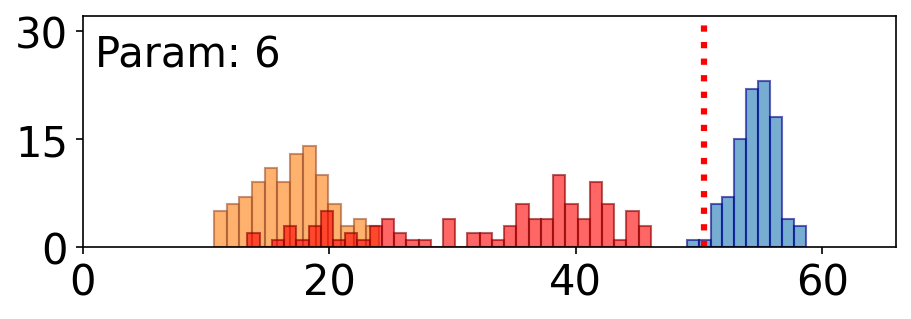}
            \includegraphics[width=1.03\linewidth]{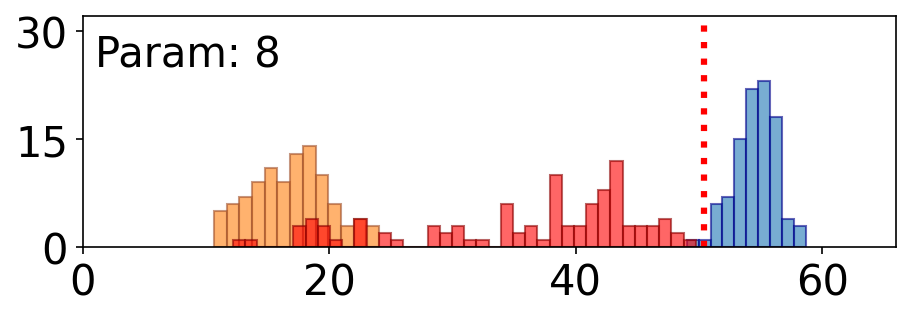}}
    
            \subfloat{\centering
            \makebox[10pt][l]{\raisebox{5pt}{\rotatebox[origin=l]{90}{\tiny JPEG}}}
            \includegraphics[width=1.03\linewidth]{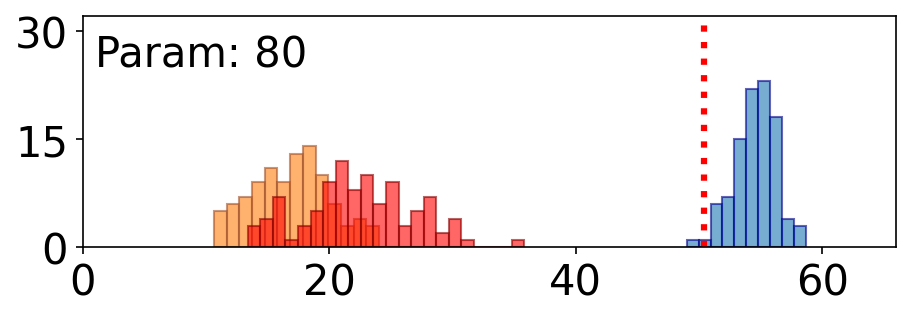}
            \includegraphics[width=1.03\linewidth]{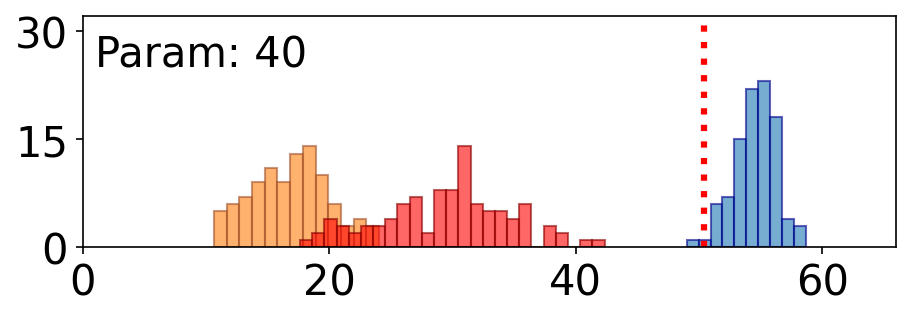}
            \includegraphics[width=1.03\linewidth]{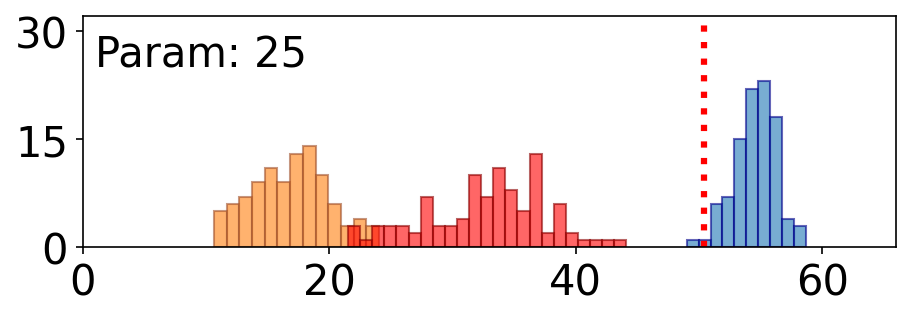}
            \includegraphics[width=1.03\linewidth]{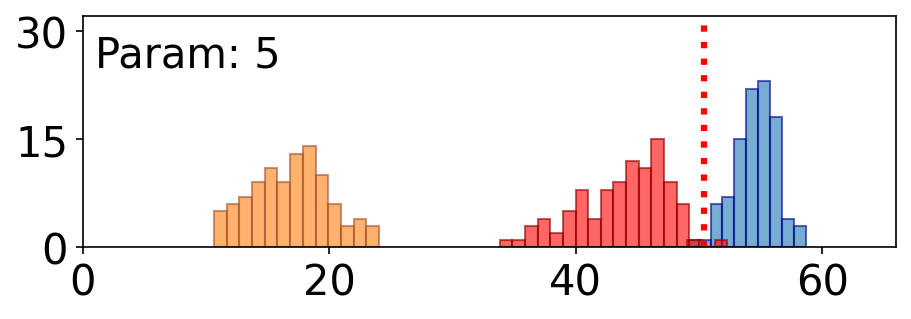}
            \includegraphics[width=1.03\linewidth]{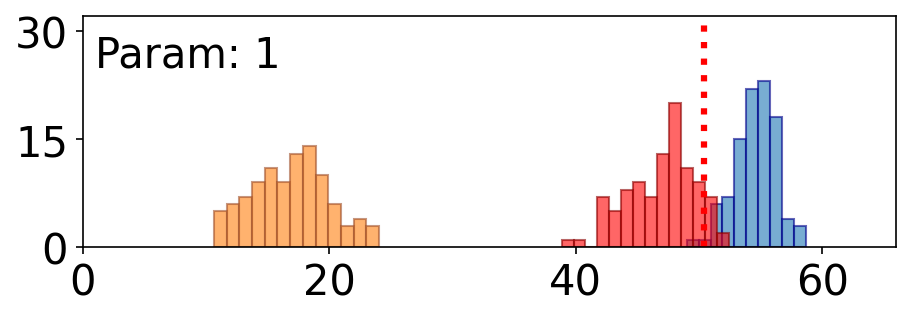}}
    
            \subfloat{\centering
            \makebox[10pt][l]{\raisebox{5pt}{\rotatebox[origin=l]{90}{\tiny Rotation}}}
            \includegraphics[width=1.03\linewidth]{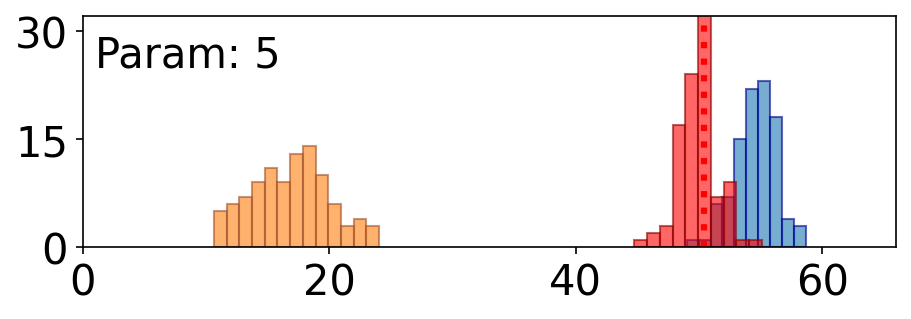}
            \includegraphics[width=1.03\linewidth]{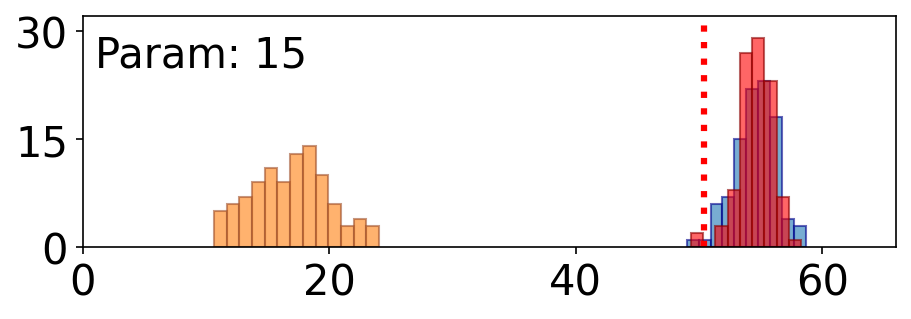}
            \includegraphics[width=1.03\linewidth]{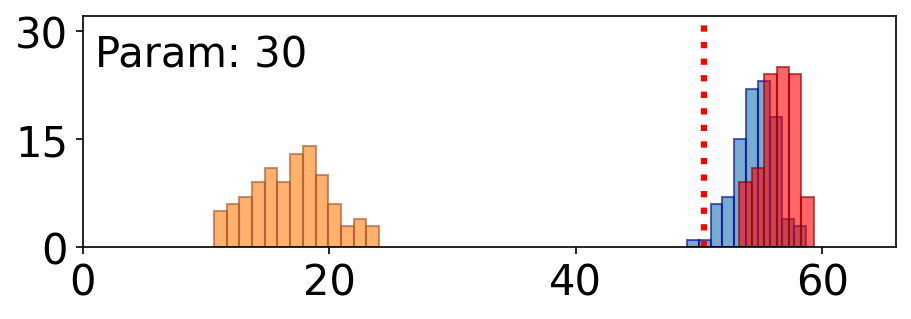}
            \includegraphics[width=1.03\linewidth]{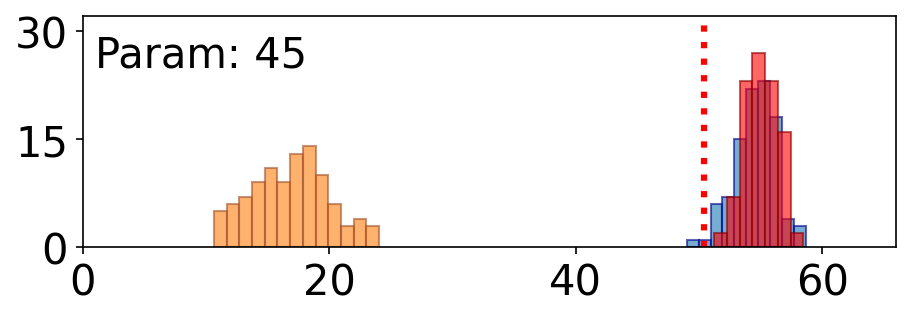}
            \includegraphics[width=1.03\linewidth]{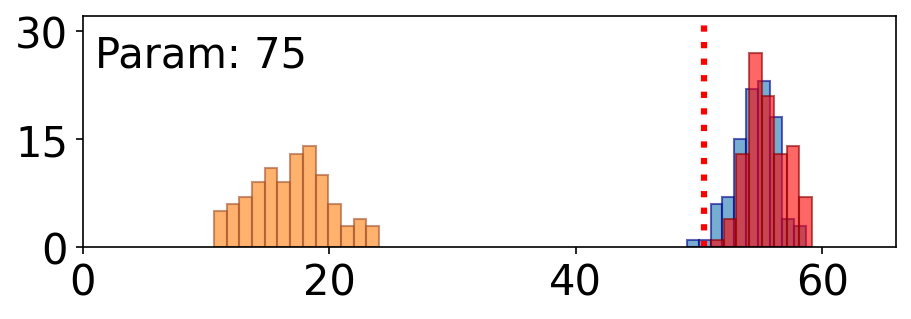}}
    
            \subfloat{\centering
            \makebox[10pt][l]{\raisebox{5pt}{\rotatebox[origin=l]{90}{\tiny Cropping}}}
            \includegraphics[width=1.03\linewidth]{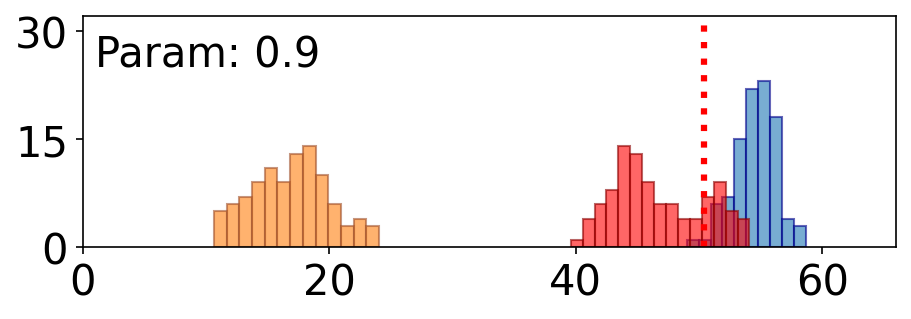}
            \includegraphics[width=1.03\linewidth]{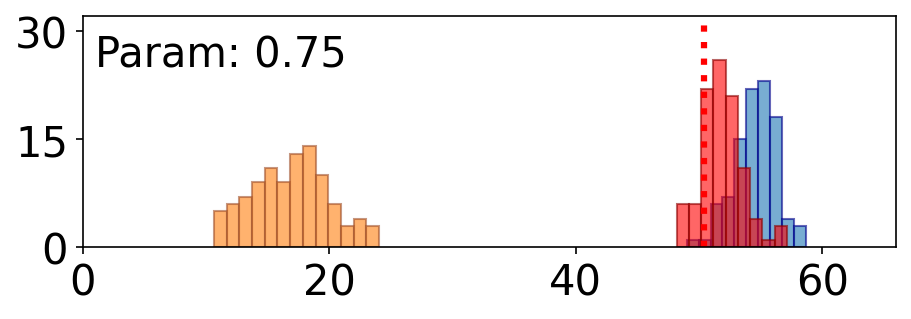}
            \includegraphics[width=1.03\linewidth]{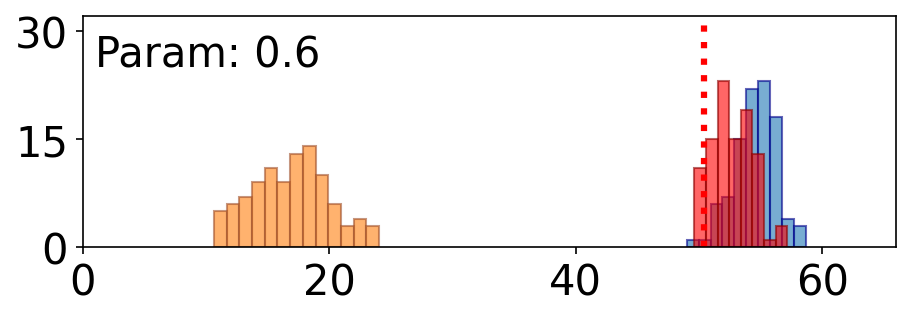}
            \includegraphics[width=1.03\linewidth]{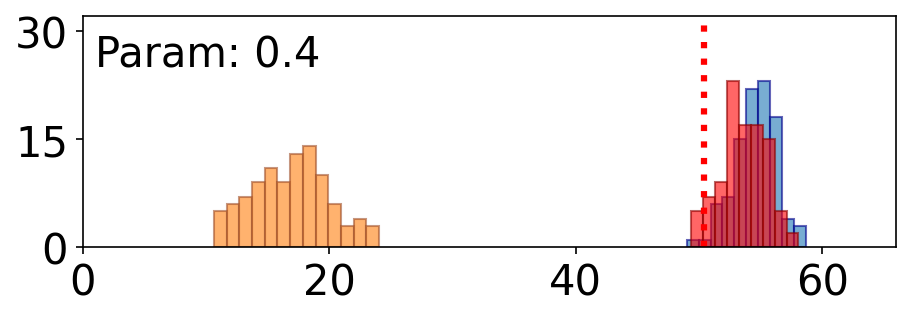}
            \includegraphics[width=1.03\linewidth]{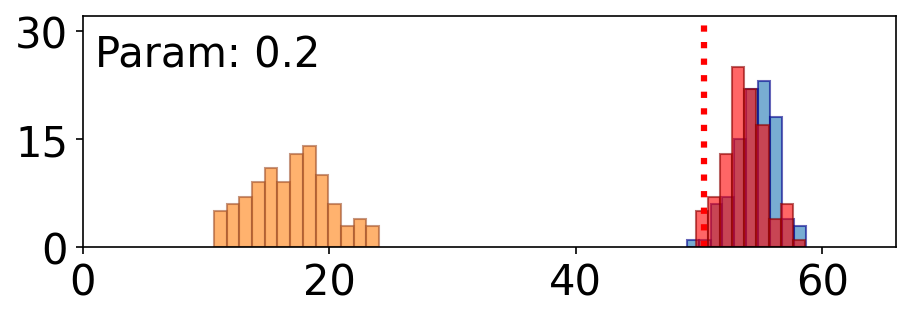}}
    
            \subfloat{\centering
            \makebox[10pt][l]{\raisebox{3pt}{\rotatebox[origin=l]{90}{\tiny Elastic}}}
            \includegraphics[width=1.03\linewidth]{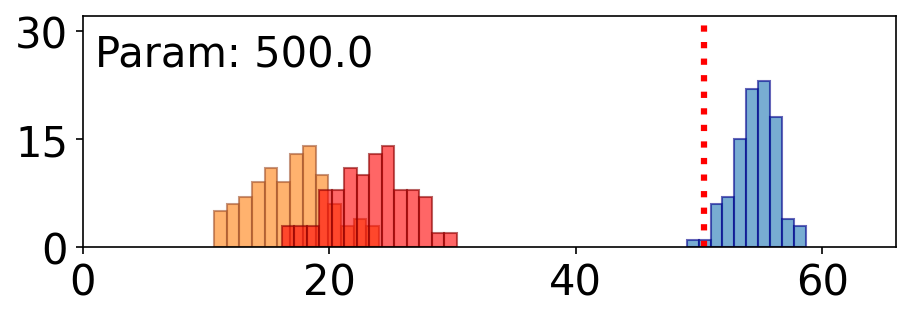}
            \includegraphics[width=1.03\linewidth]{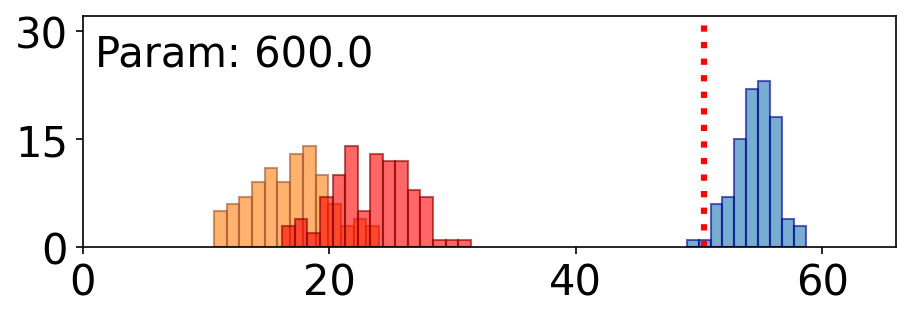}
            \includegraphics[width=1.03\linewidth]{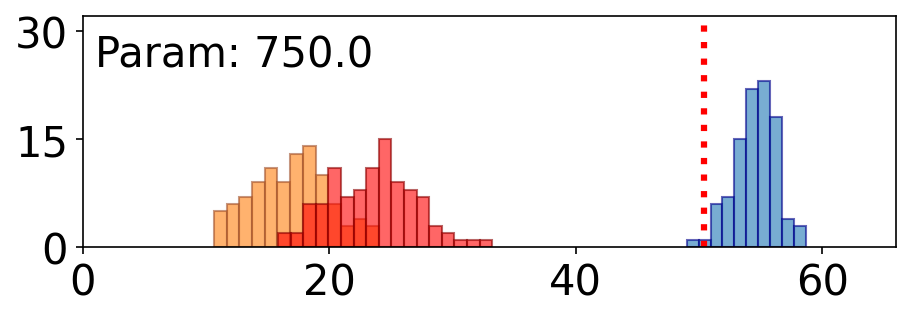}
            \includegraphics[width=1.03\linewidth]{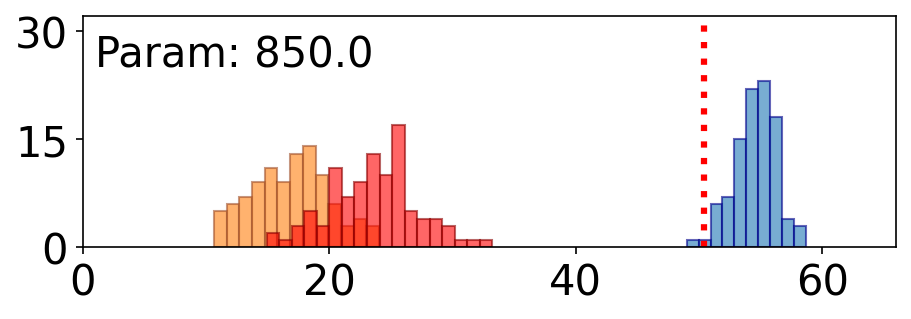}
            \includegraphics[width=1.03\linewidth]{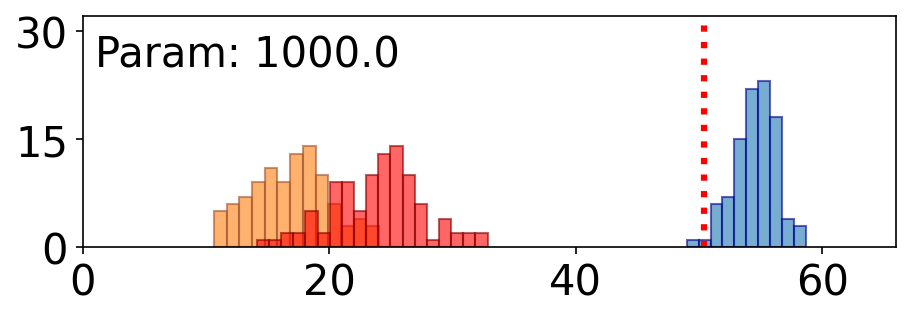}}
    
            \subfloat{\centering
            \makebox[10pt][l]{\raisebox{3pt}{\rotatebox[origin=l]{90}{\scalebox{.4}{Perspective}}}}
            \includegraphics[width=1.03\linewidth]{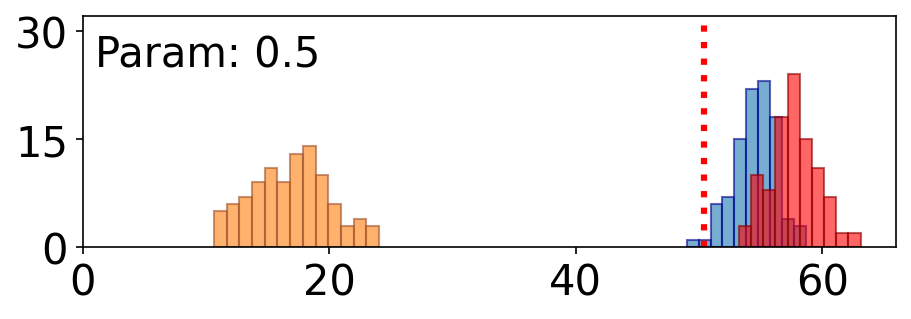}
            \includegraphics[width=1.03\linewidth]{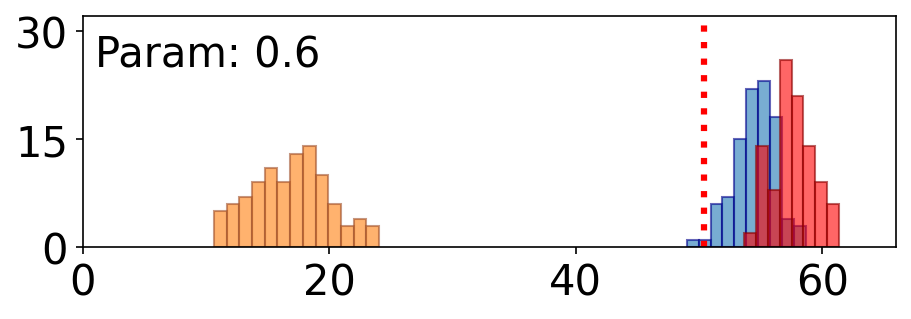}
            \includegraphics[width=1.03\linewidth]{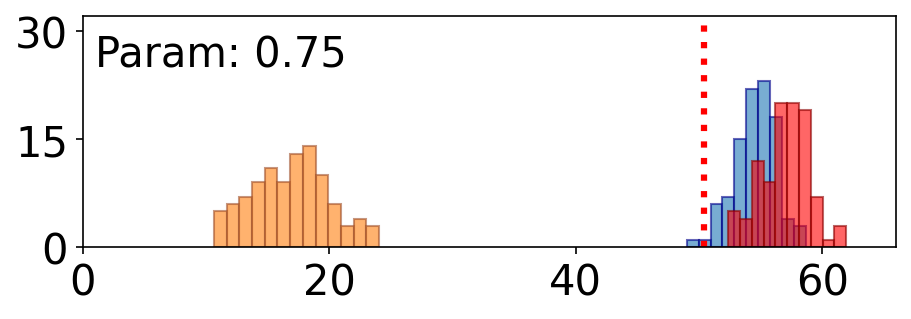}
            \includegraphics[width=1.03\linewidth]{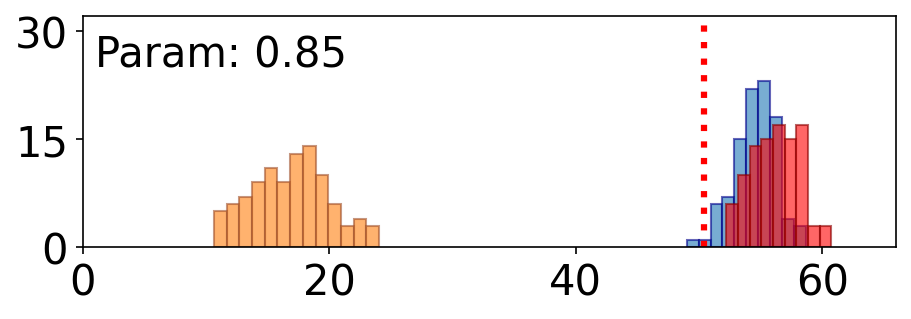}
            \includegraphics[width=1.03\linewidth]{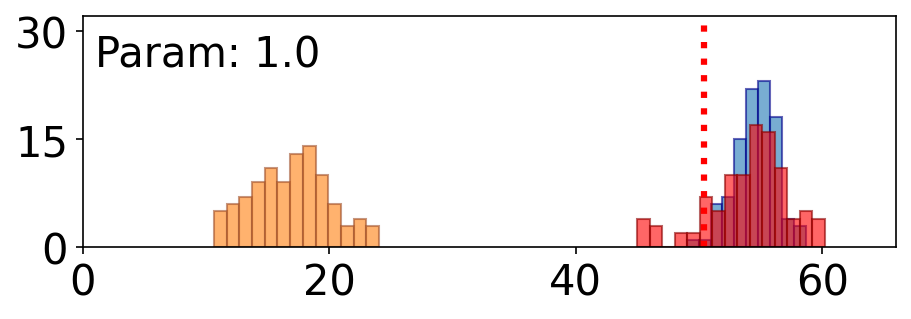}}
    
            \subfloat{\centering
            \makebox[10pt][l]{\raisebox{3pt}{\rotatebox[origin=l]{90}{\tiny Shear}}}
            \includegraphics[width=1.03\linewidth]{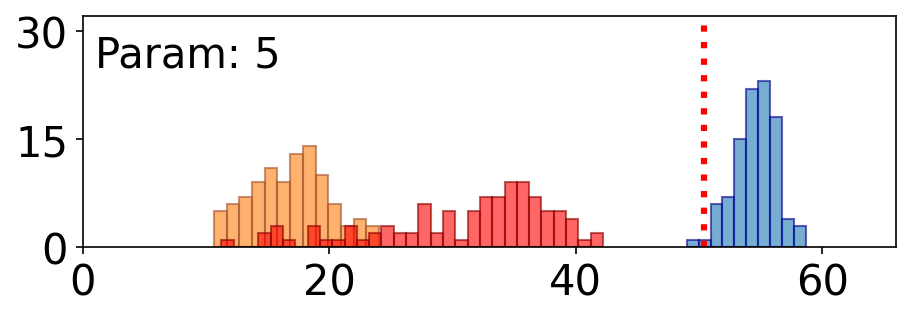}
            \includegraphics[width=1.03\linewidth]{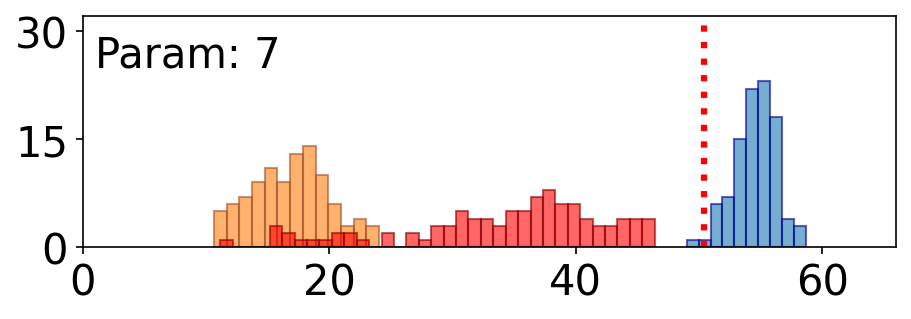}
            \includegraphics[width=1.03\linewidth]{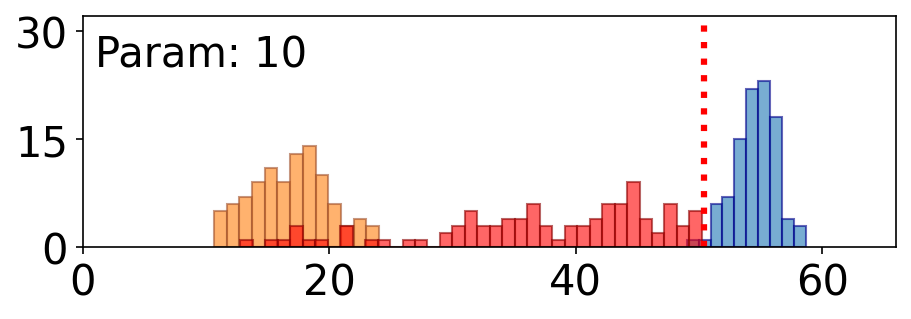}
            \includegraphics[width=1.03\linewidth]{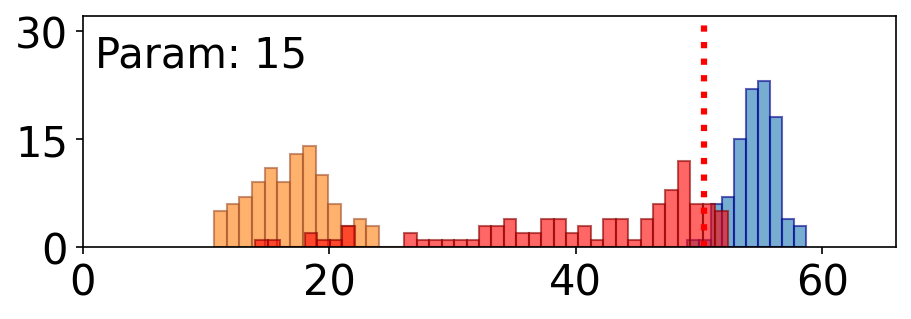}
            \includegraphics[width=1.03\linewidth]{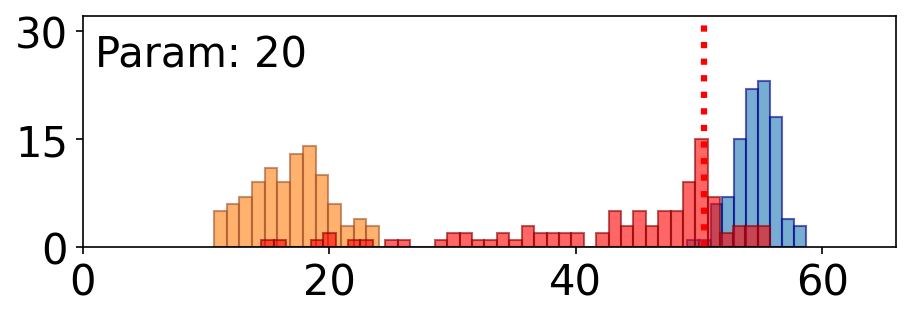}}

            \subfloat{\centering
            \makebox[10pt][l]{\raisebox{3pt}{\rotatebox[origin=l]{90}{\scalebox{.4}{Regen-VAE}}}}
            \includegraphics[width=1.03\linewidth]{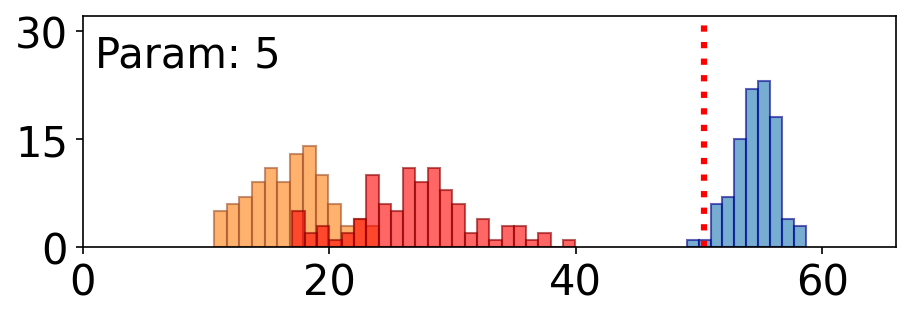}
            \includegraphics[width=1.03\linewidth]{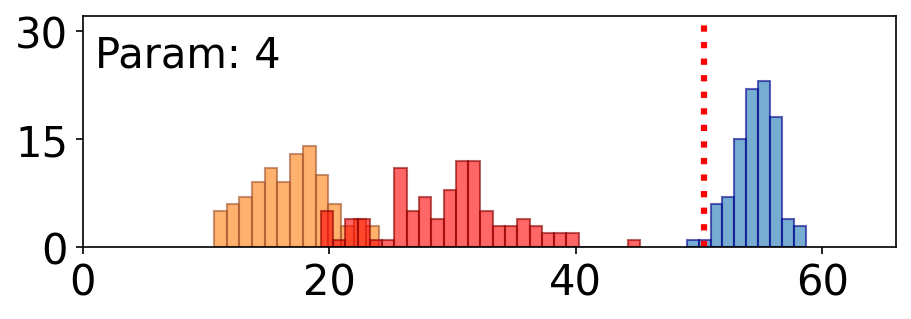}
            \includegraphics[width=1.03\linewidth]{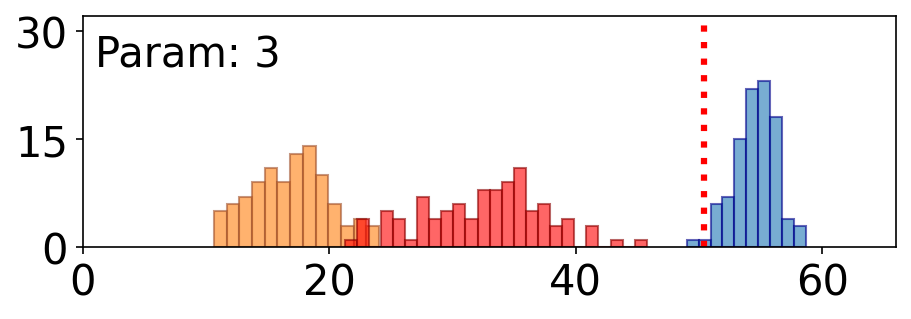}
            \includegraphics[width=1.03\linewidth]{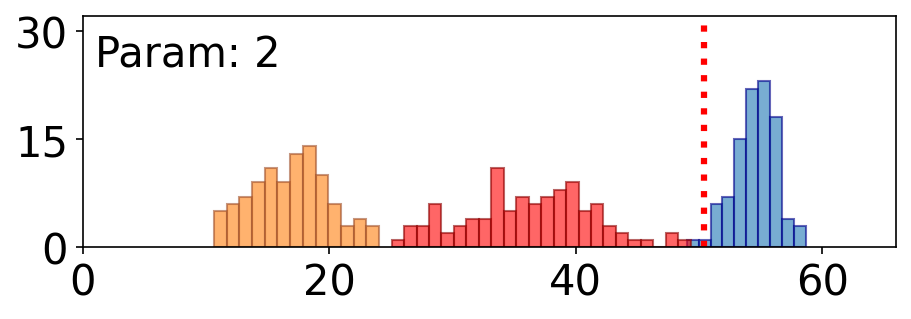}
            \includegraphics[width=1.03\linewidth]{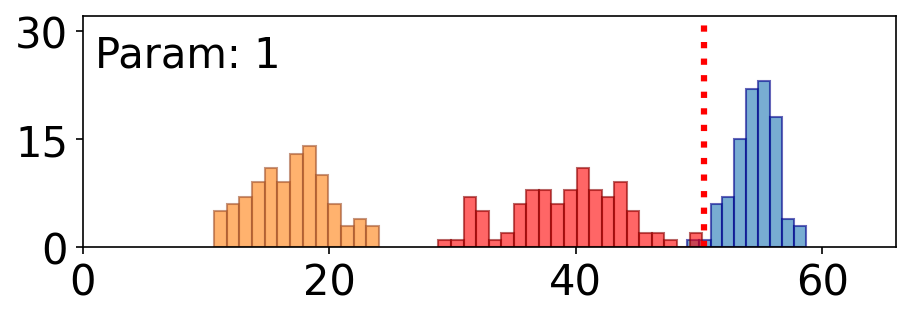}}
    
            \subfloat{\centering
            \makebox[10pt][l]{\raisebox{0pt}{\rotatebox[origin=l]{90}{\scalebox{.4}{Regen-Diff}}}}
            \includegraphics[width=1.03\linewidth]{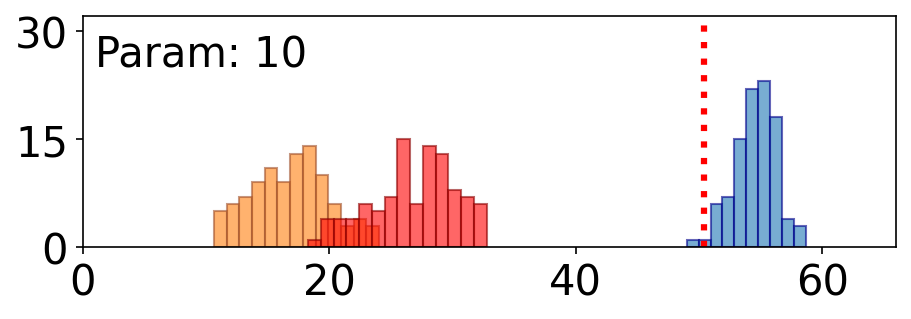}
            \includegraphics[width=1.03\linewidth]{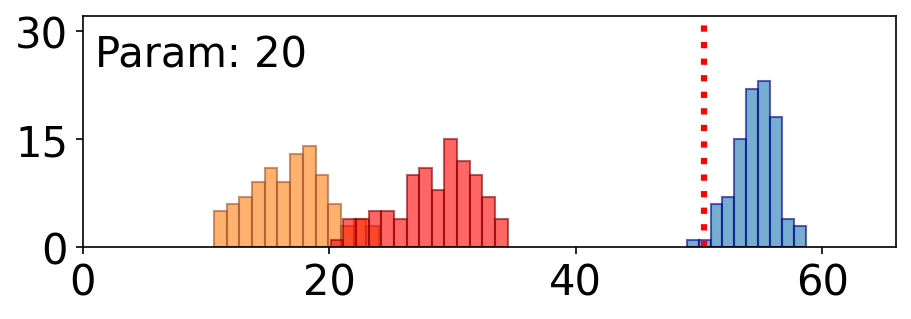}
            \includegraphics[width=1.03\linewidth]{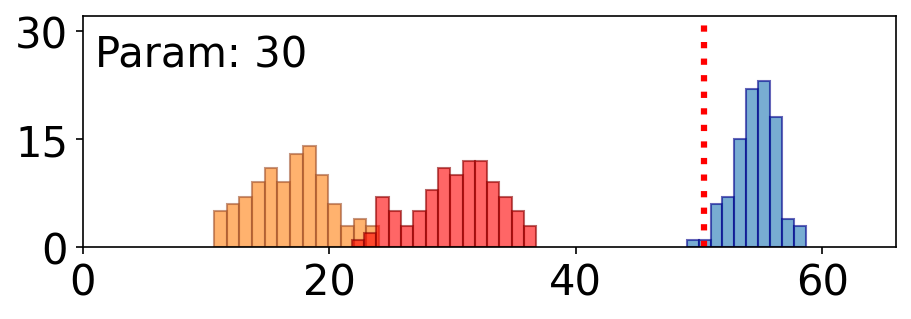}
            \includegraphics[width=1.03\linewidth]{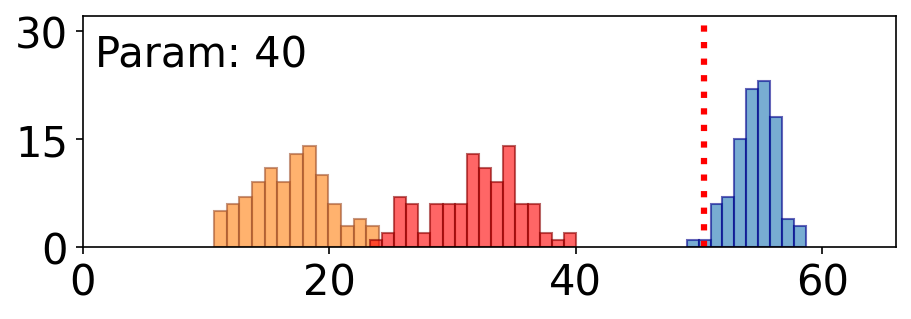}
            \includegraphics[width=1.03\linewidth]{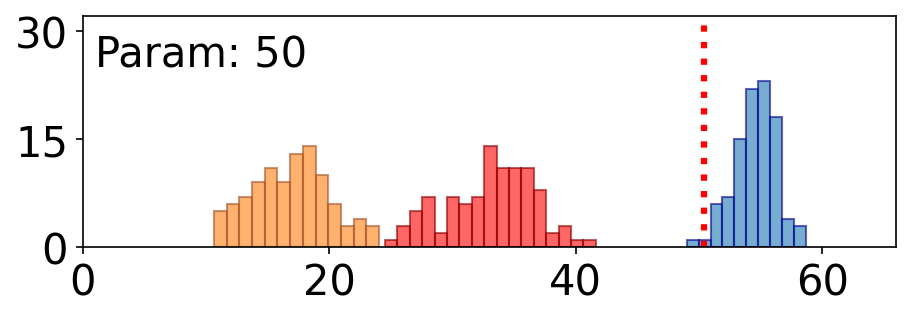}}
        \end{subfigure}
        
    \caption{The effects of different perturbations on $\|\mathbf{w}-\hat{\mathbf{w}}^{\prime}\|$. From left to right, the perturbations are set with different parameters (Param). 
    Watermark: HSTR. Thresholds: 50.39 (1e-2 FPR).}
    \label{fig: HSTR 1 1 attack shifting}
\end{figure*}

\begin{figure*}[t]
        \begin{subfigure}[c]{0.185\textwidth}
            \centering
            \subfloat{\centering
            \makebox[10pt][l]{\raisebox{5pt}{\rotatebox[origin=l]{90}{\tiny Blurring}}}
            \includegraphics[width=1.03\linewidth]{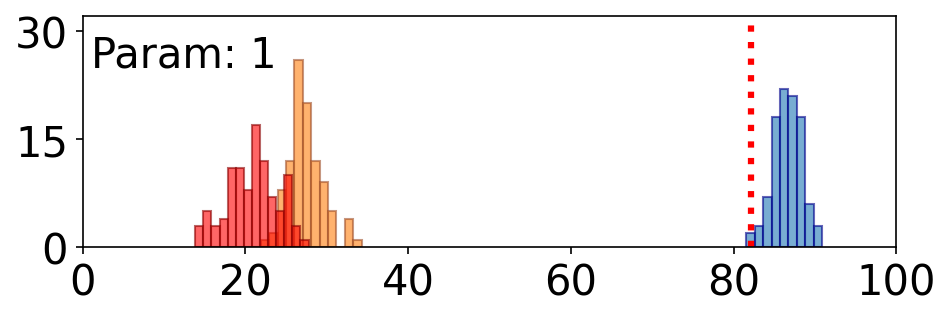}
            \includegraphics[width=1.03\linewidth]{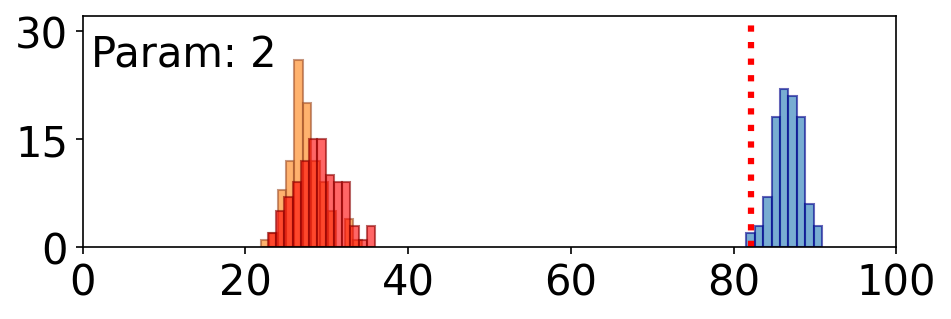}
            \includegraphics[width=1.03\linewidth]{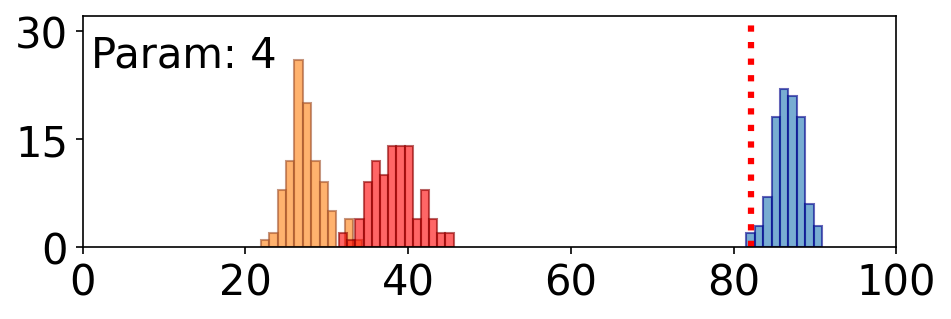}
            \includegraphics[width=1.03\linewidth]{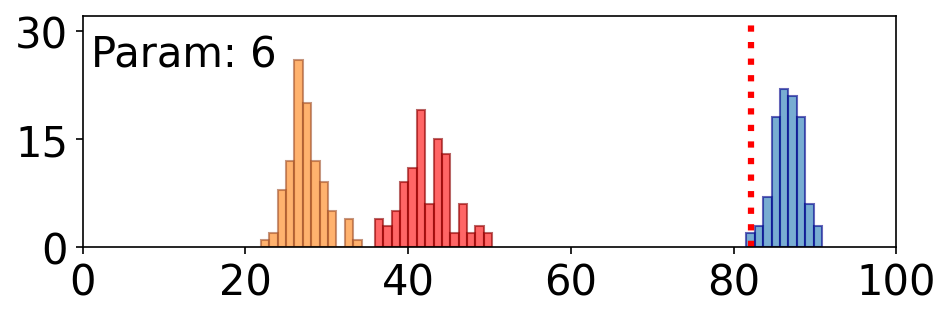}
            \includegraphics[width=1.03\linewidth]{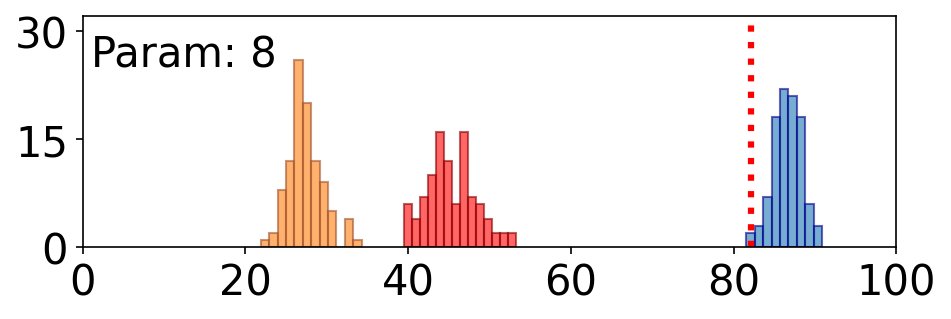}}
    
            \subfloat{\centering
            \makebox[10pt][l]{\raisebox{5pt}{\rotatebox[origin=l]{90}{\tiny Jitter}}}
            \includegraphics[width=1.03\linewidth]{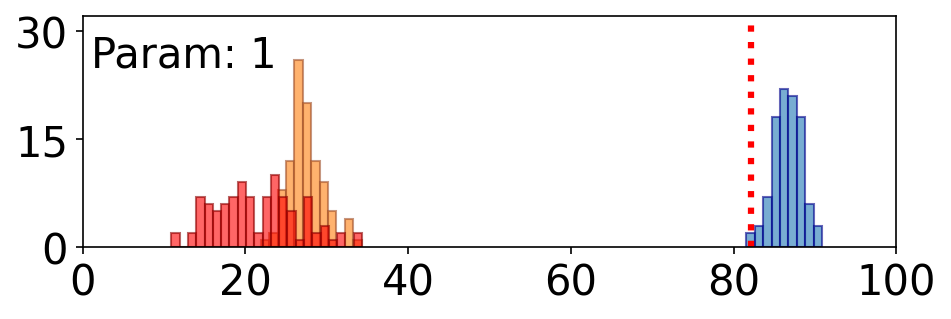}
            \includegraphics[width=1.03\linewidth]{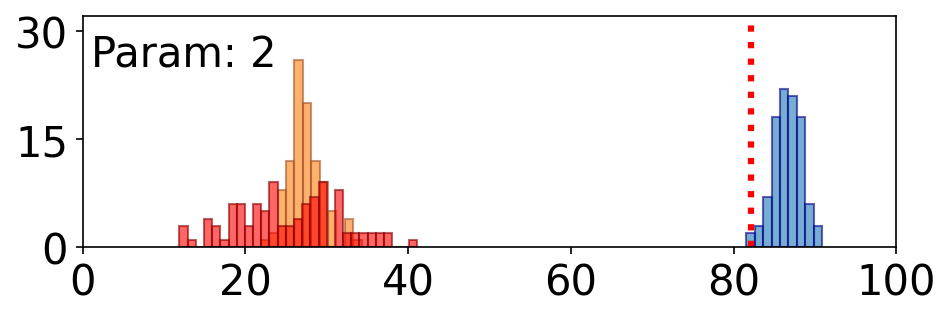}
            \includegraphics[width=1.03\linewidth]{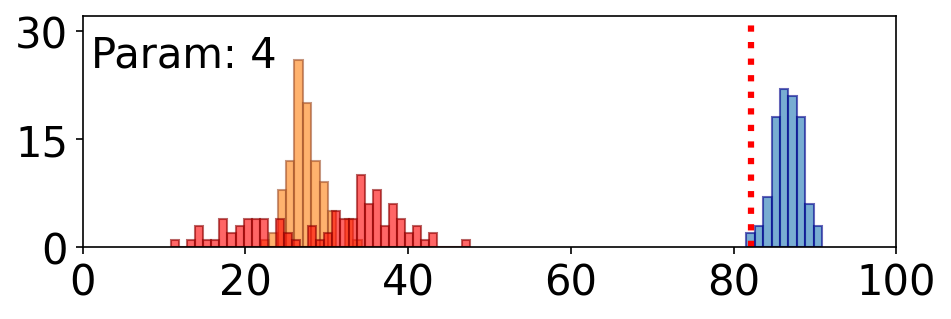}
            \includegraphics[width=1.03\linewidth]{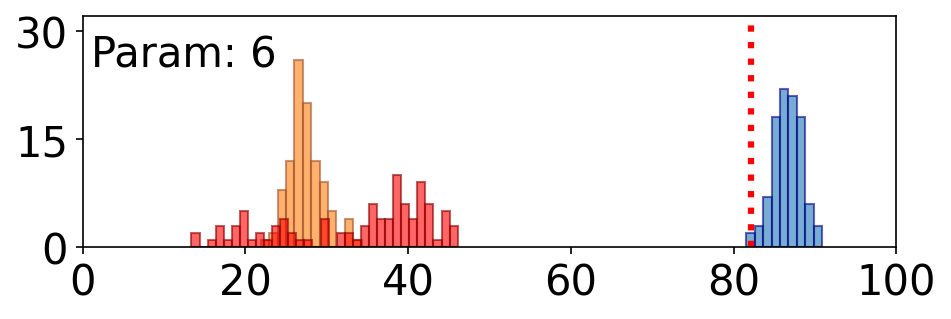}
            \includegraphics[width=1.03\linewidth]{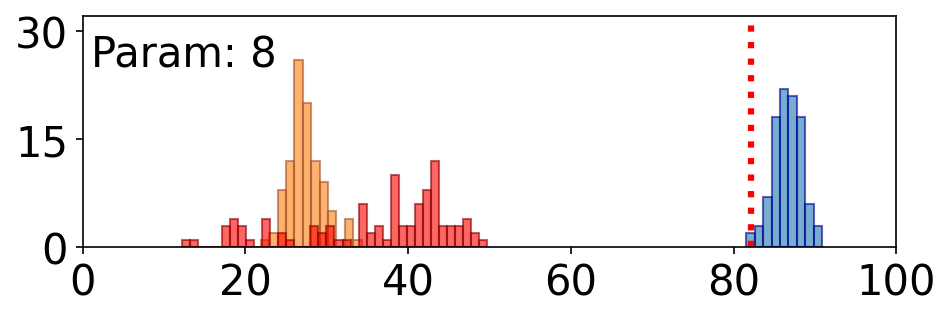}}
    
            \subfloat{\centering
            \makebox[10pt][l]{\raisebox{5pt}{\rotatebox[origin=l]{90}{\tiny JPEG}}}
            \includegraphics[width=1.03\linewidth]{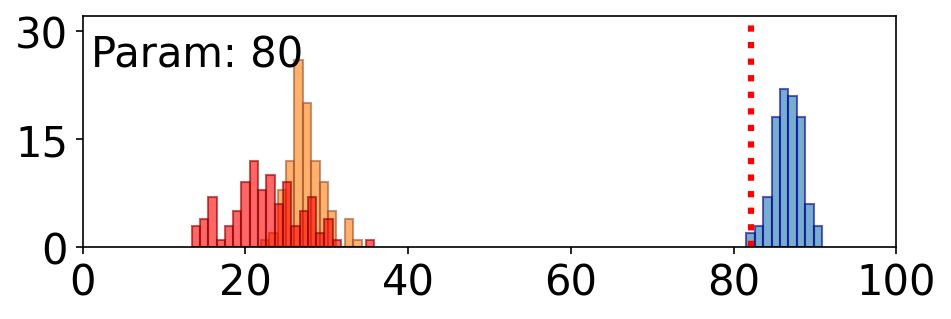}
            \includegraphics[width=1.03\linewidth]{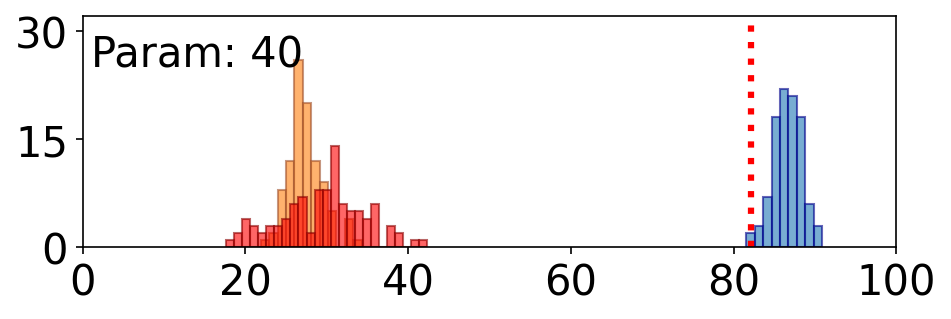}
            \includegraphics[width=1.03\linewidth]{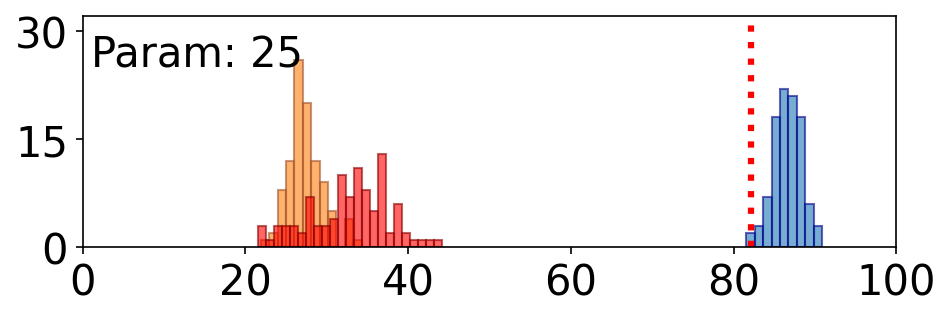}
            \includegraphics[width=1.03\linewidth]{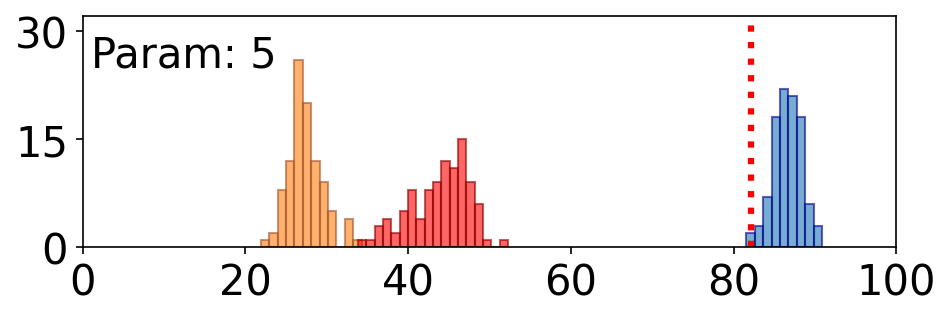}
            \includegraphics[width=1.03\linewidth]{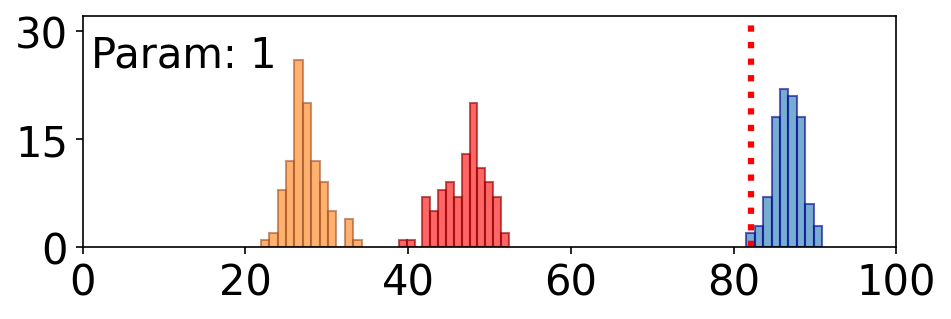}}
    
            \subfloat{\centering
            \makebox[10pt][l]{\raisebox{5pt}{\rotatebox[origin=l]{90}{\tiny Rotation}}}
            \includegraphics[width=1.03\linewidth]{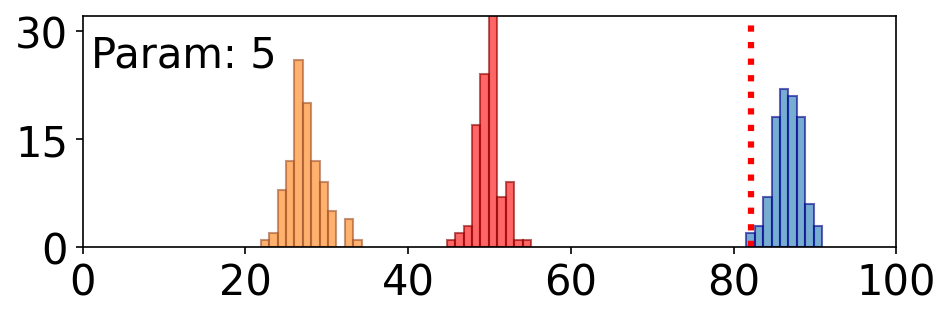}
            \includegraphics[width=1.03\linewidth]{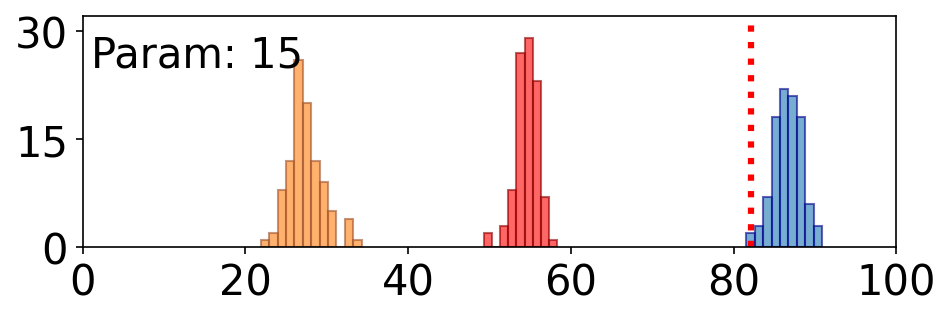}
            \includegraphics[width=1.03\linewidth]{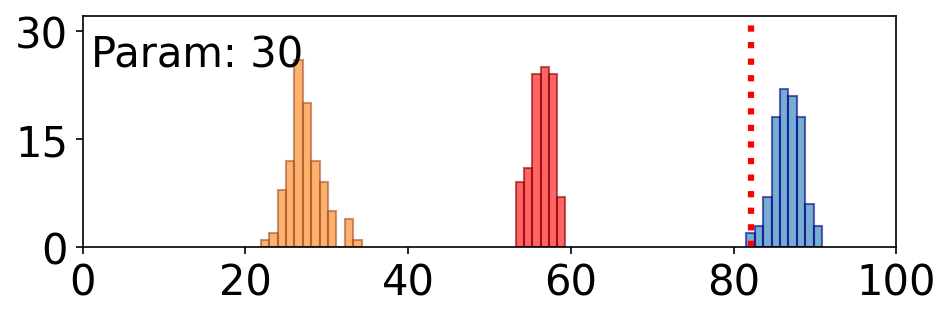}
            \includegraphics[width=1.03\linewidth]{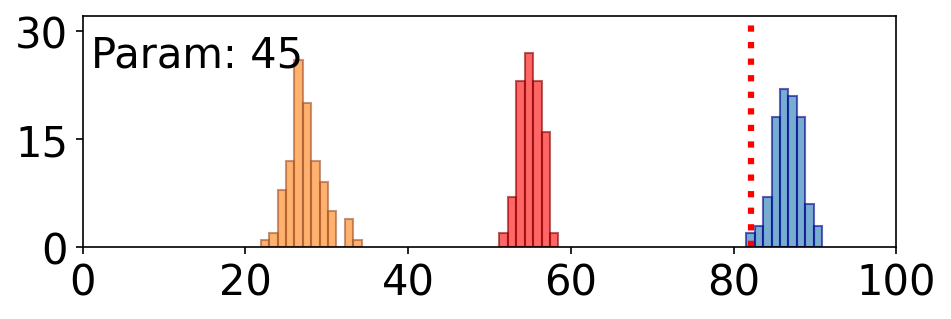}
            \includegraphics[width=1.03\linewidth]{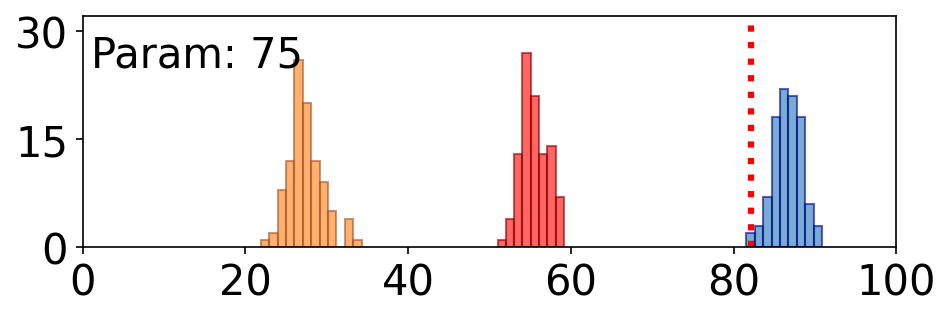}}
    
            \subfloat{\centering
            \makebox[10pt][l]{\raisebox{5pt}{\rotatebox[origin=l]{90}{\tiny Cropping}}}
            \includegraphics[width=1.03\linewidth]{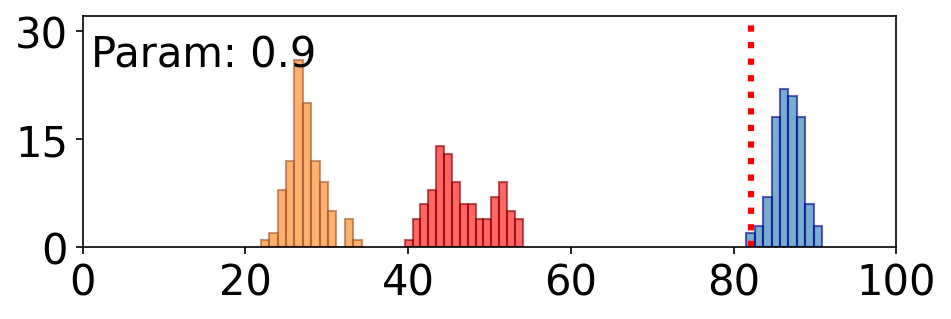}
            \includegraphics[width=1.03\linewidth]{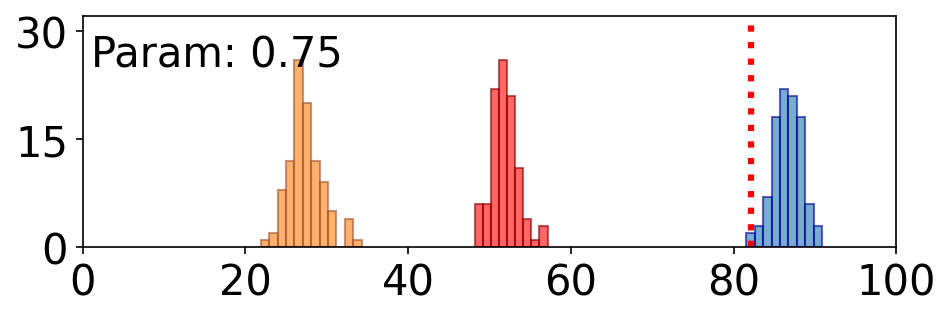}
            \includegraphics[width=1.03\linewidth]{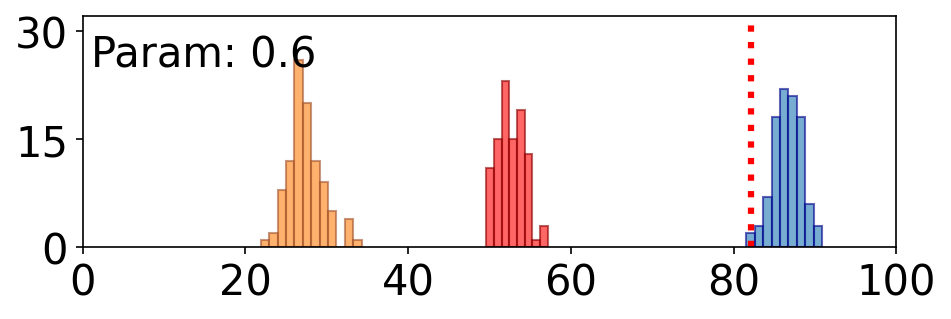}
            \includegraphics[width=1.03\linewidth]{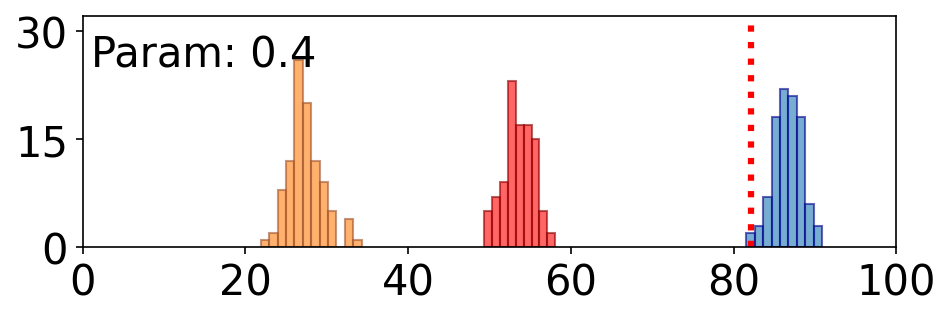}
            \includegraphics[width=1.03\linewidth]{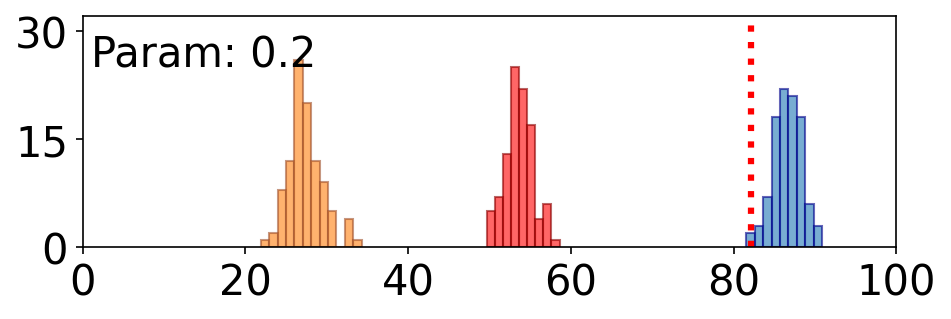}}
    
            \subfloat{\centering
            \makebox[10pt][l]{\raisebox{3pt}{\rotatebox[origin=l]{90}{\tiny Elastic}}}
            \includegraphics[width=1.03\linewidth]{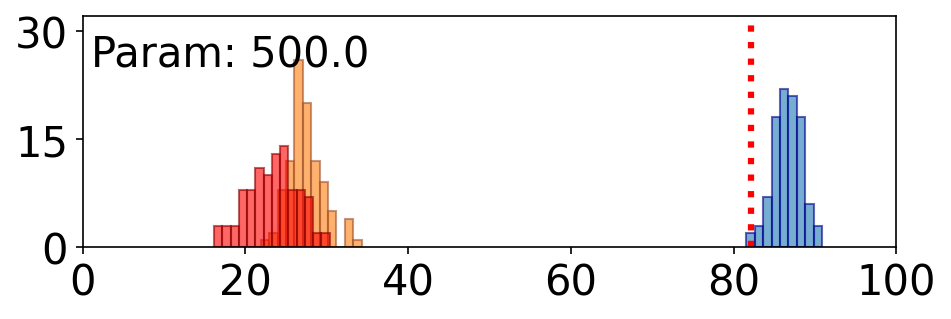}
            \includegraphics[width=1.03\linewidth]{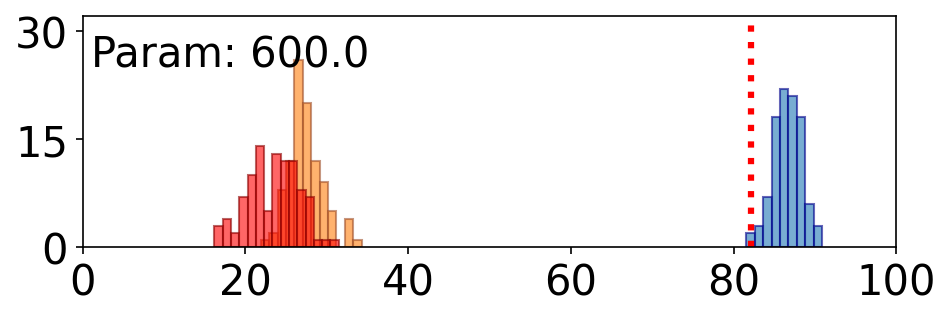}
            \includegraphics[width=1.03\linewidth]{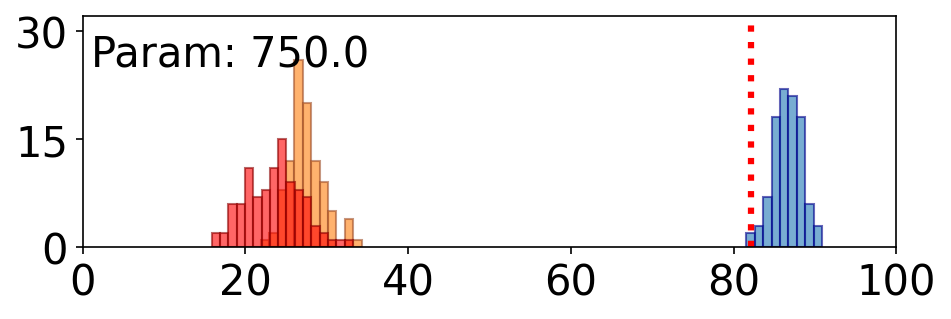}
            \includegraphics[width=1.03\linewidth]{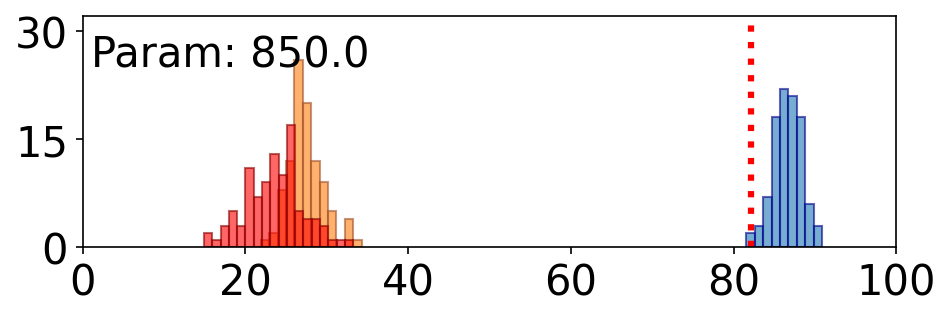}
            \includegraphics[width=1.03\linewidth]{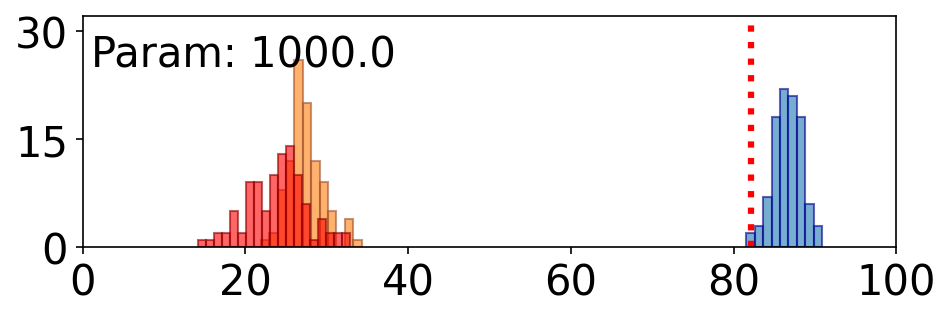}}
    
            \subfloat{\centering
            \makebox[10pt][l]{\raisebox{3pt}{\rotatebox[origin=l]{90}{\scalebox{.4}{Perspective}}}}
            \includegraphics[width=1.03\linewidth]{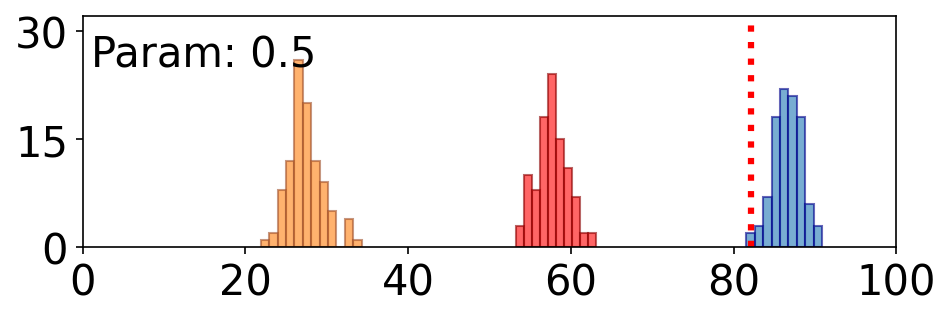}
            \includegraphics[width=1.03\linewidth]{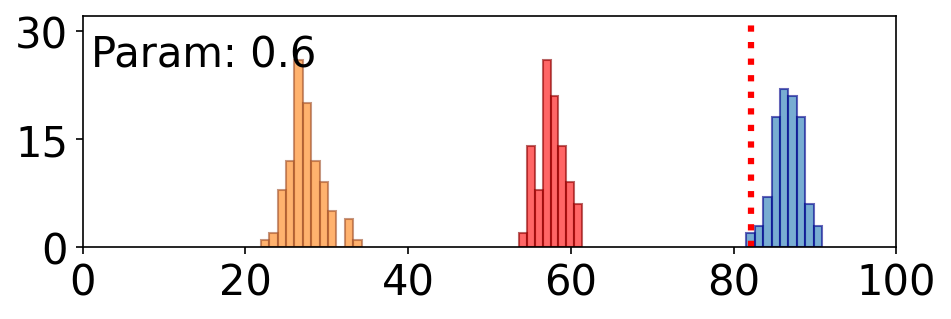}
            \includegraphics[width=1.03\linewidth]{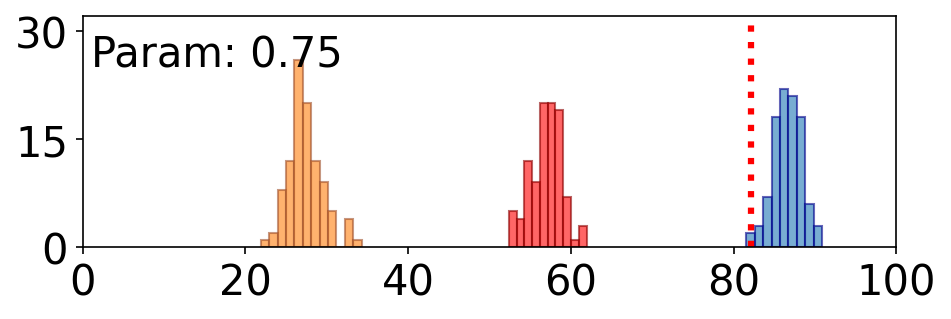}
            \includegraphics[width=1.03\linewidth]{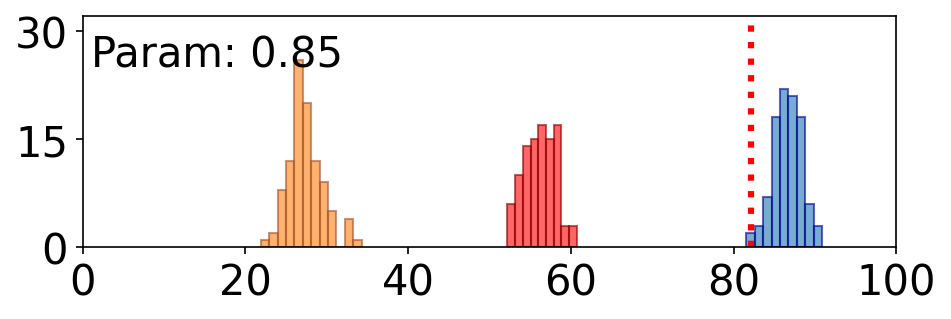}
            \includegraphics[width=1.03\linewidth]{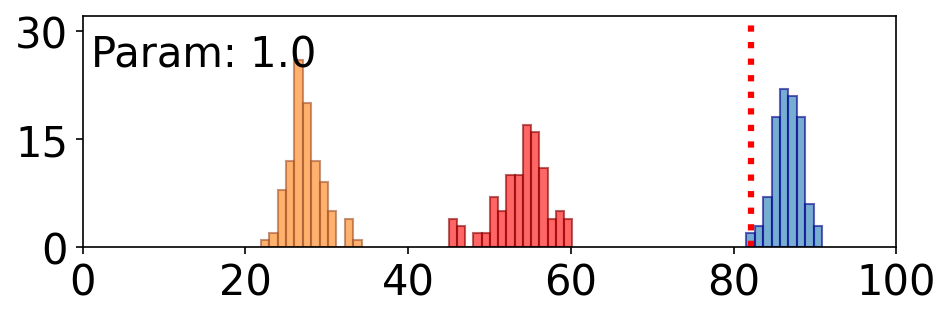}}
    
            \subfloat{\centering
            \makebox[10pt][l]{\raisebox{3pt}{\rotatebox[origin=l]{90}{\tiny Shear}}}
            \includegraphics[width=1.03\linewidth]{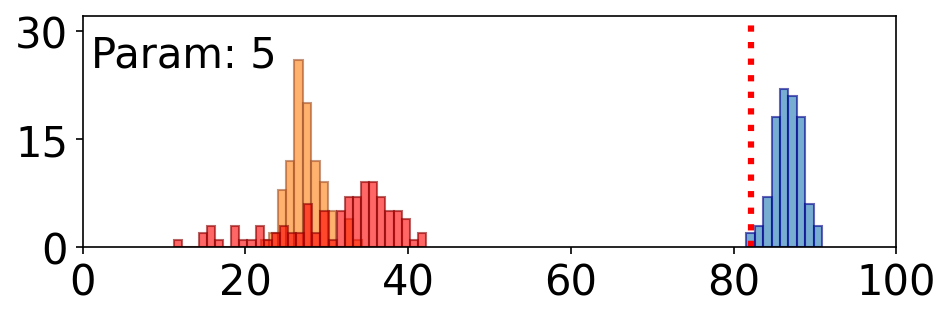}
            \includegraphics[width=1.03\linewidth]{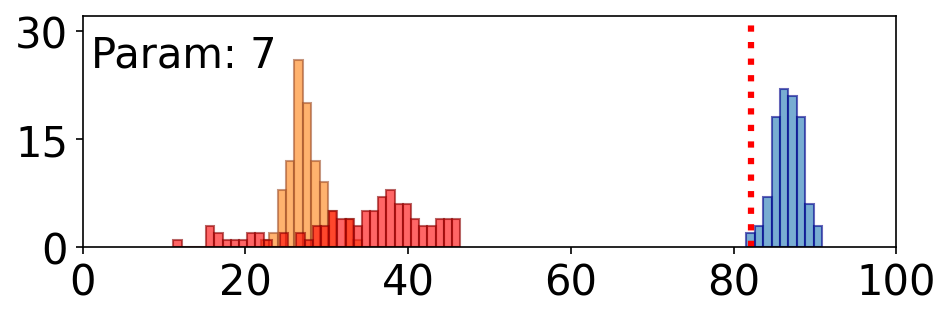}
            \includegraphics[width=1.03\linewidth]{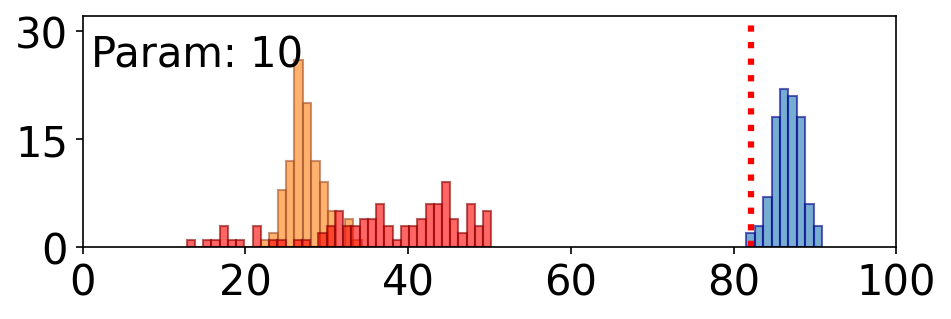}
            \includegraphics[width=1.03\linewidth]{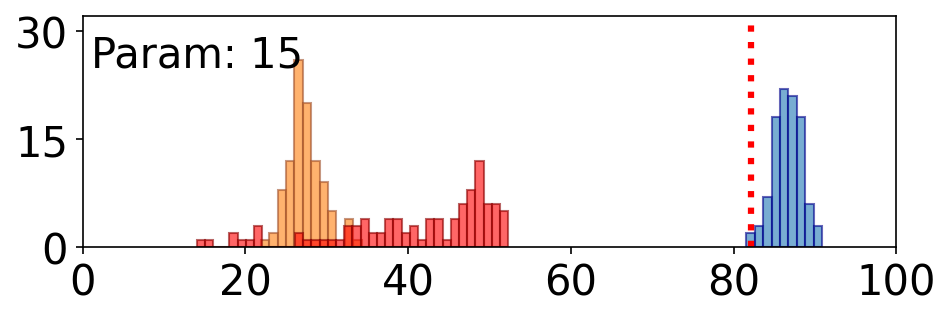}
            \includegraphics[width=1.03\linewidth]{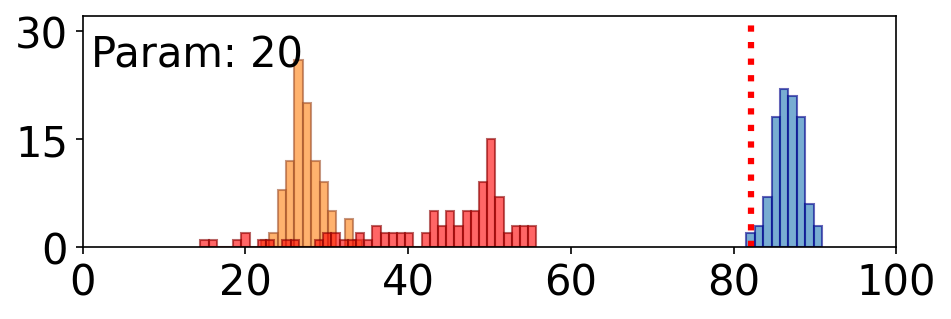}}

            \subfloat{\centering
            \makebox[10pt][l]{\raisebox{3pt}{\rotatebox[origin=l]{90}{\scalebox{.4}{Regen-VAE}}}}
            \includegraphics[width=1.03\linewidth]{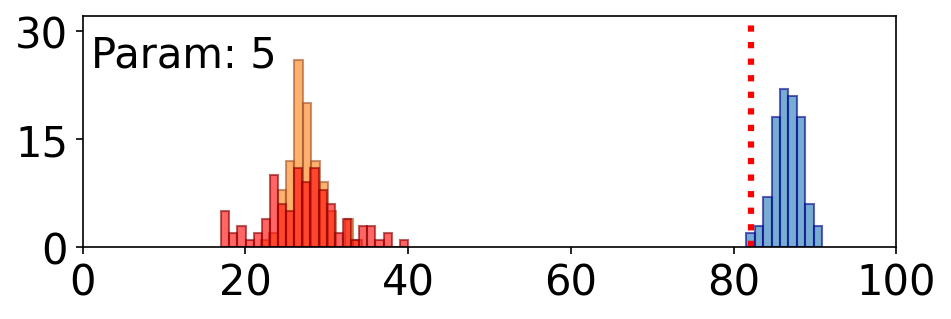}
            \includegraphics[width=1.03\linewidth]{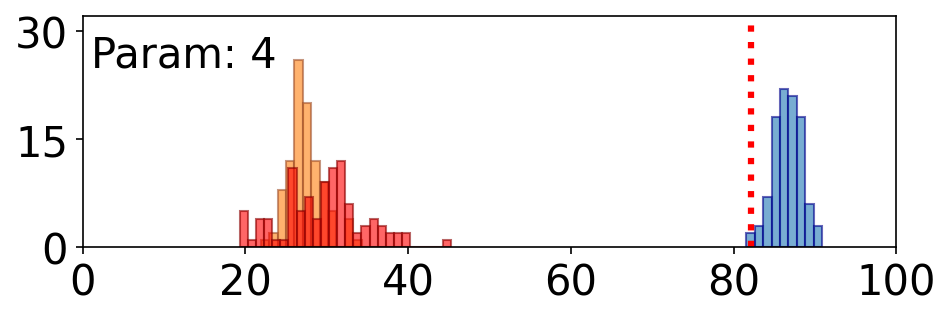}
            \includegraphics[width=1.03\linewidth]{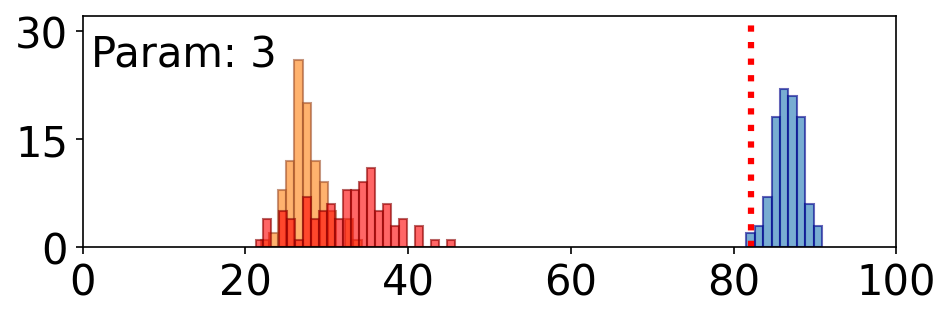}
            \includegraphics[width=1.03\linewidth]{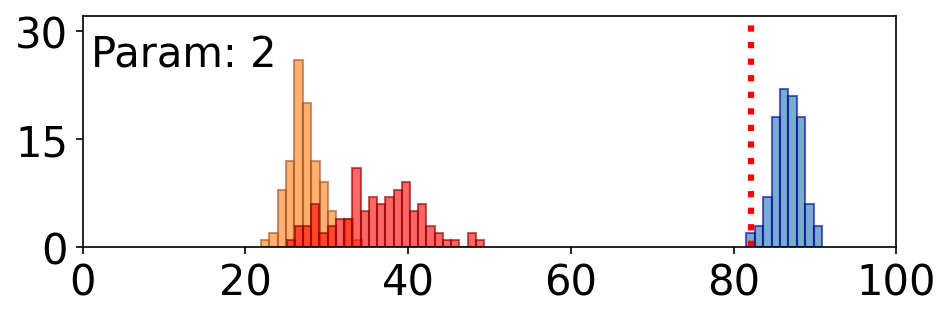}
            \includegraphics[width=1.03\linewidth]{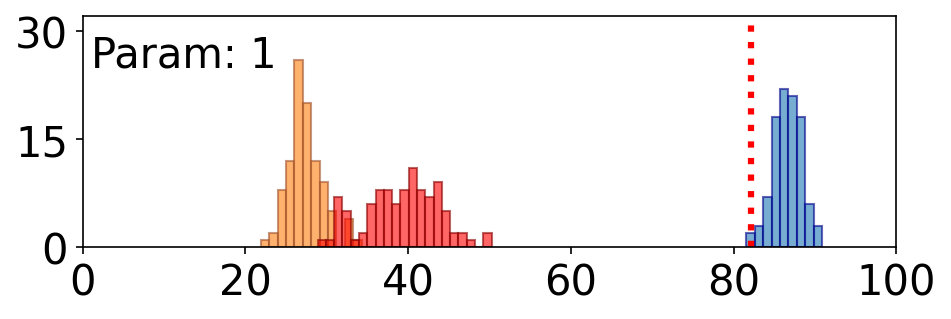}}
    
            \subfloat{\centering
            \makebox[10pt][l]{\raisebox{0pt}{\rotatebox[origin=l]{90}{\scalebox{.4}{Regen-Diff}}}}
            \includegraphics[width=1.03\linewidth]{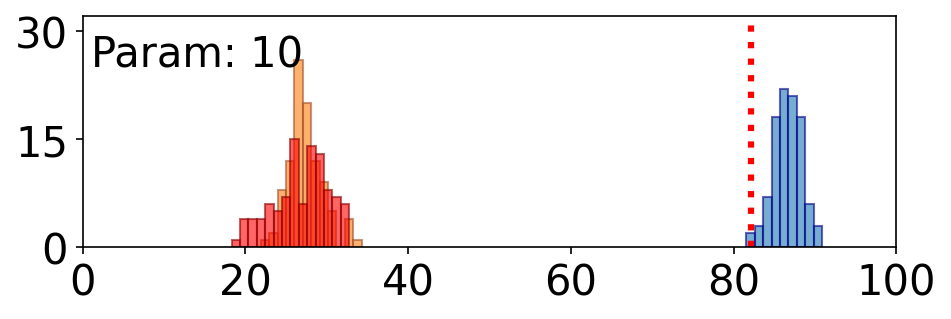}
            \includegraphics[width=1.03\linewidth]{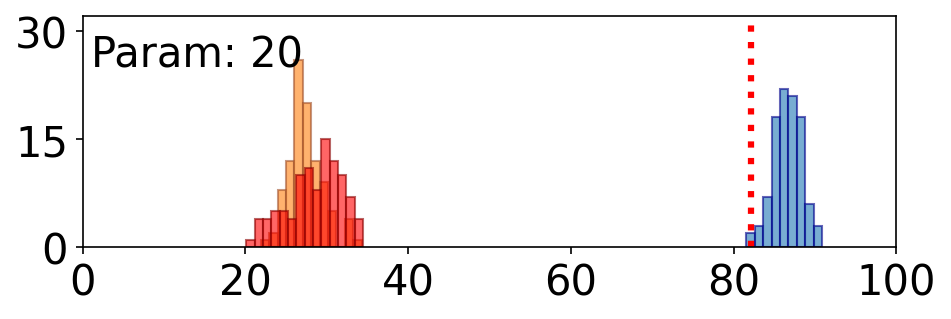}
            \includegraphics[width=1.03\linewidth]{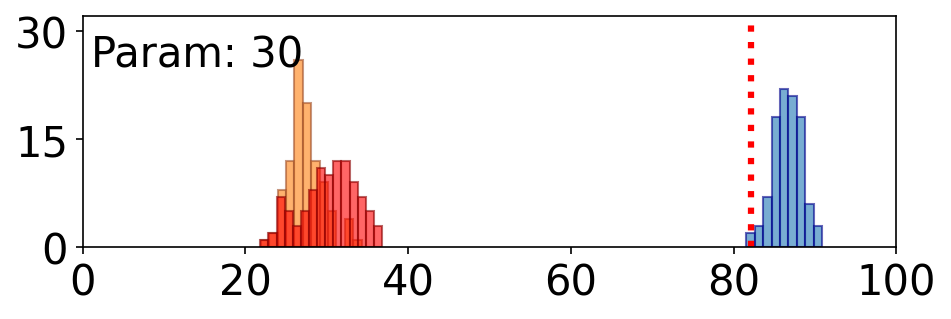}
            \includegraphics[width=1.03\linewidth]{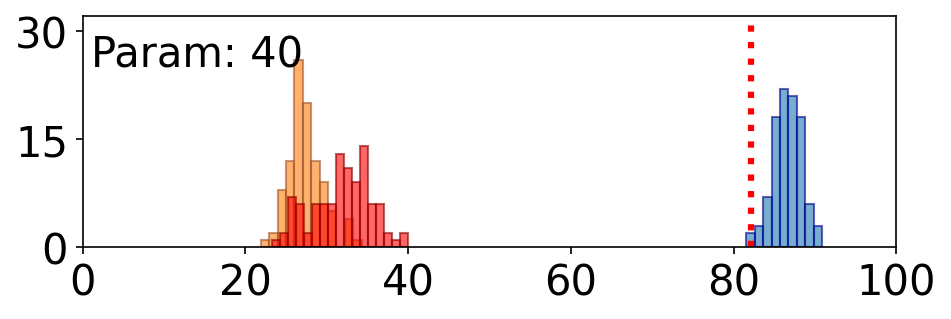}
            \includegraphics[width=1.03\linewidth]{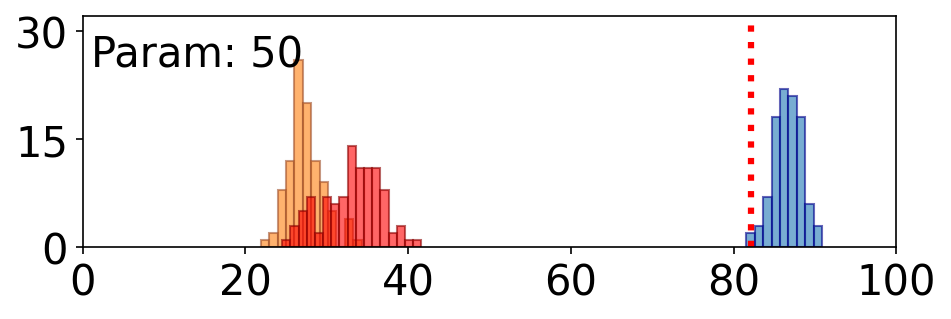}}
        \end{subfigure}
        
    \caption{The effects of different perturbations on $\|\mathbf{w}-\hat{\mathbf{w}}^{\prime}\|$. From left to right, the perturbations are set with different parameters (Param). 
    Watermark: HSTR with channels 0 and 3 pattern multiplied by $2$ (HSTR-2-2). Thresholds: 82.15 (1e-2 FPR).}
    \label{fig: HSTR 2 2 attack shifting}
\end{figure*}

\section{Experiments of the Lipschitz Constant on thresholds}\label{sec: full lip exp}
We present the estimated thresholds with different watermarking schemes. Please refer to Figs.~\ref{fig: experimental detection threshold (ringid 64)} and~\ref{fig: experimental detection threshold (HSQR)}.
\begin{figure}[ht]
    \centering
    \includegraphics[width=1.0\linewidth]{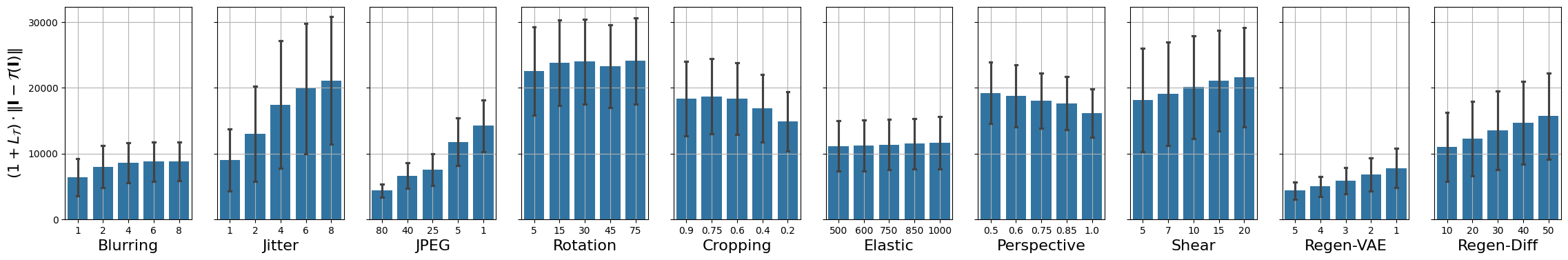}
    \caption{The experimental thresholds for different perturbations on $\|\mathbf{w}-\hat{\mathbf{w}}^{\prime}\|$, each with different parameters on x-axis. The value on y-axis represents the potential threat that $\hat{\mathbf{w}}'$ is getting far from $\mathbf{w}$, which corresponds to Eq.~(\ref{eq: ideal detection threshold}).
    Watermark: RingID.}
    \label{fig: experimental detection threshold (ringid 64)}
\end{figure}


\begin{figure}[ht]
    \centering
    \includegraphics[width=1.0\linewidth]{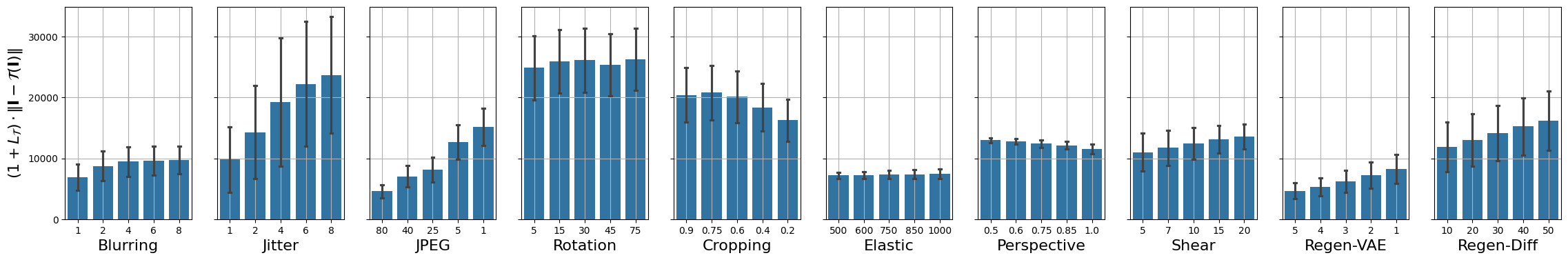}
    \caption{The experimental thresholds for different perturbations on $\|\mathbf{w}-\hat{\mathbf{w}}^{\prime}\|$, each with different parameters on x-axis. The value on y-axis represents the potential threat that $\hat{\mathbf{w}}'$ is getting far from $\mathbf{w}$, which corresponds to Eq.~(\ref{eq: ideal detection threshold}).
    Watermark: HSQR.}
    \label{fig: experimental detection threshold (HSQR)}
\end{figure}

\section{Performance of Tree-Ring \& RingID on Stirmark}\label{stirmark exp}
To check the robustness claimed by Tree-Ring~\cite{wen2023treerings} and RingID~\cite{ci2024ringid}, we verified these methods via a well-known benchmark, called Stirmark~\cite{petitcolas2000watermarking, petitcolas1998attacks}, which extensively includes both non-geometric image manipulations and geometric distortions, as shown in Tables~\ref{Table: Tree-Ring on Stirmark part I},~\ref{Table: Tree-Ring on Stirmark part II},~\ref{Table: RingID on Stirmark part I}, and~\ref{Table: RingID on Stirmark part II}.
More information can be found in Stirmark~\cite{petitcolas1998attacks}.


In this experiment, we further examine the performance in terms of detection accuracy and TPR under more extensive perturbations with FPR set to 1e-2. Here, apart from the watermark detection accuracy in main text, the ACC represented in these tables is defined as
\begin{align*}
    \mathrm{
    ACC=\max\left\{1-\frac{FPR+(1-TPR)}{2}\right\},
    }
\end{align*}
which coincides with the usual definition of ACC since the numbers of watermarked and unwatermarked samples are equal.

\begin{table*}[hb]
  \centering
  \caption{Attacker: StirMark 3.1.79. Defense: Tree-Ring~\cite{wen2023treerings} (watermark: ring). Diffusion model: LDM. Number of testing images: watermarked: 1000 / unwatermarked: 1000. Generated image size: 3x512x512. Latent size: 4x64x64.}\label{Table: Tree-Ring on Stirmark part I}
  \begin{tabular}{@{}l|ccc@{}}
    \toprule
    Attack & AUC & ACC & TPR@0.01FPR \\
    \midrule[2pt]
    Median filter 2x2 & 1.0000 & 1.0000 & 1.0000 \\
    Median filter 3x3 & 1.0000 & 1.0000 & 1.0000 \\
    Median filter 4x4 & 1.0000 & 0.9995 & 0.9990 \\
    Gaussian filter 3x3 & 1.0000 & 1.0000 & 1.0000 \\
    JPEG 90 & 1.0000 & 1.0000 & 1.0000 \\
    JPEG 80 & 1.0000 & 1.0000 & 1.0000 \\
    JPEG 70 & 1.0000 & 1.0000 & 1.0000 \\
    JPEG 60 & 1.0000 & 1.0000 & 1.0000 \\
    JPEG 50 & 1.0000 & 0.9995 & 1.0000 \\
    JPEG 40 & 1.0000 & 0.9975 & 0.9990 \\
    JPEG 30 & 0.9999 & 0.9970 & 0.9980 \\
    JPEG 20 & 0.9998 & 0.9945 & 0.9970 \\
    JPEG 10 & 0.9968 & 0.9760 & 0.9360 \\
    FMLR (Frequency Mode Laplacian Removal) & 1.0000 & 0.9990 & 1.0000 \\
    Color reduce & 1.0000 & 0.9995 & 1.0000 \\
    Sharpening 3x3 & 1.0000 & 1.0000 & 1.0000 \\
    1 column, 1 row removed & 1.0000 & 1.0000 & 1.0000 \\
    5 column, 1 row removed & 1.0000 & 1.0000 & 1.0000 \\
    1 column, 5 row removed & 1.0000 & 1.0000 & 1.0000 \\
    17 column, 5 row removed & 1.0000 & 1.0000 & 1.0000 \\
    5 column, 17 row removed & 1.0000 & 1.0000 & 1.0000 \\
    Cropping 1\% off & 1.0000 & 1.0000 & 1.0000 \\
    Cropping 2\% off & 1.0000 & 0.9995 & 1.0000 \\
    Cropping 5\% off & 1.0000 & 0.9990 & 1.0000 \\
    Cropping 10\% off & 0.7880 & 0.7190 & 0.1070 \\
    Cropping 15\% off & 0.6690 & 0.6275 & 0.0620 \\
    Cropping 20\% off & 0.8747 & 0.7975 & 0.2470 \\
    Cropping 25\% off & 0.9265 & 0.8490 & 0.3330 \\
    Cropping 50\% off & 0.7645 & 0.7090 & 0.0470 \\
    Cropping 75\% off & 0.6154 & 0.5915 & 0.0280 \\
    \bottomrule
  \end{tabular}
\end{table*}

\begin{table*}[hb]
  \centering
  \caption{Attacker: StirMark 3.1.79. Defense: Tree-Ring~\cite{wen2023treerings} (watermark: ring). Diffusion model: LDM. Number of testing images: watermarked: 1000 / unwatermarked: 1000. Generated image size: 3x512x512. Latent size: 4x64x64.}\label{Table: Tree-Ring on Stirmark part II}
  \begin{tabular}{@{}l|ccc@{}}
    \toprule
    Attack & AUC & ACC & TPR@0.01FPR \\ 
    \midrule[2pt]
    Linear (1.007, 0.010, 0.010, 1.012) & 1.0000 & 1.0000 & 1.0000 \\
    Linear (1.010, 0.013, 0.009, 1.011) & 1.0000 & 1.0000 & 1.0000 \\
    Linear (1.013, 0.008, 0.011, 1.008) & 1.0000 & 1.0000 & 1.0000 \\
    Aspect ratio change (0.80, 1.00) & 0.6752 & 0.6300 & 0.0650 \\
    Aspect ratio change (0.90, 1.00) & 0.9999 & 0.9955 & 0.9990 \\
    Aspect ratio change (1.00, 0.80) & 0.9206 & 0.8400 & 0.4470 \\
    Aspect ratio change (1.00, 0.90) & 1.0000 & 0.9995 & 1.0000 \\
    Aspect ratio change (1.00, 1.20) & 0.8612 & 0.7755 & 0.2640 \\
    Aspect ratio change (1.00, 1.10) & 1.0000 & 0.9990 & 1.0000 \\
    Aspect ratio change (1.10, 1.00) & 1.0000 & 0.9990 & 1.0000 \\
    Aspect ratio change (1.20, 1.00) & 0.9781 & 0.9270 & 0.7360 \\
    Rotation 1.00 & 1.0000 & 1.0000 & 1.0000 \\
    Rotation 2.00 & 1.0000 & 0.9985 & 1.0000 \\
    Rotation 5.00 & 0.9584 & 0.8880 & 0.5500 \\
    Rotation 10.00 & 0.8968 & 0.8150 & 0.2930 \\
    Rotation 15.00 & 0.9339 & 0.8560 & 0.4210 \\
    Rotation 30.00 & 0.8914 & 0.8100 & 0.2280 \\
    Rotation 45.00 & 0.8735 & 0.7865 & 0.1920 \\
    Rotation 90.00 & 0.9994 & 0.9935 & 0.9950 \\
    Flipping & 1.0000 & 1.0000 & 1.0000 \\
    Rotation scale 1.00 & 1.0000 & 1.0000 & 1.0000 \\
    Rotation scale 10.00 & 0.8993 & 0.8180 & 0.3770\\
    Rotation scale 15.00 & 0.9373 & 0.8620 & 0.4510 \\
    Rotation scale 30.00 & 0.8942 & 0.8155 & 0.2080 \\
    Rotation scale 45.00 & 0.8767 & 0.7885 & 0.2720 \\
    Rotation scale 90.00 & 0.9994 & 0.9935 & 0.9950 \\
    Shearing x-0\% y-1\% & 1.0000 & 1.0000 & 1.0000 \\
    Shearing x-1\% y-0\% & 1.0000 & 1.0000 & 1.0000 \\
    Shearing x-1\% y-1\% & 1.0000 & 1.0000 & 1.0000 \\
    Shearing x-0\% y-5\% & 1.0000 & 0.9970 & 1.0000 \\
    Shearing x-5\% y-0\% & 1.0000 & 0.9965 & 0.9980 \\
    Shearing x-5\% y-5\% & 0.9999 & 0.9950 & 0.9960 \\
    Random bending & 0.9999 & 0.9985 & 0.9990 \\
    \bottomrule
  \end{tabular}
\end{table*}

\begin{table*}[hb]
  \centering
  \caption{Attacker: StirMark 3.1.79. Defense: RingID~\cite{ci2024ringid} (watermark: ring \& noise). Diffusion model: LDM. Number of testing images: watermarked: 1000 / unwatermarked: 1000. Generated image size: 3x512x512. Latent size: 4x64x64.}\label{Table: RingID on Stirmark part I}
  \begin{tabular}{@{}l|ccc@{}}
    \toprule
    Attack & AUC & ACC & TPR@0.01FPR \\
    \midrule[2pt]
    Median filter 2x2 & 1.0000 & 1.0000 & 1.0000 \\
    Median filter 3x3 & 1.0000 & 1.0000 & 1.0000 \\
    Median filter 4x4 & 1.0000 & 1.0000 & 1.0000 \\
    Gaussian filter 3x3 & 1.0000 & 1.0000 & 1.0000 \\
    JPEG 90 & 1.0000 & 1.0000 & 1.0000 \\
    JPEG 80 & 1.0000 & 1.0000 & 1.0000 \\
    JPEG 70 & 1.0000 & 1.0000 & 1.0000 \\
    JPEG 60 & 1.0000 & 1.0000 & 1.0000 \\
    JPEG 50 & 1.0000 & 1.0000 & 1.0000 \\
    JPEG 40 & 1.0000 & 1.0000 & 1.0000 \\
    JPEG 30 & 1.0000 & 1.0000 & 1.0000 \\
    JPEG 20 & 1.0000 & 1.0000 & 1.0000 \\
    JPEG 10 & 1.0000 & 0.9995 & 1.0000 \\
    FMLR (Frequency Mode Laplacian Removal) & 1.0000 & 1.0000 & 1.0000 \\
    Color reduce & 1.0000 & 1.0000 & 1.0000 \\
    Sharpening 3x3 & 1.0000 & 1.0000 & 1.0000 \\
    1 column, 1 row removed & 1.0000 & 1.0000 & 1.0000 \\
    5 column, 1 row removed & 1.0000 & 1.0000 & 1.0000 \\
    1 column, 5 row removed & 1.0000 & 1.0000 & 1.0000 \\
    17 column, 5 row removed & 1.0000 & 1.0000 & 1.0000 \\
    5 column, 17 row removed & 1.0000 & 1.0000 & 1.0000 \\
    Cropping 1\% off & 1.0000 & 1.0000 & 1.0000 \\
    Cropping 2\% off & 1.0000 & 1.0000 & 1.0000 \\
    Cropping 5\% off & 1.0000 & 1.0000 & 1.0000 \\
    Cropping 10\% off & 0.8307 & 0.7505 & 0.1030 \\
    Cropping 15\% off & 0.5457 & 0.5415 & 0.0120 \\
    Cropping 20\% off & 0.8277 & 0.7535 & 0.1070 \\
    Cropping 25\% off & 0.5457 & 0.5450 & 0.0120 \\
    Cropping 50\% off & 0.3713 & 0.5020 & 0.0070 \\
    Cropping 75\% off & 0.5412 & 0.5425 & 0.0150 \\
    \bottomrule
  \end{tabular}
\end{table*}

\begin{table*}[hb]
  \centering
  \caption{Attacker: StirMark 3.1.79. Defense: RingID~\cite{ci2024ringid} (watermark: ring \& noise). Diffusion model: LDM. Number of testing images: watermarked: 1000 / unwatermarked: 1000. Generated image size: 3x512x512. Latent size: 4x64x64.}\label{Table: RingID on Stirmark part II}
  \begin{tabular}{@{}l|ccc@{}}
    \toprule
    Attack & AUC & ACC & TPR@0.01FPR \\
    \midrule[2pt]
    Linear (1.007, 0.010, 0.010, 1.012) & 1.0000 & 1.0000 & 1.0000 \\
    Linear (1.010, 0.013, 0.009, 1.011) & 1.0000 & 1.0000 & 1.0000 \\
    Linear (1.013, 0.008, 0.011, 1.008) & 1.0000 & 1.0000 & 1.0000 \\
    Aspect ratio change (0.80, 1.00) & 0.9996 & 0.9920 & 0.9920 \\
    Aspect ratio change (0.90, 1.00) & 1.0000 & 1.0000 &1.0000 \\
    Aspect ratio change (1.00, 0.80) & 0.9999 & 0.9935 & 0.9950 \\
    Aspect ratio change (1.00, 0.90) & 1.0000 & 1.0000 & 1.0000 \\
    Aspect ratio change (1.00, 1.20) & 1.0000 & 0.9980 & 1.0000 \\
    Aspect ratio change (1.00, 1.10) & 1.0000 & 1.0000 & 1.0000 \\
    Aspect ratio change (1.10, 1.00) & 1.0000 & 1.0000 & 1.0000 \\
    Aspect ratio change (1.20, 1.00) & 1.0000 & 0.9980 & 0.9980 \\
    Rotation 1.00 & 1.0000 & 1.0000 & 1.0000 \\
    Rotation 2.00 & 1.0000 & 1.0000 & 1.0000 \\
    Rotation 5.00 & 0.9777 & 0.9260 & 0.6050 \\
    Rotation 10.00 & 0.6227 & 0.5955 & 0.0080 \\
    Rotation 15.00 & 0.8210 & 0.7435 & 0.1100 \\
    Rotation 30.00 & 0.4720 & 0.5060 & 0.0070 \\
    Rotation 45.00 & 0.3698 & 0.5005 & 0.0020 \\
    Rotation 90.00 & 1.0000 & 1.0000 & 1.0000 \\
    Flipping & 1.0000 & 0.9995 & 1.0000 \\
    Rotation scale 1.00 & 1.0000 & 1.0000 & 1.0000 \\
    Rotation scale 10.00 & 0.6313 & 0.6040 & 0.0090 \\
    Rotation scale 15.00 & 0.8145 & 0.7415 & 0.1070 \\
    Rotation scale 30.00 & 0.4687 & 0.5080 & 0.0050 \\
    Rotation scale 45.00 & 0.3762 & 0.5000 & 0.0020 \\
    Rotation scale 90.00 & 1.0000 & 1.0000 & 1.0000 \\
    Shearing x-0\% y-1\% & 1.0000 & 1.0000 & 1.0000 \\
    Shearing x-1\% y-0\% & 1.0000 & 1.0000 & 1.0000 \\
    Shearing x-1\% y-1\% & 1.0000 & 1.0000 & 1.0000 \\
    Shearing x-0\% y-5\% & 1.0000 & 1.0000 & 1.0000 \\
    Shearing x-5\% y-0\% & 1.0000 & 1.0000 & 1.0000 \\
    Shearing x-5\% y-5\% & 1.0000 & 1.0000 & 1.0000 \\
    Random bending & 1.0000 & 1.0000 & 1.0000 \\
    \bottomrule
  \end{tabular}
\end{table*}


\end{document}